\documentclass[10pt, a4paper]{article}

\usepackage[utf8]{inputenc} 
\usepackage[T1]{fontenc}    
\usepackage{lmodern}
\usepackage[hidelinks]{hyperref}       
\usepackage{url}            
\usepackage{booktabs}       
\usepackage{amsfonts}       
\usepackage{nicefrac}       
\usepackage{microtype}      
\usepackage[numbers]{natbib}
\usepackage{mathtools} 
\usepackage{enumitem}
\setlist[itemize]{topsep=0pt, itemsep=-1ex,partopsep=1ex, parsep=1ex}
\usepackage{amsmath}
\usepackage{amsthm}
\newtheorem{theorem}{Theorem}
\newtheorem{lemma}{Lemma}

\usepackage{multirow} 
\usepackage{graphicx}

\usepackage{subfigure}
\usepackage{caption}
\usepackage{color}
\usepackage[dvipsnames]{xcolor}
\usepackage{algorithm}
\usepackage[compatible]{algpseudocode}
\ifpdf
  \DeclareGraphicsExtensions{.eps,.pdf,.png,.jpg}
\else
  \DeclareGraphicsExtensions{.eps}
\fi

\newcommand{\E}{\mathbb{E} \,} 
\newcommand{\R}{\mathbb{R}} 
\newcommand{\grad}{\nabla} 

\newcommand{\latent}{u} 
\newcommand{\FULLSEQ}{\mathcal{T}}
\newcommand{\SUBSEQ}{{\mathcal{S}}}
\newcommand{\BUFFSUBSEQ}{{\mathcal{S}^*}}

\newcommand{\commentout}[1]{}

\def\abovestrut#1{\rule[0in]{0in}{#1}\ignorespaces}

\DeclareMathOperator{\diag}{diag}
\DeclareMathOperator{\tr}{tr} 

\usepackage{tikz}
\DeclareRobustCommand\dashedline{\tikz[baseline=-0.6ex]\draw[dashed] (0,0)--(0.34,0);}
\DeclareRobustCommand\fullline{\tikz[baseline=-0.6ex]\draw[] (0,0)--(0.27,0);}

\begin{document}

\title{Stochastic Gradient MCMC for State Space Models}

\author{Christopher Aicher\thanks{Department of Statistics, University of Washington, WA} \and Yi-An Ma\thanks{Department of Electrical Engineering and Computer Sciences, UC Berkeley, CA.}
\and
Nicholas J. Foti\thanks{Paul G. Allen School of Computer Science and Engineering, University of Washington, WA. \newline
\hspace*{1.5em} Email: \texttt{[aicherc, nfoti, ebfox]@uw.edu}, \texttt{yianma@berkeley.edu}}
\and
Emily B. Fox\footnotemark[1] \footnotemark[3]
}
\date{}
\maketitle

\begin{abstract}
State space models (SSMs) are a flexible approach to modeling complex time series.
However, inference in SSMs is often computationally prohibitive for long time series.
Stochastic gradient MCMC (SGMCMC) is a popular method for scalable Bayesian inference for large independent data.
Unfortunately when applied to dependent data, such as in SSMs, SGMCMC's stochastic gradient estimates are biased as they break crucial temporal dependencies.
To alleviate this, we propose stochastic gradient estimators that control this bias by performing additional computation in a `buffer' to reduce breaking dependencies.
Furthermore, we derive error bounds for this bias and show a geometric decay under mild conditions.
Using these estimators, we develop novel SGMCMC samplers for discrete, continuous and mixed-type SSMs with analytic message passing.
Our experiments on real and synthetic data demonstrate the effectiveness of our SGMCMC algorithms compared to batch MCMC, allowing us to scale inference to long time series with millions of time points.
\end{abstract}

\section{Introduction}
\label{sec:intro}
State space models (SSMs) are ubiquitous in the analysis of time series
in fields as diverse as biology~\cite{wulsin2013bayesian},
finance and economics~\cite{kim1999state, zeng2013state},
and systems and control~\cite{elliott2008hidden}.
As a defining feature,
SSMs augment the observed time series with a \emph{latent state sequence}
to model complex time series dynamics
with a latent Markov chain dependence structure.
Given a time series,
inference of model parameters involves sampling or marginalizing this latent state sequence.
Unfortunately, both the runtime and memory required scale with the length of the time series,
which is prohibitive for long time series
(e.g. high frequency stock prices~\cite{goodhart1997high},
genome sequences~\cite{eddy1998profile},
or neural impulse recordings~\cite{davis2016mining}).
In practice, given a long time series, one could `segment' or `downsample' to reduce length;
however, this preprocessing can destroy or change important signals and
computational considerations should ideally not limit scientific modeling.

To help scale inference in SSMs, we consider stochastic gradient Markov chain Monte Carlo (SGMCMC), a popular method for scaling Bayesian inference to large data sets~\cite{chen2015convergence, ma2015complete, welling2011bayesian}.
The key idea of SGMCMC is to employ stochastic gradient estimates based on subsets or `minibatches' of data, avoiding costly computation of gradients on the full dataset,
such that the resulting dynamics produce samples from the posterior distribution over SSM parameters.
This approach has found much success in \emph{independent} data models, where the stochastic gradients are \emph{unbiased} estimates of the true gradients.
However, when applying SGMCMC to SSMs, naive stochastic gradients are \emph{biased}, as subsampling the data breaks dependencies in the SSM's latent state sequence.
This bias can destroy the dynamics of SGMCMC causing it to fail when applied to SSMs.
The challenge is to correct these stochastic gradients for SSMs while maintaining the computational benefits of SGMCMC.

In this work,
we develop computationally efficient stochastic gradient estimators for inference in general discrete-time SSMs.
To control the bias of stochastic gradients, we marginalize the latent state sequence in a \emph{buffer} around each subsequence,
propagating critical information from outside each subsequence to its local gradient estimate while avoiding costly full-chain computations.
Similar buffering ideas have been previously considered for belief propagation~\cite{gonzalez2009residual}, variational inference~\cite{foti2014stochastic}, and in our earlier work on SGMCMC for hidden Markov models (HMMs)~\cite{ma2017stochastic}, but all are limited to discrete latent states.
Here, we present buffering as an approximation to \emph{Fisher's identity}~\cite{cappe2005inference},
allowing us to naturally extend buffering trick to continuous and mixed-type latent states.

We further develop analytic bounds on the bias of our proposed gradient estimator that, under mild conditions, decay geometrically in the buffer size.
To obtain these bounds we prove that the latent state sequence posterior distribution has an \emph{exponential forgetting} property~\cite{cappe2005inference, del2010forward}.
However unlike classic results which prove a geometric decay between the approximate and exact marginal posterior distributions in total variation distance,
we use Wasserstein distance~\cite{villani2008optimal} to allow analysis of continuous and mixed-type latent state SSMs.
Our approach is similar to proofs of Wasserstein ergodicity in homogeneous Markov chains~\cite{durmus2015quantitative, madras2010quantitative, rudolf2015perturbation}; however we extend these ideas to the \emph{nonhomogeneous} Markov chains defined by the latent state sequence posterior distribution.
These geometrically decaying bounds guarantee that we only need a small buffer size in practice, allowing scalable inference in SSMs.

Although our proposed gradient estimator can be generally applied to any stochastic gradient method, here, we develop SGMCMC samplers for Bayesian inference in a variety of SSMs such as HMMs, linear Gaussian SSMs (LGSSM), and switching linear dynamical systems (SLDS)~\cite{cappe2005inference,fox2009bayesian}.
We also derive preconditioning matrices to take advantage of information geometry,
which allows for more rapid mixing and convergence of our samplers~\cite{girolami2011riemann, patterson2013stochastic}.
Finally, we validate our algorithms and theory on a variety of synthetic and real data experiments,
finding that our gradient estimator can provide orders of magnitude run-time speed ups compared to batch sampling.

This paper significantly expands upon our initial work~\cite{ma2017stochastic}, by
(i) connecting buffering to Fisher's identity, simplifying its presentation and analysis,
(ii) non-trivially generalizing the approach to SSMs beyond the HMM, including continuous and mixed-type latent states,
(iii) developing a general framework for bounding the error of buffered gradient estimators using Wasserstein distance,
and (iv) providing extensive validation on a number of real and synthetic datasets.

The paper is organized as follows.
First, we review background on SSMs and SGMCMC methods in Section~\ref{sec:background}.
We then present our framework of constructing buffered gradient estimators to extend SGMCMC to SSMs in Section~\ref{sec:framework}.
We prove the geometrically decaying bounds for our proposed buffered gradient estimate in Section~\ref{sec:bounds}.
We apply our framework and error bounds to discrete, continuous and mixed-type latent state SSMs in Section~\ref{sec:models}.
Finally, we investigate our algorithms on both synthetic and real data in Section~\ref{sec:experiments}.

\section{Background}
\label{sec:background}
\subsection{State Space Models for Time Series}
State space models (SSMs) for time series are a class of discrete-time bivariate stochastic process
$\{\latent_t, y_t \}_{t \in \FULLSEQ}$, $\FULLSEQ = \{1, \ldots, T\}$, consisting of a latent state sequence $\latent := \latent_{1:T}$ generated by a homogeneous Markov chain and an observation sequence $y := y_{1:T}$ generated independently conditioned on $\latent$~\cite{cappe2005inference}.
Examples of state space models include: HMMs, LGSSMs, and SLDSs (see Section~\ref{sec:models} for details).
For a generic SSM, the joint distribution of $y$ and $\latent$ factorizes as
\begin{equation}
\label{eq:complete_data_like}
p(y, \latent \, | \, \theta) =
    \prod_{t = 1}^T p(y_t \, | \, \latent_t, \theta)  p(\latent_t \, | \, \latent_{t-1}, \theta) \cdot p_0(\latent_0) \enspace,
\end{equation}
where $\theta$ are model-specific parameters,
$p(y_t  |  \latent_t, \theta)$ is the \emph{emission density},
$p(\latent_t  |  \latent_{t-1}, \theta)$ is the \emph{transition density},
and $p_0(\latent_0)$ is a prior for the latent states. 
As the latent state sequence $u$ is unobserved,
the likelihood of $\theta$ given only the observations $y$ (marginalizing $u$) is
\begin{equation}
\label{eq:marginal_like}
p(y \, | \, \theta) =
    \int \, \prod_{t = 1}^T p(y_t \, | \, \latent_t, \theta)  p(\latent_t \, | \, \latent_{t-1}, \theta) \cdot p_0(\latent_0) \enspace d\latent  \enspace,
\end{equation}
Unconditionally, the observations $y$ are not independent and the graphical model of this \emph{marginal likelihood}, Eq.~\eqref{eq:marginal_like}, has many long term dependencies, Figure~\ref{fig:ssm_graphical} (right).
In contrast, when conditioned on $u$ the observations $y$ are independent and the \emph{complete-data likelihood}, Eq.~\eqref{eq:complete_data_like}, has a simpler chain structure, Figure~\ref{fig:ssm_graphical} (left).

\begin{figure}[t]
\centering
\begin{minipage}[c]{.45\textwidth}
    \centering
    \includegraphics[width=0.99\linewidth]{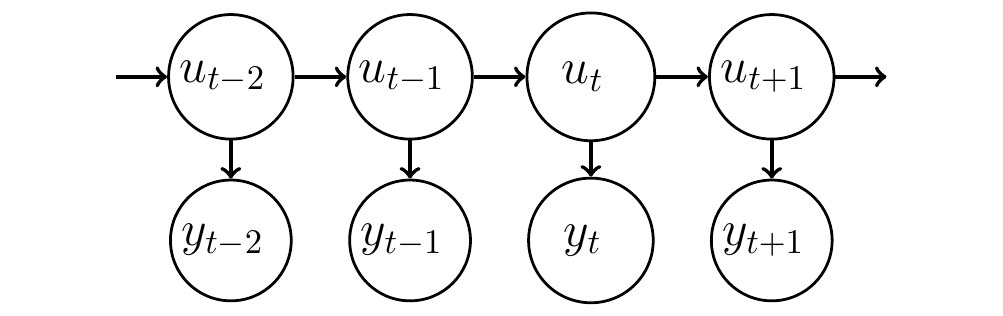}
\end{minipage}
\begin{minipage}[c]{.5\textwidth}
    \centering
    \includegraphics[width=0.99\linewidth]{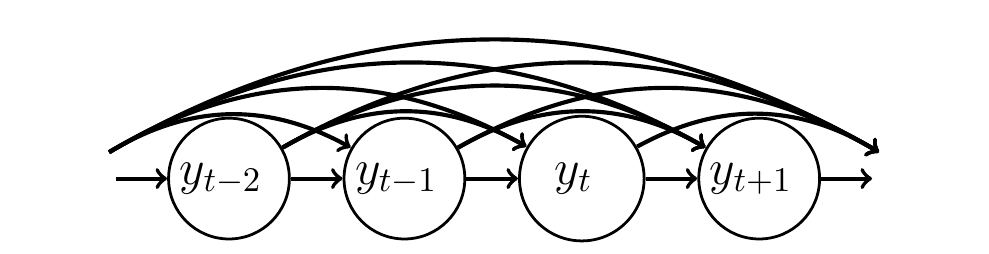}
\end{minipage}
\caption{Graphical Model of a SSM: (left) the joint process $u, y$, Eq.~\eqref{eq:complete_data_like} and (right) $y$ marginalizing out $u$, Eq.~\eqref{eq:marginal_like}.
The parameters $\theta$ are not shown, but connect to all nodes. 
}
\label{fig:ssm_graphical}
\end{figure}

To infer $\theta$ given $y$, we can maximize the marginal likelihood $p(y \, | \, \theta)$ or, given a prior $p(\theta)$, sample from the posterior $p(\theta \, | \, y) \propto p(y \, | \, \theta) p(\theta)$.
However, traditional inference methods for $\theta$, such as expectation maximization (EM), variational inference, or Gibbs sampling,
take advantage of the conditional independence structure in $p(y, u \, | \, \theta)$, Eq.~\eqref{eq:complete_data_like}, rather than working directly with $p(y \, | \, \theta)$, Eq.~\eqref{eq:marginal_like}~\cite{beal2003variational, scott2002bayesian}.
To use $p(y, u \, | \, \theta)$ with unobserved $u$,
these methods rely on sampling or taking expectations of $\latent$ from the posterior $\gamma(\latent) := p(\latent \, | \, y, \theta)$.
As an example, gradient-based methods take advantage of 
\emph{Fisher's identity}
~\cite{cappe2005inference}
\begin{equation}
\label{eq:fisher_identity}
    \grad \log p(y \, | \, \theta) = \E_{u | y, \theta}[ \grad \log p(y, \latent \, | \, \theta)] = \E_{u\sim\gamma}[\grad \log p(y, \latent \, | \, \theta)]\enspace,
\end{equation}
which allows gradients of Eq.~\eqref{eq:marginal_like} to be computed in terms of Eq.~\eqref{eq:complete_data_like}.
To compute the posterior $\gamma(u)$,
these methods use the well-known \emph{forward-backward algorithm}~\cite{cappe2005inference, scott2002bayesian}.
The algorithm works by recursively computing a sequence of
forward messages $\alpha_t(\latent_t)$
and backward messages $\beta_t(\latent_t)$
which are used to compute the pairwise marginals of $\gamma$.
More specifically,
\begin{align}
\label{eq:forward_message}
\alpha_t(\latent_t) &:= p(\latent_t, y_{\leq t} \, | \, \theta) = \int p(y_t, \latent_t \, | \, \latent_{t-1}, \theta) \alpha_{t-1}(\latent_{t-1}) \, d \latent_{t-1} \\
\label{eq:backward_message}
\beta_t(\latent_t) &:= p(y_{>t} \, | \, \latent_t, \theta) = \int p(y_{t+1}, \latent_{t+1} \, | \, \latent_t, \theta) \beta_{t+1}(\latent_{t+1}) \, d \latent_{t+1} \\
\label{eq:pairwise_marginal}
\gamma_{t-1:t}(\latent_{t-1}, \latent_{t}) &:= p(\latent_{t-1}, \latent_t \, | \, y, \theta) \propto \alpha_{t-1}(\latent_{t-1}) p(y_t, \latent_t \, | \, \latent_{t-1}, \theta) \beta_t(\latent_t) 
\enspace.
\end{align}
When message passing is tractable (i.e., when Eqs.~\eqref{eq:forward_message}-\eqref{eq:backward_message} involve discrete or conjugate likelihoods), 
the forward-backward algorithm can be calculated in closed form.
When message passing is intractable, 
the messages can be approximated using Monte-Carlo sampling methods (e.g. blocked Gibbs sampling~\cite{carter1994gibbs, fox2011bayesian}, particle methods~\cite{andrieu2010particle, briers2010smoothing, doucet2009tutorial, sudderth2010nonparametric}).
In both cases, when the length of the time series $|\FULLSEQ|$ is much larger than the dimension of $\theta$, the forward-backward algorithm (running over the entire sequence) requires $O(|\FULLSEQ|)$ time and memory at \emph{each iteration}.

The SSM challenge is to scale inference of model parameters $\theta$ to long time series when the computation and storage per iteration $O(|\FULLSEQ|)$ is prohibitive.

\subsection{Stochastic Gradient MCMC}
One popular method for scalable Bayesian inference is \emph{stochastic gradient} Markov chain Monte Carlo (SGMCMC)~\cite{chen2015convergence, ma2015complete, welling2011bayesian}.
The idea behind gradient-based MCMC is to simulate continuous dynamics for a \emph{potential energy} function $U(\theta) \propto -\log p(y, \theta)$
such that the dynamics generate samples from the posterior distribution $p(\theta \, | \, y)$.
For example, the Langevin diffusion over $U(\theta)$ is given by the stochastic differential equation (SDE)
\begin{equation}
\label{eq:langevin_diffusion}
d\theta_s = g(\theta) ds + \sqrt{2} dW_s \enspace,
\end{equation}
where $dW_s$ is Brownian motion, $g(\theta) = -\grad U(\theta) = \grad_\theta \log p(y, \theta)$, and $s$ indexes continuous time.
As $s \rightarrow \infty$, the distribution of $\theta_s$ converges to the SDE's stationary distribution,
which by the Fokker-Planck equation is the posterior $p(\theta \, | \, y)$~\cite{ma2015complete}.
Because we cannot perfectly simulate Eq.~\eqref{eq:langevin_diffusion},
in practice we use a discretized numerical approximation.
One straightforward approximation is the Euler-Mayurma discretization
\begin{equation}
\label{eq:LMC}
\theta^{(s+1)} \leftarrow \theta^{(s)} + h g(\theta^{(s)}) + \mathcal{N}(0, 2h) \enspace,
\end{equation}
where $h$ is the stepsize and $s$ indexes discrete time steps.
This recursive update defines the
Langevin Monte-Carlo (LMC) algorithm.
Typically, a Metropolis-Hastings correction step is added to account for the discretization error~\cite{roberts1996exponential, roberts1998optimal}.

For large datasets, computing $g(\theta)$ at every step in Eq.~\eqref{eq:LMC} is computationally prohibitive.
To alleviate this, the key ideas of \emph{stochastic gradient} Langevin dynamics (SGLD) are 
to replace $g(\theta)$ with a quick-to-compute unbiased estimator $\hat{g}(\theta)$ and 
to use a decreasing stepsize $h^{(s)}$ to avoid costly Metropolis-Hastings correction steps~\cite{welling2011bayesian}
\begin{equation}
\label{eq:sgld_step}
\theta^{(s+1)} \leftarrow \theta^{(s)} + h^{(s)} \hat{g}(\theta^{(s)}) + \mathcal{N}(0, 2h^{(s)}) \enspace.
\end{equation} 
For i.i.d. data, an example of $\hat{g}(\theta)$ is to use a random minibatch $\SUBSEQ \subset \FULLSEQ$, $|\SUBSEQ|\ll|\FULLSEQ|$
\begin{equation}
\label{eq:approx_potential_iid}
\hat{g}(\theta) = -\frac{1}{\Pr(\SUBSEQ)}
    \sum_{t \in {\SUBSEQ}} \grad\log p(y_t \, | \, \theta) - \grad\log p(\theta) \enspace,
\end{equation}
which only requires $O(|\SUBSEQ|)$ time to compute.
When $\hat{g}(\theta)$ is unbiased and with an appropriate decreasing stepsize schedule $h^{(s)}$, the distribution of $\theta^{(s)}$ asymptotically converges to the posterior distribution~\cite{chen2015convergence, teh2016consistency}.
However, in practice one uses a small, finite step-size for greater efficiency, which introduces a small bias~\cite{dalalyan2017user}.

A Riemannian extension of SGLD (SGRLD) simulates the Langevin diffusion
over a Riemannian manifold with metric $D(\theta)^{-1}$ by preconditioning the gradient and noise of Eq.~\eqref{eq:sgld_step} by $D(\theta)$.
By incorporating geometric information about structure of $\theta$, SGRLD aims for a diffusion which mixes more rapidly.
Suggested examples of the metric $D(\theta)^{-1}$ are the Fisher information matrix $\mathcal{I}(\theta) = \E_{y}[\grad^2\log p(y\,|\,\theta)]$ or a noisy Hessian estimate $\widehat{\grad^2 \log p}(y \, | \,\theta)$
~\cite{girolami2011riemann, patterson2013stochastic}.
Given $D(\theta)$, each step of SGRLD is
\begin{equation}
\label{eq:sgrld_step}
\theta^{(s+1)} \leftarrow \theta^{(s)} + h \left[D(\theta^{(s)}) \cdot \hat{g}(\theta^{(s)}) + \Gamma(\theta^{(s)})\right] + \mathcal{N}\left(0, 2h D(\theta^{(s)})\right)
\end{equation}
where the vector $\Gamma(\theta)$ is a correction term $\Gamma(\theta)_i = \sum_j  \frac{\partial D(\theta)_{ij}}{\partial \theta_{j}}$
to ensure the dynamics converge to the target posterior~\cite{ma2015complete, xifara2014langevin}.
Many extensions to SGMCMC have been proposed such as using control variates to reduce the variance of $\hat{g}(\theta)$~\cite{baker2017control,chatterji2018theory,nagapetyan2017true} or augmented dynamics to improve mixing~\cite{chen2015convergence,chen2014stochastic, ding2014bayesian, li2016preconditioned}.
Although our ideas extend to these formulations as well,
we focus on the popular SGLD and SGRLD algorithms.

To apply SGMCMC to SSMs, we must choose whether to use the complete-data loglikelihood or the marginal data loglikelihood in the potential $U(\theta)$.
If we use the complete-data loglikelihood, then we treat $(\latent, \theta)$ as the parameters.
Although the observations $y$ conditioned on $(\latent, \theta)$ are independent,
we must calculate gradients for $\latent_{-T:T}$ at each iteration, which is prohibitive for long sequences $|\FULLSEQ|$ and intractable for discrete or mixed-type $\latent$.
On the other hand, if we use the marginal loglikelihood, then we only need to take gradients in $\theta$.
However, the observations $y$ conditioned on $\theta$ alone are \emph{not} independent and therefore the minibatch gradient estimator Eq.~\eqref{eq:approx_potential_iid} breaks crucial dependencies causing it to be biased.
Our SGMCMC challenge is correcting the bias in stochastic gradient estimates $\grad \widetilde U(\theta)$ when applied to SSMs.

\section{General Framework}
\label{sec:framework}
We now present our framework for scalable Bayesian inference in SSMs with long observation sequences.
Our approach is to extend SGMCMC to SSMs by developing a gradient estimator that ameliorates the issue of broken temporal dependencies.
In particular, we develop a computationally efficient gradient estimator that uses a \emph{buffer} to avoid breaking crucial dependencies, only breaking weak dependencies.
We first present a (computationally prohibitive) unbiased estimator of $g(\theta) = \grad \log p(y \, | \, \theta)$ for SSMs using Fisher's identity. 
We then derive a general computationally efficient gradient estimate $\tilde{g}(\theta)$ that accounts for the dependence in observations using a buffer. 
We also propose preconditioning matrices for SGRLD with SSMs. 
Finally, we present our general SGMCMC pseudocode for SSMs. 

\subsection{Unbiased Gradient Estimate}
\label{sec:framework-unbiased-gradient}
The main challenge in constructing an efficient estimate $\tilde{g}(\theta)$ of $g(\theta)$
for SSMs is handling the lack of independence (marginally) in $y$.
Because the observations in SSMs are not independent,
we cannot produce an unbiased estimate of $g(\theta)$ with a randomly selected subset of data points as in Eq.\eqref{eq:approx_potential_iid}.
For example, a naive estimate is to take the gradient of a random contiguous \emph{subsequence} $\SUBSEQ = \{t_1, \ldots, t_{S}\} \subset \FULLSEQ$ with $t_i = t_{i-1}+1$
\begin{equation}
\hat{g}(\theta) = - \frac{1}{\Pr(\SUBSEQ)}\grad \log p(y_\SUBSEQ \, | \, \theta) - \grad \log p(\theta) \enspace,
\end{equation}
where $p(y_\SUBSEQ \, | \, \theta)$ is computed with $p(\latent_{t_0}) = p_0(\latent_{t_0})$.
This estimate only requires $O(S)$ time compared to the $O(T)$ for $g(\theta)$.
However because the marginal likelihood does not factorize as in the independent observations case, this estimate is biased $\E_\SUBSEQ[\hat{g}(\theta)] \neq g(\theta)$.
In addition, as $\SUBSEQ$ are contiguous subsequences of $\FULLSEQ$, the scaling factor $\Pr(\SUBSEQ)^{-1}$ is no longer correct as time points in the center of $\mathcal{T}$ are sampled more frequently than the endpoints; instead each time point should be scaled point-wise.

To obtain an unbiased estimate for $g(\theta)$,
we use Fisher's identity Eq.~\eqref{eq:fisher_identity} to
rewrite $g(\theta)$ in terms of the complete-data loglikelihood
as a sum over time points
\begin{align}
g(\theta) &= - \grad \log p(y \, |\, \theta) - \grad \log p(\theta)\\
\nonumber
    &= - \E_{\latent | y, \theta} \left[ \grad \log p(y, \latent \, | \,\theta) \right] - \grad \log p(\theta) \\
\nonumber
    &= -  \sum_{t \in \FULLSEQ}\E_{\latent | y, \theta} \left[ \grad \log p(y_t, \latent_t \, | \, \latent_{t-1}, \theta) \right] - \grad \log p(\theta)
\end{align}
From this,
we straightforwardly identify an unbiased estimator for a subsequence $\SUBSEQ$
\begin{equation}
\label{eq:unbiased_potential_estimate}
\bar{g}(\theta) = - \sum_{t \in \SUBSEQ} \frac{1}{\Pr(t \in \SUBSEQ)}\E_{\latent | y, \theta} \left[ \grad \log p(y_t, \latent_t \, | \, \latent_{t-1}, \theta) \right] - \grad \log p(\theta) \enspace,
\end{equation}
where $\Pr(t \in \SUBSEQ)$ is the probability $t$ is in the random subsequence $\SUBSEQ$.

Although Eq.~\eqref{eq:unbiased_potential_estimate} reduces the number of gradient terms to compute from $T$ to $S$,
the summation terms require calculating expectations of $u \, | \, y, \theta$.
More specifically, Eq.~\eqref{eq:unbiased_potential_estimate} requires expectations with respect to the pairwise marginal posteriors $p(\latent_t, \latent_{t-1}\, | \, y_\FULLSEQ)$ for $t \in \SUBSEQ$.
Recall that computing these marginals take $O(T)$ time to pass messages over the entire sequence $\FULLSEQ$.
This defeats the purpose of using a subsequence.
If we instead only pass messages over the subsequence $\SUBSEQ$,
then the pairwise marginals are $p(\latent_t, \latent_{t-1} \, | \, y_\SUBSEQ)$ and we return to a biased gradient estimator
\begin{equation}
\label{eq:naive_grad_approx}
\hat{g}(\theta) = -\sum_{t \in \SUBSEQ}  \frac{1}{\Pr(t \in \SUBSEQ)} \E_{u | y_\SUBSEQ, \theta}[\grad \log p(y_t, \latent_t \, | \, \latent_{t-1}, \theta)] - \grad \log p(\theta) \enspace.
\end{equation}

\subsection{Approximate Gradient Estimate}
\label{sec:framework-approx-gradient}
\begin{figure}[t]
\centering
\begin{minipage}[c]{.9\textwidth}
    \centering
    \includegraphics[width=0.99\linewidth]{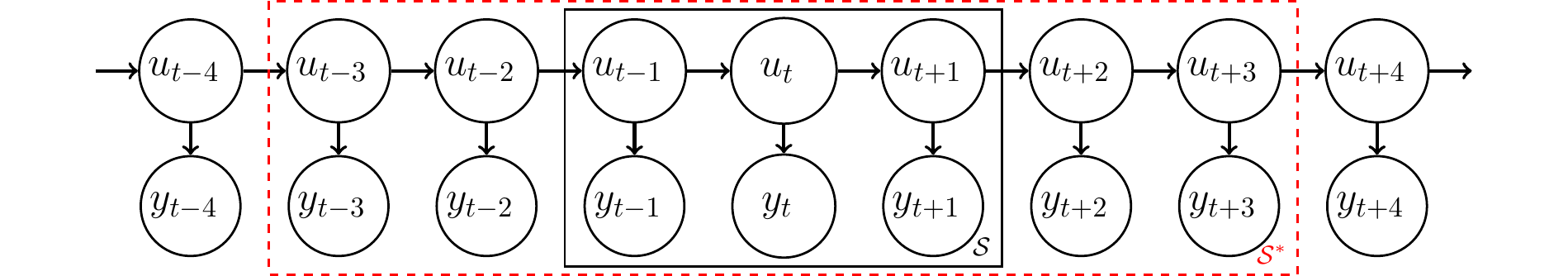}
\end{minipage}
\caption{Graphical model of a buffered subsequence with $S = 3$ and $B=2$.}
\label{fig:ssm_buffer}
\end{figure}

We instead propose passing messages over a \emph{buffered} subsequence
$\BUFFSUBSEQ := \{t_{-B}, \ldots, t_{S+B} \}$ for some positive buffer size $B$,
with $\SUBSEQ \subset \BUFFSUBSEQ \subset \FULLSEQ$ (see Figure~\ref{fig:ssm_buffer}).
The idea is that there exists a large enough $B$ such that $p(u_\SUBSEQ \, | \, y_\BUFFSUBSEQ, \theta) \approx p(u_\SUBSEQ \, | \, y_\FULLSEQ, \theta)$.
Our \emph{buffered gradient estimator} sums only over $\SUBSEQ$, but takes expectations over $u_\SUBSEQ \, | \, y_\BUFFSUBSEQ, \theta$ instead of $u_\SUBSEQ \, | \, y_\FULLSEQ, \theta$
\begin{equation}
\label{eq:efficient_potential_estimate}
\tilde{g}(\theta) = - \sum_{t \in \SUBSEQ}\frac{1}{\Pr(t \in \SUBSEQ)} \E_{\latent | y_{\BUFFSUBSEQ}, \theta} \left[ \grad \log p(y_t, \latent_t \, | \, \latent_{t-1}, \theta) \right] - \grad \log p(\theta) \enspace,
\end{equation}
where $p(\latent_{t_{-B-1}}) = p_0(\latent_{t_{-B-1}})$.
When $B = 0$ this is equivalent to the biased estimator $\hat{g}(\theta)$ of Eq.~\eqref{eq:naive_grad_approx}.
When $B = T$ this is equivalent to the unbiased estimator $\bar{g}(\theta)$ of Eq.~\eqref{eq:unbiased_potential_estimate}.

The trade-off between accuracy (bias) and runtime depends on the size of the buffer $B$ and current model parameters $\theta^{(s)}$.
Intuitively, when $\theta^{(s)}$ produces pairwise marginals that are similar to i.i.d.\ data, we can use a small buffer $B$.
When $\theta^{(s)}$ produces strongly dependent pairwise marginals, we must use a larger buffer $B$.
In Section~\ref{sec:bounds}, we analyze, for a fixed value of $\theta$, how quickly the bias between $\bar{g}(\theta)$ and $\tilde{g}(\theta)$ decays with increasing $B$.
We show a geometric decay
\begin{equation}
\label{eq:geom_decay_rate}
\E_\SUBSEQ \| \bar{g}(\theta) - \tilde{g}(\theta) \|_2 \leq C_\theta \rho_\theta^{-B} \enspace, \enspace \text{ for some } C_\theta > 0 \enspace,
\end{equation}
where $\rho_\theta$ is large for i.i.d.\ data and small for strongly dependent data.
The term $C_\theta$ depends on the smoothness of $g(\theta)$ and how accurately $p_0(\latent_{t_{-B-1}})$ approximates $p(\latent_{t_{-B-1}} \, | \, y_{\FULLSEQ\backslash\BUFFSUBSEQ})$.
For a gradient accuracy of $\epsilon$, we only need a logarithmic buffer size $O(\log\epsilon^{-1})$.\footnote{As $\epsilon \geq C_\theta \rho_\theta^{-B} \Rightarrow B \geq -\log \epsilon/\log\rho_\theta + \log C_\theta/\log\rho_\theta \Rightarrow B$ is $O(\log \epsilon^{-1})$.}
Therefore our buffered gradient estimator reduces the computation time from $O(T)$
to $O(S + \log\epsilon^{-1})$.
By using buffered stochastic gradients $\tilde{g}$ with an appropriate buffer size $B$ in SGMCMC (Eq.~\eqref{eq:sgld_step} or~\eqref{eq:sgrld_step}), we can generate samples $\theta^{(s)}$ that are close to the samples that would be generated if we were to use the unbiased (but intractable) stochastic gradients $\bar{g}$.
In our experiments (Section~\ref{sec:experiments}), we find that modest buffers significantly correct for bias.

Our approach is similar to fixed-lag smoothing methods in the particle filter literature~\cite{chan2016theory,del2017biased,olsson2008sequential}, which approximate $p(\latent_t \, | \, y_{1:T}, \theta)$ using a right buffer $p(\latent_t \, | \, y_{1:t+B}, \theta)$ in a streaming fashion.
However, our approach, Eq.~\eqref{eq:efficient_potential_estimate}, differs by using both a left and a right buffer $p(\latent_t \, | \, y_{1:T}, \theta) = p(\latent_t \, | \, y_{t-B:t+B})$, which allow us to avoid a full passes over the data.

\subsection{Preconditioning and Fisher Information}
\label{sec:framework-sgrld}
The desirable properties for the preconditioning matrix $D(\theta)$ for SGRLD are
(i) the resulting dynamics takes advantage of the geometric structure of $\theta$,
(ii) both $D(\theta)$ and $\Gamma(\theta)$ can be efficiently computed, 
and (iii) neither $D(\theta)g(\theta)$ nor $\Gamma(\theta)$ are numerically unstable.

The \emph{expected Fisher information} $\mathcal{I}_y$ is the Riemannian metric proposed in \cite{girolami2011riemann}
\begin{equation}
\label{eq:preconditioning_with_fisher}
D^{-1}(\theta) = \mathcal{I}_y = \E_{y\,|\,\theta}\left[\grad^2 \log p(y \, | \, \theta) \right] \enspace.
\end{equation}
Unfortunately for SSMs, the lack of independence in the marginal likelihood
requires a double sum over $\FULLSEQ$ to compute $\mathcal{I}_y$, which is computationally intractable for long time series.
We instead replace $I_{y}$ with the \emph{complete data Fisher information} $I_{\latent, y}$
\begin{equation}
\mathcal{I}_{\latent,y} = \E_{\latent, y \,|\, \theta}\left[\grad^2 \log p(y, \latent \, \ \, \theta) \right] =
T \cdot \E_{\latent, y \, | \, \theta}\left[\grad^2 \log p(y_t, \latent_t \, | \, \latent_{t-1}, \theta)\right]  \enspace.
\end{equation}
Because $I_{\latent,y}$ can be calculated analytically for the SSMs we consider (Section~\ref{sec:models}),
we use $D(\theta) = I_{\latent,y}^{-1}$ when possible or approximations of $I_{\latent, y}^{-1}$ when not (see the Supplement for details).
In our experiments, we find that in practice, using preconditioning works well and outperforms vanilla SGLD.

\subsection{Algorithm Pseudocode}
\label{sec:framework-algorithm}
Algorithms~\ref{alg:sgrld} and \ref{alg:noisygradient-analytic} summarize our generic SGMCMC method for SSMs\footnote{Python code for our method is available at~\url{https://github.com/aicherc/sgmcmc_ssm_code}}.

\begin{algorithm}
\caption{SGRLD}
\label{alg:sgrld}
\begin{algorithmic}
\STATE{Input: data $y$, parameters $\theta^{(0)}$, stepsize $h$, subsequence length $S$, error tolerance $\epsilon$}
\FOR{$s = 0, 1, 2, \ldots, N_\text{steps}-1$}
\STATE{$\tilde{g}(\theta^{(s)})$ = \texttt{NoisyGradient}$(y, \theta^{(s)}, S, \epsilon)$}
\COMMENT{Algorithm~\ref{alg:noisygradient-analytic} or \ref{alg:noisygradient-slds}}
\STATE{$D^{(s)}, \Gamma^{(s)}$ =  \texttt{GetPreconditioner}$(\theta^{(s)})$}
\COMMENT{e.g. Eq.~\eqref{eq:preconditioning_with_fisher}}
\STATE{$\theta^{(s+1)} \leftarrow \theta^{(s)} + h^{(s)} \left[ D^{(s)}\tilde{g}(\theta^{(s)}) + \Gamma^{(s)}\right] + \mathcal{N}\left(0, 2h^{(s)} D^{(s)}\right)$}
\COMMENT{Eq.~\eqref{eq:sgrld_step}}
\ENDFOR
\STATE{Return $\theta^{(N_\text{steps})}$}
\end{algorithmic}
\end{algorithm}

\begin{algorithm}
\caption{\texttt{NoisyGradient} for analytic message passing}
\label{alg:noisygradient-analytic}
\begin{algorithmic}
\STATE{Input: data $y$, parameters $\theta$, subsequence length $S$, error tolerance $\epsilon$}
\STATE{$B$ = \texttt{BufferSize}$(\theta, S, \epsilon)$}
\STATE{$\SUBSEQ, \BUFFSUBSEQ$ = \texttt{GetBufferedSubsequence}$(y, S, B)$}
\STATE{$p(\latent_\SUBSEQ \, | \, y_\BUFFSUBSEQ, \theta)$ = \texttt{ForwardBackward}$(y,\, \BUFFSUBSEQ, \theta)$}
\COMMENT{Message Passing}
\STATE{$\tilde{g}(\theta) =
-\sum_{t \in \mathcal{S}}\frac{1}{\Pr(t \in \SUBSEQ)} \, \E_{\latent_\SUBSEQ | y_\BUFFSUBSEQ, \theta}[ \grad_\theta \log p(y_t, \latent_t | \latent_{t-1}) ]$}
\Comment{Eq.~\eqref{eq:efficient_potential_estimate}}
\STATE{Return $\tilde{g}(\theta)$}
\end{algorithmic}
\end{algorithm}

To select the buffer size $B$ in Algorithm~\ref{alg:noisygradient-analytic},
we choose $B$ large enough
such that the error using $B$ and a larger buffer size $B^*$ is small:
\begin{equation}
\label{eq:adaptive_B}
B = \min \left\{\hat{B} \in [0, B^*] \, : \, \E_\SUBSEQ \|\, \tilde{g}(\theta, \SUBSEQ, \hat{B}) - \tilde{g}(\theta, \SUBSEQ, B^*) \| < \epsilon  \right\}
\end{equation}
where $\tilde{g}(\theta, \SUBSEQ, B) = \E_{\latent | y_{\BUFFSUBSEQ}, \theta}[ \grad \log p(y_\SUBSEQ, \latent_\SUBSEQ \, | \, \theta) ]$ and the expectation over $\SUBSEQ$ is approximated with an empirical average over $N_S$ subsequences. 
Eq.~\eqref{eq:adaptive_B} uses $\tilde{g}(\theta, \SUBSEQ, B^*)$ as a proxy for $\tilde{g}(\theta, \SUBSEQ, T)$. 
As the error decays geometrically (Section~\ref{sec:bounds}), we found using $B^* = 100$ was conservative in practice.
Calculating $B$ using Eq.~\eqref{eq:adaptive_B} at every iteration for a new $\theta^{(s)}$ is impractical; therefore for our experiments, we use a fixed $B$, estimated using $\theta$ from a pilot run with $B=B^*$ and $N_S = 1000$.
In addition, instead of evaluating each $\hat{B}$ in $[0, B^*]$, we can estimate the required $B$ for a target error tolerance $\epsilon$ after estimating the error $\hat{\epsilon}$ of a single $\hat{B}$, by taking advantage of the geometric error scaling rate, Eq.~\eqref{eq:geom_decay_rate}, to obtain $B = \hat{B} + \log_{\rho_\theta}(\hat{\epsilon}/\epsilon)$ where $\rho_\theta$ is a bound on the geometric decay rate from theory.

\section{Buffered Gradient Estimator Error Bounds}
\label{sec:bounds}
In this section, we establish a bound on the expected error between
the unbiased gradient $\bar{g}(\theta)$ and our buffered gradient estimator $\tilde{g}(\theta)$ Eq.~\eqref{eq:efficient_potential_estimate}.
Given such a bound, we can control the overall error in our SGLD or SGRLD scheme when the SGMCMC dynamics possess a contraction property~\cite{johndrow2017error}.
Specifically,
if we can uniformly bound $\| \bar{g}(\theta) - \tilde{g}(\theta) \|_2 < \delta$, then the difference in a single step of SGMCMC, Eq.~\eqref{eq:sgrld_step}, using the unbiased and approximate gradients $\bar{g}$ and $\tilde{g}$ is bounded by $\delta h$. Therefore we can apply Theorem 1.11 of~\cite{johndrow2017error} which states the sample average of a test function evaluated on samples of the approximate-gradient $\tilde{g}$ chain, $\sum_{i < s} \varphi(\theta^{(i)})/s$, converges to the posterior expected value of the unbiased-gradient $\bar{g}$ chain, $\E_\theta[\varphi(\theta)]$, plus an additional error term proportional to $\delta h$.
For our analysis, we first consider the simple case of uniformly sampling a single sequence from $T/S$ separate subsequences (i.e. $\Pr(t \in \SUBSEQ) = S/T$ for all $t$) and assume the prior $p_0$ is stationary (i.e. $p_0(u_t) = \int p(u_t | u_{t-1}) p_0(u_{t-1}) du_{t-1}$).

Our approach is to bound $\| \bar{g}(\theta) - \tilde{g}(\theta)\|_2$ 
in terms of the Wasserstein distance between the exact posterior $\gamma_t(\latent_t) = p(\latent_t \, | \, y_\FULLSEQ, \theta)$ and our approximate posterior $\widetilde{\gamma}_t(\latent_t) = p(\latent_t \, | \, y_\BUFFSUBSEQ, \theta)$ and then show this Wasserstein distance decays geometrically.
To bound the Wasserstein distance,
we follow existing work on bounding Markov processes in Wasserstein distance~\cite{durmus2015quantitative, madras2010quantitative, rudolf2015perturbation}.
However, unlike previous work that focuses on the homogeneous Markov process of the joint model $\{\latent, y \, | \, \theta\}$,
we instead focus on the induced \emph{nonhomogeneous} Markov process of the conditional model $\{\latent \, | \, y, \theta\}$.
To do so, we use the forward ($f_t$) and backward ($b_t$) \emph{random maps} of $\{\latent \, | \, y, \theta\}$~\cite{diaconis1999iterated}
\begin{align}
\latent_t \sim p(\latent_t \, | \, y, \theta) \enspace &\Rightarrow \enspace (f_t(\latent_t), \latent_t) \sim p(\latent_{t+1}, \latent_t \, | \, y, \theta) \\
\latent_t \sim p(\latent_t \, | \, y, \theta) \enspace &\Rightarrow \enspace (b_t(\latent_t), \latent_t) \sim p(\latent_{t-1}, \latent_t \, | \, y, \theta) \enspace,
\end{align}
If $f_t$ and $b_t$ satisfy a contractive property, then we can bound the Wasserstein distance between $\gamma_t, \widetilde\gamma_t$ in terms of $\gamma_{t-1}, \widetilde\gamma_{t-1}$ and $\gamma_{t+1}, \widetilde\gamma_{t+1}$ respectively.
Bounding the error of the induced nonhomogeneous Markov process has been previously studied in the SSM literature using total variation (TV) distance~\cite{cappe2005inference, del2010forward, le2000exponential, tong2012ergodicity}.
These works bound the error in total variation distance by quantifying how quickly the smoothed posterior forgets the initial condition.
However, these bounds typically require stringent regularity conditions, which are hard to prove outside of finite or compact spaces\footnote{These bounds have been extended to non-compact spaces for the \emph{filtered} posterior, when the SSM satisfies a multiplicative drift condition~\cite{whiteley2013stability}.}.
In particular, these bounds are not immediately applicable for LGSSMs.
In contrast, we bound the error in Wasserstein distance by proving contraction properties of $f_t$ and $b_t$,
allowing us to handle continuous and mixed-type SSMs such as the LGSSM (Section~\ref{sec:models:LGSSM-bounds}).

Our main result is that if, for each fixed $\theta$, the gradient of $\log p(y, \latent \, | \, \theta)$ satisfies a Lipschitz condition and the random maps $\{f_t, b_t\}_{t \in \BUFFSUBSEQ}$ all satisfy a contraction property,
then the error $\| \bar{g}(\theta) - \tilde{g}(\theta)\|_2$ decays geometrically in the buffer size $B$.
\begin{theorem}
\label{theorem:geometric_error_bound}
Let $\epsilon_{\rightarrow}$ and $\epsilon_{\leftarrow}$ be the 1-Wasserstein distances between $\gamma_t$ and $\widetilde{\gamma}_t$ at the left and right ends of $\BUFFSUBSEQ$ respectively.
Let $\epsilon_1 = \max_{\BUFFSUBSEQ \subset \FULLSEQ}\{\epsilon_\rightarrow, \epsilon_\leftarrow\}$.
If the gradients of $\log p(y_t, \latent_t \, | \, \latent_{t-1}, \theta)$ are all Lipschitz in $\latent_{t-1:t}$ with constant $L_U$,
and random maps $f_t$ and $b_t$ are all Lipschitz\footnote{
The random mapping $\psi$ is Lipschitz with constant $L$ if ${\E_\psi\| \psi(\latent) - \psi(\latent') \|_2} \leq L {\|\latent - \latent' \|_2} \ \forall \latent, \latent'$.
} in $\latent_t$ with constant $L < 1$,
then we have
\begin{equation}
\| \bar{g}(\theta) - \tilde{g}(\theta)\|_2 \leq T \cdot L_U \cdot \frac{1+L}{1-L} \cdot \frac{1-L^S}{S} \cdot L^{B} \cdot 2\epsilon_1 \enspace.
\end{equation}
\end{theorem}
A similar result for when the gradient of the complete data loglikelihood is Lipschitz in $uu^T$ instead of $u$ (as needed for LGSSM)
will be proved in Section~\ref{sec:bounds-main}. 

As $L < 1$, Theorem~\ref{theorem:geometric_error_bound} states that the error of the buffered gradient estimator decays geometrically as $O(L^B)$.
Therefore, the required buffer size $B$ for an error tolerance of $\delta$ scales logarithmically as $O(\log \delta^{-1})$.
In contrast, the error of the gradient estimator decays only linearly in the subsequence length, $O(S^{-1})$;
therefore much longer subsequences, $O(\delta^{-1})$, are required to reduce bias. 
This agrees with the intuition that the bias is dominated by the error at the endpoints of subsequence.

Theorem~\ref{theorem:geometric_error_bound} requires bounding the Lipschitz constants of the gradient of the complete data loglikelihood and the random maps $f_t, b_t$ given the parameters $\theta$ and observations $y_\FULLSEQ$.
We show examples of these bounds for specific models in Section~\ref{sec:models:HMM-bounds} (HMMs) and \ref{sec:models:LGSSM-bounds} (LGSSMs).
Theorem~\ref{theorem:geometric_error_bound} also depends on the maximum Wasserstein distance $\epsilon_1$ between $\gamma_t$ and $\widetilde{\gamma}_t$ for all $\BUFFSUBSEQ \subset \FULLSEQ$ and $t \in \FULLSEQ$, which is finite.

The remainder of this section is as follows.
First, in Section~\ref{sec:bounds-wasserstein}, we show how to bound the error in $\bar{g}, \tilde{g}$ in terms of Wasserstein distances between $\gamma, \widetilde\gamma$.
Second, in Section~\ref{sec:bounds-geometric-decay}, we show these Wasserstein distances decay geometrically in $B$.
Finally, in Section~\ref{sec:bounds-main}, we prove our main results: Theorems~\ref{theorem:geometric_error_bound} and \ref{theorem:geometric_error_bound2}, and discuss relaxations of the assumptions on the sampling of subsequences $\SUBSEQ$ and the prior $p_0$.
To keep the presentation clean, we leave proofs of Lemmas to the Supplement.

\subsection{Functional Bound in terms of Wasserstein}
\label{sec:bounds-wasserstein}
We first review the definition of Wasserstein distance.
Let $\mathcal{W}_p(\gamma, \widetilde\gamma)$ be the $p$-Wasserstein distance
\begin{equation}
\mathcal{W}_p(\gamma, \widetilde\gamma) := \left[\inf_\xi \int \| \latent - \widetilde\latent \|_2^p \, d\xi(\latent, \widetilde\latent) \right]^{1/p}
\end{equation}
where $\xi$ is a joint measure or \emph{coupling} over $(\latent, \widetilde\latent)$ with marginals
$\int_{\widetilde{\latent}} d\xi(\latent, \widetilde{\latent}) = d\gamma(\latent) $
and
$\int_{\latent} d\xi(\latent, \widetilde{\latent}) = d\widetilde{\gamma}(\widetilde{\latent})$.
Wasserstein distance satisfies all the properties of a metric.
A useful property of the $1$-Wasserstein distance is the following Kantorovich-Rubinstein duality formula for the difference of expectations of Lipschitz functions~\cite{villani2008optimal}
\begin{equation}
\label{eq:kantorovich-rubinstein}
\mathcal{W}_1(\gamma, \widetilde\gamma) = \sup_{\|\psi\|_{Lip} \leq 1}\left\{ \int \psi \, d\gamma - \int \psi \, d\widetilde\gamma \right\} \enspace \Rightarrow \enspace |\E_\gamma[\psi] - \E_{\widetilde\gamma}[\psi] | \leq \| \psi \|_{Lip} \cdot \mathcal{W}_1(\gamma, \widetilde\gamma) \enspace,
\end{equation}
where $\| \psi \|_{Lip}$ denotes the Lipchitz constant of $\psi$.

We connect the error $\|\bar{g}-\tilde{g}\|_2$ to the Wasserstein distances between $\gamma, \widetilde\gamma$,
by applying this duality formula Eq.~\eqref{eq:kantorovich-rubinstein} to the difference of Eqs.~\eqref{eq:unbiased_potential_estimate} and \eqref{eq:efficient_potential_estimate}
\begin{equation}
\label{eq:diff_gradient_estimates}
\bar{g}(\theta) - \tilde{g}(\theta) = \frac{T}{S} \sum_{t \in \SUBSEQ} \E_{\gamma_{t-1:t}}\left[\grad \log p(y_t, \latent_t | \latent_{t-1}, \theta )\right] -  \E_{\widetilde\gamma_{t-1:t}}\left[\grad \log p(y_t, \latent_t | \latent_{t-1}, \theta )\right] .
\end{equation}
Applying the triangle inequality gives Lemma~\ref{lemma:grad_error_to_wasserstein}.

\begin{lemma}
\label{lemma:grad_error_to_wasserstein}
If $\grad \log p(y_t, \latent_t | \latent_{t-1}, \theta)$ are Lipschitz in $\latent_{t-1:t}$ with constant $L_U$,
\begin{equation}
\| \bar{g}(\theta) - \tilde{g}(\theta) \|_2 \leq \frac{T}{S} \cdot L_U \cdot \sum_{t \in \SUBSEQ} \mathcal{W}_1(\gamma_{t-1:t}, \widetilde\gamma_{t-1:t}).
\end{equation}
\end{lemma}

If $\grad \log p(y_t, \latent_t \, | \, \latent_{t-1}, \theta)$ is not Lipschitz in $\latent_{t-1:t}$, but is Lipschitz in $\latent_{t-1:t}\latent_{t-1:t}^T$ (as in LGSSMs),
then the following Lemma lets us bound the $1$-Wasserstein distance of $\latent\latent^T$ in terms of the $2$-Wasserstein distance of $\latent$.
\begin{lemma}
\label{lemma:second_wass_moment}
Let $\gamma'$ be the distribution of $\latent\latent^T$.
Let $\widetilde\gamma'$ be the distribution of $\widetilde\latent\widetilde\latent^T$.
Let $M = \E_\gamma[\|\latent\|_2^2] < \infty$.
(Note $\mathcal{W}_2(\gamma, \widetilde\gamma) < \infty$ implies $\E_\gamma[\|\latent\|_2^2] < \infty$.)
Then,
\begin{equation*}
\mathcal{W}_1(\gamma', \widetilde\gamma')
    \leq  (2\sqrt{M} + 1) \cdot \max\left\{\mathcal{W}_2(\gamma, \widetilde\gamma)^{1/2},\mathcal{W}_2(\gamma, \widetilde\gamma)\right\}  \enspace.
\end{equation*}
\end{lemma}

\subsection{Geometric Wasserstein Decay}
\label{sec:bounds-geometric-decay}
We first review why contractive random maps induce Wasserstein bounds.
If two distributions $\gamma_t, \gamma_t'$ have identically distributed random maps $f_t, f_t'$, that is there exists a random function $f_t$ satisfying
\begin{equation}
    u \sim \gamma_t \text{ and }  u' \sim \gamma_t' \enspace \Rightarrow \enspace
    f_t(u) \sim \gamma_{t+1} \text{ and } f_t(u') \sim \gamma_{t+1}' \enspace,
\end{equation}
then we can bound the Wasserstein distance of $\gamma_{t+1}, \gamma_{t+1}'$ in terms of the Wasserstein distance of $\gamma_t, \gamma_t'$ given a bound on the random map's Lipschitz constant $\| f_t \|_{Lip} < L$
\begin{align}
\label{eq:wasserstein_random_map}
\mathcal{W}_p(\gamma_{t+1}, \gamma_{t+1}')^p &= \inf_{\xi_{t+1}} \int \| \latent_{t+1} - \latent_{t+1}'\|_2^p \, d\xi_{t+1}(\latent_{t+1}, \latent_{t+1}') \\
\nonumber
&\leq \inf_{\xi_t} \int \| f_t(\latent_t) - f_t(\latent_t') \|_2^p \, d\xi_t(\latent_t, \latent_t') df_t \\
\nonumber
&\leq \inf_{\xi_t} \int L^p \cdot \| \latent_t - \latent_t' \|_2^p \, d\xi_t(\latent_t, \latent_t') = L^p \cdot \mathcal{W}_p(\gamma_t, \gamma_t')^p\enspace.
\end{align}

Unfortunately for SSMs, Eq.~\eqref{eq:wasserstein_random_map} does not apply as
the random maps $f_t, b_t$ of $\gamma$ and $\widetilde{f}_t, \widetilde{b}_t$ of $\widetilde\gamma$ are \emph{not identically} distributed.
To see this,
we first review the conditional probability distributions used to define $f_t, b_t$.
The forward random map $f_t$ draws $\latent_{t+1} \, | \, \latent_t$ from the \emph{forward smoothing kernel}
\begin{equation}
\label{eq:forward_smoothing_kernel}
\mathcal{F}_t(\latent_{t+1} \, | \, \latent_{t}) := p(\latent_{t+1} \, | \, \latent_t, y_{>t}) = p(\latent_{t+1} \, | \, \latent_t) p(y_{t+1} \, | \, \latent_{t+1}) \beta_{t+1}(\latent_{t+1}) / \beta_t(\latent_{t})
\end{equation}
and the backward random map $b_t$ draws $\latent_{t-1} \, | \, \latent_t$ from the \emph{backward smoothing kernel}
\begin{equation}
\label{eq:backward_smoothing_kernel}
\mathcal{B}_t(\latent_{t-1} \, | \, \latent_{t}) := p(\latent_{t-1} \, | \, \latent_t, y_{\geq t}) = p(\latent_t \, | \, \latent_{t-1})p(y_t \, | \, \latent_t) \alpha_{t-1}(\latent_{t-1})/\alpha_t(\latent_t) \enspace.
\end{equation}
Because $\widetilde\gamma$ uses different forward and backward messages $\widetilde\alpha$, $\widetilde\beta$ in Eqs.~\eqref{eq:forward_smoothing_kernel} and \eqref{eq:backward_smoothing_kernel},
the kernels $\widetilde{\mathcal{F}}_t, \widetilde{\mathcal{B}}_t$ are not identical to $\mathcal{F}_t, \mathcal{B}_t$
(and the random maps are \emph{not} identically distributed).
This is unlike homogeneous Markov chains, where the kernels are identical at each time $t$ (and the random maps are identically distributed).

Instead of connecting $\gamma$ to $\widetilde\gamma$ directly,
we use the triangle inequality to connect them through an intermediate distribution $\widehat{\gamma} := p(\latent \, | \, y_{t \geq t_{-B}}, \theta)$
\begin{equation}
\label{eq:wasserstein_triangle}
\mathcal{W}_p(\gamma, \widetilde\gamma) \leq \mathcal{W}_p(\gamma, \widehat\gamma)  + \mathcal{W}_p(\widehat\gamma, \widetilde\gamma) \enspace.
\end{equation}
Introducing this particular intermediate distribution $\widehat{\gamma}$ is the key step for our Wasserstein bounds between $\gamma$ and $\widetilde\gamma$.
Because $\widehat{\gamma}$ conditions on all $y_t$ after $y_\BUFFSUBSEQ$, $\widehat{\gamma}$ and $\gamma$ have identical backward messages $\beta_t$ and therefore identically distributed forward random maps $f_t$.
Similarly, because $\widehat{\gamma}$ does not condition on $y_t$ before $y_\BUFFSUBSEQ$, $\widehat{\gamma}$ and $\widetilde{\gamma}$ have identical forward messages $\widetilde{\alpha_t}$ and identically distributed backward random maps $\widetilde{b}_t$.

Therefore, we can bound $\mathcal{W}_p(\gamma, \widehat\gamma)$ using $f_t$ and bound $\mathcal{W}_p(\widehat\gamma, \widetilde\gamma)$ using $\widetilde{b}_t$
with the contraction trick Eq.~\eqref{eq:wasserstein_random_map} giving us Lemma~\ref{lemma:geometric_wasserstein}.

\begin{lemma}
\label{lemma:geometric_wasserstein}
If there exists $L_f, L_b < 1$ such that for all $t \in \BUFFSUBSEQ$,
$\| f_t\|_{Lip} < L_f$ and $\| \widetilde{b}_t \|_{Lip} < L_b$, then for all $t \in \SUBSEQ$ we have
\begin{align}
\label{eq:geometric_wasserstein_forward}
\mathcal{W}_p(\gamma_{t-1:t}, \widehat\gamma_{t-1:t}) &\leq (1 + L_f^p)^{1/p} \cdot \mathcal{W}_p(\gamma_{t-1}, \widehat\gamma_{t-1}) \\
    \nonumber
    &\leq (1 + L_f^p)^{1/p} \cdot L_f^{t-1-t_{-B}} \cdot \mathcal{W}_p(\gamma_{t_{-B}}, \widehat\gamma_{t_{-B}}) \\
\label{eq:geometric_wasserstein_backward}
\mathcal{W}_p(\widehat\gamma_{t-1:t}, \widetilde\gamma_{t-1:t}) &\leq (1 + L_b^p)^{1/p} \cdot \mathcal{W}_p(\widehat\gamma_{t}, \widetilde\gamma_{t}) \\
    \nonumber
    &\leq (1 + L_b^p)^{1/p} \cdot L_b^{t_{S+B}-t} \cdot \mathcal{W}_p(\widehat\gamma_{t_{S+B}}, \widetilde\gamma_{t_{S+B}})
\end{align}
\end{lemma}
We show sufficient conditions for the random maps to be contractions (i.e. $L_f, L_b < 1$) for specific models in Section~\ref{sec:models:HMM-bounds} (HMMs) and \ref{sec:models:LGSSM-bounds} (LGSSMs).

\subsection{Proof of Main Theorems}
\label{sec:bounds-main}
Putting together the results of the previous two subsections gives us our geometric error bounds: Theorem~\ref{theorem:geometric_error_bound} when the gradient terms are Lipschitz in $\latent$ and Theorem~\ref{theorem:geometric_error_bound2} when the gradient terms are Lipschitz in $\latent\latent^T$.
Both theorems require the random maps of the forward and backward smoothing kernels are contractions.
We first prove Theorem~\ref{theorem:geometric_error_bound}.

\begin{proof}[Proof of Theorem \ref{theorem:geometric_error_bound}]
Combining Lemmas~\ref{lemma:grad_error_to_wasserstein} and \ref{lemma:geometric_wasserstein} with some algebra
\begin{align*}
\| \bar{g}(\theta) - \tilde{g}(\theta)\|_2
    &\leq \frac{T}{S} \cdot L_U \cdot \sum_{t \in \SUBSEQ} \mathcal{W}_1(\gamma_{t-1:t}, \widetilde\gamma_{t-1:t})
    \\
    &\leq \frac{T}{S} \cdot L_U \cdot \sum_{t \in \SUBSEQ} \mathcal{W}_1(\gamma_{t-1:t}, \widehat\gamma_{t-1:t}) + \mathcal{W}_1(\widehat\gamma_{t-1:t}, \widetilde\gamma_{t-1:t})
    \\
    &\leq \frac{T}{S} \cdot L_U \cdot \sum_{t = 1}^S (1+L_f) L_f^{B+t-1} \epsilon_1  + (1+L_b) L_b^{B+S-t}\epsilon_1
    \\
    &\leq T \cdot L_U \cdot \frac{1+L}{1-L} \cdot \frac{1-L^S}{S} \cdot L^{B} \cdot 2\epsilon_1 \enspace,
\end{align*}
where $\max_{\BUFFSUBSEQ \subset \FULLSEQ}\{\mathcal{W}_1(\gamma_{t_{-B}}, \widehat{\gamma}_{t_{-B}}), \mathcal{W}_1(\widehat\gamma_{t_{S+B}}, \widetilde{\gamma}_{t_{S+B}})\} = \max_{\BUFFSUBSEQ \subset \FULLSEQ}\{\epsilon_{\rightarrow}, \epsilon_{\leftarrow}\} = \epsilon_1$.
\end{proof}

We now prove a similar result for when $\grad \log p(y, \latent_t \,|\, \latent_{t-1} \theta)$ is Lipschitz in $uu^T$.
\begin{theorem}
\label{theorem:geometric_error_bound2}
Let $\epsilon_2 = \max_{\BUFFSUBSEQ \subset \FULLSEQ}\{\mathcal{W}_2(\gamma_{t_{-B}}, \widehat{\gamma}_{t_{-B}}), \, \mathcal{W}_2(\widehat\gamma_{t_{S+B}}, \widetilde{\gamma}_{t_{S+B}})\}$.
If the gradients are Lipschitz in $\latent\latent^T$ with constant $L_U'$,
and there exists $L_f, L_b < 1$ for Lemma~\ref{lemma:geometric_wasserstein},
then with $L = \max\{L_f, L_b\}$ and $L_U = (2 \sqrt{\E_\gamma\|\latent\|^2_2} + 1) L_U'$
\begin{equation*}
\| \bar{g}(\theta) - \tilde{g}(\theta)\|_2 \leq T \cdot L_U \cdot \frac{\sqrt{1+L^2}}{1-L^{1/2}} \cdot \frac{1-L^{S/4}}{S/2} \cdot L^{B/2} \cdot \max_{r \in \{ 1/2, \, 1 \}} (2\epsilon_2)^r \,.
\end{equation*}
\end{theorem}

Similar to Theorem~\ref{theorem:geometric_error_bound},
Theorem~\ref{theorem:geometric_error_bound2} states that the squared error of the buffered gradient estimator decays geometrically if the complete-data loglikelihood is Lipschitz in $uu^T$ instead of $u$.
However, the price we pay is a square-root: the error decays $O(L^{B/2})$ instead of $O(L^{B})$.

\begin{proof}[Proof of Theorem \ref{theorem:geometric_error_bound2}]
Applying Lemmas~\ref{lemma:second_wass_moment} and \ref{lemma:geometric_wasserstein}, we have
\begin{align*}
\| \bar{g}(\theta) - \tilde{g}(\theta)\|_2
    &\leq \frac{T}{S} \cdot L_U \cdot \sum_{t \in \SUBSEQ} \max_{r \in \{1/2, \, 1\}} \left[\mathcal{W}_2(\gamma_{t-1:t}, \widehat\gamma_{t-1:t}) + \mathcal{W}_2(\widehat\gamma_{t-1:t}, \widetilde\gamma_{t-1:t})\right]^r \\
    &\leq \frac{T}{S} \cdot L_U \cdot \sum_{t = 1}^S \max_{r \in \{1/2, \, 1\}} \left[(L^{B +t-1} + L^{B+S-t}) \sqrt{1+L^2} \epsilon_2\right]^r \\
    &\leq \frac{T}{S} \cdot L_U \cdot \sum_{t = 1}^S L^{(B + \min\{t-1, S-t\})/2} \cdot \sqrt{1+L^2} \cdot  \max_{r \in \{1/2, \, 1\}} (2 \epsilon_2)^r \\
    &\leq \frac{T}{S} \cdot L_U \cdot 2\cdot\frac{1-L^{S/4}}{1-L^{1/2}} \cdot L^{B/2}  \cdot \sqrt{1+L^2} \cdot \max_{r \in \{ 1/2, \, 1 \}} (2 \epsilon_2)^r
\end{align*}
\end{proof}

Our error analysis (Theorems~\ref{theorem:geometric_error_bound} and \ref{theorem:geometric_error_bound2}) indicates that only a logarithmic buffer size is required to control the bias to a fixed error tolerance $\delta$.

\subsubsection{Relaxations of Assumptions}
We now briefly discuss relaxations of the assumptions on $\Pr(t \in \SUBSEQ)$ and $p_0$.

If the contiguous subsequences are not sampled from a strict partition (i.e. $\Pr(t \in \SUBSEQ) \neq S/T$ for all $t$),
then we can replace the factor of $T/S$ in Theorems~\ref{theorem:geometric_error_bound} and \ref{theorem:geometric_error_bound2} with $\max_t \Pr(t \in \SUBSEQ)^{-1}$. 
Additional details on different sampling methods for $\SUBSEQ$ can be found in the Supplement.

If the initial distribution for $u_{t_{-B-1}}$ of our buffered stochastic gradient, $p_0$, is not stationary, then our approximate posterior over the latent states $\tilde{\gamma}_t(u_t)$ is not equal to $p(u_t \, | \, y_\BUFFSUBSEQ, \theta)$.
However Theorems~\ref{theorem:geometric_error_bound} and \ref{theorem:geometric_error_bound2} will still apply;
the choice of initial distribution only affects the Wasserstein distance between $\gamma_t, \tilde{\gamma}_t$ and therefore the terms $\epsilon_1, \epsilon_2$ in the Theorems.
In fact, the optimal initial distribution is $p(u_{t_{-B}} | y_{\FULLSEQ \backslash \BUFFSUBSEQ})$, which minimizes the Wasserstein distance of $\gamma, \tilde{\gamma}$.

\section{Example Models}
\label{sec:models}

In this section, we provide examples of how to apply the generic framework of Section~\ref{sec:framework} and bounds of Section~\ref{sec:bounds} to common SSMs.

\subsection{Gaussian HMM}
\label{sec:models:HMM}
We consider discrete latent state HMMs with Gaussian emissions.
The complete data likelihood of a Gaussian HMM is as follows
\begin{equation}
\label{eq:gaussian_hmm_model}
    p(y, z \, | \, \theta) = \prod_{t=1}^T \Pi_{z_{t-1}, z_{t}} \cdot \mathcal{N}(y_t \, | \, \mu_{z_t} , \Sigma_{z_k}) \enspace,
\end{equation}
where
$y_t \in \R^m$ are the observations,
$\latent_t \equiv z_t \in \{1, \ldots, K\}$ are the discrete latent variables,
and
$\theta = \{\Pi, \mu, \Sigma\}$ are the parameters with
$\Pi_k \in \Delta^K$ (simplex over $K$ states), $\mu_k \in \R^{m}$, $\Sigma_k \in \mathbb{S}^{m}_+$ (positive definite matrices) for $k = 1,\ldots, K$.
In practice, we use the \emph{expanded mean} parameters of $\Pi$ instead of $\Pi$ (as in~\cite{patterson2013stochastic})
and the \emph{Cholesky decomposition} of $\Sigma^{-1}_k$ instead of $\Sigma_k$ to ensure positive definiteness.
As the latent states are discrete over a finite space,
the forward backward algorithm for an HMM can be done in closed-form; thus, pairwise latent marginals $\gamma_{t-1:t}(z_{t-1}, z_t)$, gradients $\grad U(\theta)$ and preconditioning terms $D(\theta)$ and $\Gamma(\theta)$ are straightforward to calculate.
Complete details are provided in 
the Supplement.

\subsubsection{Error Bound Coefficients}
\label{sec:models:HMM-bounds}
In the finite discrete variable case, conditions for bounding the Lipschitz coeffficient of the smoothing kernels $\mathcal{F}_t, \mathcal{B}_t$ (as needed for Section~\ref{sec:bounds-geometric-decay}) are equivalent to conditions for bounding their \emph{Dobrushin coefficients}~\cite{cappe2005inference, del2010forward}.
The Dobrushin coefficient for a transition kernel $\mathcal{Q}$ is
\begin{equation}
\label{eq:dobrushin}
\delta(\mathcal{Q}) = \sup_{z, z'} \frac{1}{2} \| \mathcal{Q}(z, \cdot) - \mathcal{Q}(z', \cdot) \|_{TV} = \frac{\| \mathcal{Q}(z, \cdot) - \mathcal{Q}(z', \cdot) \|_{TV}}{\| \delta_z - \delta_{z'}\|_{TV}} \enspace.
\end{equation}
The final term of Eq.~\eqref{eq:dobrushin} show the connection between Dobrushin coefficients and Lipschitz coefficients: it is the ratio of the distance of between kernels $\mathcal{Q}(z, \cdot), \mathcal{Q}(z', \cdot)$ with the distance between point masses at $z$ and $z'$.
Therefore for discrete latent states, $L_f = \max_t \delta(\mathcal{F}_t)$ and $L_b = \max_t \delta(\mathcal{B}_t)$.

In the discrete case, sufficient conditions for $L_f, L_b < 1$
are well known (See~\cite{cappe2005inference} Chapter 4.3).
If the transition matrix $\Pi$ satisfies the \textit{strong mixing condition},
that is, there exists constants $\sigma^{-}$ and $\sigma^{+}$ with $0 < \sigma^{-} \leq \sigma^{+}$ and a probability distribution $\kappa \in \Delta^{K}$ over $z$ such that
$\sigma^{-} \kappa(z') \leq \Pi_{z, z'} \leq \sigma^{+} \kappa(z')$ and $\E_\kappa[p(y\, |\, z)] < \infty$,
then the Dobrushin coefficients are bounded by $L = 1-\sigma^{-}/\sigma^{+}$.
Relaxations of this condition can be found in~\cite{cappe2005inference, del2010forward}.
Alternatively, we can obtain tighter bounds for HMMs via estimating the Lyapunov exponents for the underlying random dynamical systems defined by random maps $f_t$ and $b_t$~\cite{ye2017estimate, ma2017stochastic}.

Finally, the Lipschitz constant $L_U$ for Lemma~\ref{lemma:grad_error_to_wasserstein} is
\begin{equation}
L_U = \max_{t \in \SUBSEQ, z_t, z_t'} \| \grad \log p(y_t, z_t \, | \, z_{t-1}, \theta) - \grad \log p(y_t, z_t' \, | \, z_{t-1}', \theta) \| \enspace.
\end{equation}
This is easy to compute since at each iteration $y$ and $\theta=\theta^{(s)}$ are fixed.
Given these bounds on $L_U$ and $L$, we can use Theorem~\ref{theorem:geometric_error_bound} to select the buffer size $B$ to ensure approximate convergence to the stationary distribution.

\subsection{Autoregressive HMM}
\label{sec:models:ARHMM}
We now consider ARHMMs,
a generalization of the discrete state HMM
where each observation depends not only on the latent state,
but also on the last $p$ observations.
Specifically,
the discrete latent state $z_t$ determines which AR($p$) process models the dynamics of $y$ at time $t$.
The complete data likelihood of an ARHMM is as follows
\begin{equation}
\label{eq:arhmm_model}
    p(y, z \, | \, \theta) = \prod_{t=1}^T \Pi_{z_{t-1}, z_{t}} \cdot \mathcal{N}(y_t \, | \, A_{z_t} \overline{y_{t}}, Q_{z_k}) \enspace,
\end{equation}
where
$y_t \in \R^m$ are the observations,
$\overline{y_{t}} = y_{t-1:t-p}$ are the $p$-lagged observations,
$\latent_t \equiv z_t \in \{1, \ldots, K\}$ are the discrete latent variables, and
$\theta = \{\Pi, A, Q\}$ are the parameters with
$\Pi_k \in \Delta^K$, $A_k \in \R^{m \times mp}$,
$Q_k \in \mathbb{S}^{m}_+$ for $k = 1,\ldots, K$.
From Eq.~\eqref{eq:arhmm_model}, we see that the ARHMM is a time-dependent mixture of $K$ AR processes of order $p$.
The pairwise latent marginals, gradients, and preconditioning terms for an ARHMM are calculated similarly to the Gaussian HMM.
Further details are provided in 
the Supplement.
The theory and constants for the error bounds of Section~\ref{sec:bounds} are
identical to those presented for the Gaussian HMM.

\subsection{Linear Gaussian SSM}
\label{sec:models:LGSSM}
A linear Gaussian SSM (LGSSM), also called a linear dynamical system (LDS),
consists of a latent Gaussian (vector) autoregressive process over states $\latent_t \equiv x_t \in \R^n$ and conditionally Gaussian emissions $y_t \in \R^m$~\cite{bishop2006pattern, lutkepohl2005new}.
Specifically,
\begin{equation}
    p(y, x \, | \, \theta) = \prod_{t=1}^T \mathcal{N}(x_t \, | \, A x_{t-1}, Q) \cdot \mathcal{N}(y_t \, | \, C x_t, R) \enspace,
\end{equation}
where $A \in \R^{n \times n}$ is the latent state transition matrix,
$Q \in \mathbb{S}^{n}_+$ is the transition noise covariance,
$C \in \R^{m \times n}$ is the emission matrix,
and $R \in \mathbb{S}^m_+$ is the emission noise covariance.
Together $A, Q, C, R$ are the model parameters $\theta$.
The matrices $A$, $C$, and $Q$ are unidentifiable without additional restriction,
as applying an orthonormal transformation $M$ gives an equivalent representation $\tilde{A} = MAM^{-1}$, $\tilde{C}=CM^{-1}$, $\tilde{Q} = MQM^T$.
To enforce identifiability, we choose to restrict the first $\min(n,m)$ rows and columns of $C$ to the identity matrix.
In practice, we use the Cholesky decompositions $\psi_Q, \psi_R$ of $Q^{-1}, R^{-1}$ (respectively) instead of $Q, R$.
The recursions for the forward backward algorithm for LGSSMs is known as the Kalman smoother~\cite{cappe2005inference, bishop2006pattern, fox2009bayesian}.
Because the transition and emission processes are linear Gaussian, all forward messages, backward messages, and pairwise latent marginals $\gamma_{t-1:t}(x_{t-1}, x_{t})$ are Gaussian; therefore, the gradients and preconditioning matrix can be calculated analytically.
Further details are provided in 
the Supplement.

\subsubsection{Error Bound Coefficients}
\label{sec:models:LGSSM-bounds}
The random maps of an LGSSM are strict contractions under mild conditions (Lemmas~\ref{lemma:LGSSM_forward_maps}, \ref{lemma:LGSSM_backward_maps}) and
the gradients are Lipschitz in $xx^T$ (Lemma~\ref{lemma:LGSSM_Lipschitz}).
Therefore, Theorem~\ref{theorem:geometric_error_bound2} applies.

\begin{lemma}
\label{lemma:LGSSM_forward_maps}
The forward random maps of an LGSSM are Gaussian linear maps.
Specifically, $f_t(x_t) = F^{f}_t x_t + \zeta^{f}_t$,
where $\zeta^{f}_t$ is a Gaussian random intercept and $F^{f}_t$ is a matrix function of $\theta$ and $y_{>t}$.
As a linear map, the Lipschitz constant of $f_t$ is
\begin{equation}
\|f_t \|_{Lip} = \| F^{f}_t \|_2 \leq \| A (I_n + Q C^T R^{-1} C)^{-1} \|_2 = L_f \enspace.
\end{equation}
As $\| (I_n + Q C^T R^{-1} C)^{-1} \|_2 < 1$,
if $\| A\|_2 < 1$, then $\| f_t \|_{Lip} \leq L_f < 1$ for all $t$.
\end{lemma}
\begin{lemma}
\label{lemma:LGSSM_backward_maps}
The backward random maps of an LGSSM are Gaussian linear maps.
Specifically, $b_t(x_t) = F^b_t x_t + \zeta^b_t$,
where $\zeta^b_t$ is a Gaussian random intercept and $F^b_t$ is a matrix function of $\theta$ and $y_{<t}$.
As a linear map, the Lipschitz constant of $b_t$ is
\begin{equation}
\|b_t \|_{Lip} = \| F^b_t \|_2 \leq \|A (Q A^T Q^{-1} A + Q C^T R^{-1} C)^{-1} \|_2 = L_b \enspace.
\end{equation}
If $\| A\|_2 < \| (Q A^T Q^{-1} A  + Q C^T R^{-1} C)^{-1} \|_2$, then $\| f_t \|_{Lip} \leq L_f < 1$ for all $t$.
In addition, when the variance of the prior $p_0(x)$ is less than the steady state variance $V_\infty = Q + A V_\infty A^T$ and $A$ commutes with $Q$, we obtain a tighter bound
\begin{equation}
\|b_t \|_{Lip} = \| F^b_t \|_2 \leq \|A (I_n + Q C^T R^{-1} C)^{-1} \|_2 = L_b \enspace.
\end{equation}
In this case, if $\| A\|_2 < 1$, then $\| b_t \|_{Lip} \leq L_b < 1$ for all $t$.
\end{lemma}

Lemmas~\ref{lemma:LGSSM_forward_maps} and \ref{lemma:LGSSM_backward_maps} agree with intuition, when $\|A\|_2 \approx 0$ (no connection between $x_{t-1}$ and $x_t$) or $\| Q \|_2 \gg \| R \|_2$ (transition noise is much larger than emission noise), then $L_f, L_b \approx 0$ (observations can be treated independently). Conversely, when $\|A\|_2 \approx 1$ and $\| Q \|_2 \ll \| R\|_2$, then $L_f, L_b \approx 1$ and buffering is necessary.

\begin{lemma}
\label{lemma:LGSSM_Lipschitz}
As $x,y$ are jointly Gaussian in the LGSSM,
the gradient of the complete data loglikelihood is a quadratic form in $xx^T$
with matrices
\begin{align}
 \Omega = \{& I_n \otimes Q^{-1}, I_n \otimes Q^{-1} A, Q^{-1/2} \otimes I_n, Q^{-1/2} A \otimes I_n, Q^{-1/2} \otimes A, Q^{-1/2}A \otimes A, \\
 \nonumber
 & I_n \otimes R^{-1}, I_n \otimes R^{-1} C, R^{-1/2}C \otimes C \} \enspace,
\end{align}
where $Q^{-1/2} = \psi_Q$ and $R^{-1/2} = \psi_R$.
Therefore a bound for the Lipschitz constant is $L_U' = \max_{\omega \in \Omega} \|\omega\|_2$.
This bound grows in $\|A\|, \|C\|, \|Q\|^{-1}, \|R\|^{-1}$.
\end{lemma}
The proofs can be found in the Supplement.

\subsection{Switching Linear Dynamical System (SLDS)}
\label{sec:models:SLDS}

Switching linear dynamical systems (SLDSs) are an example of a state space model with both discrete and continuous latent variables.
The form of SLDS models that we consider is
\begin{equation}
    \label{eq:slds_complete_likelihood}
    p(y, x, z \, | \, \theta) = \prod_{t=1}^T \Pi_{z_{t-1}, z_{t}} \cdot \mathcal{N}(x_t \, | \, A_{z_t} x_{t-1}, Q_{z_t}) \cdot \mathcal{N}(y_t \, | \, C x_t, R) \enspace,
\end{equation}
where $y_t \in \R^m$ are the observations,
$\latent_t \equiv (x_t, z_t) \in \R^n \times \{1, \ldots, K\}$ are the mixed-type latent state sequence,
and $\theta = \{\Pi, A, Q, C, R\}$ the model parameters with
$\Pi_k \in \Delta^K$, $A_k \in \R^{n \times n}$, $Q_k \in \mathbb{S}^{n}_+$ for $k = 1,\ldots, K$, $C \in \R^{m \times n}$ and $R \in \mathbb{S}^m_+$.
The SLDS of Eq.~\eqref{eq:slds_complete_likelihood} can be viewed either as a latent AR(1)-HMM with conditional Gaussian emissions or
as hidden Markov switches of a LGSSM.
As an extension of the ARHMM, the latent continuous state sequence $x_t$ can \emph{smooth} noisy observations.
As an extension of the LGSSM, the latent discrete state sequence $z_t$ allows modeling of more complex dynamics by \emph{switching} between different states (or regimes).

\begin{figure}[t]
\centering
\begin{minipage}[c]{.5\textwidth}
    \centering
    \includegraphics[width=0.99\linewidth]{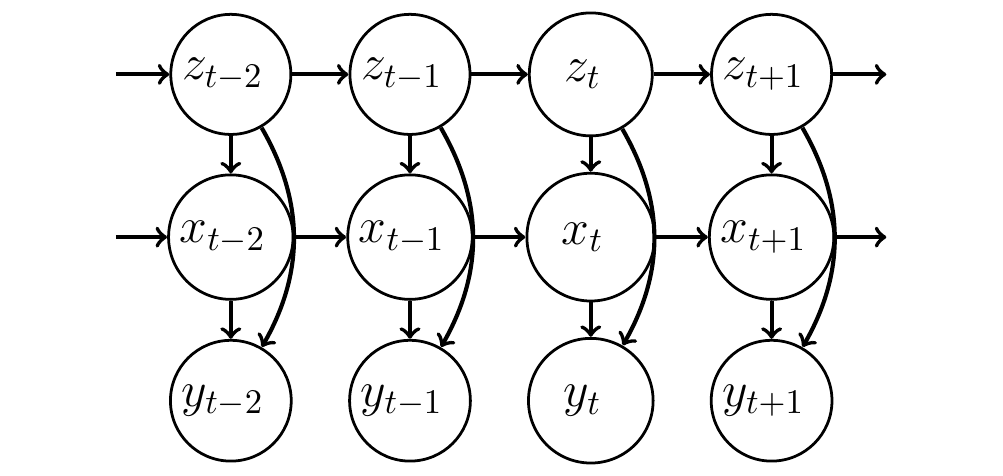}
\end{minipage}
\caption{Graphical Model of a SLDS.}
\label{fig:slds_graphical_model}
\end{figure}

\subsubsection{Gradient Estimators}
\label{sec:models:SLDS:gradients}
Unlike previous models, the forward-backward algorithm for the latent variables $(x, z)$ in an SLDS does not have a closed form.
Specifically, the transition kernel for $x$ is a Gaussian mixture,
so the forward and backward messages of $x$ are Gaussian mixtures with an exponentially increasing number of components (e.g. $\alpha_t$ has $K^t$ components).
Because the forward-backward algorithm is intractable for SLDSs,
we rely on sampling $(x,z)$ and forming a Monte Carlo estimate of the expectation in Fisher's identity Eq.~\eqref{eq:efficient_potential_estimate}.
We consider various options of this Monte Carlo estimate below.
To sample $(x,z)$, we use a blocked Gibbs scheme as in~\cite{fox2011bayesian}, detailed in the Supplement.

Given a collection of $N$ samples from blocked Gibbs $\{x^{(r)}, z^{(r)}\} \sim x, z \,|\, y, \theta$,
we construct three different estimators for the marginal loglikelihood.
The first estimator, replaces the expectation in Eq.~\eqref{eq:efficient_potential_estimate} with a Monte Carlo average
\begin{equation}
\label{eq:slds_naive}
\E_{x, z | y, \theta}[\grad \log p(y, x, z \, | \, \theta)] \approx \frac{1}{N}\sum_{r = 1}^{N} \grad \log p(y, x^{(r)}, z^{(r)} \, | \, \theta) \enspace.
\end{equation}
We construct two additional estimators by analytically integrating out either one of the two latent variables.
These estimators are the \textit{Rao-Blackwellization} of the naive Monte Carlo estimate~\cite{cappe2005inference}.
Integrating out either $x$ or $z$, gives us
\begin{align}
\label{eq:slds_z_marginal}
\E_{x, z | y, \theta}[\grad \log p(y, x, z \, | \, \theta)]
 &= \frac{1}{N}\sum_{r = 1}^{N} \E_{x | y, z^{(r)}, \theta}[\grad \log p(y, x, z^{(r)} \, | \, \theta)] \\
\label{eq:slds_x_marginal}
\E_{x, z | y, \theta}[\grad \log p(y, x, z \, | \, \theta)]
 &= \frac{1}{N}\sum_{r = 1}^{N} \E_{z | y, x^{(r)}, \theta}[\grad \log p(y, x^{(r)}, z \, | \, \theta)] \enspace.
\end{align}
Because Eq.~\eqref{eq:slds_z_marginal} integrates out $x$, it has lower variance for the gradient terms involving $x$ (i.e. $A$, $Q$ $R$).
Similarly, because Eq.~\eqref{eq:slds_x_marginal} integrates out $z$, it has lower variance for the gradient terms involving $z$ (i.e. $\Pi$).

Selecting one of the above Monte Carlo estimates of $\grad U(\theta)$,
we can deploy the same buffered subsampling estimator Eq.~\eqref{eq:efficient_potential_estimate}, obtaining Algorithm~\ref{alg:noisygradient-slds}.
Algorithm~\ref{alg:noisygradient-slds} replaces the forward-backward subroutine in Algorithm~\ref{alg:noisygradient-analytic} with blocked Gibbs sampling over $\BUFFSUBSEQ$.
Although this is more computationally costly than the exact forward-backward algorithms of the previous sections, it still provides memory saving and runtime speed ups compared to running a full blocked Gibbs sampler over $\FULLSEQ$.
The explicit forms of Eqs.~\eqref{eq:slds_naive}-\eqref{eq:slds_x_marginal},
precondition matrix $D(\theta)$, and correction term $\Gamma(\theta)$ for SLDS used in Alg.~\ref{alg:sgrld} are a combination of those for ARHMMs and LGSSMs.
Complete details are provided in the Supplement.

\begin{algorithm}[ht]
\caption{\texttt{NoisyGradient} using blocked Gibbs (SLDS)}
\label{alg:noisygradient-slds}
\begin{algorithmic}
\STATE{\textbf{input:} data $y$, parameters $\theta$, subsequence length $S$, error tolerance $\epsilon$, }
\STATE{$B$ = \texttt{BufferLength}$(\theta, S, \epsilon)$}
\COMMENT{From Theory or Adaptive}
\STATE{$\SUBSEQ, \BUFFSUBSEQ$ = \texttt{GetBufferedSubsequence}$(y, S, B)$}
\STATE{$z^{(0)}_\BUFFSUBSEQ$ = \texttt{InitLatent}$(\BUFFSUBSEQ, \theta)$}
\COMMENT{With `burn-in'}
\FOR{$r = 1, 2, \ldots, N$}
\STATE{sample $x_{\BUFFSUBSEQ}^{(r)} \sim x_{\BUFFSUBSEQ} \, | \, y_\BUFFSUBSEQ, z_\BUFFSUBSEQ^{(r-1)}$}
\COMMENT{Blocked Gibbs}
\STATE{sample $z_{\BUFFSUBSEQ}^{(r)} \sim z_{\BUFFSUBSEQ} \, | \, y_\BUFFSUBSEQ, x_\BUFFSUBSEQ^{(r)}$}
\ENDFOR
\STATE{ calculate $\widetilde{U}(\theta)$ using a Monte Carlo estimate}
\COMMENT{Eq.~\eqref{eq:slds_naive}, \eqref{eq:slds_z_marginal}, or \eqref{eq:slds_x_marginal}}
\STATE{\textbf{return} $\grad\widetilde{U}(\theta)$}
\end{algorithmic}
\end{algorithm}

\subsubsection{Error Bounds}
There are two primary challenges for the error analysis of the SLDS:
(i) the forward and backward smoothing kernels for the SLDS are mixtures and
(ii) the error from the finite-step blocked Gibbs sampler needs to be quantified.
Conditions for contraction in the forward and backward smoothing random maps of switching models may follow from the conditions in~\cite{cloez2015exponential}.
Combining the convergence rate of the blocked Gibbs sampler with the error bound is an area we leave for future work.
Our experiments in Section~\ref{sec:experiments-SLDS} provide empirical evidence of the potential benefits of the algorithm.
\section{Experiments}
\label{sec:experiments}

We evaluate the performance of our proposed SGRLD algorithm (Section~\ref{sec:framework}) using both synthetic and real data.
We organize our experiments by the corresponding models of Section~\ref{sec:models}.
Our evaluation focuses on the following three topics:
(1) the computational speed-up of SGMCMC over batch MCMC,
(2) the effectiveness of buffering in correcting bias,
and
(3) the effectiveness of the complete-data Fisher information preconditioning of SGRLD over SGLD.

For batch MCMC, we consider block-Gibbs sampling (Gibbs) and unadjusted Langevin Monte-Carlo -- both with preconditioning (RLD) and without precondition (LD).
Note that LD and RLD are SGLD and SGRLD with $S = T$.

To assess the performance of our samplers, we measure the marginal loglikelihood of samples $\theta^{(s)}$ at different runtimes on a heldout test sequence.
In synthetic data, where the true parameter $\theta^*$ is known, we also measure the mean-squared error (MSE) of the sample average $\hat\theta^{(s)} = \sum_{i \leq s} \theta^{(i)}/s$ to $\theta^*$.
To assess the quality of our MCMC samples at approximating the posterior $\Pr(\theta \, | \, y)$, we measure the kernel Stein discrepancy (KSD) of each chain after burn-in given equal computation time~\cite{gorham2017measuring, liu2016kernelized},
rather than effective sample size (ESS)~\cite{brooks2011handbook, gelman2013bayesian}, as KSD accounts for bias in the samples.
As with all gradient-based methods,
our SGMCMC methods require a hyper-parameter search over the fixed step-size tuning parameter $h$.
We present results for the best step-size as assessed via heldout loglikelihood on a validation set.
As the potential $U(\theta)$ for SSMs is non-convex, initialization is important.
For the HMM and ARHMM, we initialize the parameters $\Pi, A, Q$ using $z$ given from $K$-means clustering of the observations $y$ (or $y_{t-p:t}$).
For the LGSSM, we initialize the parameters from the prior.
For the mixed-type SLDS, we first sample $R$ from the prior and initialize $\Pi, A, Q$ using $z$ from $K$-means.
Finally, in our experiments, we use flat and non-informative priors for $\theta$. 
For complete details see the Supplement.

\subsection{Gaussian HMM \& ARHMM}
\label{sec:experiments-HMM}
\subsubsection{Synthetic ARHMM}
\label{sec:experiments-ARHMM-synth}
We first consider synthetic data generated from a $2$-state ARHMM in two dimensions $m=2$.
The true model parameters $\theta^*$ are
\begin{equation*}
\Pi = \begin{bmatrix}
    0.1 & 0.9 \\
    0.9 & 0.1
    \end{bmatrix}
\enspace, \enspace
Q_1 = Q_2 = 0.1 \cdot \begin{bmatrix}
    1 & 0 \\
    0 & 1
    \end{bmatrix}
\end{equation*}
\begin{equation*}
A_1 = 0.9 \cdot \begin{bmatrix}
    \cos(-\vartheta) & -\sin(-\vartheta) \\
    \sin(-\vartheta) & \cos(-\vartheta)
    \end{bmatrix}
\enspace, \enspace
A_2 = 0.9 \cdot \begin{bmatrix}
    \cos(\vartheta) & -\sin(\vartheta) \\ \sin(\vartheta) & \cos(\vartheta)
    \end{bmatrix}
\enspace .
\end{equation*}
The model's two states are alternating rotations of $y \in \R^2$ with angle $\vartheta = \pi/4$ and the latent state sequence has a high transition rate $\Pr(z_t \neq z_{t-1}) = 0.9$.
From this model we generate time series of length $T = 10^4$ and $10^6$.

\begin{figure}[!htb]
    \centering
    \begin{minipage}[t]{.45\textwidth}
    \centering
        \includegraphics[width=\textwidth]{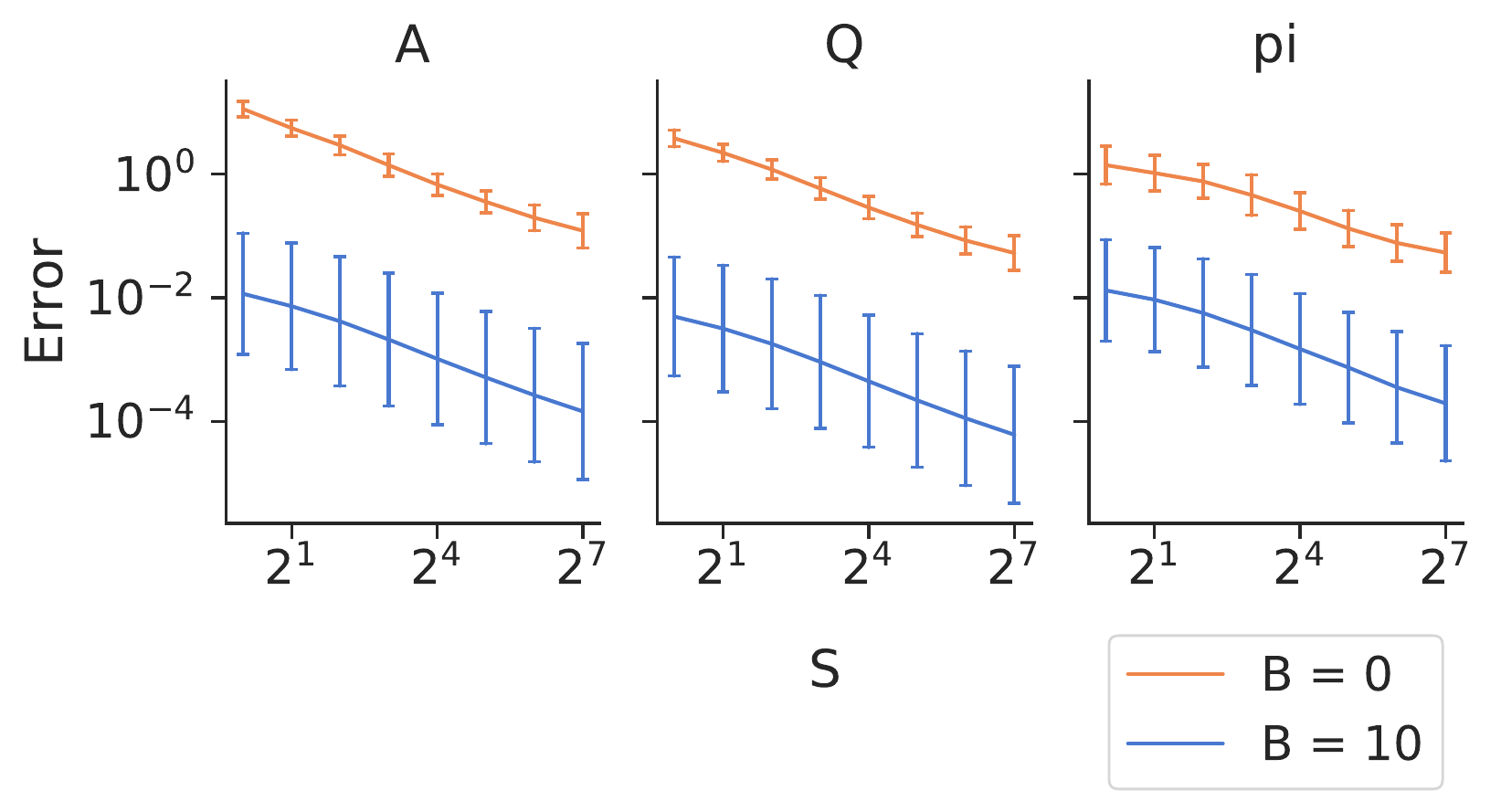}
    \end{minipage}
    \hspace{0.5cm}
    \begin{minipage}[t]{.45\textwidth}
    \centering
        \includegraphics[width=\textwidth]{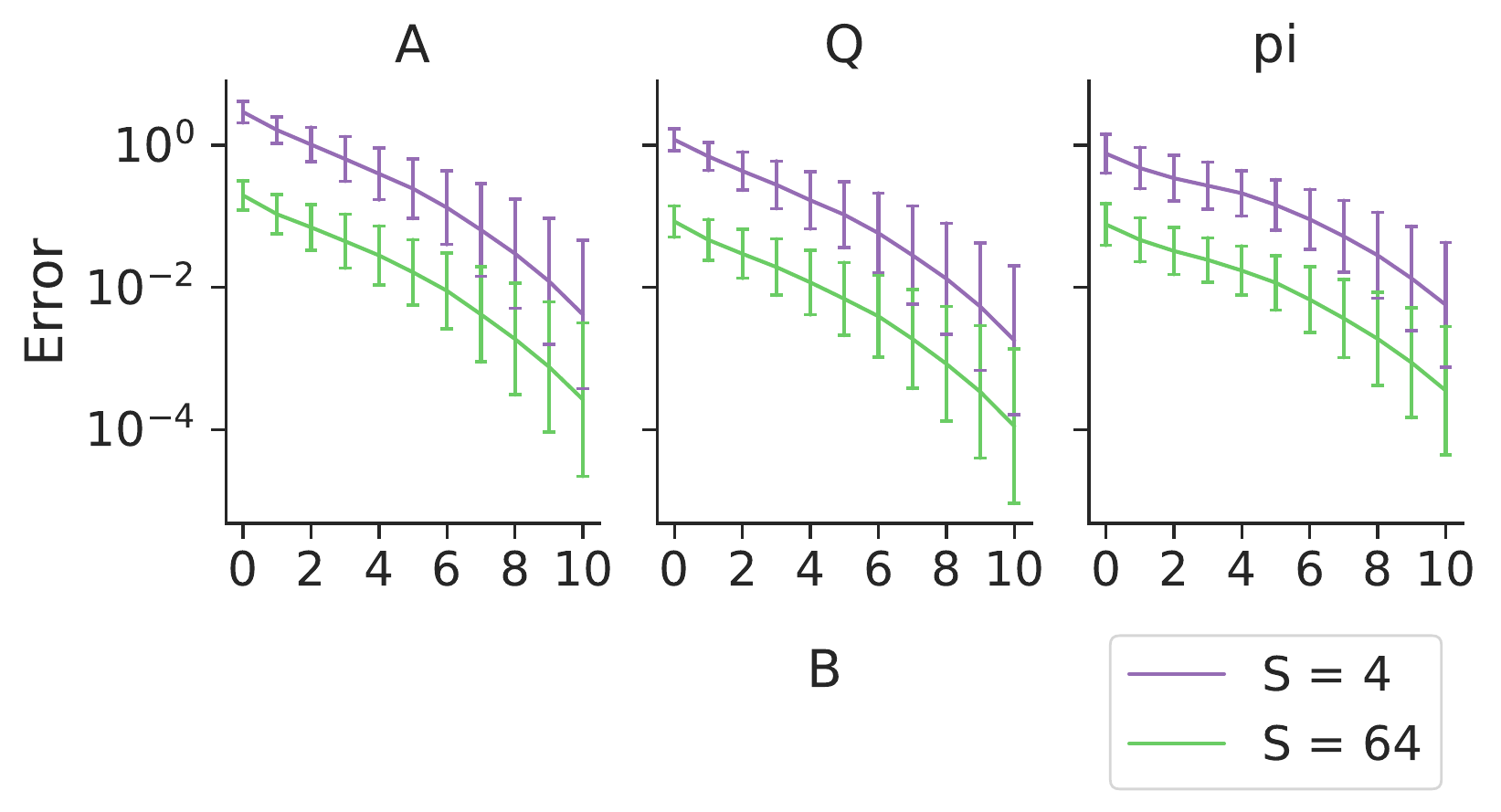}
    \end{minipage}
    \caption{Stochastic gradient error $\E_\SUBSEQ\|\bar{g}(\theta) - \tilde{g}(\theta)\|_2$. (Left) varying subsequence length $S$ for no-buffer $B = 0$ and buffer $B=10$. (Right) varying buffer size $B$ for $S = 4$ and $S=64$ subsequence lengths. Error bars are SD over $100$ datasets.
}
\label{fig:arhmm_grad_error}
\end{figure}

Figure~\ref{fig:arhmm_grad_error} are plots of the stochastic gradient error $\E_\SUBSEQ \|\bar{g}(\theta) - \tilde{g}(\theta)\|_2$ between the unbiased and buffered estimates evaluated at the true model parameters $\theta=\theta^*$.
From Figure~\ref{fig:arhmm_grad_error} (left), we see that the error decays $O(1/S)$ and that the error in estimates without buffering $B=0$ (orange) are orders of magnitude larger than the estimates with moderate buffering $B=10$ (blue).
From Figure~\ref{fig:arhmm_grad_error} (right), we see that the error decays geometrically in buffer size $O(L^B)$.

\begin{figure}[!htb]
\centering
    \begin{minipage}[c]{.1\textwidth}
    \centering
    \hbox{\rotatebox{90}{\hspace{1em} $T = 10^4$}}
    \end{minipage}
    \hspace{0.2em}
    \begin{minipage}[c]{.4\textwidth}
    \centering
        \includegraphics[width=\textwidth]{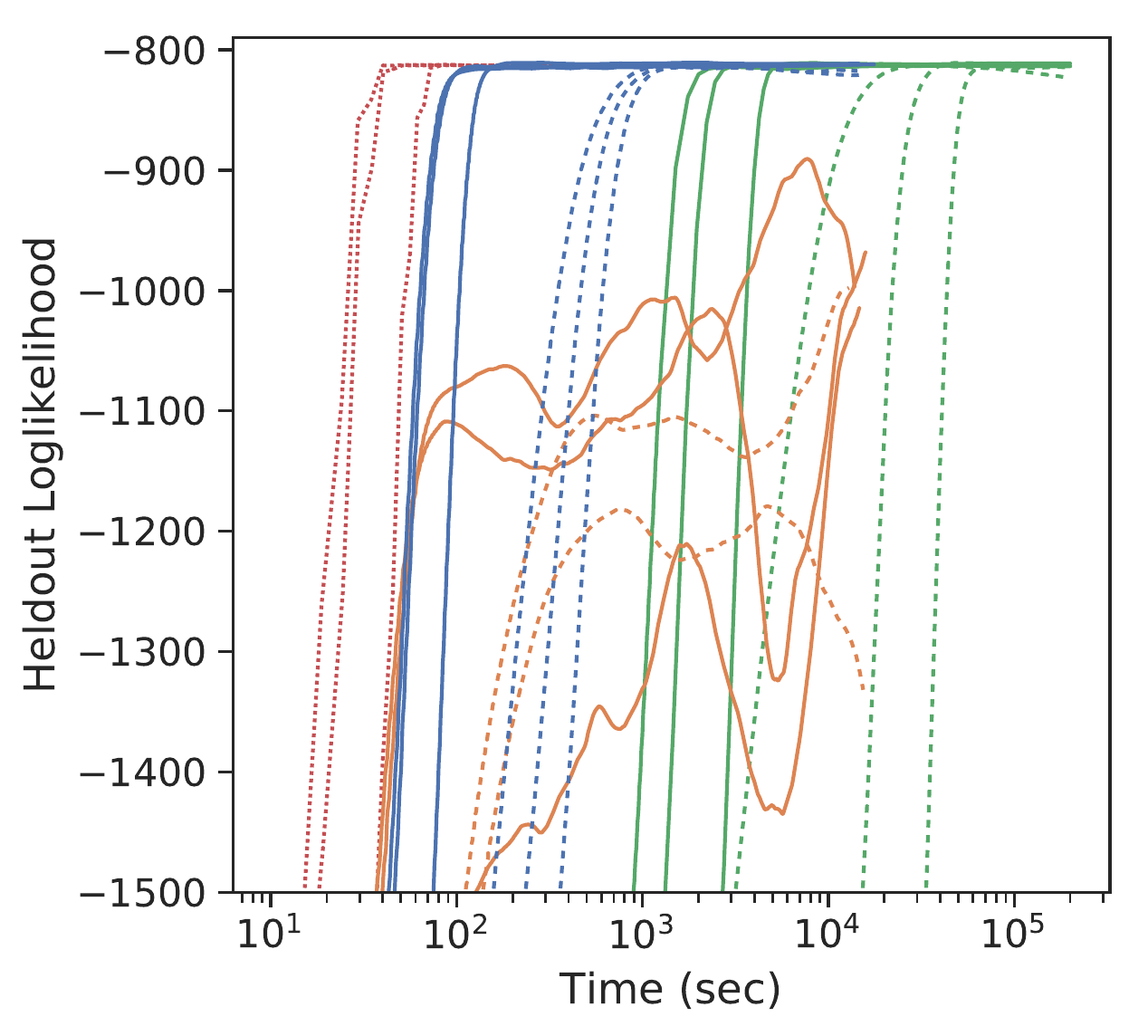}
    \end{minipage}
    \begin{minipage}[c]{.4\textwidth}
    \centering
        \includegraphics[width=\textwidth]{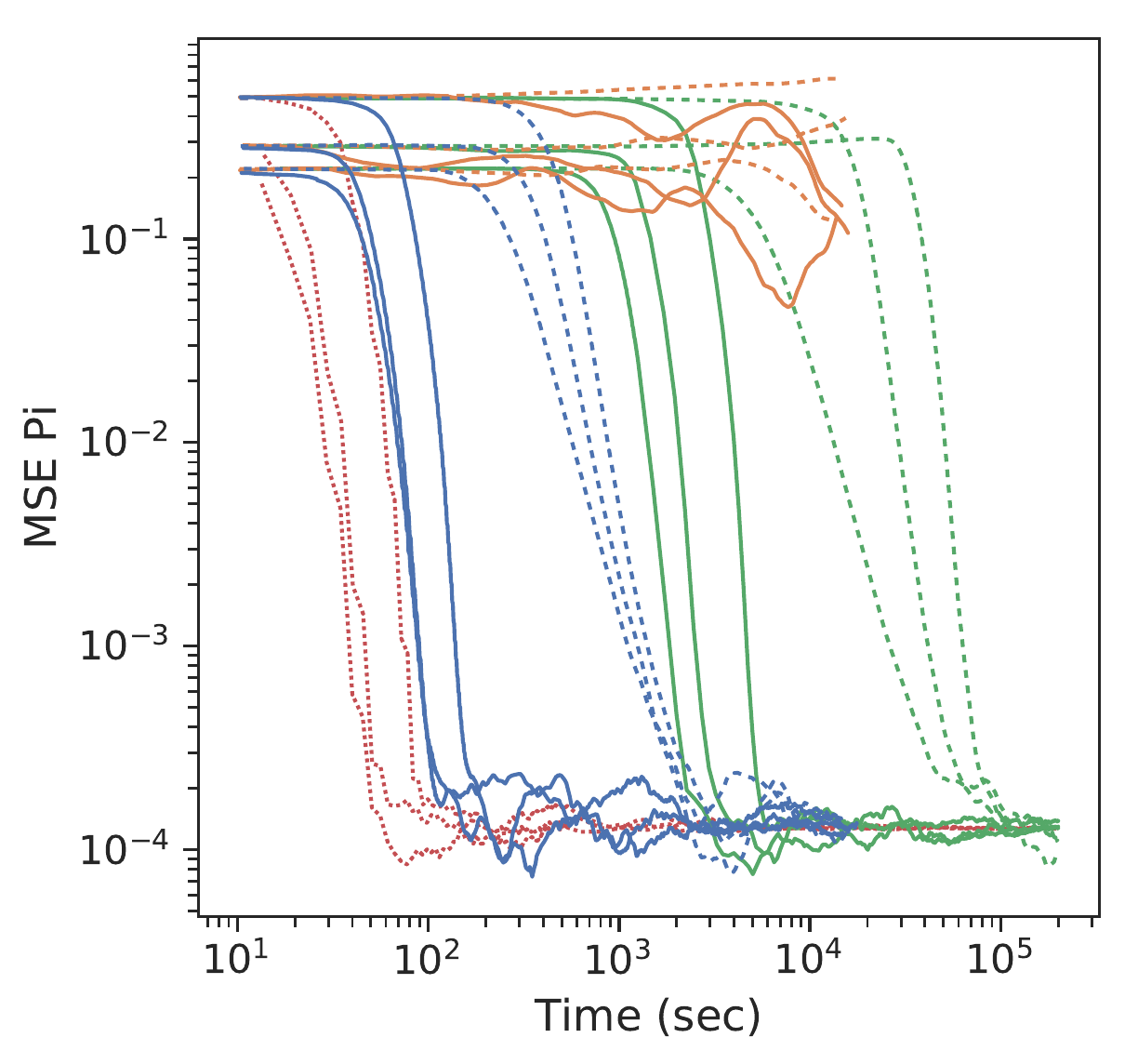}
    \end{minipage}

    \begin{minipage}[c]{.1\textwidth}
    \centering
    \hbox{\rotatebox{90}{\hspace{1em} $T = 10^6$}}
    \end{minipage}
    \hspace{0.2em}
    \begin{minipage}[c]{.4\textwidth}
    \centering
        \includegraphics[width=\textwidth]{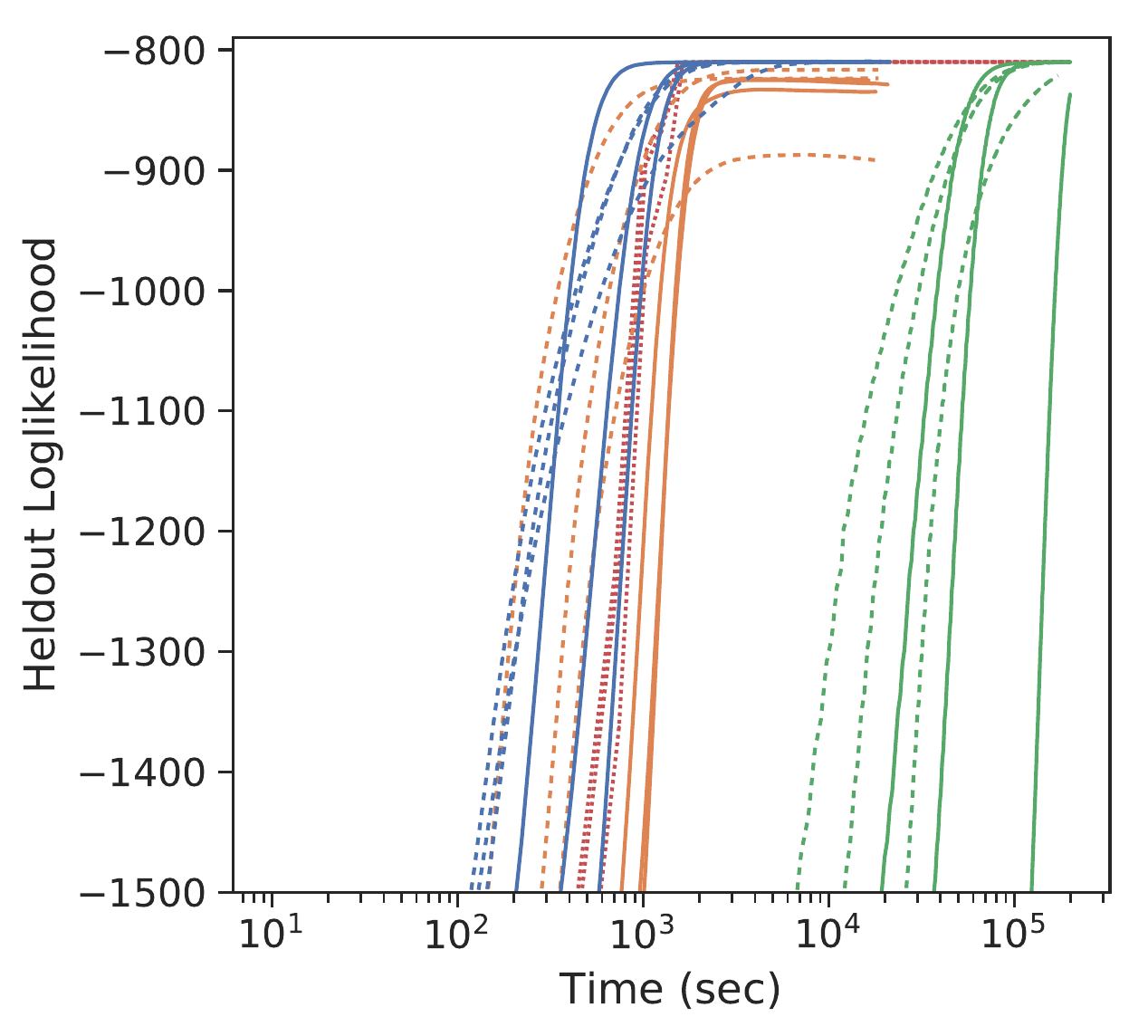}
    \end{minipage}
    \begin{minipage}[c]{.4\textwidth}
    \centering
        \includegraphics[width=\textwidth]{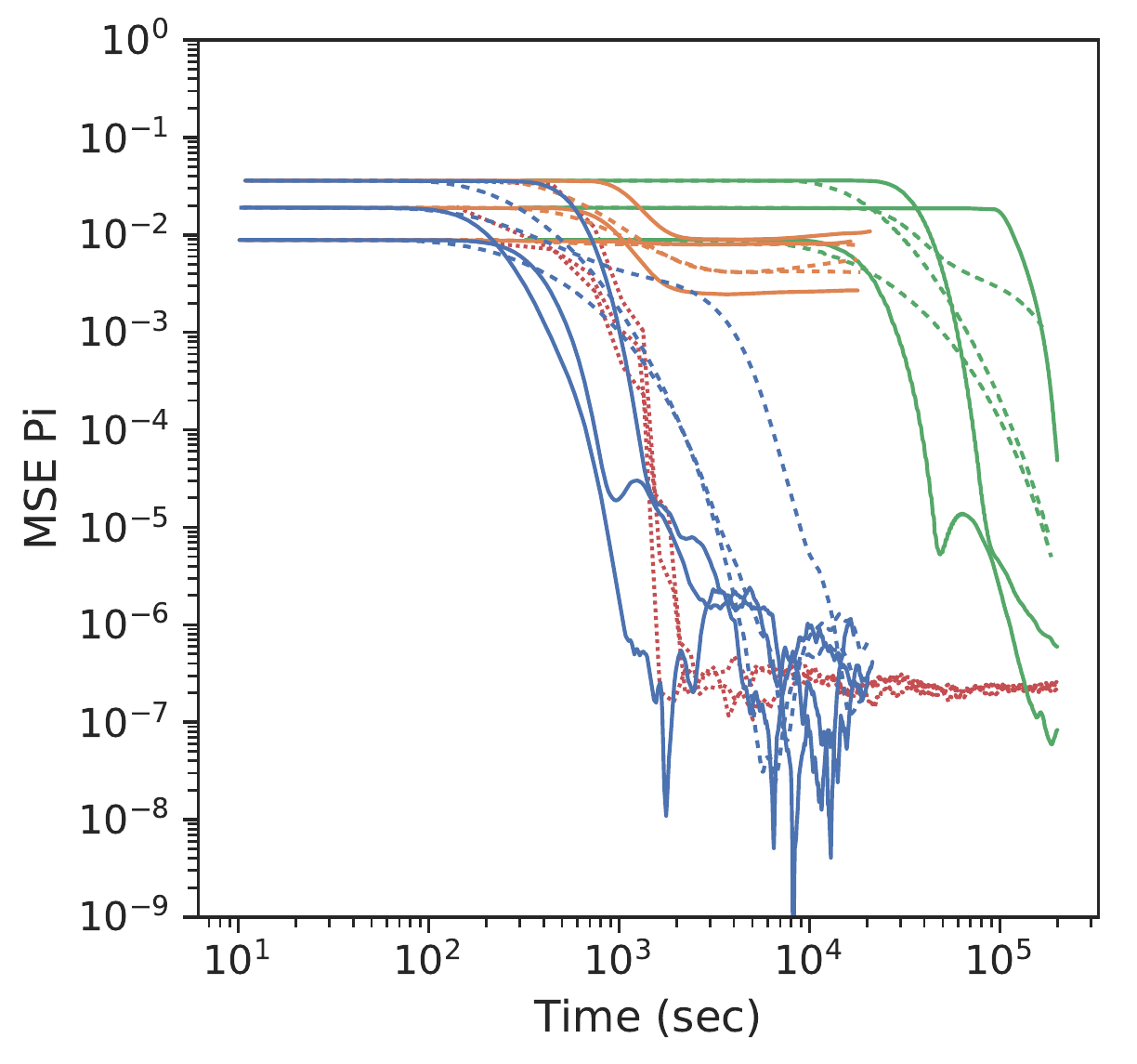}
    \end{minipage}

    \caption{Metrics vs Runtime on ARHMM data with $T = 10^4$ (top), $T=10^6$ (bottom), for different methods:
    \textcolor{BrickRed}{(Gibbs)}, \textcolor{ForestGreen}{(Full)},  \textcolor{Orange}{(No Buffer)} and \textcolor{Blue}{(Buffer)} SGMCMC.
    For SGMCMC methods, solid (\fullline) and dashed (\dashedline) lines indicate SGRLD and SGLD respectively.
    The different metrics are: (left) heldout loglikelihoood and (right) transition matrix estimation error $MSE(\hat\Pi^{(s)}, \Pi^*)$.
}
    \label{fig:arhmm_synth_metrics}
\end{figure}

\begin{figure}[!htb]
    \centering
    \begin{minipage}[c]{.9\textwidth}
    \centering
        \includegraphics[width=\textwidth]{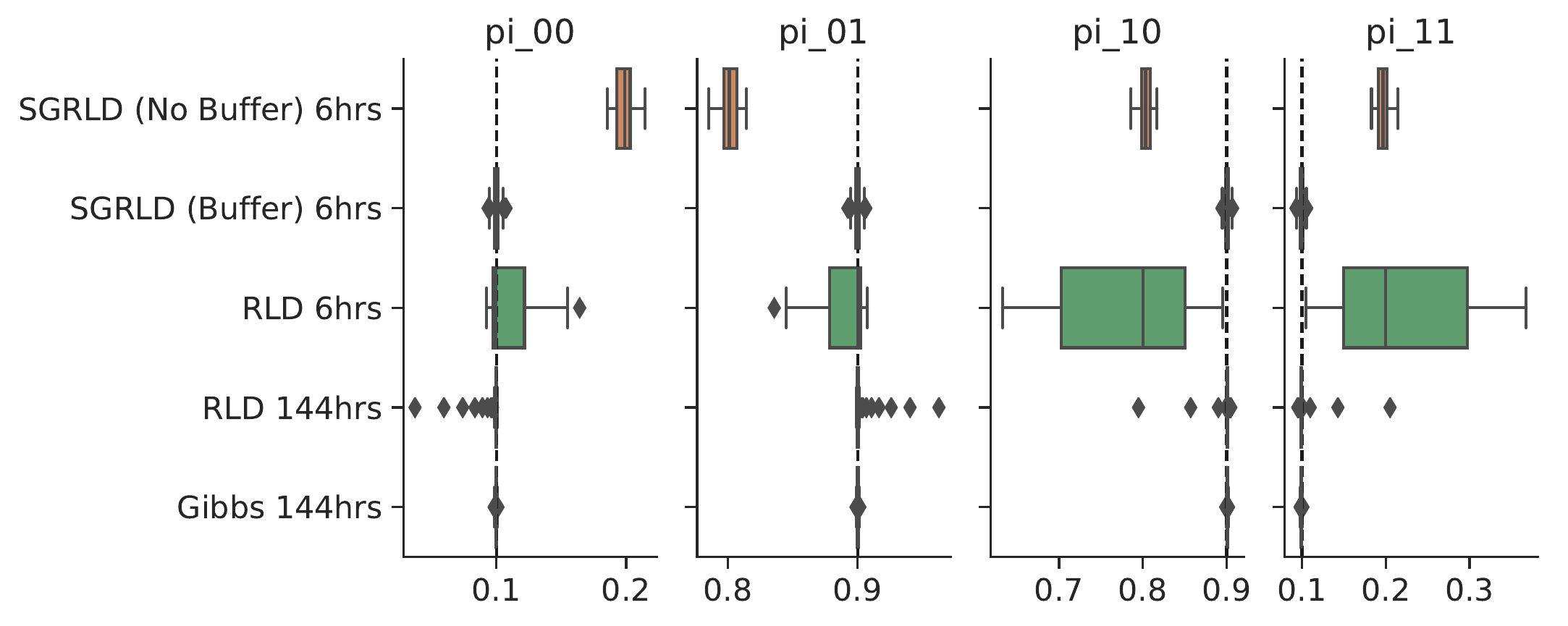}
        \includegraphics[width=\textwidth]{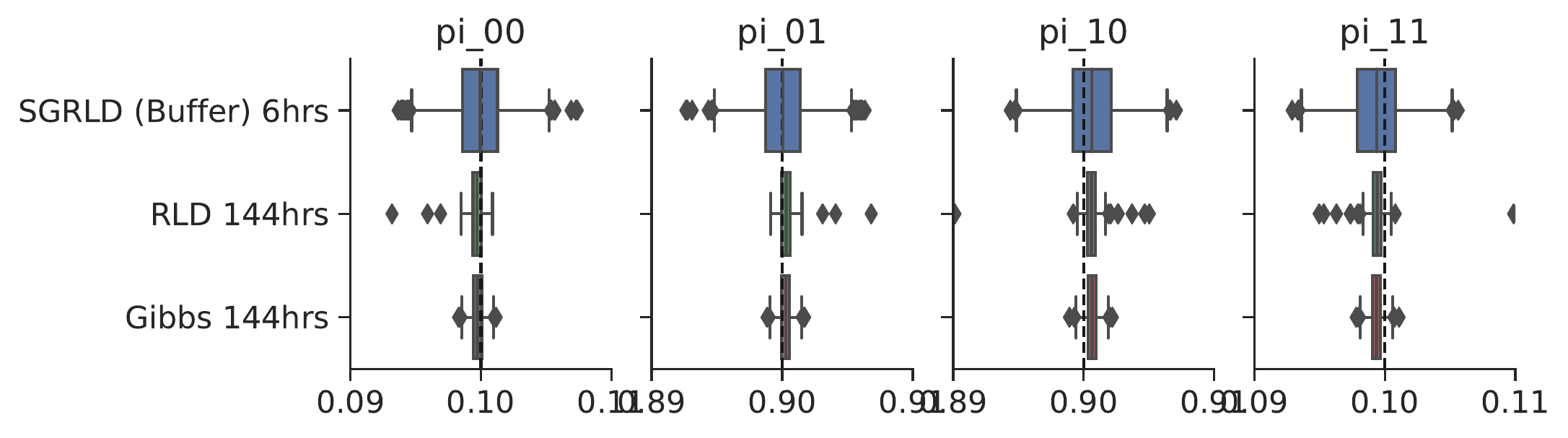}
    \end{minipage}
    \caption{Boxplot of MCMC samples for ARHMM data $T = 10^6$. (Top) comparison of all samplers, (bottom) zoom-in for top three. The half of each chain is discarded as burn-in. SGRLD with buffering in 6 hrs is comparable to RLD or Gibbs in 144 hrs.}
\label{fig:arhmm_boxplot}
\end{figure}

\begin{table}[ht]
\centering
\caption{$\log_{10}$(KSD) by variable of ARHMM samplers at 6 hrs. Mean and (SD) over runs in Figure~\ref{fig:arhmm_synth_metrics}.}
\label{table:arhmm_ksd}
\begin{tabular}{rrlll}
  \hline
 & Sampler & $\pi$ & $A$ & $\Sigma$ \\
  \hline
 \multirow{7}{*}{\rotatebox[origin=c]{90}{$T = 10^4$}}
\abovestrut{1em}
& SGLD (No Buffer)  & 3.15 (0.46) & 2.47 (0.51) &  2.33 (0.30) \\
& SGLD (Buffer)     & 0.99 (0.13) & 1.60 (0.20) &  1.80 (0.13) \\
& LD                & 1.77 (0.72) & 1.86 (0.32) &  2.12 (0.36) \\
\cline{2-5}
\abovestrut{1em}
& SGRLD (No Buffer) & 3.15 (0.39) & 2.02 (0.24) &  1.91 (0.24) \\
& SGRLD (Buffer)    & 0.89 (0.04) & 1.53 (0.10) &  1.60 (0.30) \\
& RLD               & 0.67 (0.27) & 2.02 (0.14) &  1.60 (0.18) \\
\cline{2-5}
\abovestrut{1em}
& Gibbs             & 0.36 (0.07) & 1.30 (0.20) &  0.61 (0.13) \\
 \hline
 \hline
 \multirow{7}{*}{\rotatebox[origin=c]{90}{$T = 10^6$}}
\abovestrut{1em}
& SGLD (No Buffer)  & 4.73 (0.07) & 4.07 (0.22) &  3.67 (0.25) \\
& SGLD (Buffer)     & 2.62 (0.06) & 3.30 (0.20) &  2.77 (0.31) \\
& LD                & 3.59 (0.22) & 4.73 (0.33) &  4.78 (0.34) \\
\cline{2-5}
\abovestrut{1em}
& SGRLD (No Buffer) & 4.75 (0.15) & 4.02 (0.06) &  3.61 (0.12) \\
& SGRLD (Buffer)    & 2.27 (0.08) & 3.38 (0.08) &  2.89 (0.09) \\
& RLD               & 3.31 (0.05) & 4.22 (0.12) &  3.56 (0.07) \\
\cline{2-5}
\abovestrut{1em}
& Gibbs             & 3.17 (0.30) & 4.18 (0.07) &  3.30 (0.07) \\
\hline
\end{tabular}
\end{table}

In Figures~\ref{fig:arhmm_synth_metrics} and~\ref{fig:arhmm_boxplot},
we compare subsequence-based MCMC methods: SGLD (no-buffer and buffer) and SGRLD (no-buffer and buffer), with full-sequence MCMC methods: LD, RLD, and Gibbs.
We fit our samplers on one training sequence and evaluate performance on one test sequence.
We consider two training sequences of lengths $T = 10^4$ and $T=10^6$ and evaluate on the same test sequence of length $T = 10^4$.
For the SGMCMC methods we use a subsequence size of $S = 2$ and a buffer size of $B=0$ (no-buffer) or $B=2$ (buffer).
We ran the subsequence methods for 6 hours and full-sequence methods for 144 hours.

From Figure~\ref{fig:arhmm_synth_metrics}, we see that
our buffered SGMCMC (blue) helps convergence and mixing orders of magnitude faster than the full-sequence gradient MCMC (green).
We also see that buffering is necessary to properly estimate $\Pi$ as the no-buffer SGMCMC methods (orange) do not properly learn $\Pi$.
We also see that preconditioning helps convergence and mixing as SGRLD (solid) outperforms SGLD (dashed).
Although Gibbs outperforms SGMCMC for $T=10^4$,
Gibbs performs worse for $T=10^6$, as each iteration requires a full pass over the data set.

Figure~\ref{fig:arhmm_boxplot} are boxplots comparing the marginal distribution for the different methods on the synthetic ARHMM data $T = 10^6$.
From Figure~\ref{fig:arhmm_boxplot}, we see that SGRLD with buffering in 6 hours is comparable to RLD or Gibbs in 144 hours;
however, SGRLD without buffering is biased and RLD in 6 hours has not had enough time to mix.

Table~\ref{table:arhmm_ksd} displays the KSD of the samples to the posterior after discarding half the samples as burn-in. The standard deviation is over MCMC chains with different initializations.
Although RLD and Gibbs perform well for $T = 10^4$, both perform worse for larger $T = 10^6$ due to the increased time between samples.
We also see that the non-buffered methods do poorly for all $T$ due to sampling from the incorrect distribution.
Although SGLD (buffer) and SGRLD (buffer) perform comparably after burn-in, Figure~\ref{fig:arhmm_synth_metrics} suggests SGRLD converges more rapidly.

In the Supplement,
we present a synthetic data experiment for the Gaussian HMM,
and find similar results.

\subsubsection{Ion Channel Recordings}
\label{sec:experiments-HMM-ion}
We investigate the behavior of SGMCMC samplers on ion channel recording data.
In particular, we consider a 1MHz recording of a
single alamethicin channel~\cite{Rosenstein:2013}.
This data was previously investigated using a Bayesian
nonparametric HMM in~\cite{Palla:2014} and \cite{Tripuraneni:2015}.
In that work, the authors downsample the data by a factor of
$100$ and only used $10,000$ and $2,000$ observations due to the challenge of scaling computations to the full sequence.
We present the results on the data without downsampling ($10$ million observations),
where Gibbs sampling runs into memory issues.
Figure~\ref{fig:ion_channel} presents our results,
after applying a log-transform and normalizing the observations.
We train on the first 90\% and evaluate on the last 10\%.
For our SGMCMC methods we use a subsequence size of $S = 10$ and a buffer size of $B=0$ (no-buffer) or $B=10$ (buffer).
In addition to heldout loglikelihood, we also evaluate on 10-step ahead predictive loglikelihood $\sum_{t} \log \Pr(y_{t+10} \, | \, \theta, y_{\leq t})$, which is more sensitive to $\Pi$.
We see that SGRLD quickly converges compared to SGLD.
Although the buffered methods take longer to compute ($S+2B = 30$ vs $S=10$), we see that buffering is necessary to perform well.
In the Supplement, we present results comparing SGMCMC methods with Gibbs sampling on a downsampled version.

\begin{figure}[!htb]
    \centering
    \begin{minipage}[c]{.33\textwidth}
    \centering
        \includegraphics[width=\textwidth]{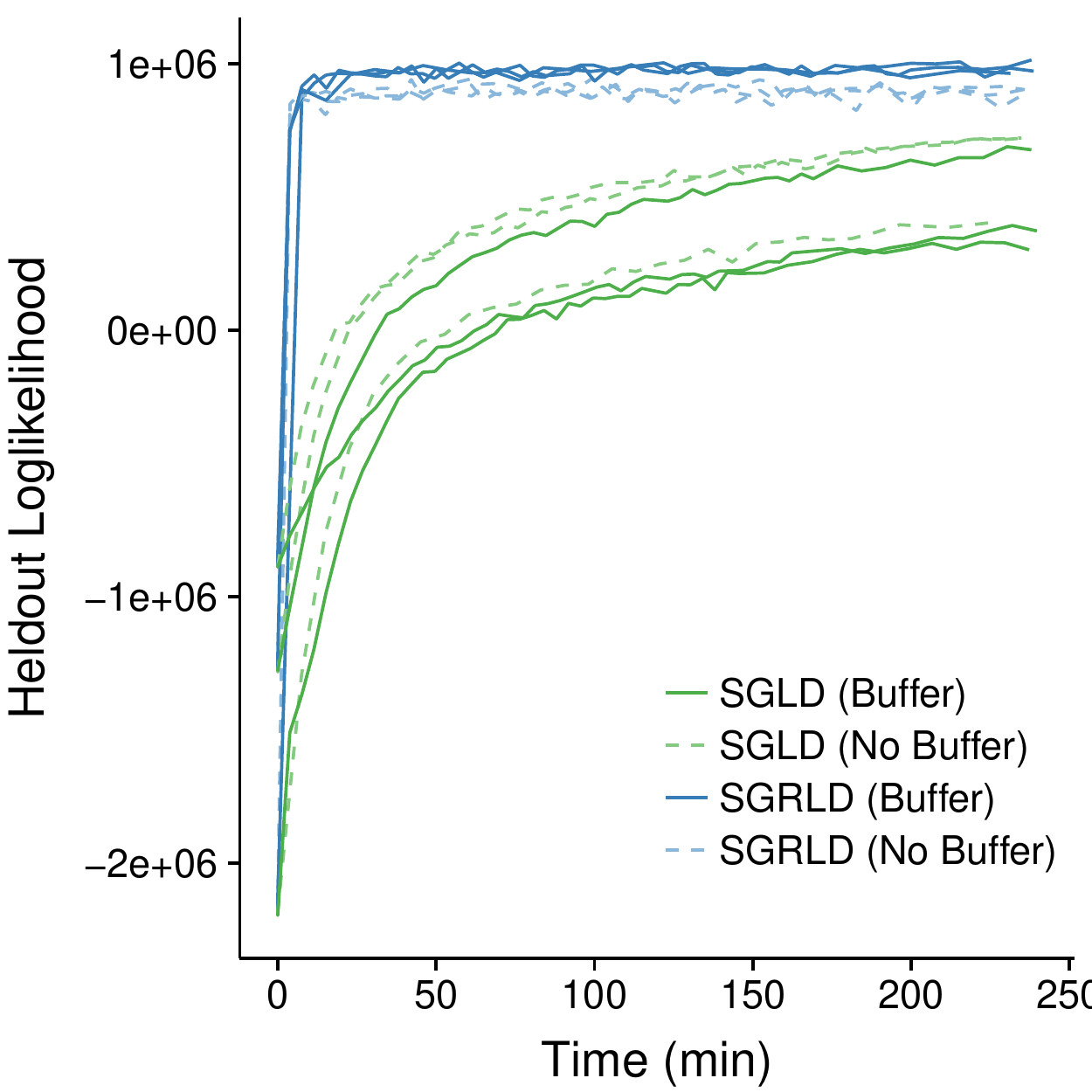}
    \end{minipage}
    \begin{minipage}[c]{.33\textwidth}
    \centering
        \includegraphics[width=\textwidth]{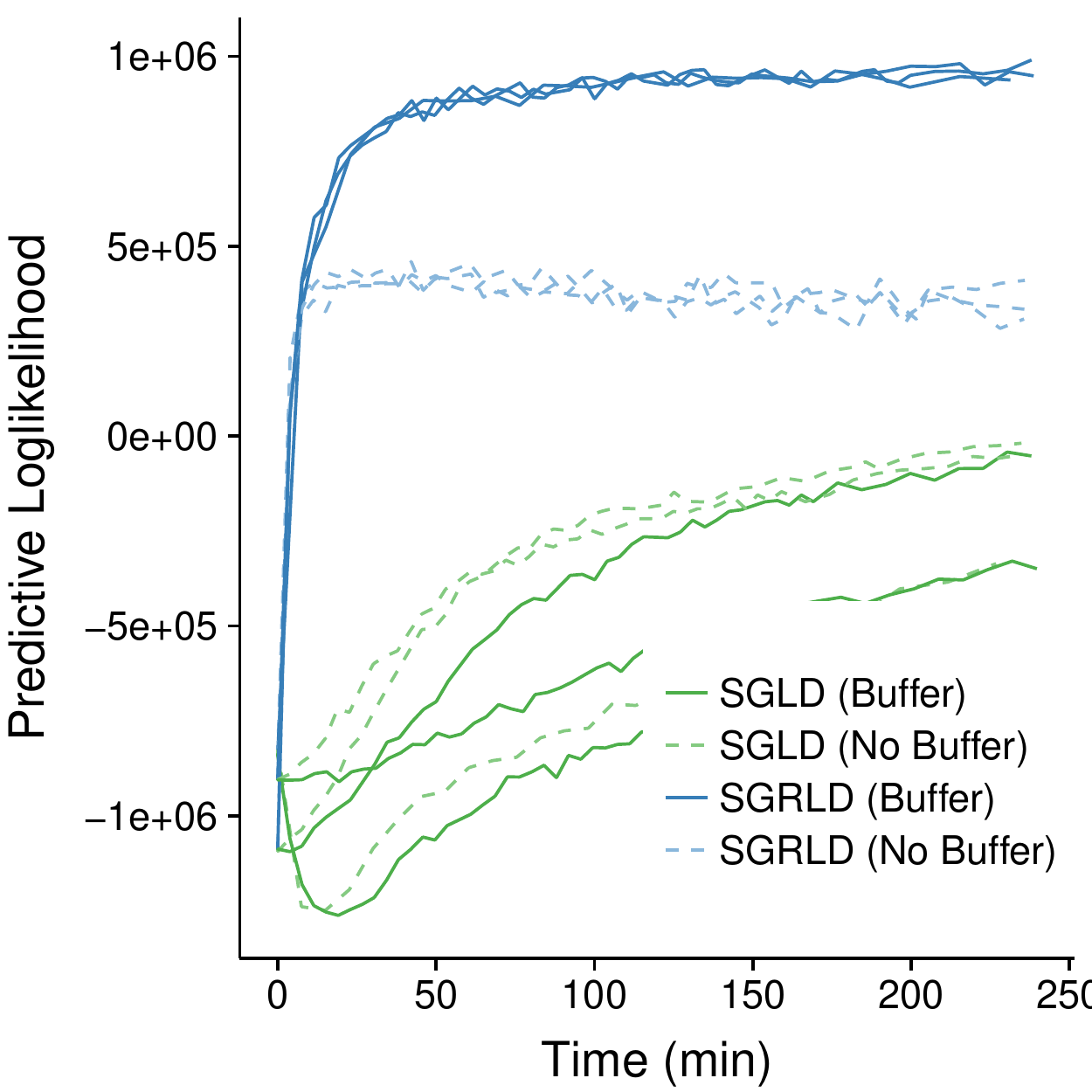}
    \end{minipage}
    \begin{minipage}[c]{.31\textwidth}
    \centering
    \includegraphics[width=\textwidth]{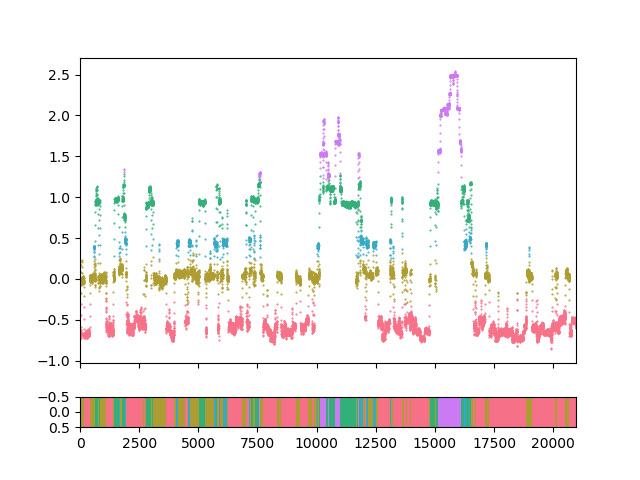}
    \end{minipage}
    \caption{Ion Channel Recordings: (Left) heldout loglikelihood vs runtime. (Center) 10-step predictive loglikelihood $\sum_{t} \log \Pr(y_{t+10} \, | \, \theta, y_{\leq t})$ vs runtime. (Right) segmentation by SGRLD (Buffer).
}
\label{fig:ion_channel}
\end{figure}

\subsubsection{Canine Seizure iEEG}
\label{sec:experiments-ARHMM-seizure}
\begin{figure}[!htb]
    \centering
    \begin{minipage}[t]{.29\textwidth}
    \centering
        \includegraphics[width=\textwidth]{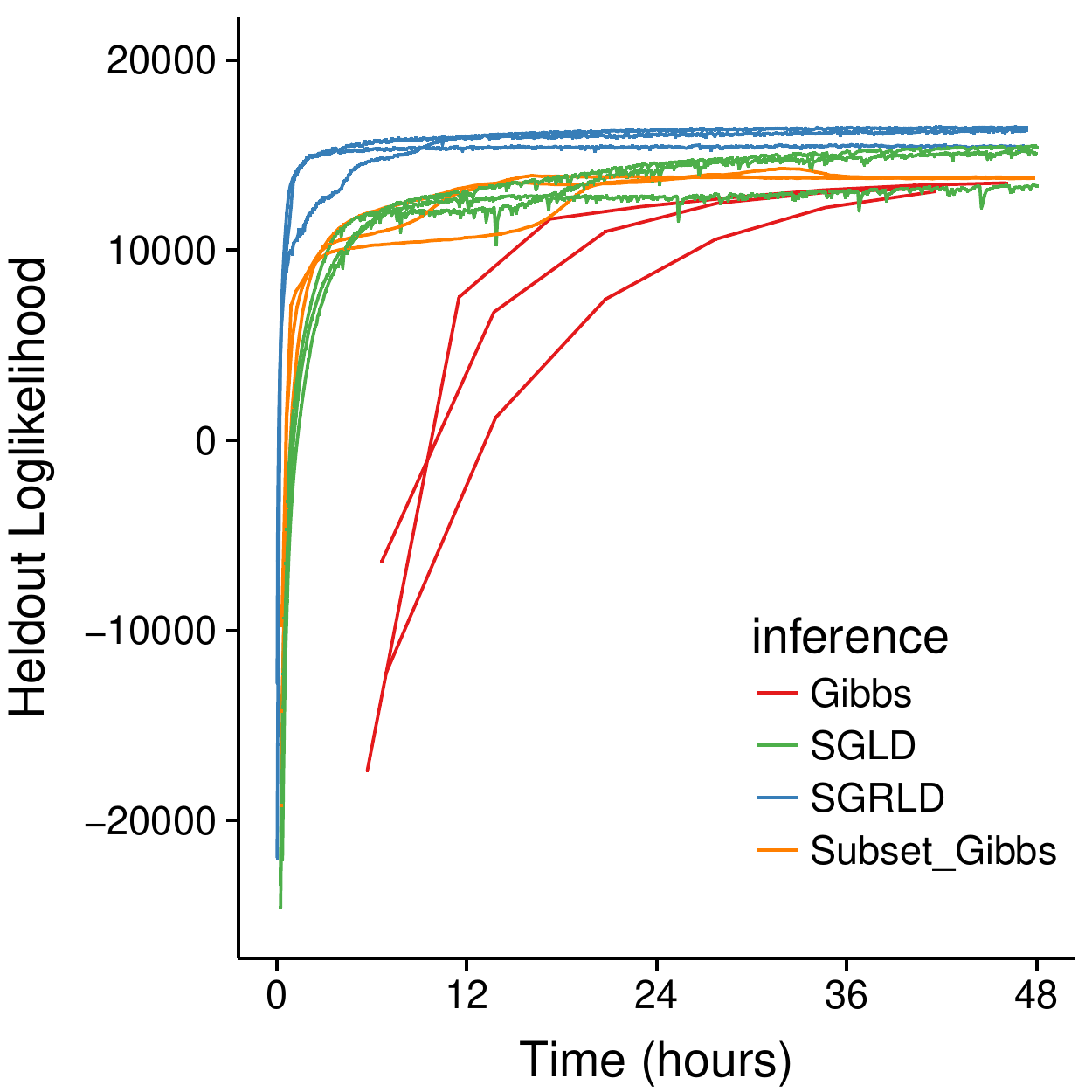}
    \end{minipage}
    \begin{minipage}[t]{.29\textwidth}
    \centering
        \includegraphics[width=\textwidth]{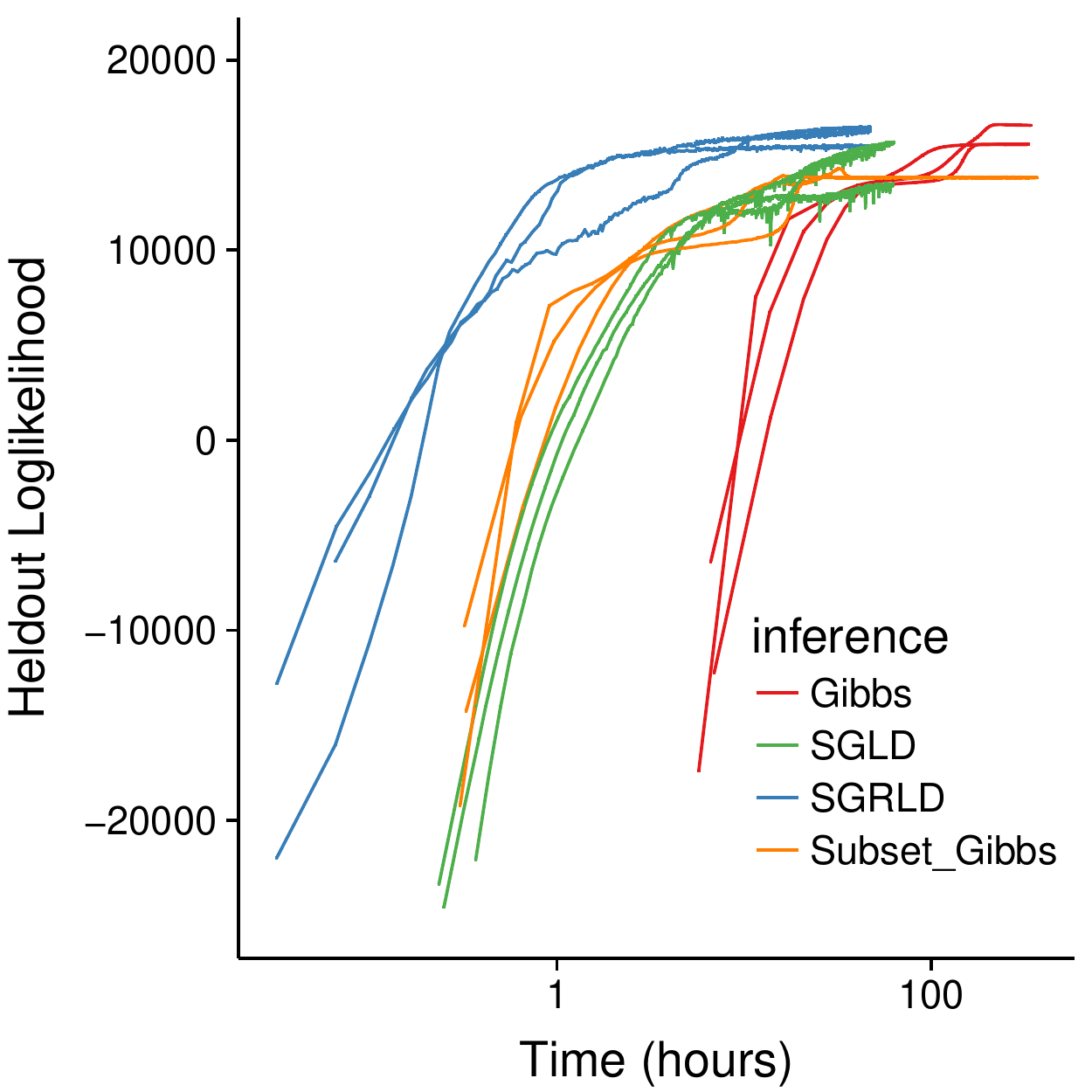}
    \end{minipage}
    \begin{minipage}[t]{.4\textwidth}
    \centering
    \includegraphics[width=\textwidth]{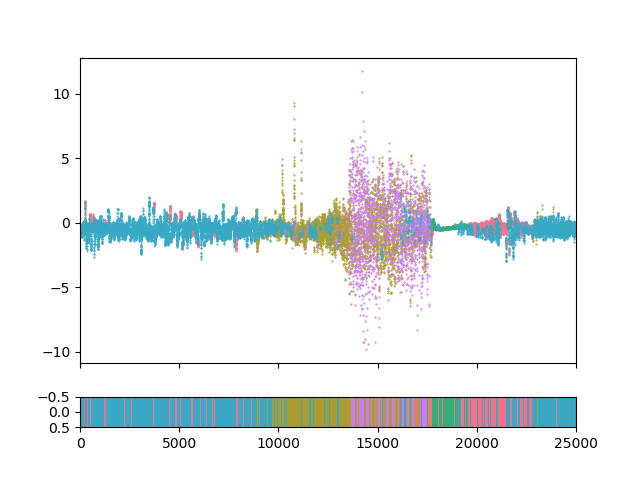}
    \end{minipage}
    \caption{ARHMM for Canine Seizure Data: (left) heldout loglikelihood vs time, (center) heldout loglikelihood vs time on log-scale (right) example segmentation of a test seizure channel by SLDS fit with SGRLD.
    The MCMC methods compared are \textcolor{BrickRed}{Gibbs}, \textcolor{YellowOrange}{Subset Gibbs}, \textcolor{ForestGreen}{SGLD}, and \textcolor{Blue}{SGLRD}.
    }
    \label{fig:canine}
\end{figure}

We now consider applying SGMCMC samplers to intracranial EEG (iEEG) data.
In particular, we consider data from a study on canines with epilepsy available at \texttt{ieeg.org} \cite{davis2016mining}.
We focus on one canine, which over the course of 45.1 days was continuously monitored at 200Hz over 16 channels and recorded 90 seizures.
This data was analyzed in prior work that compared a baseline ARHMM to nonparametric extensions using Gibbs sampling~\cite{wulsin2013bayesian}.
Following~\cite{wulsin2013bayesian}, we process the data into 4 minute windows around each seizure to focus on the seizure dynamics resulting in 90 time series of 48,000 points in $\R^{16}$.
We use an ARHMM with $K=5$ latent states and $p = 5$ lags treating each channel independently.
We perform an 80-20 train-test split over 90 seizures, running inference on the training set and evaluating log-likelihood on the heldout test set.
We compare SGLD and SGRLD samplers with $S = 100$ and $B=10$ with the baseline Gibbs sampler on the full data set.
Because of the large data size,
we also consider a \emph{subset} Gibbs sampler that only uses $10\%$ of the training set seizures.

In Figure~\ref{fig:canine}, we see that SGRLD converges much more rapidly than the other methods.
As each iteration of the Gibbs sampler takes ~6 hours,
it takes a couple weeks for the Gibbs sampler to converge to the solution SGRLD converges to in a few hours.
Although the subset Gibbs sampler is 10x faster than Gibbs,
it does not converge to the full data posterior and its generalization error to the heldout test set is poorer than the other methods.
From this experiment we see that SGMCMC methods provide order of magnitude improvements (compared to subsetting the data).

\subsection{LGSSM and SLDS}
\label{sec:experiments-SLDS}
We first validate the LGSSM (SLDS with $K=1$) on synthetic data.
We then consider the SLDS sampler on a synthetic dataset and two real datasets:
the seizure data of Section~\ref{sec:experiments-ARHMM-seizure} and a weather dataset.

\subsubsection{Synthetic LGSSM}
\label{sec:experiments-LGSSM-synth}
We consider synthetic data from a LGSSM with observations and latent state dimension $m = n = 2$.
In particular, we consider, a rotating state sequence with noisy observations.
The true model parameter $\theta^*$ are
\begin{equation*}
A = 0.7 \cdot \begin{bmatrix}
    \cos(\vartheta) & -\sin(\vartheta) \\
    \sin(\vartheta) & \cos(\vartheta)
    \end{bmatrix}
\enspace, \enspace
Q = 0.1 \cdot \begin{bmatrix}
    1 & 0 \\
    0 & 1
    \end{bmatrix}
\enspace, \enspace
C = \begin{bmatrix}
    1 & 0 \\
    0 & 1
    \end{bmatrix}
\enspace, \enspace
R =\begin{bmatrix}
    1 & 0 \\
    0 & 1
    \end{bmatrix}
\enspace,
\end{equation*}
where $\vartheta=\pi/4$.
Because the transition error $Q$ is smaller than the emission error $R$, 
inclusion of previous and future observations is necessary to accurately infer the continuous latent state $x_t$.

\begin{figure}[!htb]
\centering
    \begin{minipage}[t]{.45\textwidth}
    \centering
        \includegraphics[width=\textwidth]{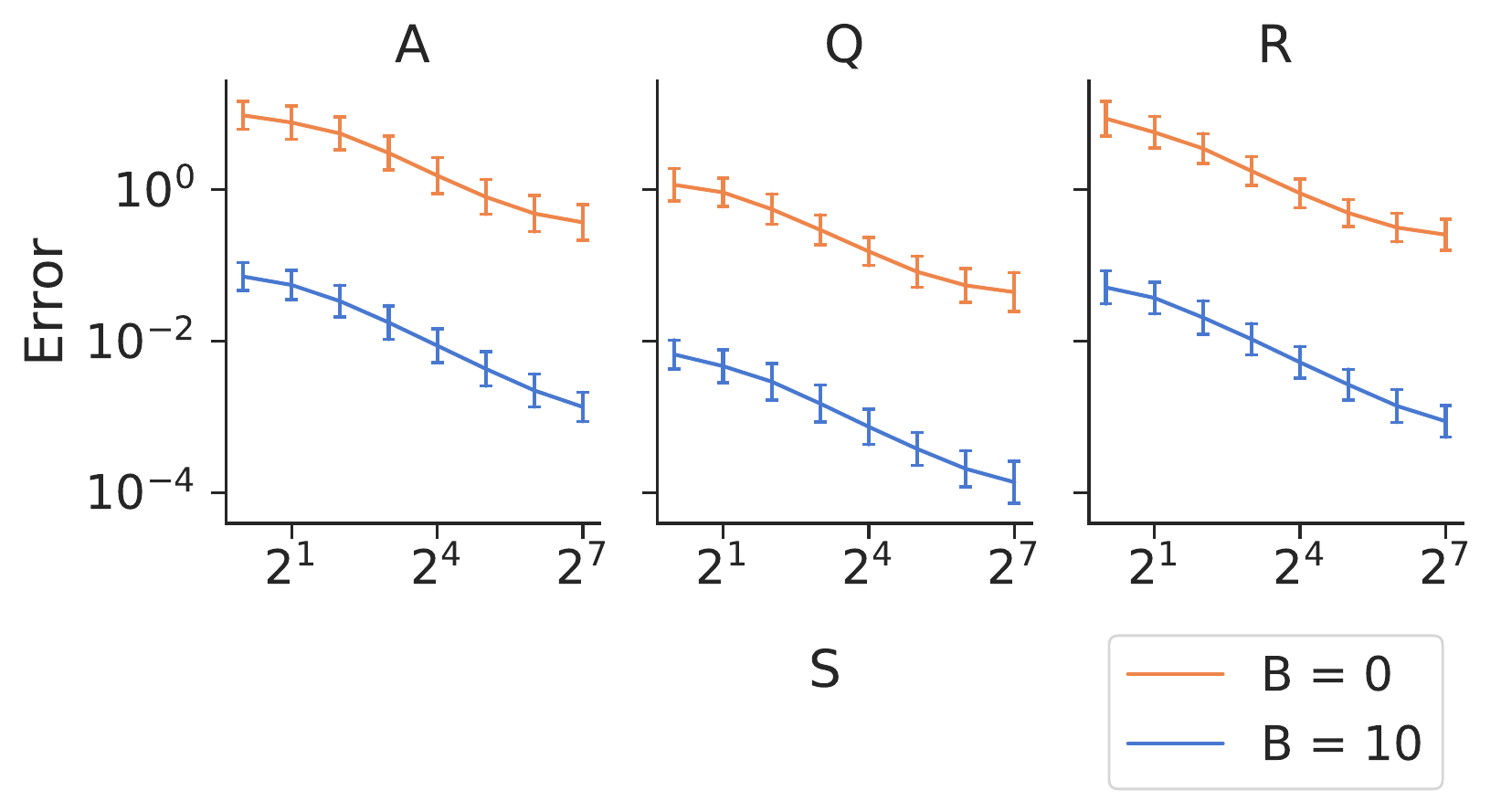}
    \end{minipage}
    \hspace{0.5cm}
    \begin{minipage}[t]{.45\textwidth}
    \centering
        \includegraphics[width=\textwidth]{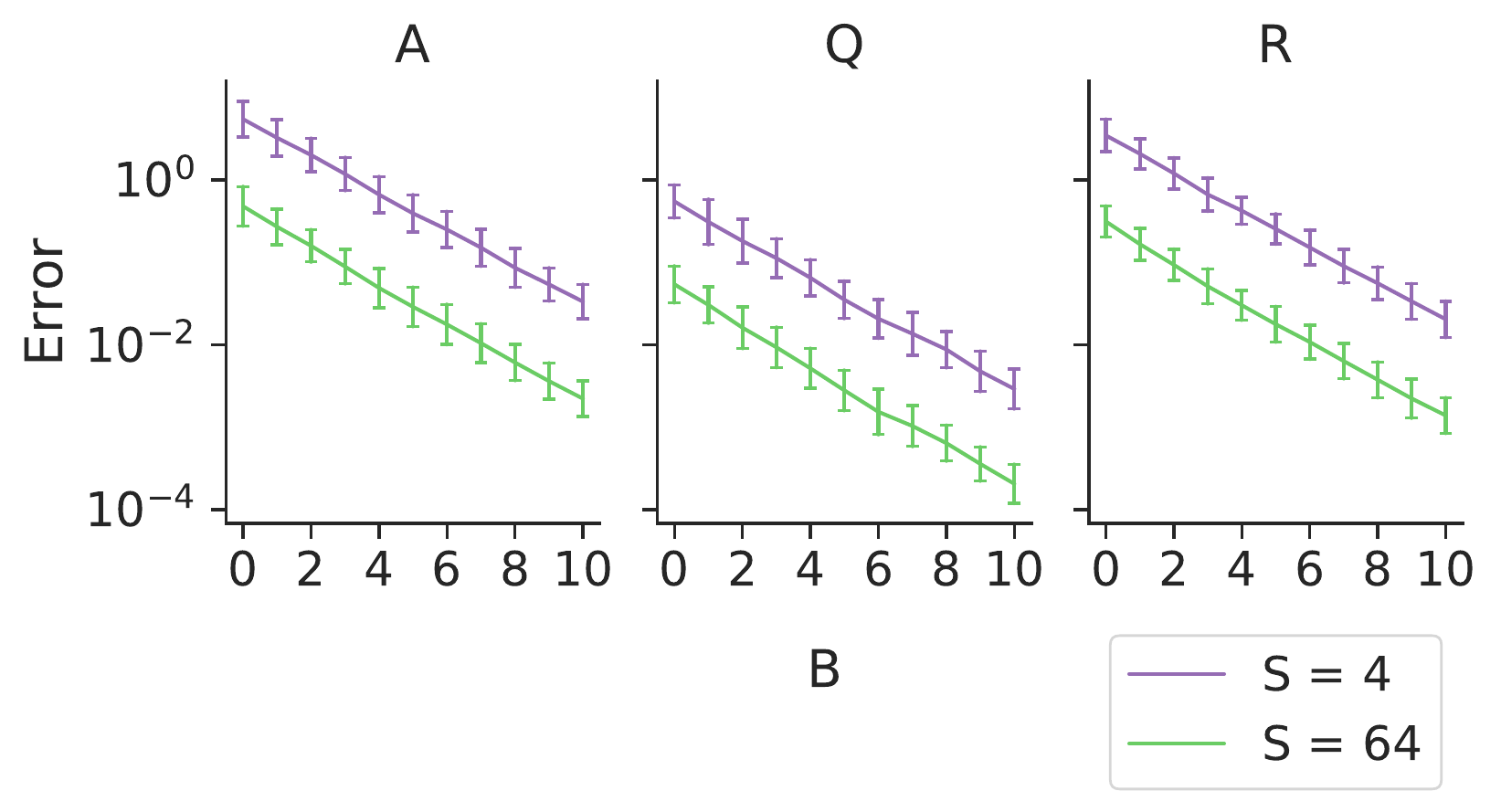}
    \end{minipage}
    \caption{Stochastic gradient error $\E_\SUBSEQ\|\bar{g}(\theta) - \tilde{g}(\theta)\|_2$. (Left) varying subsequence length $S$ for no-buffer $B = 0$ and buffer $B=10$. (Right) varying buffer size $B$ for $S = 4$ and $S=64$ subsequence lengths. Error bars are SD over $100$ datasets.}
\label{fig:lds_grad_error}
\end{figure}

Figure~\ref{fig:lds_grad_error} are plots of the stochastic gradient error $\E_\SUBSEQ \|\bar{g}(\theta) - \tilde{g}(\theta)\|_2$ between the unbiased and buffered estimates evaluated at the true model parameters $\theta=\theta^*$.
Similar to the ARPHMM, we see that the error decays $O(1/S)$ and that moderate buffering (e.g. $B=10$) deceases the error by orders of magnitude in Figure~\ref{fig:lds_grad_error} (left).
And we see that the error decays geometrically in buffer size $O(L^B)$ in Figure~\ref{fig:lds_grad_error} (right).

\begin{figure}[!htb]
\centering
    \begin{minipage}[c]{.1\textwidth}
    \centering
    \hbox{\rotatebox{90}{\hspace{1em} $T = 10^6$}}
    \end{minipage}
    \hspace{0.2em}
    \begin{minipage}[c]{.4\textwidth}
    \centering
        \includegraphics[width=\textwidth]{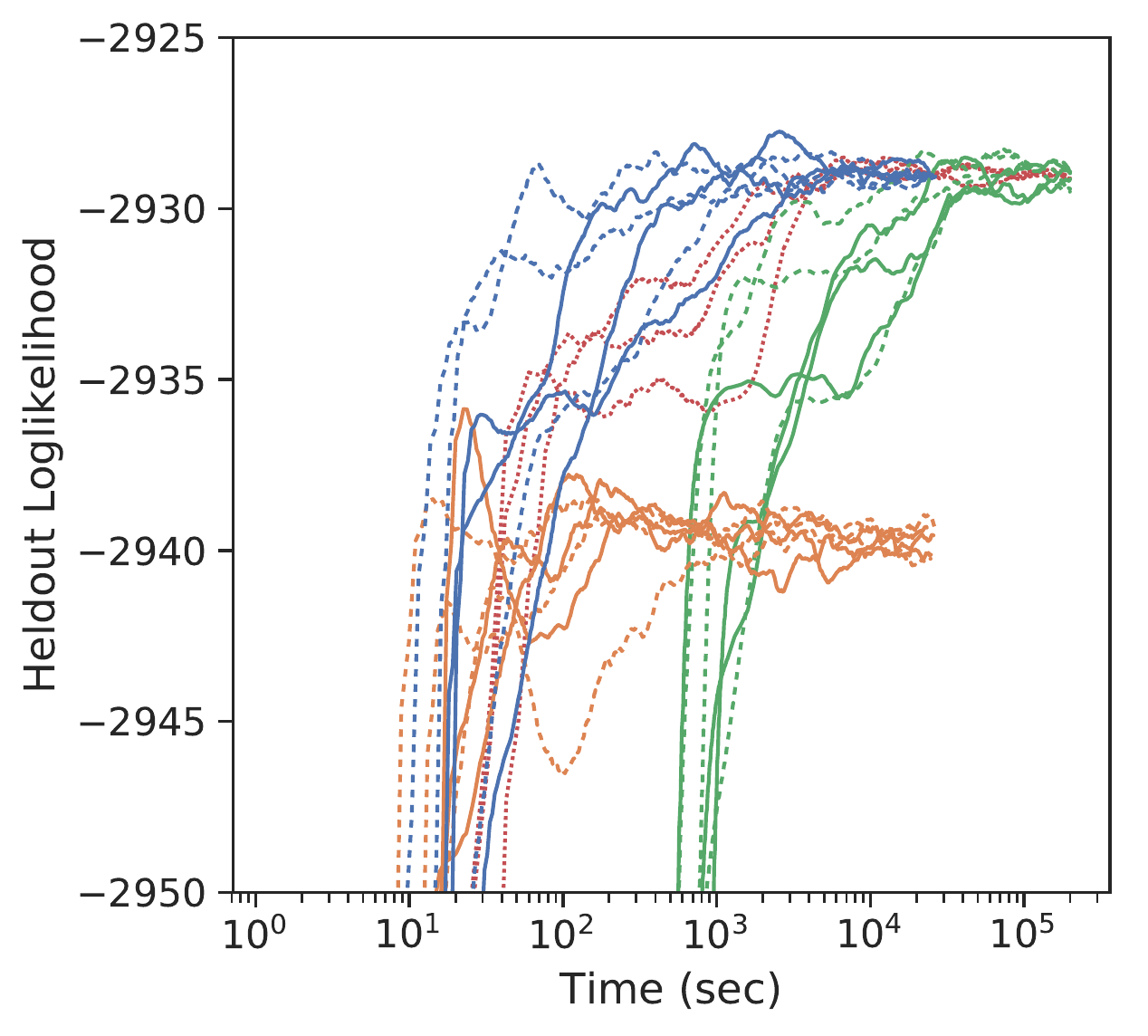}
    \end{minipage}
    \begin{minipage}[c]{.4\textwidth}
    \centering
        \includegraphics[width=\textwidth]{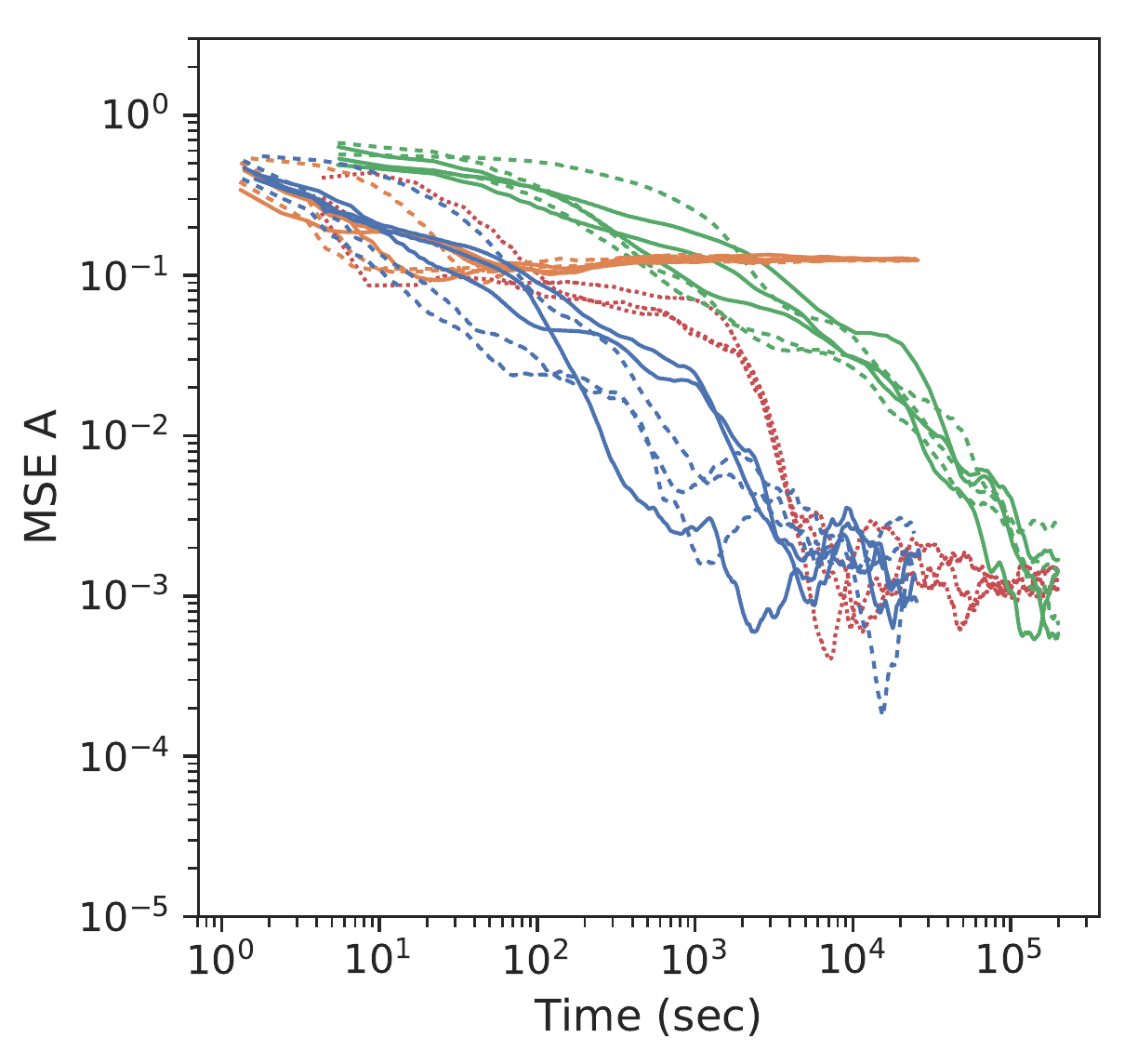}
    \end{minipage}

    \begin{minipage}[c]{.1\textwidth}
    \centering
    \hbox{\rotatebox{90}{\hspace{1em} $T = 10^6$}}
    \end{minipage}
    \hspace{0.2em}
    \begin{minipage}[c]{.4\textwidth}
    \centering
        \includegraphics[width=\textwidth]{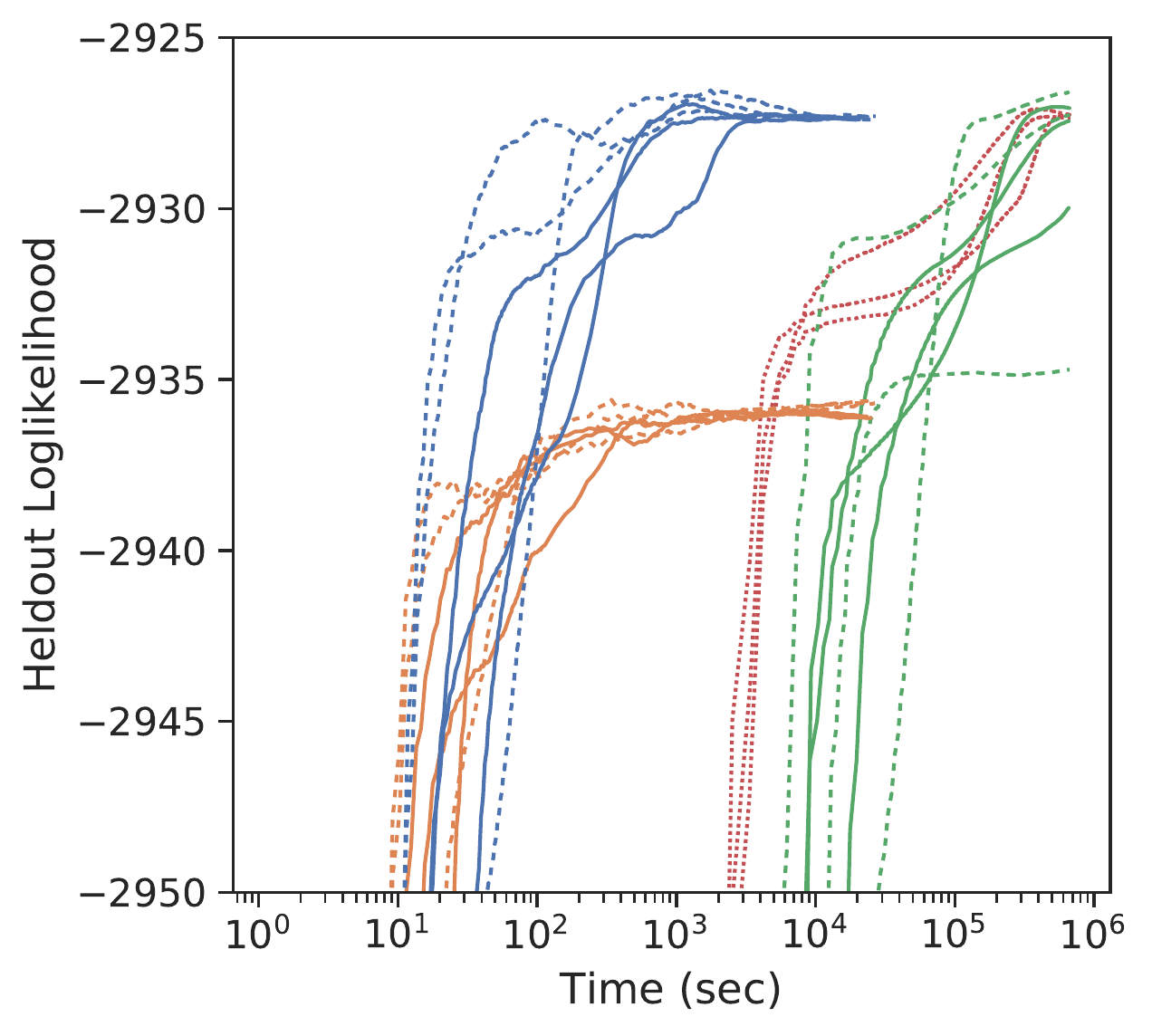}
    \end{minipage}
    \begin{minipage}[c]{.4\textwidth}
    \centering
        \includegraphics[width=\textwidth]{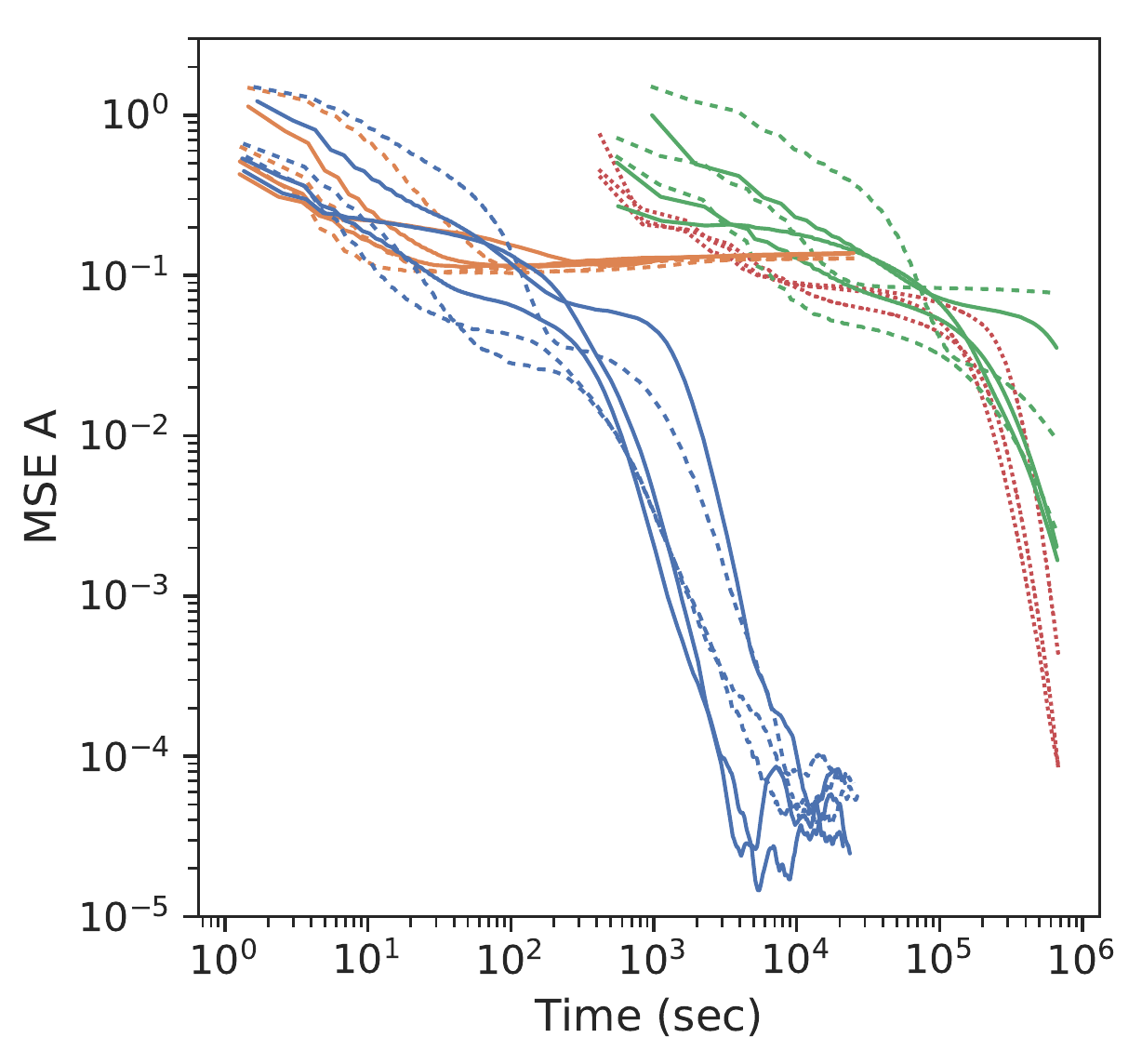}
    \end{minipage}
    \caption{Metrics vs Runtime on LGSSM with $T = 10^4$ (top), $T = 10^6$ (bottom) for different methods:
    \textcolor{BrickRed}{(Gibbs)}, \textcolor{ForestGreen}{(Full)},  \textcolor{Orange}{(No Buffer)} and \textcolor{Blue}{(Buffer)} SGMCMC.
    For SGMCMC methods, solid (\fullline) and dashed (\dashedline) lines indicate SGRLD and SGLD respectively.
    The different metrics are: (left) heldout loglikelihoood and (right) transition matrix estimation error $MSE(\hat{A}^{(s)}, A^*)$.
}
\label{fig:lds_synth}
\end{figure}

\begin{figure}[!htb]
    \centering
    \begin{minipage}[c]{.9\textwidth}
    \centering
        \includegraphics[width=\textwidth]{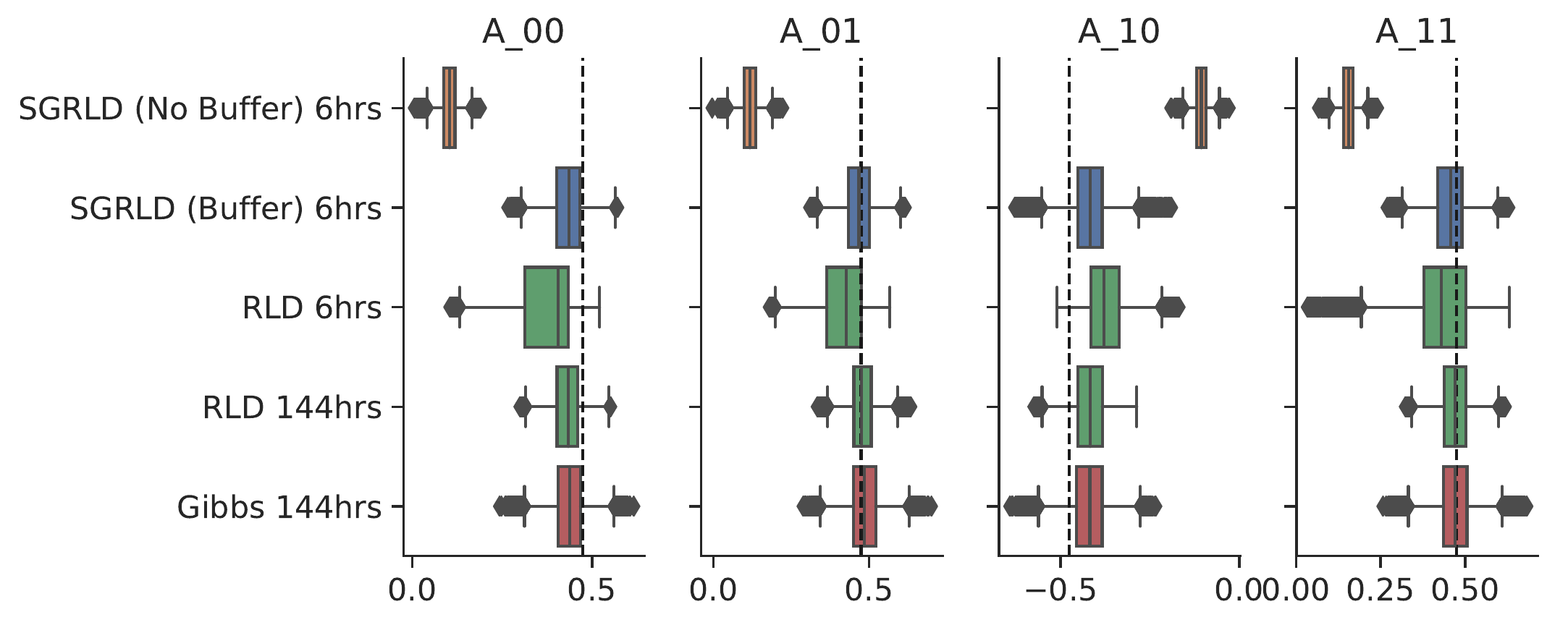}
    \end{minipage}

    \caption{Boxplot of MCMC samples of transition matrix $A$ for LGSSM data $T = 10^4$. SGRLD with buffering in 6 hours is comparable to RLD or Gibbs in 144 hours. SGRLD without buffering is biased and RLD in 6 hours has not fully mixed.}
\label{fig:lds_boxplot}
\end{figure}

In Figures~\ref{fig:lds_synth} and \ref{fig:lds_boxplot},
we compare SGLD (no-buffer and buffer), SGRLD (no-buffer and buffer), LD, RLD,
and a blocked Gibb sampler.
We fit our samplers on one training sequence and evaluate performance on one test sequence.
We consider two training sequences of lengths $T = 10^4$ and $T=10^6$ and evaluate on the same test sequence of length $T = 10^4$.
For the SGMCMC methods, we use a subsequence size of $S = 20$ with $B=0$ (no buffer) and $B=10$ (buffer).
We see that even with a large subsequence size, buffering is crucial for accurate inference as
SGMCMC methods without buffering converge to a different stationary distribution than the posterior.

In Table~\ref{tab:lds_ksd}, we evaluate the KSD of the different MCMC methods. We see that SGMCMC with buffering slightly outperforms the full sequence methods for $T = 10^4$ and significantly outperforms the full sequence methods for $T = 10^6$, while SGMCMC without buffering performs poorly due to bias.

\begin{table}[ht]
\caption{$\log_{10}$(KSD) by variable of LGSSM samplers at 6 hrs. Mean and (SD) over runs in Figure~\ref{fig:lds_synth}.}
\label{tab:lds_ksd}
\centering
\begin{tabular}{rrlll}
  \hline
\abovestrut{1em}
 & Sampler & $A$ & $Q$ & $R$ \\
  \hline
 \multirow{7}{*}{\rotatebox[origin=c]{90}{$T = 10^4$}}
\abovestrut{1em}
& SGLD (No Buffer)  &  2.39 (0.01) &  1.73 (0.03) &  1.48 (0.03) \\
& SGLD (Buffer)     &  0.88 (0.11) &  0.41 (0.11) &  0.86 (0.08) \\
& LD                &  0.99 (0.13) &  1.12 (0.19) &  1.10 (0.17) \\
\cline{2-5}
\abovestrut{1em}
& SGRLD (No Buffer) &  2.38 (0.01) &  1.70 (0.02) &  1.43 (0.02) \\
& SGRLD (Buffer)    &  0.85 (0.08) &  0.18 (0.12) &  0.77 (0.14) \\
& RLD               &  0.99 (0.12) &  0.90 (0.19) &  1.10 (0.17) \\
\cline{2-5}
\abovestrut{1em}
& Gibbs             &  0.74 (0.20) &  0.33 (0.18) &  1.06 (0.27) \\
 \hline
 \hline
 \multirow{7}{*}{\rotatebox[origin=c]{90}{$T = 10^6$}}
\abovestrut{1em}
& SGLD (No Buffer)  &  4.32 (0.01) &  3.79 (0.02) &  3.50 (0.02) \\
& SGLD (Buffer)     &  2.30 (0.19) &  1.61 (0.18) &  2.84 (0.03) \\
& LD                &  4.26 (0.35) &  4.00 (0.39) &  4.14 (0.19) \\
\cline{2-5}
\abovestrut{1em}
& SGRLD (No Buffer) &  4.27 (0.01) &  3.77 (0.02) &  3.23 (0.03) \\
& SGRLD (Buffer)    &  2.17 (0.33) &  1.64 (0.21) &  3.03 (0.12) \\
& RLD               &  4.34 (0.23) &  3.76 (0.25) &  4.03 (0.23) \\
\cline{2-5}
\abovestrut{1em}
& Gibbs             &  3.46 (0.28) &  3.52 (0.14) &  3.50 (0.28) \\
 \hline
\end{tabular}
\end{table}

\subsubsection{Synthetic SLDS}
We now consider synthetic data from a model we can view 
as switching extension of the LGSSM in Section~\ref{sec:experiments-LGSSM-synth}
or as a noisy  version of the ARHMM in the Supplement.
The true model parameters $\theta^*$ are
\begin{equation*}
\Pi = \begin{bmatrix}
    0.9 & 0.1 \\
    0.1 & 0.9
    \end{bmatrix}
\enspace, \enspace
Q_1 = Q_2 = 0.1 \cdot \begin{bmatrix}
    1 & 0 \\
    0 & 1
    \end{bmatrix}
\enspace, \enspace
C = \begin{bmatrix}
    1 & 0 \\
    0 & 1
    \end{bmatrix}
\enspace, \enspace
R = 0.1 \cdot \begin{bmatrix}
    1 & 0 \\
    0 & 1
    \end{bmatrix}
\enspace,
\end{equation*}
\begin{equation*}
A_1 = 0.9 \cdot \begin{bmatrix}
    \cos(-\vartheta) & -\sin(-\vartheta) \\
    \sin(-\vartheta) & \cos(-\vartheta)
    \end{bmatrix}
\enspace, \enspace
A_2 = 0.9 \cdot \begin{bmatrix}
    \cos(\vartheta) & -\sin(\vartheta) \\
    \sin(\vartheta) & \cos(\vartheta)
    \end{bmatrix}\enspace,
\end{equation*}
where again $\vartheta=\pi/4$.
We generate sequences of length $T = 10^4$ and $10^6$.

\begin{figure}[!htb]
    \centering
    \begin{minipage}[t]{.48\textwidth}
    \centering
        \includegraphics[width=\textwidth]{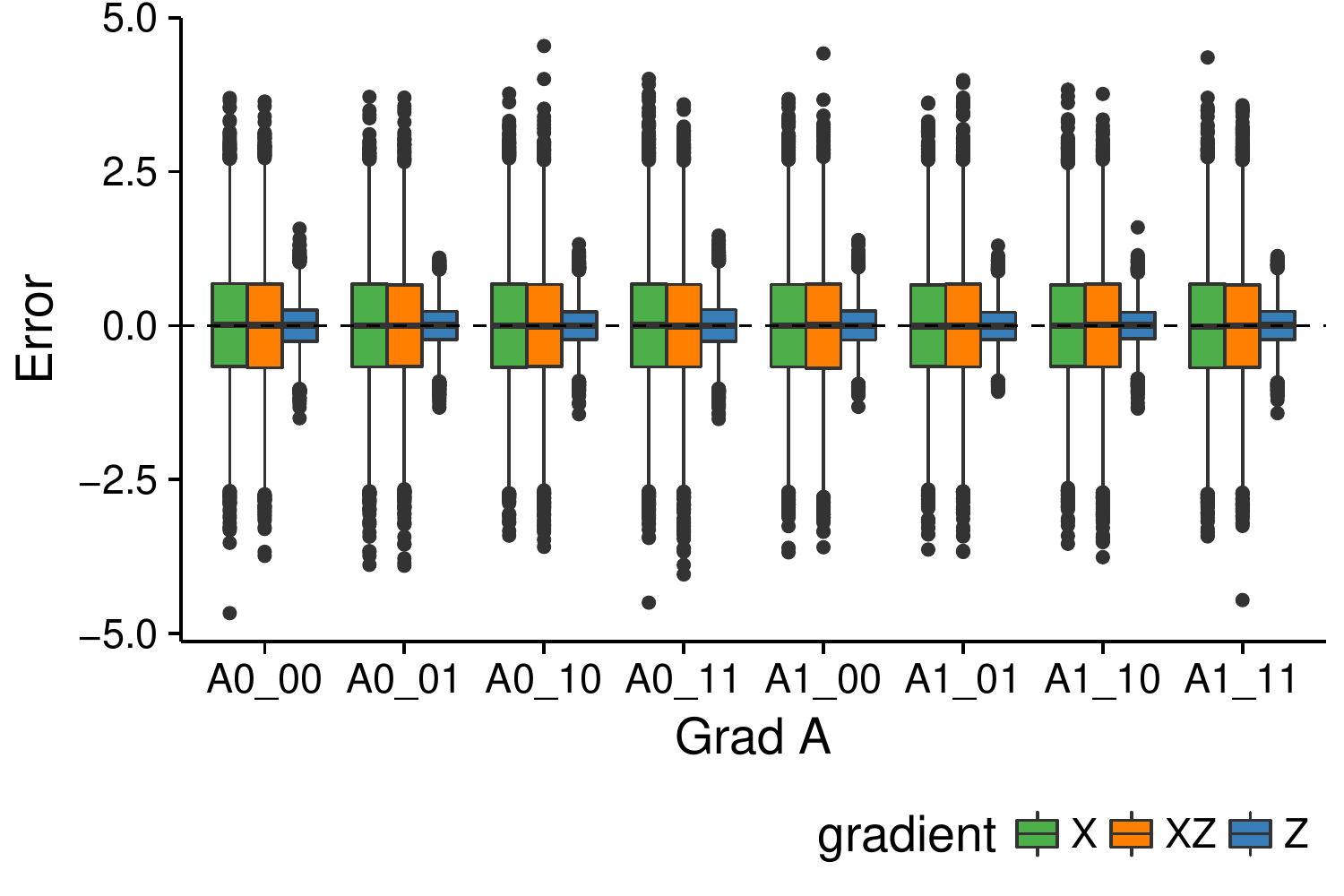}
    \end{minipage}
    \hspace{0.1cm}
    \begin{minipage}[t]{.48\textwidth}
    \centering
        \includegraphics[width=\textwidth]{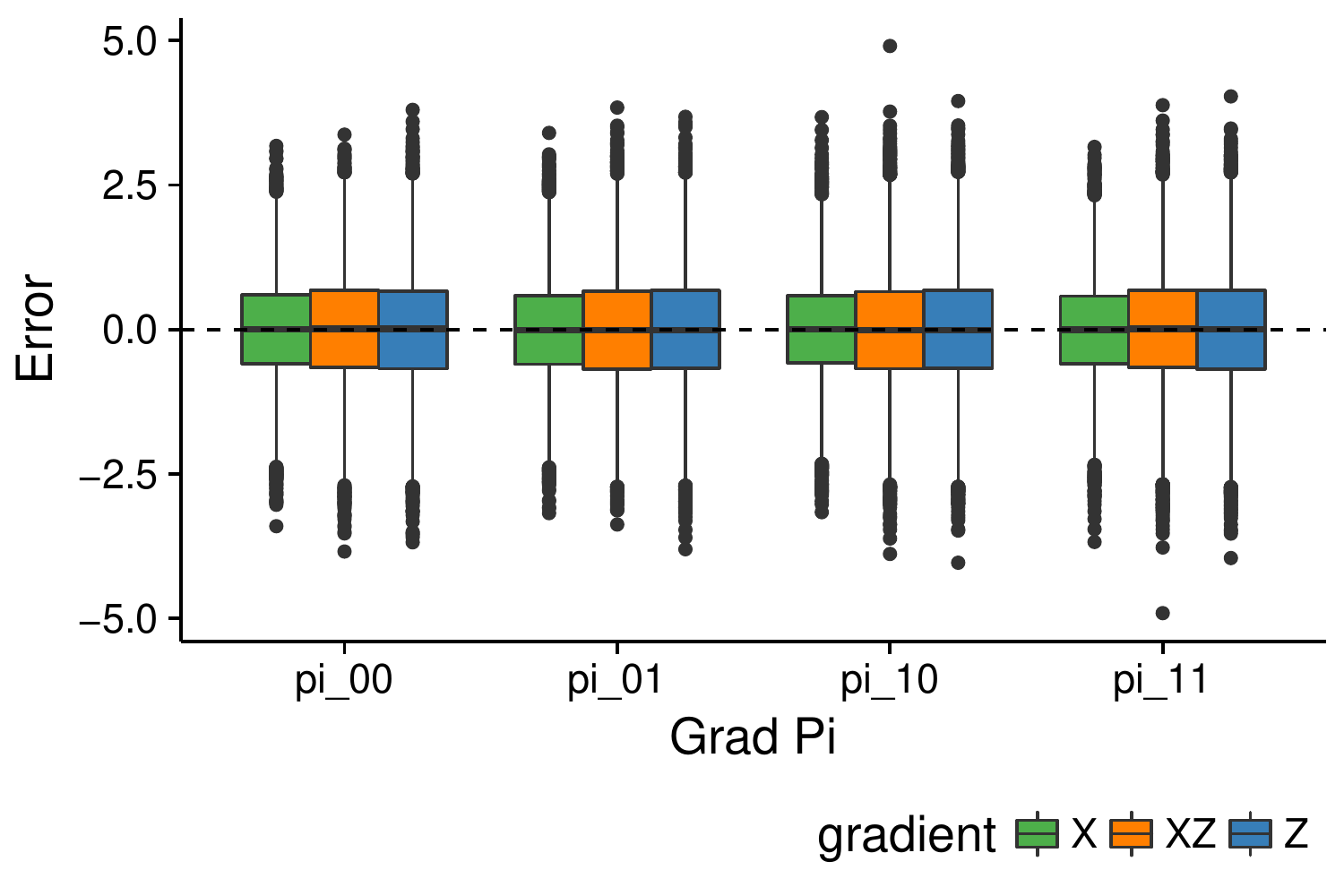}
    \end{minipage}

    \caption{SLDS gradient error for the different estimators Eqs.~\eqref{eq:slds_naive}-\eqref{eq:slds_x_marginal}. (Left) Boxplots of $\tilde{g}(\theta)_A - g(\theta)_A$. (Right) Boxplots of $\tilde{g}(\theta)_\Pi - g(\theta)_\Pi$. 
}
    \label{fig:slds_gradients}
\end{figure}

\begin{figure}[!htb]
    \centering
    \begin{minipage}[t]{.48\textwidth}
    \centering
        \includegraphics[width=\textwidth]{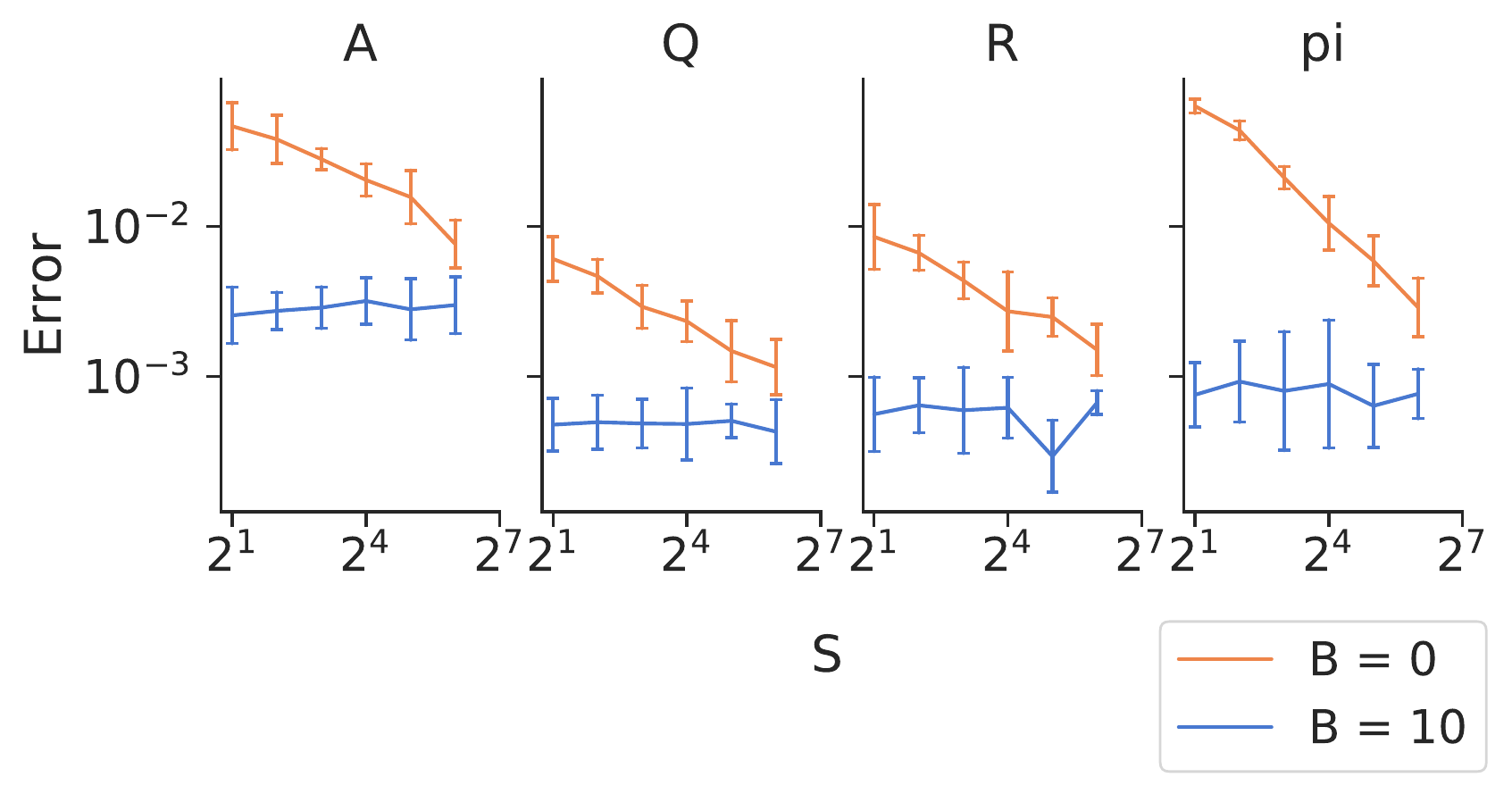}
    \end{minipage}
    \hspace{0.1cm}
    \begin{minipage}[t]{.48\textwidth}
    \centering
        \includegraphics[width=\textwidth]{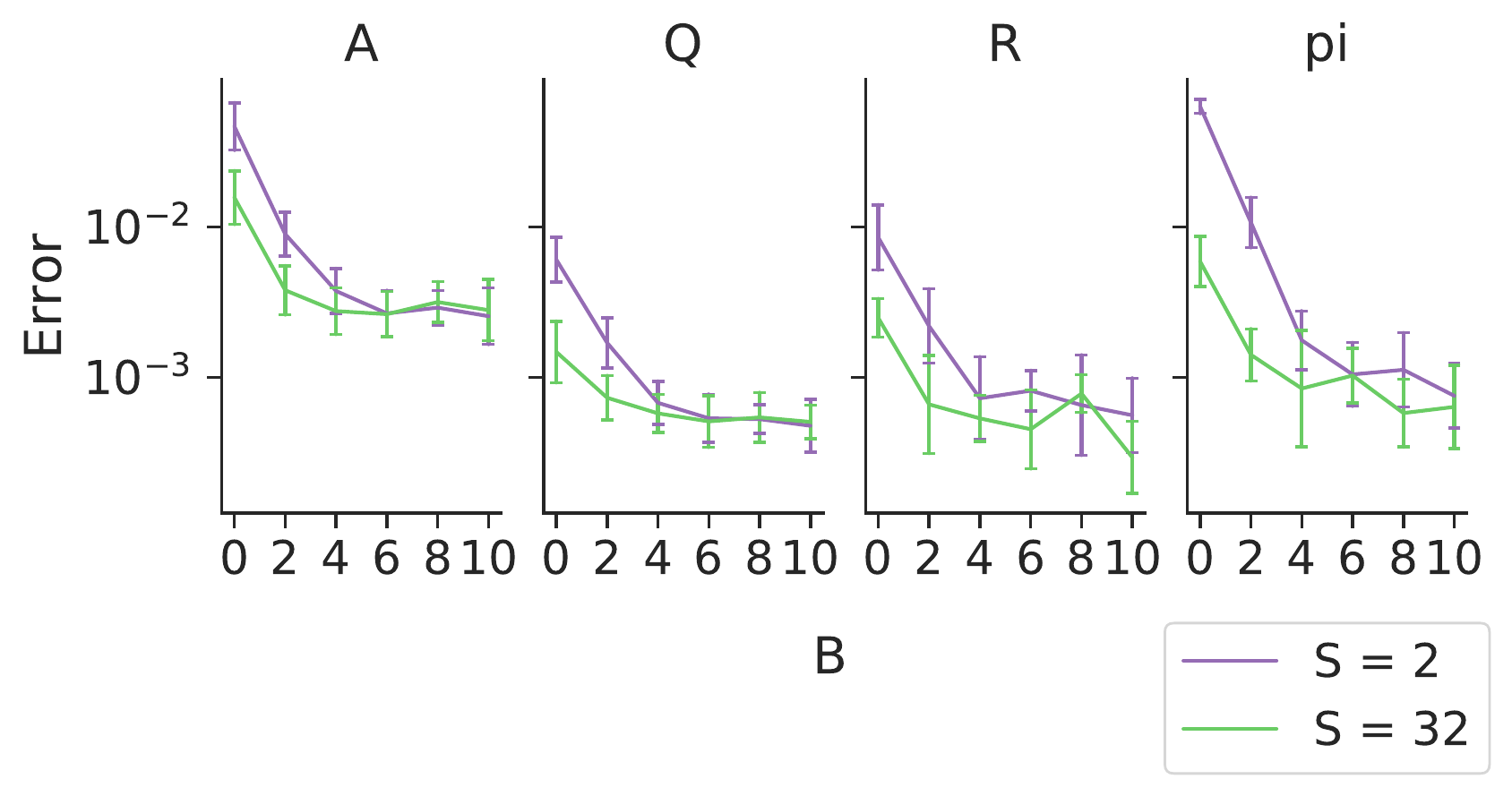}
    \end{minipage}

    \caption{Stochastic gradient error $\E_\SUBSEQ\|\bar{g}(\theta) - \tilde{g}(\theta)\|_2$ for \texttt{z Gradient}. (Left) error varying subsequence length $S$ for no-buffer $B=0$ and buffer $B=4$. (Right) error varying buffer size $B$ for small $S = 2$ and long $S=32$ subsequences. Error bars are SD over $100$ datasets. 
}
    \label{fig:slds_stochastic_gradients}
\end{figure}

\begin{figure}[!htb]
    \centering

    \begin{minipage}[c]{.05\textwidth}
    \centering
    \hbox{\rotatebox{90}{\hspace{1em} $T = 10^4$}}
    \end{minipage}
    \begin{minipage}[c]{.3\textwidth}
    \centering
        \includegraphics[width=\textwidth]{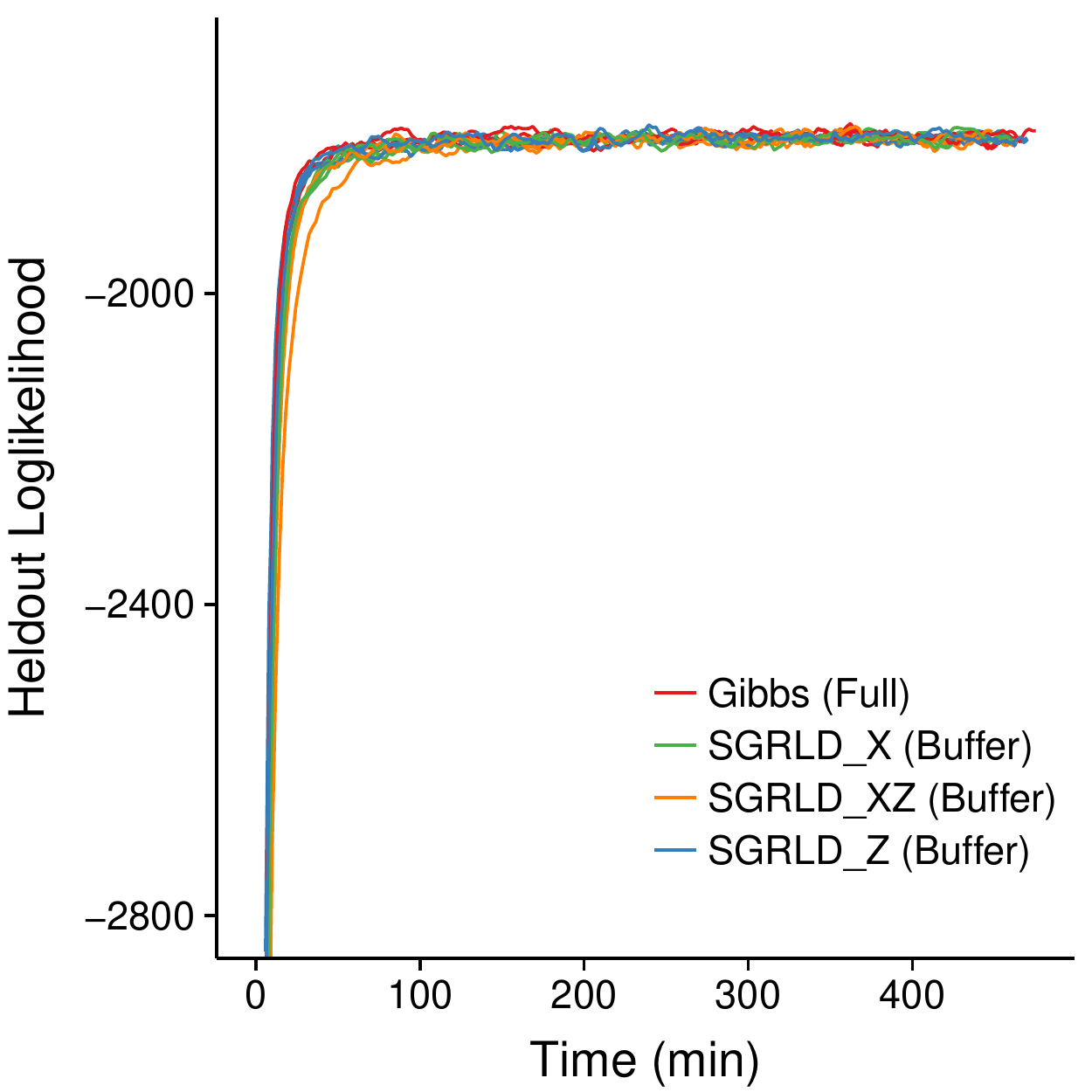}
    \end{minipage}
    \begin{minipage}[c]{.3\textwidth}
    \centering
        \includegraphics[width=\textwidth]{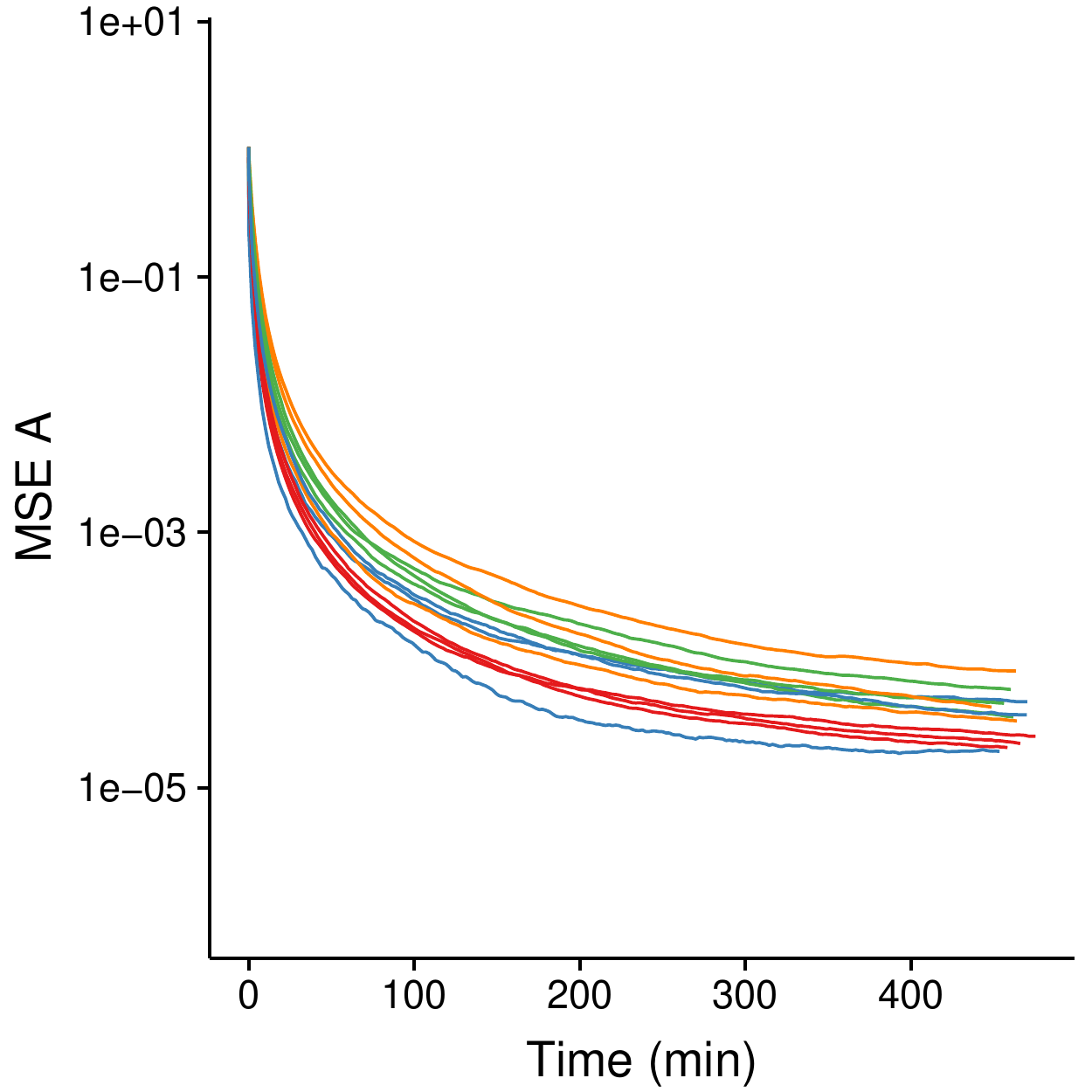}
    \end{minipage}
    \begin{minipage}[c]{.3\textwidth}
    \centering
        \includegraphics[width=\textwidth]{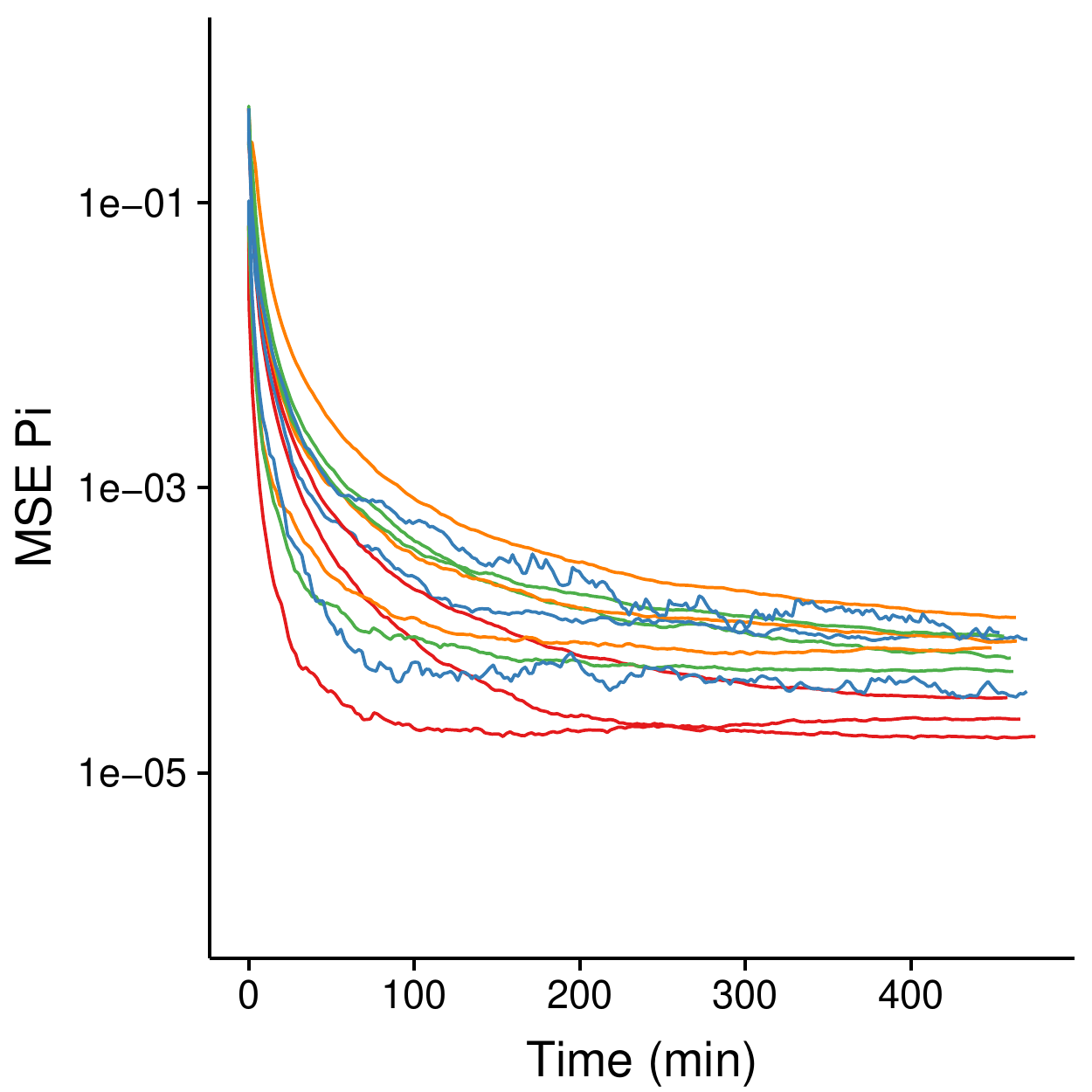}
    \end{minipage}

    \begin{minipage}[c]{.05\textwidth}
    \centering
    \hbox{\rotatebox{90}{\hspace{1em} $T = 10^6$}}
    \end{minipage}
    \begin{minipage}[c]{.3\textwidth}
    \centering
        \includegraphics[width=\textwidth]{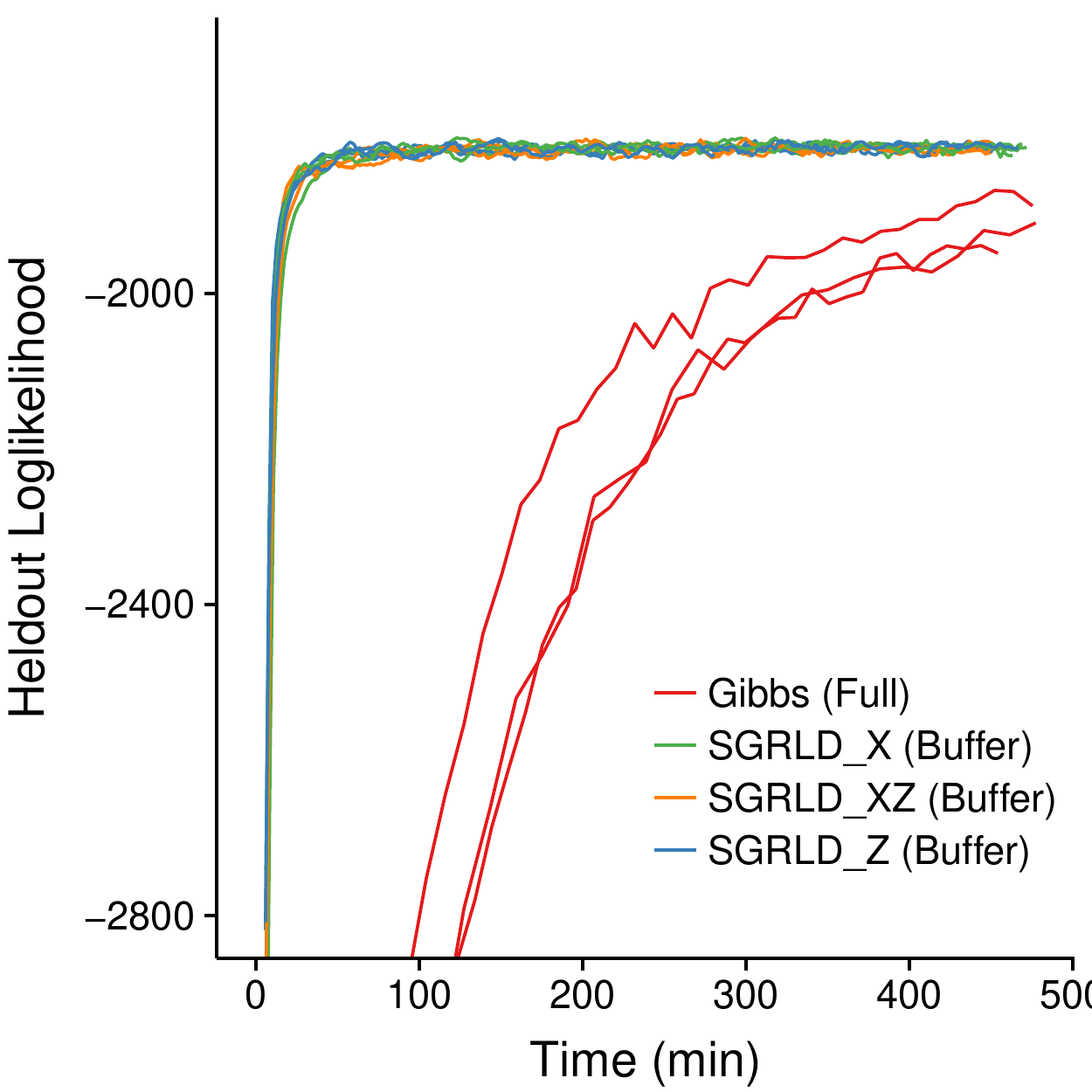}
    \end{minipage}
    \begin{minipage}[c]{.3\textwidth}
    \centering
        \includegraphics[width=\textwidth]{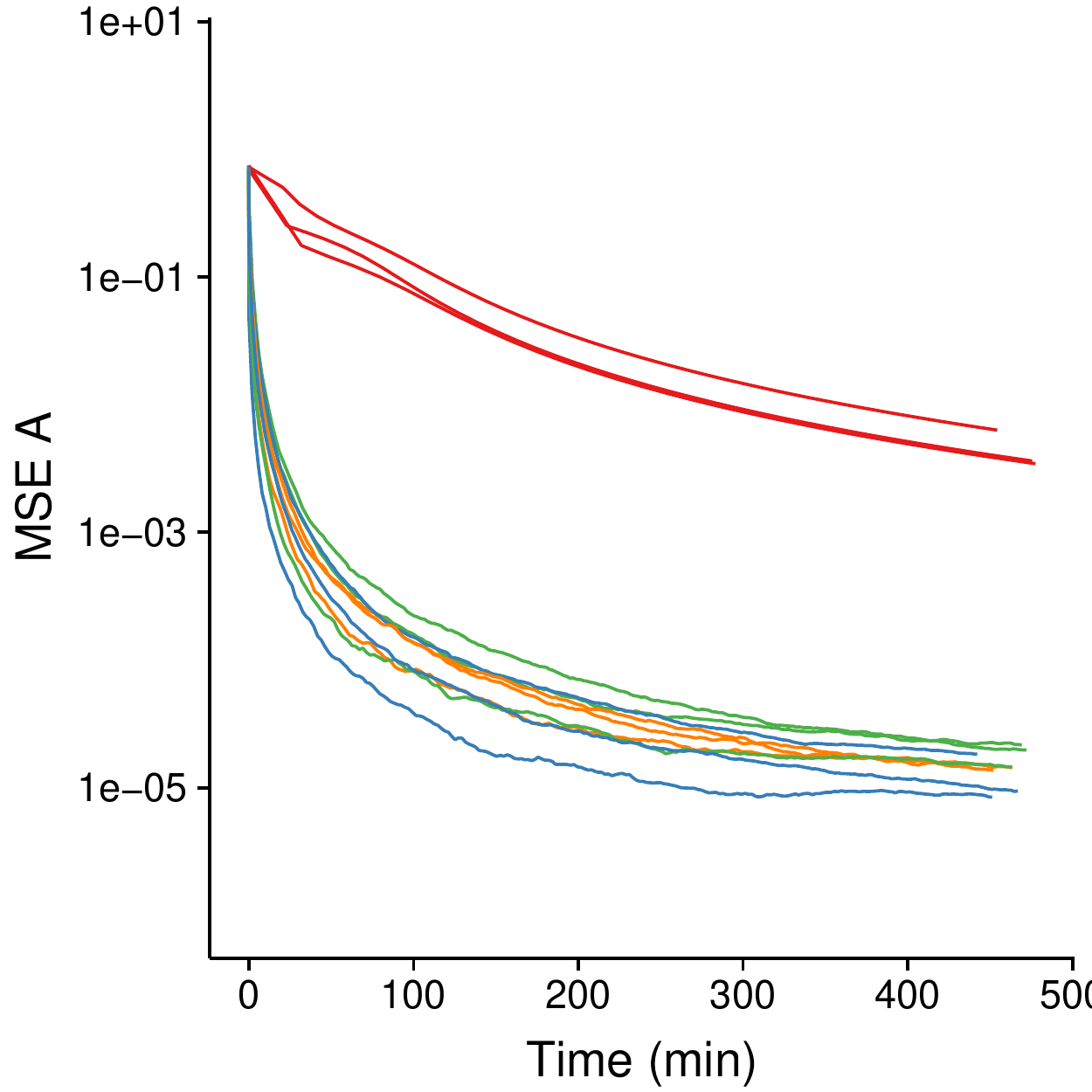}
    \end{minipage}
    \begin{minipage}[c]{.3\textwidth}
    \centering
        \includegraphics[width=\textwidth]{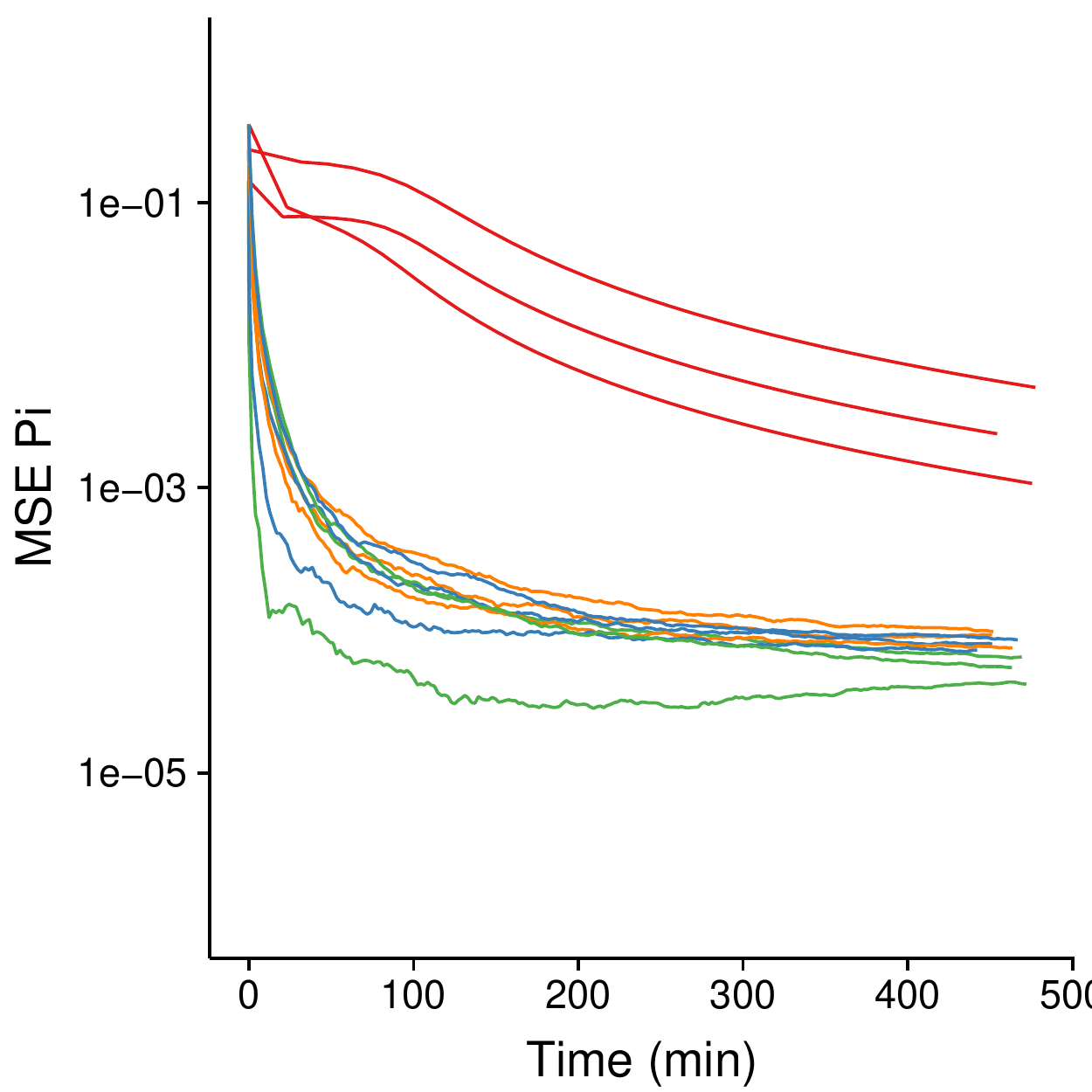}
    \end{minipage}
    \caption{Metrics vs Runtime on SLDS data for different inference methods:
    \textcolor{BrickRed}{Gibbs}, \textcolor{ForestGreen}{SGRLD X}, \textcolor{YellowOrange}{SGRLD XZ}, and \textcolor{Blue}{SGLRD Z}.
    (Top) $T = 10^4$ (Bottom) $T = 10^6$.
    The metrics are: (left) heldout loglikelihood, (center) estimation error $MSE(\hat{A}^{(s)}, A^*)$, (right) estimation error $MSE(\hat\Pi^{(s)}, \Pi^*)$. 
    }
    \label{fig:slds_lowtrans_metrics}
\end{figure}

We first compare the variance of the three difference Monte-Carlo gradient estimators for SLDS: using $(x,z)$ samples (\texttt{xz Gradient}) as in Eq.~\eqref{eq:slds_naive}, only using $z$ samples (\texttt{z Gradient}) as in Eq.~\eqref{eq:slds_z_marginal}, and only using $x$ samples (\texttt{x Gradient}) as in Eq.~\eqref{eq:slds_x_marginal}.
Figure~\ref{fig:slds_gradients} presents boxplots of $\tilde{g}(\theta) - g(\theta)$ for the three different estimators at $\theta = \theta^*$. 
From Figure~\ref{fig:slds_gradients} (left), we see that \texttt{z Gradient} (blue) has much lower variance than the other two estimators for the gradient of $A$.
This also holds for the gradients of $Q$ and $R$ (see Supplement).
From Figure~\ref{fig:slds_gradients} (right), we see that all three estimators have similar variance for the gradient of $\Pi$ (with \texttt{x Gradient} (green) slightly better than the other two).
This agrees with intuition described in Section~\ref{sec:models:SLDS:gradients}.
Because \texttt{z Gradient} has lower variance than the other two estimators, we can use larger step-sizes, leading to faster convergence and mixing.

Figure~\ref{fig:slds_stochastic_gradients} are plots of the stochastic gradient error $\E_\SUBSEQ \|\bar{g}(\theta) - \tilde{g}(\theta)\|_2$ between the unbiased and buffered estimates (for \texttt{z Gradient}) evaluated at the true model parameters $\theta=\theta^*$.
For short buffered subsequences (e.g. small $S$ and $B$), the error decays as expected $O(L^B/S)$; however, for longer buffered subsequences the error is dominated by the Monte Carlo error in the number of Gibbs steps used in sampling $z$ for calculating $\tilde{g}$ in Eq.~\ref{eq:slds_z_marginal} .

In Figure~\ref{fig:slds_lowtrans_metrics},
we compare SGRLD (with buffer) using each of the gradient estimators Eqs.~\eqref{eq:slds_naive}-\eqref{eq:slds_x_marginal}, and a blocked Gibb sampler.
We run our samplers on one training sequence and evaluate performance on another test sequence.
For all SGRLD samplers, we used subsequence size of $S = 10$ and $B=10$.
As the marginal loglikelihood is not available in closed form for SLDSs,
we instead use a Monte Carlo approximation of the EM lower bound $\log \Pr(y \, | \, \theta) \geq \E_{x, z | y,\theta}[ \log \Pr(y, x, z \, | \, \theta) ]$
where the expectation is approximated with samples of $x,z$ drawn using blocked Gibbs for each fixed $\theta$.
From Figure~\ref{fig:slds_lowtrans_metrics}, we see that SGRLD methods perform similarly to Gibbs for $T = 10^4$, but vastly outperform Gibbs for $T = 10^6$.

\subsubsection{Canine Seizure iEEG}
Recall the data from Section~\ref{sec:experiments-ARHMM-seizure}.
For our SLDS analysis,
we set the continuous latent variable dimension to $n = 1$.
The number of latent states remains $K = 5$.
We again compare SGLD and SGRLD samplers with $S = 100$ and $B=10$ to Gibbs samplers on both the full data set and a $10\%$ subset of seizures.
In Figure~\ref{fig:canine-slds},
we see again that the SGRLD sampler converges much more rapidly than the other methods.
In comparison to Figure~\ref{fig:canine},
we also see that the SLDS is a better model for this data than the ARHMM (as measured by heldout likelihood). 
Qualitatively, the SLDS segmentations of seizures (Figure~\ref{fig:canine-slds} (right)) is more contiguous than the ARHMM segmentation (Figure~\ref{fig:canine} (right)).
\begin{figure}[!htb]
    \centering
    \begin{minipage}[c]{.29\textwidth}
    \centering
        \includegraphics[width=\textwidth]{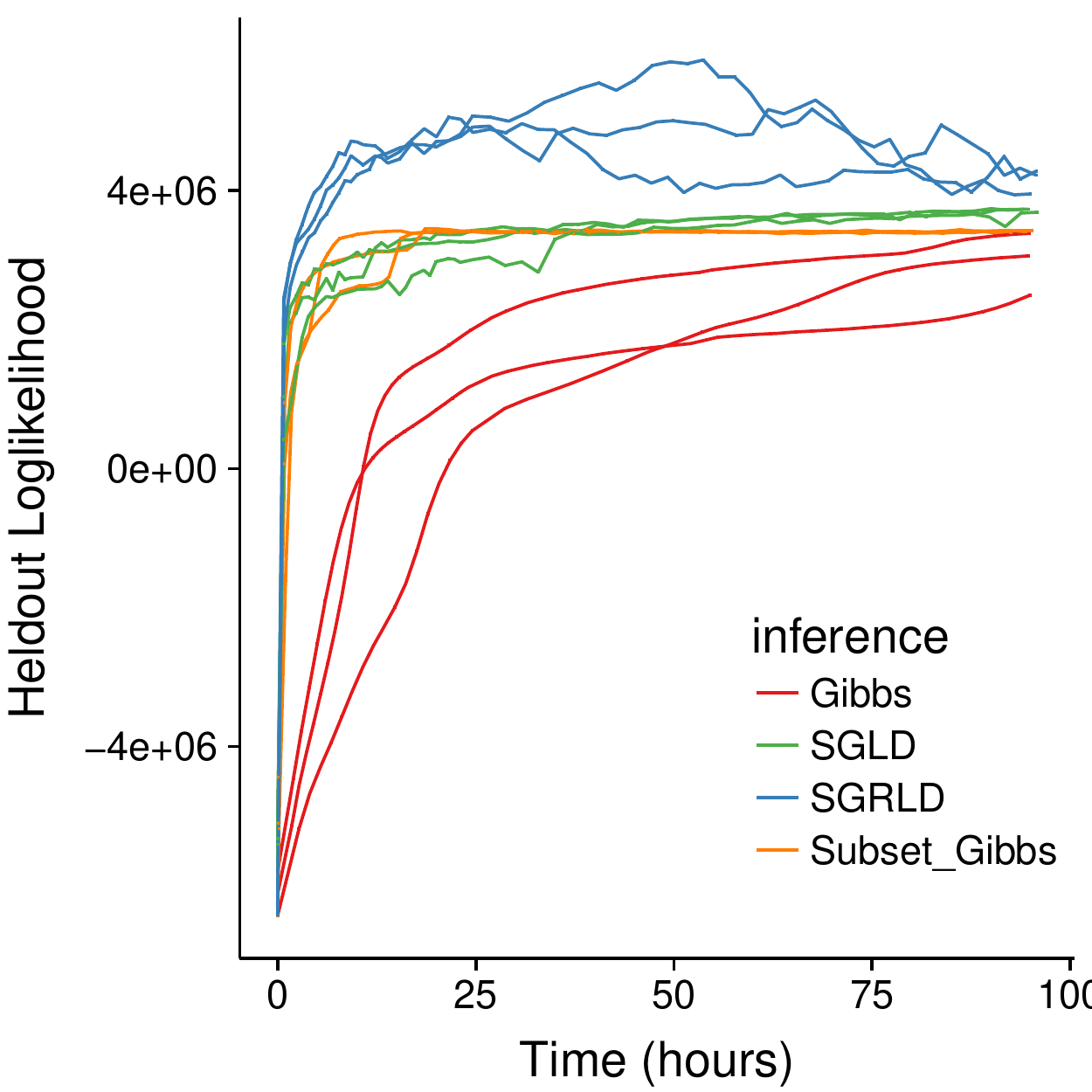}
    \end{minipage}
    \centering
    \begin{minipage}[c]{.29\textwidth}
    \centering
        \includegraphics[width=\textwidth]{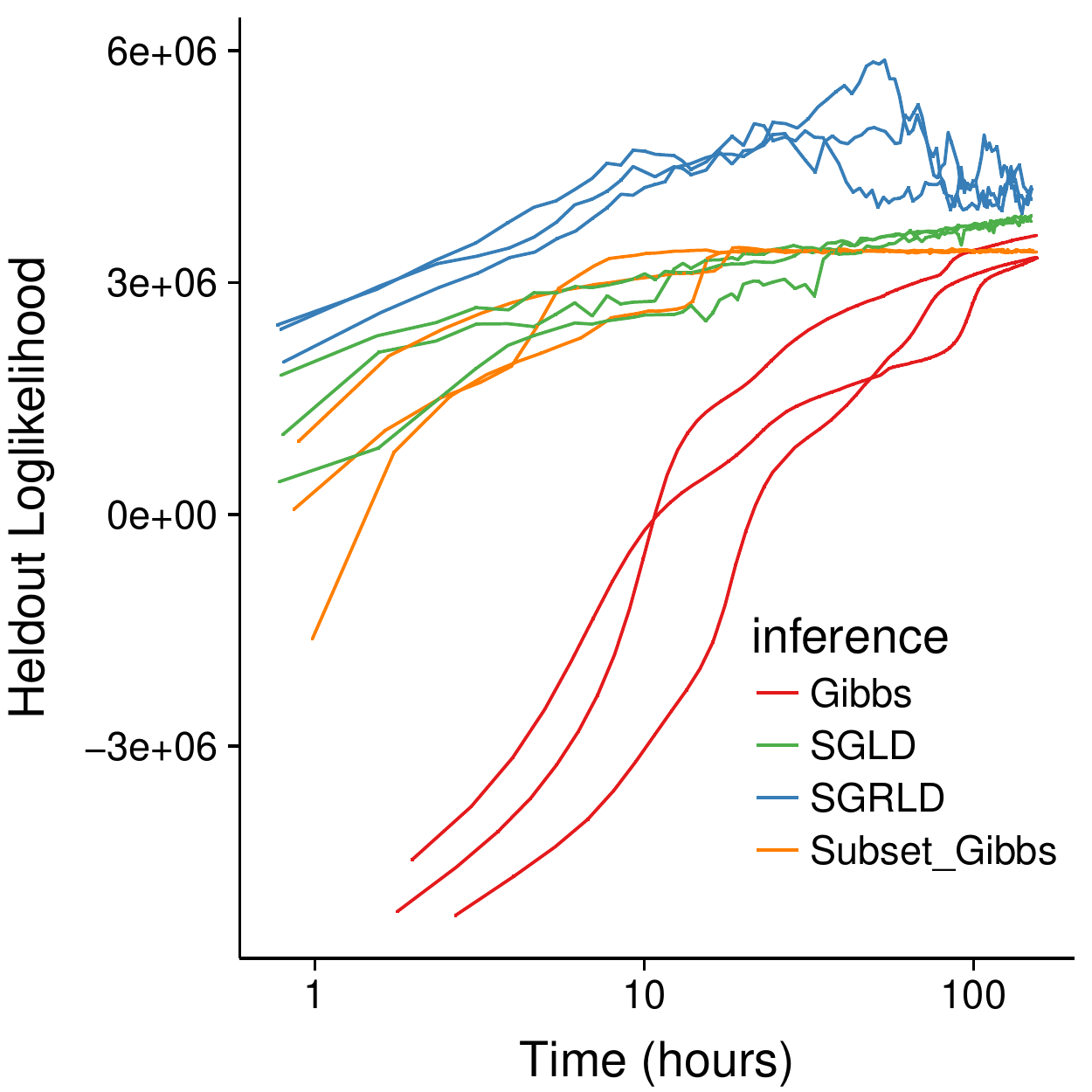}
    \end{minipage}
    \begin{minipage}[c]{.4\textwidth}
    \centering
    \includegraphics[width=\textwidth]{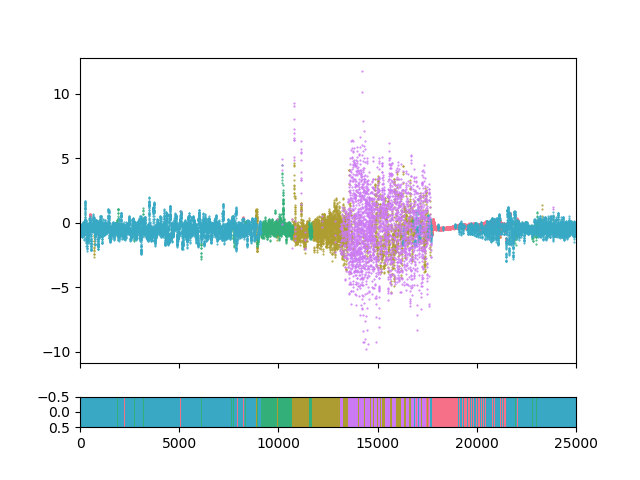}
    \end{minipage}
    \caption{SLDS Canine Seizure Data: (left) heldout loglikelihood vs time, (center) heldout loglikelihood vs time on log-scale (right) example segmentation by ARHMM fit with SGRLD.
    The MCMC methods compared are \textcolor{BrickRed}{Gibbs}, \textcolor{YellowOrange}{Subset Gibbs}, \textcolor{ForestGreen}{SGLD}, and \textcolor{Blue}{SGLRD}.
    }
    \label{fig:canine-slds}
\end{figure}

\subsubsection{Historical Cities Weather Data}
We apply SGMCMC to historical city weather data from Kaggle~\cite{kaggleweather}.
The data consists of hourly temperature, pressure and humidity measurements $(m=3)$ for 20 US cities over 5 years with $T=44,000$ hourly observations per city.
We fit SLDS models with $n=3$ and $K=4$ to both the hourly and daily average observations, treating the cities independently.
For both sets of observations, we perform an 80-20 train-test split over 20 cities, running inference on the training set (16 cities) and evaluating loglikelihood on the test set (4 cities).

\begin{figure}[!htb]
    \centering
    \begin{minipage}[b]{.28\textwidth}
    \centering
        \includegraphics[width=\textwidth]{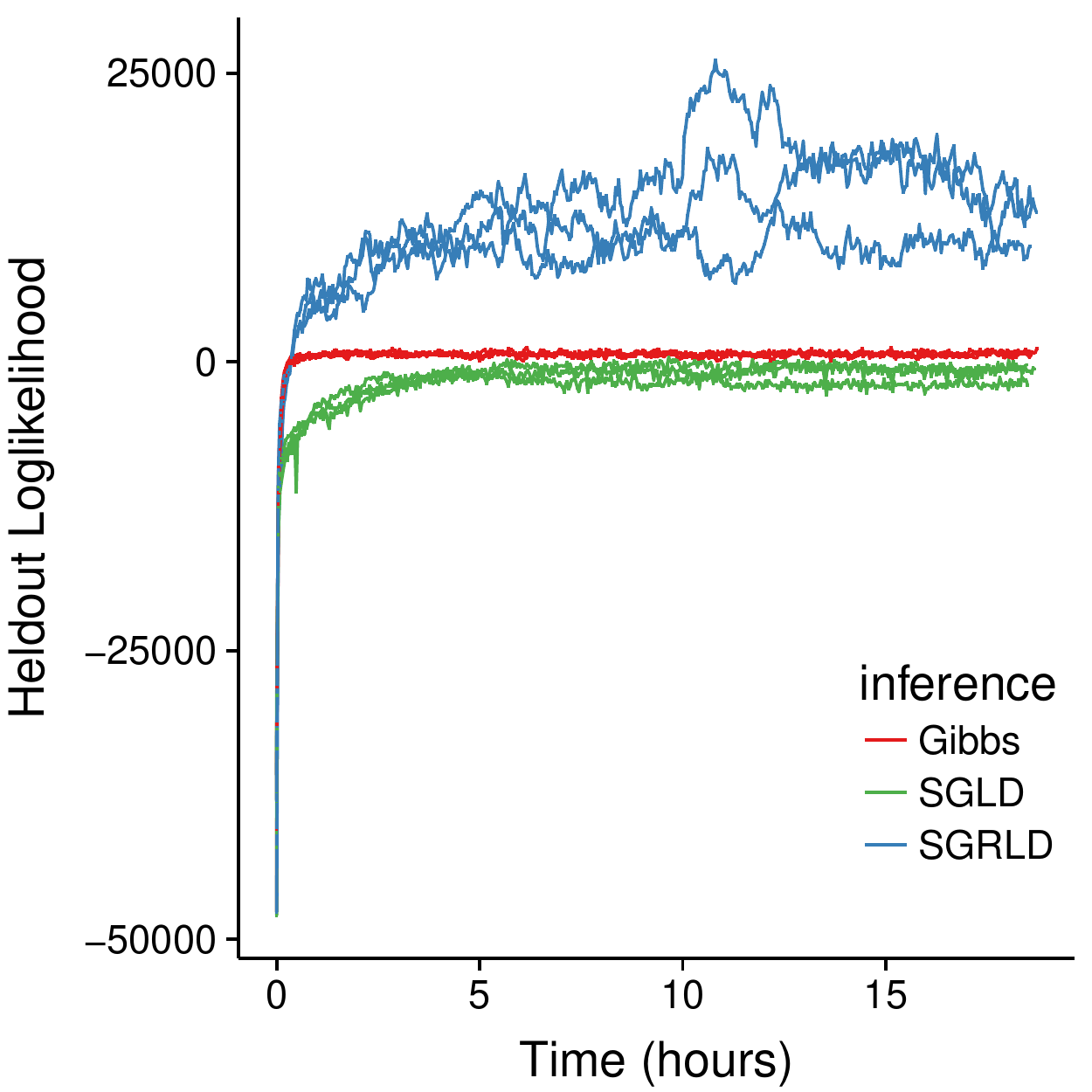}
    \end{minipage}
    \hspace{0.2em}
    \begin{minipage}[t]{.34\textwidth}
    \centering
        \includegraphics[width=\textwidth]{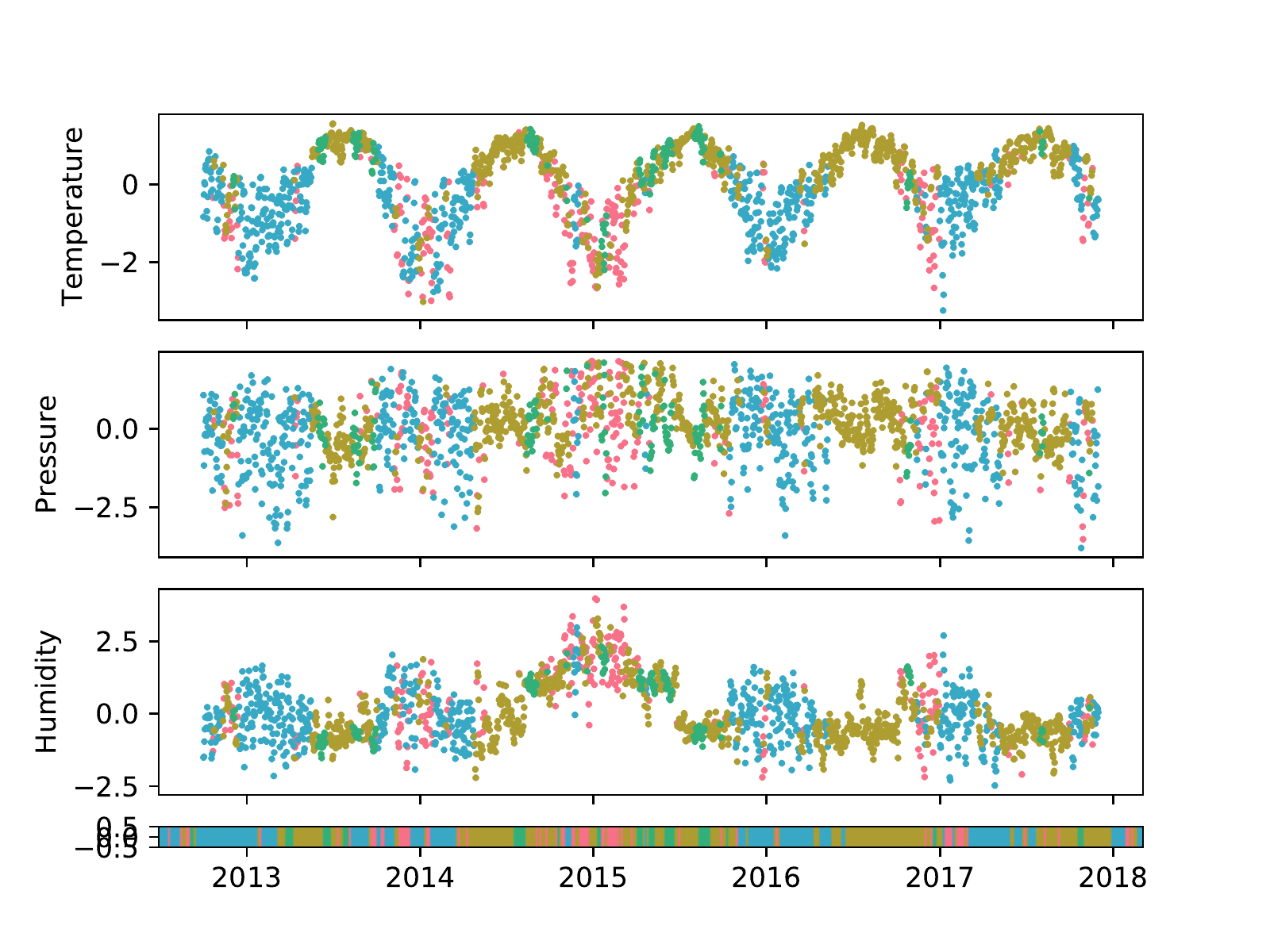}
    \end{minipage}
    \begin{minipage}[t]{.34\textwidth}
    \centering
        \includegraphics[width=\textwidth]{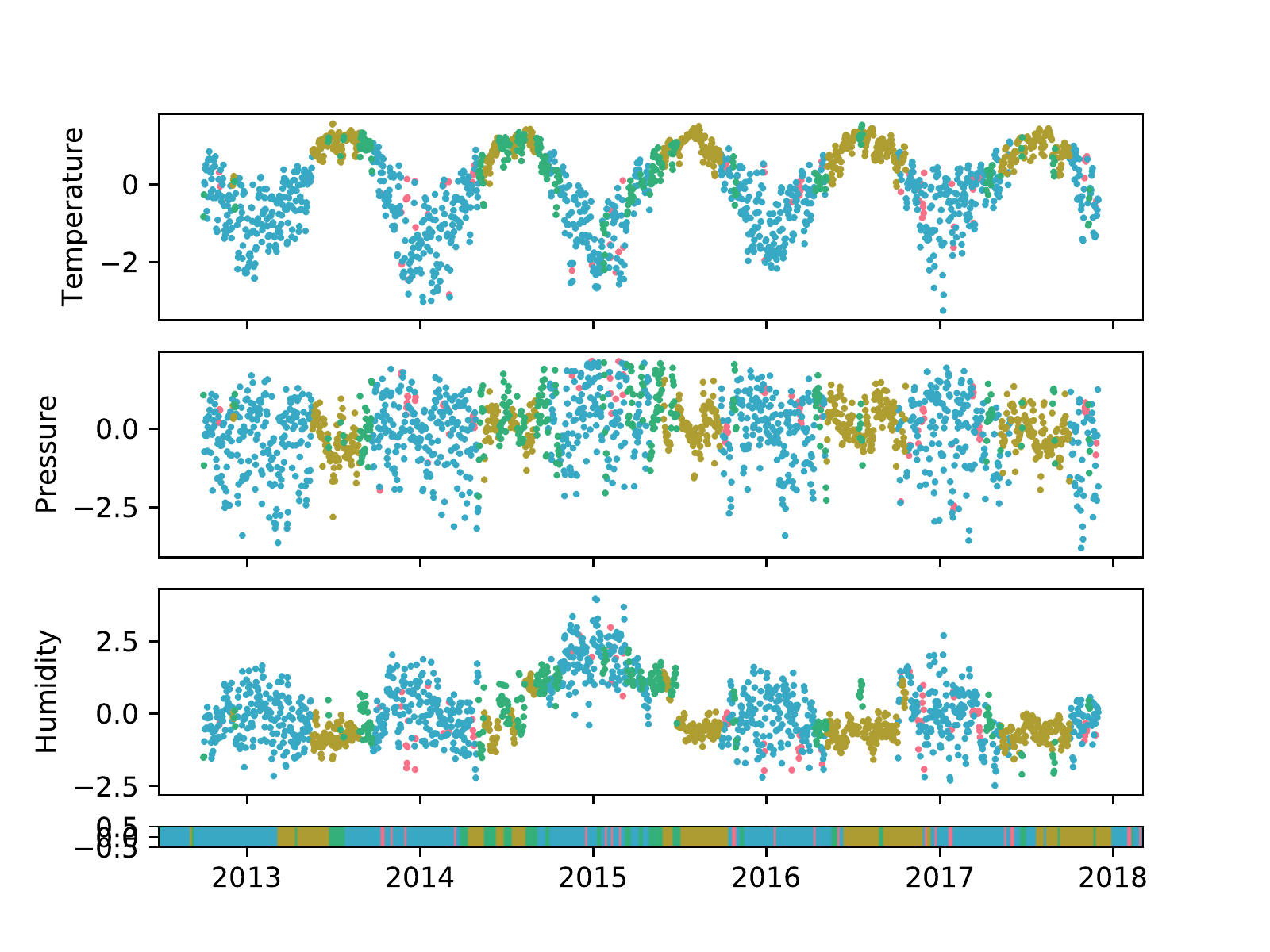}
    \end{minipage}

    \begin{minipage}[b]{.28\textwidth}
    \centering
        \includegraphics[width=\textwidth]{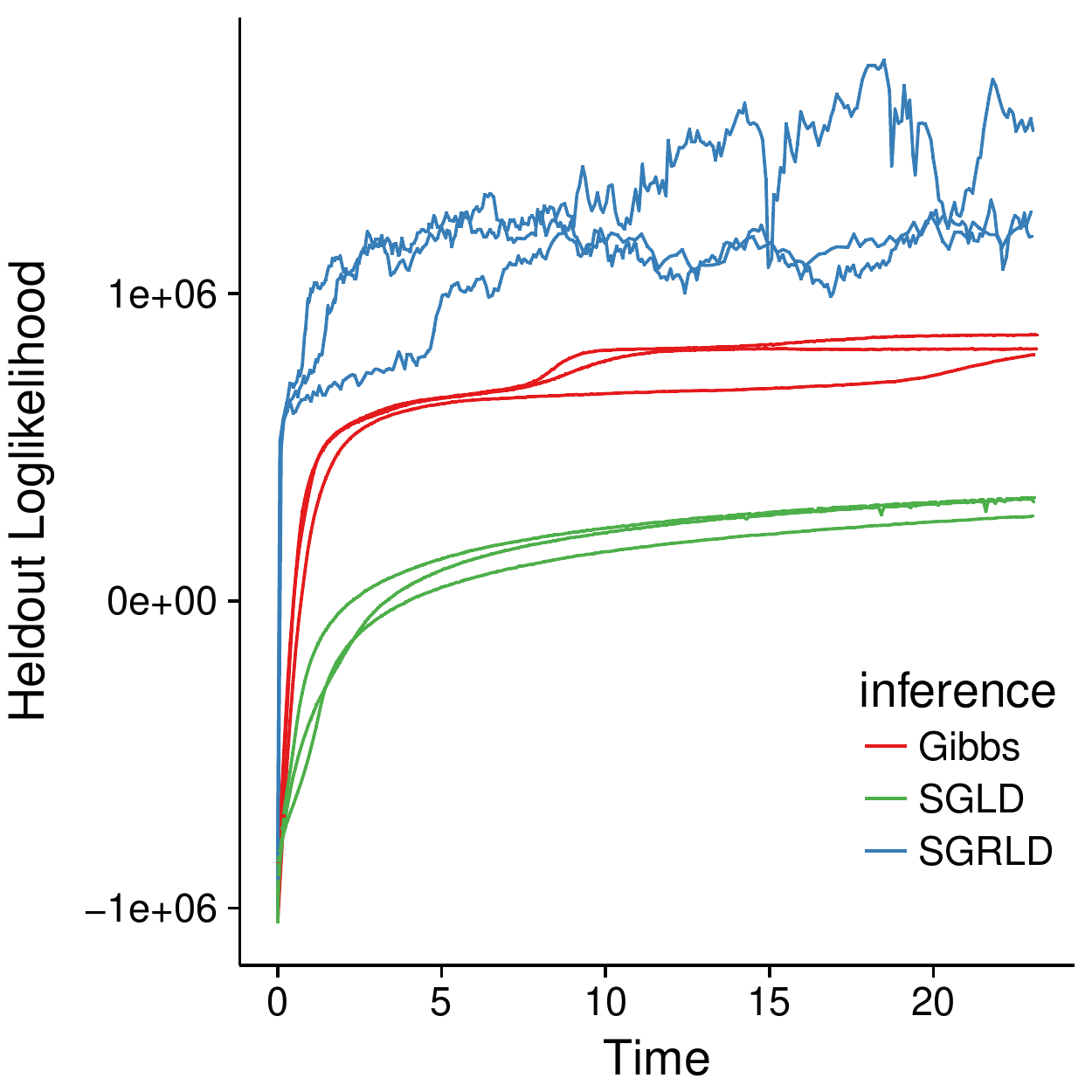}
    \end{minipage}
    \hspace{0.2em}
    \begin{minipage}[t]{.34\textwidth}
    \centering
        \includegraphics[width=\textwidth]{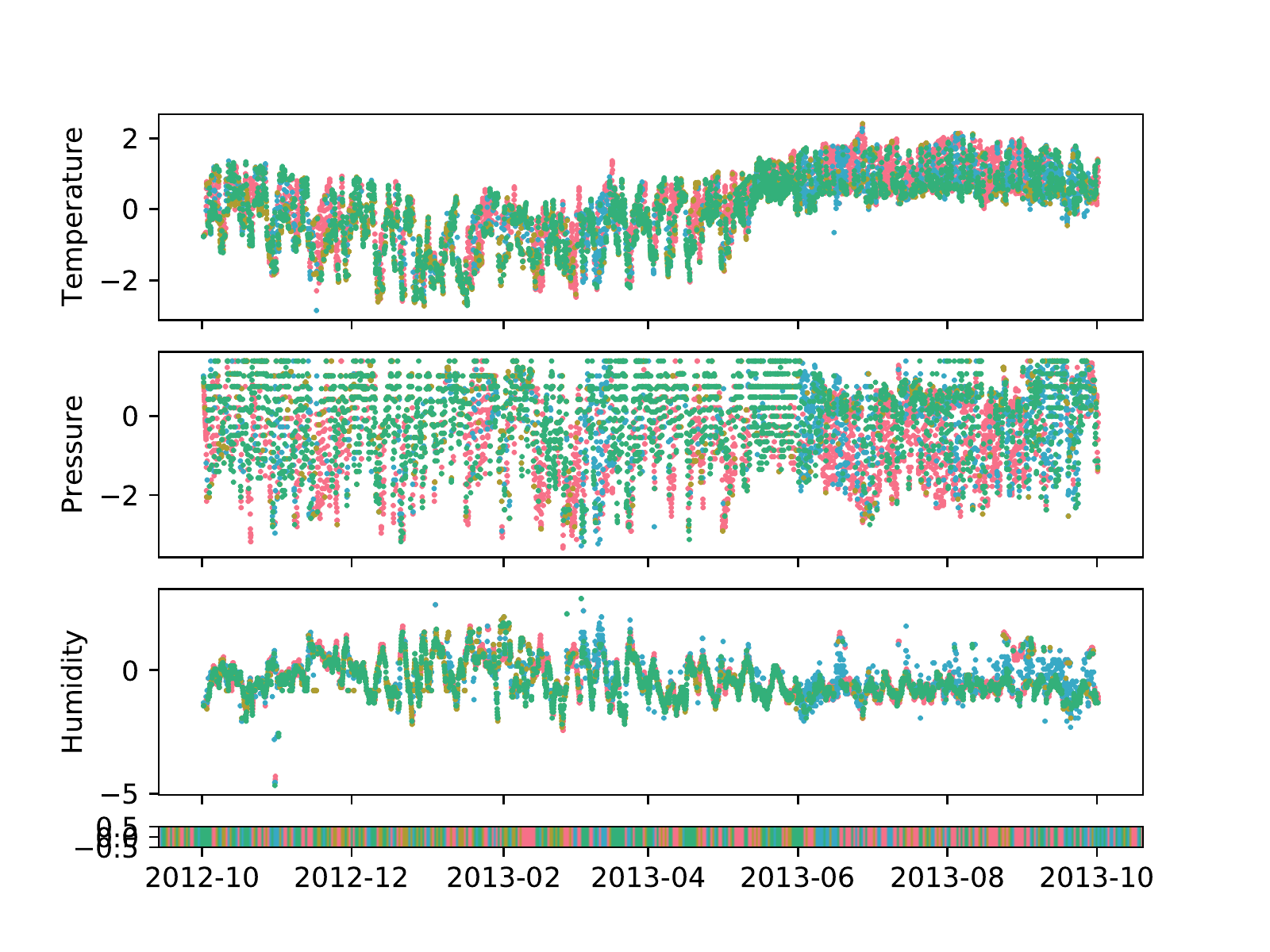}
    \end{minipage}
    \begin{minipage}[t]{.34\textwidth}
    \centering
        \includegraphics[width=\textwidth]{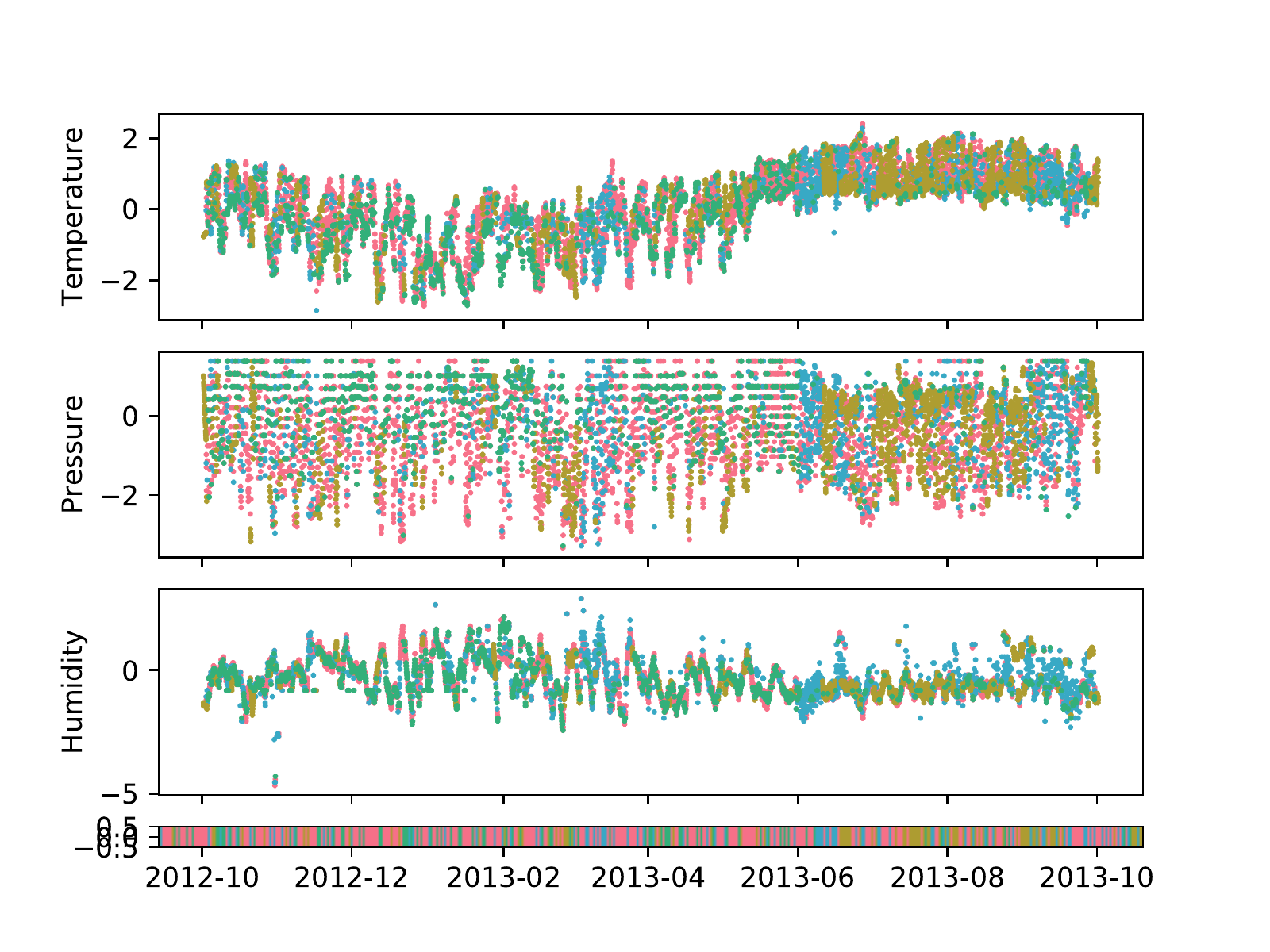}
    \end{minipage}
    \caption{SLDS Weather Data. (Top) daily aggregated data, (bottom) hourly data.
   (Left) heldout loglikelihood vs runtime, (center) Gibbs Houston fit, (right) SGRLD Houston fit.}
    \label{fig:slds-weather}
\end{figure}
Figure~\ref{fig:slds-weather} (top-left) shows the heldout loglikelihood vs the runtime for the different samplers on the daily data.
From this plot, we see that SGRLD clearly outperforms Gibbs.
Although Gibbs converges quickly on the daily data, it gets stuck in local optima.
In particular, the Gibbs runs converge to a suboptimal parametrization that mixes over three states, while SGRLD converges to a two state (summer-winter) solution (with the remaining states for sudden shifts or jumps).
For example, Figure~\ref{fig:slds-weather} (top-center and right) are fits of the daily model to the Houston time series for both Gibbs and SGRLD respectively.
Figure~\ref{fig:slds-weather} (bottom-left) shows the heldout loglikelihood vs the runtime of the different samplers for the hourly data.
SGRLD again outperforms Gibbs and, for the hourly data,
the Gibbs sampler is significantly slower than the SGMCMC samplers.
\section{Conclusions}
\label{sec:conclusions}
In this work,
we developed stochastic gradient MCMC samplers for state space models of sequential data.
Our key contribution is a \emph{buffered} gradient estimator $\tilde{g}(\theta)$ for general discrete-time SSMs based on Fisher's identity.
We developed bounds for the error of this buffered gradient estimator and showed that the error decays geometrically in the buffer size under mild conditions.
Using this estimator and bound, we developed SGRLD samplers for discrete (Gaussian HMM, ARHMM), continuous (LGSSM), and mixed-type (SLDS) state space models.
In our experiments, we find that our methods can provide orders of magnitude run-time speed ups compared to Gibbs sampling,
control bias with modest buffer size, and converge and mix more rapidly using preconditioning.
In particular, our SGRLD method only uses subsequences at each iteration and
is able to take advantage of geometric structure using the complete-data Fisher information matrix.

There are many interesting directions for future work.
This buffered gradient estimator for sequential data could be applied to other stochastic gradient methods such as
maximum likelihood estimation or variational inference~\cite{archer2015black, krishnan2017structured}.
The approach could also be extended to non-linear continuous SSMs (e.g. stochastic volatility models) replacing message passing with particle filtering~\cite{andrieu2010particle, cappe2005inference, doucet2009tutorial, olsson2008sequential}.
The buffered gradient estimator could likewise be applied to diffusions with control variates~\cite{baker2017control,chatterji2018theory} or with augmented dynamics, such as using momentum (SGHMC)~\cite{chen2014stochastic} or temperature (SGNHT)~\cite{ding2014bayesian}.
In terms of analysis, the standard SGLD error analysis could be extended to analyze the optimal trade-off between buffer size and subsequence length.

\section*{Acknowledgments}
We would like to thank Drausin Wulsin, Jack Baker, Chris Nemeth and other members of the Dynamode lab at UW for their helpful discussions.
This work was supported in part by ONR Grant N00014-15-1-2380 and NSF CAREER Award IIS-1350133.
Nicholas J. Foti was supported by a Washington Research Foundation Innovation Postdoctoral Fellowship in Neuroengineering and Data Science.

\nocite{*}
\bibliographystyle{abbrvnat}
\bibliography{references.bib}

\clearpage
\appendix
\renewcommand{\theequation}{\Alph{section}.\arabic{equation}}
\begin{center}
\LARGE Supplement for SGMCMC for State Space Models
\end{center}
\vspace{1em}

This supplement is organized as follows.
In Section~\ref{supp-sec:lemma_proofs}, we provide the proofs of Lemmas for Section~\ref{sec:bounds}.
In Section~\ref{supp-sec:models}, we provide additional details for how to calculate the forward backward messages, gradients, and preconditioning terms for the models in Section~\ref{sec:models}.
In particular, in ~\ref{supp-sec:LGSSM:lemma_proofs}, we provide the proofs of the error bound lemmas from Section~\ref{sec:models:LGSSM-bounds}.
Finally, in Section~\ref{supp-sec:experiments}, we provide additional details and figures of experiments.

\section{Proof of Lemmas in Section~\ref{sec:bounds}}
\label{supp-sec:lemma_proofs}
We now provide proofs to the Lemmas in section~\ref{sec:bounds}.

We first present a proof of Lemma~\ref{lemma:grad_error_to_wasserstein} that relates the error in the difference of expectations in Eq.~\eqref{eq:diff_gradient_estimates} to Wasserstein distance.
\begin{proof}[Proof of Lemma~\ref{lemma:grad_error_to_wasserstein}]
Let $g_t(\latent_{t-1:t}) = \grad \log p(y_t, \latent_t \, | \, \latent_{t-1}, \theta)$.

Recall $\| g_t(\latent_{t-1:t}) \|_{Lip} \leq L_U$ for all $t$ by assumption.
Then, by the Kantorovich-Rubinstein duality formula Eq.~\eqref{eq:kantorovich-rubinstein}, we have
\begin{equation}
\left\| \E_{\gamma_{t-1:t}}\left[g_t(\latent_{t-1:t})\right] -  \E_{\widetilde\gamma_{t-1:t}}\left[g_t(\latent_{t-1:t})\right] \right\|_2  \leq L_U \cdot \mathcal{W}_1(\gamma_{t-1:t}, \widetilde\gamma_{t-1:t}) \enspace.
\end{equation}
Therefore,
\begin{align}
\| \bar{g}(\theta) - \tilde{g}(\theta) \|_2 &\leq 
\left\| \frac{T}{S} \sum_{t \in \SUBSEQ} \E_{\gamma_{t-1:t}}\left[g_t(\latent_{t-1:t})\right] -  \E_{\widetilde\gamma_{t-1:t}}\left[\grad U_t(\latent_{t-1:t})\right] \right\|_2 \\
&\leq \frac{T}{S} \sum_{t \in \SUBSEQ} \left\| \E_{\gamma_{t-1:t}}\left[g_t(\latent_{t-1:t})\right] -  \E_{\widetilde\gamma_{t-1:t}}\left[\grad U_t(\latent_{t-1:t})\right] \right\|_2 \\
 &\leq\frac{T}{S} \cdot L_U \cdot \sum_{t \in \SUBSEQ} \mathcal{W}_1(\gamma_{t-1:t}, \widetilde\gamma_{t-1:t}) \enspace.
\end{align}
\end{proof}

We now present the proof of Lemma~\ref{lemma:second_wass_moment} that relates the $1$-Wasserstein distance between distributons $(\gamma', \widetilde\gamma')$ of $\latent\latent^T$ to the $2$-Wasserstein distance between $(\gamma,\widetilde\gamma)$ over $\latent$.

\begin{proof}[Proof of Lemma~\ref{lemma:second_wass_moment}]
Let $\xi$ be a joint distribution over $\latent$ and $\widetilde\latent$ with marginals $\gamma$ and $\widetilde\gamma$.
Let $w := \widetilde\latent - \latent$, which implies $\widetilde\latent = \latent + w$.

Then we have
\begin{align}
\E \| \widetilde\latent\widetilde\latent^T - \latent\latent^T \|_F &= \E \| \latent w^T + w \latent^T + w w^T \|_F \\
    &\leq \E \| \latent w^T \|_F + \E \|w \latent^T\|_F + \E\|w w^T \|_F \\
    &= 2\E | \latent^T w | + \E[\|w\|^2] \\
    &\leq 2\sqrt{\E[\|\latent\|^2] \E[\|w\|^2]} + \E[\|w\|^2] \\
    &\leq (2\sqrt{M} + 1) \max \, \{\E[\|w\|^2]^{1/2}, \E[\|w\|^2] \} \\
    &= (2\sqrt{M} + 1) \max \, \{\E[\|\widetilde\latent - \latent\|^2]^{1/2}, \E[\|\widetilde\latent - \latent\|^2] \}
\end{align}
where we observe $\|xx^T\|_F = \| x x^T\|_2 = \|x\|_2^2 = x^T x$ and we use Cauchy-Schwartz.

Taking the infimum over all $\xi$ gives the result
\begin{align}
\mathcal{W}_1(\gamma', \widetilde\gamma') &= \inf_\xi  \E\| \widetilde\latent\widetilde\latent^T - \latent\latent^T \|_F \\
    &\leq \inf_\xi \left[(2\sqrt{M} + 1) \max \, \{\E[\|\widetilde\latent - \latent\|^2]^{1/2}, \, \E[\|\widetilde\latent - \latent\|^2] \} \right] \\
    &= (2\sqrt{M} + 1) \max \, \{ \inf_\xi \E[\|\widetilde\latent - \latent\|^2]^{1/2}, \, \inf_\xi \E[\|\widetilde\latent - \latent\|^2] \} \\
    &= (2\sqrt{M} + 1) \cdot \max_{r \in {1, 1/2}} \mathcal{W}_2(\gamma, \widetilde\gamma)^{r}
\enspace.
\end{align}
\end{proof}

We now prove Lemma~\ref{lemma:geometric_wasserstein} that bounds $\mathcal{W}_p(\gamma_{t-1:t}, \widetilde\gamma_{t-1:t})$ in terms of buffer size, if the forward and backward random maps $f_t, b_t$ are Lipschitz.

\begin{proof}[Proof of Lemma~\ref{lemma:geometric_wasserstein}]
We will first prove Eq.~\eqref{eq:geometric_wasserstein_forward}.
Recall $f_t$ is Lipschitz with constant $L_f < 1$ for all $t \in \BUFFSUBSEQ$.

Let $\xi_{t:t+1}$ be a joint distribution over $\latent_{t:t+1}$ and $\widehat\latent_{t:t+1}$ with marginals $\gamma_{t:t+1}$ and $\widehat\gamma_{t:t+1}$.
Let $\xi_{t}$ be a joint distribution over $\latent_{t}$ and $\widehat\latent_{t}$ with marginals $\gamma_{t}$ and $\widehat\gamma_{t}$.
Then for all $t \in \SUBSEQ$, we have
\begin{align}
\mathcal{W}_p(\gamma_{t:t+1}, \widehat\gamma_{t:t+1})^p &\leq \inf_{\xi_{t:t+1}} \int \| \latent_{t} - \hat\latent_{t}\|_2^p + \| \latent_{t+1} - \hat\latent_{t+1}\|_2^p  \, d\xi_{t:t+1}(\latent_{t:t+1}, \hat\latent_{t:t+1}) \\
&\leq \inf_{\xi_t} \int \| \latent_{t} - \hat\latent_{t}\|_2^p + \| f_t(\latent_t) - f_t(\hat\latent_t) \|_2^p \, d\xi_t(\latent_t, \hat\latent_t) df_t \\
&\leq \inf_{\xi_t} \int \| \latent_{t} - \hat\latent_{t}\|_2^p + L_f^p \cdot \| \latent_t -\hat\latent_t \|_2^p \, d\xi_t(\latent_t, \hat\latent_t) \\
&\leq (1 + L_f^p) \cdot \mathcal{W}_p(\gamma_{t}, \widehat\gamma_{t})^p
\end{align}

Repeatedly applying Eq.~\eqref{eq:wasserstein_random_map} completes the proof for Eq.~\eqref{eq:geometric_wasserstein_forward}
\begin{align}
\mathcal{W}_p(\gamma_{t-1:t}, \widehat\gamma_{t-1:t}) &\leq (1 + L_f^p)^{1/p} \cdot \mathcal{W}_p(\gamma_{t-1}, \widehat\gamma_{t-1}) \\
    &\leq (1 + L_f^p)^{1/p} \cdot L_f \cdot \mathcal{W}_p(\gamma_{t-2}, \widehat\gamma_{t-2}) \\
    &\leq (1 + L_f^p)^{1/p} \cdot L_f^2 \cdot \mathcal{W}_p(\gamma_{t-3}, \widehat\gamma_{t-3}) \\
    &\leq \ldots \\
    &\leq (1 + L_f^p)^{1/p} \cdot L_f^{B+t-1} \cdot \mathcal{W}_p(\gamma_{-B}, \widehat\gamma_{-B})
\enspace.
\end{align}

The proof of Eq.~\eqref{eq:geometric_wasserstein_backward} is identical.
\end{proof}
\section{Additional Model Details}
\label{supp-sec:models}
\subsection{Gaussian HMM}
\label{supp-sec:hmm}
See Sections~\ref{sec:models:HMM} 
for notation.

\subsubsection{Forward Backward}
The forward and backward recursions (Eqs.~\eqref{eq:forward_message} and \eqref{eq:backward_message}) for an HMM are
\begin{align}
\label{eq:hmm_forward_message}
\alpha_t & := p(z_t, y_{\leq t}) = \alpha_{t-1} \cdot \Pi \cdot P_t \\
\label{eq:hmm_backward_message}
\beta_t & := p(y_{>t} \, | \, z_t) = \Pi \cdot P_{t+1} \beta_{t+1} \enspace,
\end{align}
where $\alpha_{-T} = \mathbf{1}/K$, $\beta_T = \mathbf{1}$,
and
\begin{equation}
P_t := \diag\{\mathcal{N}(y_t \, | \, \mu_k, \Sigma_k)\}_{k=1}^K
\enspace.
\end{equation}
Given the messages $\alpha_t, \beta_t$,
the marginal and pairwise posteriors of the latent states are computed as
\begin{align}
\gamma_t(z_t) & := p(z_t \, | \, y) \propto \alpha_t \odot \beta_t \\
\gamma_{t:t-1}(z_{t-1}, z_t) & := p(z_{t-1}, z_t \, | \, y) \propto
\diag(\alpha_{t-1}) \cdot \Pi \cdot P_t \cdot \diag(\beta_t) \enspace.
\label{eq:hmm_pairwise_marginal}
\end{align}

\subsubsection{Gradient Estimator}
As stated in Sec.~\ref{sec:models:HMM},
we use the `expanded mean' parameters of $\Pi$ instead of $\Pi$ (as in~\cite{patterson2013stochastic})
and the Cholesky decomposition of $\Sigma^{-1}_k$ instead of $\Sigma_k$ to ensure positive definiteness.
The expanded mean parametrization is $\phi \in \R^{K \times K}_+$ where $\Pi_{k,\cdot} = \phi_{k,\cdot} / \sum_{k'} \phi_{k, k'}$.
The Cholesky decomposition of the precision $\Sigma^{-1}_k$ is $\psi_{\Sigma_k}$ such that $\psi_{\Sigma_k} \psi_{\Sigma_k}^T = \Sigma_k^{-1}$.

The gradient of the marginal loglikelihood takes the form
\begin{align}
\label{eq:gaussian_hmm_marginal_loglike_start}
\grad_{\phi_k} \log p(y \, | \, \theta) &= \sum_{t \in \mathcal{T}} \E_{z_t, z_{t-1} | y}[\mathbb{I}(z_{t-1}=k) \cdot \phi_k^{-1} \odot (\vec{e}_{z_t} - \Pi_{k})]
\\
\grad_\mu \log p(y | \theta) &=
\sum_{t=1}^T \E_{z_t | y} \left[ \Sigma_{z_t}^{-1}(y_t - \mu_{z_t}) \right]
\\
\grad_{\psi_\Sigma} \log p(y \, | \, \theta)&= \sum_{t \in \mathcal{T}} \E_{z_t | y}\left[
\left(\Sigma_{z_t} - (y_t - \mu_{z_t})(y_t - \mu_{z_t})^T\right) {\psi_\Sigma}_{z_t}
\right] \enspace.
\label{eq:gaussian_hmm_marginal_loglike_end}
\end{align}
As $z$ is discrete and these expectations only involve pairwise elements of $z$, they can be tractably computed as weighted average using $\gamma(z_{t}, z_{t-1})$ from forward backward.

\subsubsection{Preconditioning}
\label{supp-sec:hmm_precond}
For the Gaussian HMM, the complete-data Fisher information matrix is block diagonal.
With some algebra,
the Fisher information matrix, precondition matrices, and correction term are
\begin{align}
\label{eq:gaussian_hmm_precondition_start}
\mathcal{I}_{\phi_k} &= (\diag(\Pi_k) - 11^T) \cdot (1^T\phi_k)^{-2} \, \Rightarrow \, D(\theta)_{\phi_k} = \diag(\phi_k) \enspace \text{and} \enspace \Gamma(\theta)_{\phi_k} = 1
\\
\mathcal{I}_{\mu_k} &= \Sigma_k^{-1}
\, \Rightarrow \, D(\theta)_{\mu_k} = \Sigma_k \enspace \text{and} \enspace \Gamma(\theta)_{\mu_k} = 0
\\
\mathcal{I}_{\psi_{\Sigma_k}} &= 2 (I_m \otimes \Sigma_k) \, \Rightarrow \, D(\theta)_{\psi_{\Sigma_k}} = \frac{1}{2} (I_m \otimes \Sigma_k^{-1}) \enspace \text{and} \enspace \Gamma(\theta)_{\psi_{\Sigma_k}} = \psi_{\Sigma_k}
\label{eq:gaussian_hmm_precondition_end}
\end{align}
For $\phi_k$, we use $D(\theta)_{\phi_k} = \diag(\phi_k)$ and $\Gamma(\theta)_{\phi_k} = 1$, following past work~\cite{patterson2013stochastic, ma2017stochastic}.
However, we observed that $\phi_k$ will be absorbed at $0$, whenever $\phi_k$ approaches to closely to $0$.
To fix this we recommend adding a small identity matrix $\nu_\phi I_K$ (for some $\nu_\phi > 0$) to $D(\theta)_\phi$.
An alternative solution is to use a stochastic Cox-Ingersoll-Ross process to sample $\pi$ instead~\cite{baker2018large}.

\subsection{ARHMM}
\label{supp-sec:arhmm}
See Section~\ref{sec:models:ARHMM} for notation.
\subsubsection{Forward Backward}
The forward backward recursions for the ARHMM are identical to the Gaussian HMM Eqs.~\eqref{eq:hmm_forward_message}-\eqref{eq:hmm_pairwise_marginal}, where $P_t$ is now
\begin{equation}
P_t := \diag\{\mathcal{N}(y_t \, | \, A_k \overline{y_t}, Q_k)\}_{k=1}^K
\enspace.
\end{equation}

\subsubsection{Gradient Estimator}
The gradient of the marginal loglikelihood is similar to the Gaussian HMM Eqs.~\eqref{eq:gaussian_hmm_marginal_loglike_start}-\eqref{eq:gaussian_hmm_marginal_loglike_end}
with $\mu_k$ replaced with $A_k \overline{y}_t$
\begin{align}
\label{eq:arhmm_marginal_loglike_start}
\grad_{\phi_k} \log p(y \, | \, \theta) &= \sum_{t \in \mathcal{T}} \E_{z_t, z_{t-1} | y}[\mathbb{I}(z_{t-1}=k) \cdot \phi_k^{-1} \odot (\vec{e}_{z_t} - \Pi_{k})]
\\
\grad_A \log p(y | \theta) &=
\sum_{t=1}^T \E_{z_t | y} \left[ Q_{z_t}^{-1}(y_t - A_{z_t} \overline{y}_t) \overline{y}_t^T \right]
\\
\grad_{\psi_Q} \log p(y \, | \, \theta)&= \sum_{t \in \mathcal{T}} \E_{z_t | y}\left[
\left(Q_{z_t} - (y_t - A_{z_t} \overline{y}_t)(y_t - A_{z_t} \overline{y}_t)^T\right) {\psi_Q}_{z_t}
\right] \enspace.
\label{eq:arhmm_marginal_loglike_end}
\end{align}

\subsubsection{Preconditioning}
The preconditioning terms for the ARHMM is similar to the Gaussian HMM
\begin{align}
\label{eq:arhmm_precondition_start}
\mathcal{I}_{\phi_k} &= (\diag(\Pi_k) - 11^T) \cdot (1^T\phi_k)^{-2} \, \Rightarrow \, D(\theta)_{\phi_k} = \diag(\phi_k) \enspace \text{and} \enspace \Gamma(\theta)_{\phi_k} = 1
\\
\mathcal{I}_{A_k} &= \E_{y, z |\theta}[\overline{y}_t \overline{y}_t^{T}] \otimes Q_k^{-1}
\, \Rightarrow \, D(\theta)_{A_k} = I_m \otimes Q_k \enspace \text{and} \enspace \Gamma(\theta)_{A_k} = 0
\\
\mathcal{I}_{\psi_{Q_k}} &= 2 (I_m \otimes Q_k) \, \Rightarrow \, D(\theta)_{\psi_{Q_k}} = \frac{1}{2} (I_m \otimes Q_k^{-1}) \enspace \text{and} \enspace \Gamma(\theta)_{\psi_{Q_k}} = \psi_{Q_k}
\label{eq:arhmm_precondition_end}
\end{align}
The expectation $\E[\overline{y}_t \overline{y}_t^T]$ does not have a closed form as the expectation is over $z$ is a combinatorial sum.
Therefore, we choose to replace $\E[\overline{y}_t \overline{y}_t^T]$ with the identity matrix $I_m$ in our preconditioning matrix $D(\theta)_A$.

\subsection{LGSSM}
\label{supp-sec:LGSSM}
See Section~\ref{sec:models:LGSSM} for notation.

\subsubsection{Forward Backward}
The recursions for the forward backward algorithm for LGSSMs is known as the Kalman smoother~\cite{cappe2005inference, bishop2006pattern, fox2009bayesian}.
Because the transition and emission processes are linear Gaussian, all forward messages, backward messages, and pairwise latent marginals $\gamma(x_t, x_{t-1})$ are Gaussian.

\begin{align}
\label{eq:lds_kalman_forward}
\alpha_t &:= p(x_t, y_{\leq t}) = \mathcal{N}(x_t \, | \, \mu_{\alpha_t} = \Lambda_{\alpha_t}^{-1} h_{\alpha_t}, \Sigma_{\alpha_t} = \Lambda_{\alpha_t}^{-1}) \\
\label{eq:lds_kalman_backward}
\beta_t &:= p(y_{>t} \, | \, x_t) \propto \mathcal{N}(x_t \, | \, \mu_{\beta_t} = \Lambda_{\beta_t}^{-1} h_{\beta_t}, \Sigma_{\beta_t} = \Lambda_{\beta_t}^{-1}) \enspace,
\end{align}
where $h_{\alpha_t}, \Lambda_{\alpha_t}$ are the Gaussian natural parameters of $\alpha$ that satisfy the recursion 
\begin{align}
\label{eq:lds_forward_lambda}
\Lambda_{\alpha_t} &= C^T R^{-1} C + (Q + A\Lambda_{\alpha_{t-1}}^{-1} A^T)^{-1} \\
\label{eq:lds_forward_h}
h_{\alpha_t} &= C^T R^{-1} y_{t} + (Q + A\Lambda_{\alpha_{t-1}}^{-1}A^T)^{-1} A \Lambda_{\alpha_{t-1}}^{-1} h_{\alpha_{t-1}} \enspace,
\end{align}
and $h_{\beta_t}, \Lambda_{\beta_t}$ are the Gaussian natural parameters of $\beta$ that satisfy the recursion
\begin{align}
\label{eq:lds_backward_lambda}
\Lambda_{\beta_t} &= A^TQ^{-1} A - A^T Q^{-1}(Q^{1} + C^T R^{-1} C + \Lambda_{\beta_{t+1}})^{-1} Q^{-1} A \\
\label{eq:lds_backward_h}
h_{\beta_t} &= A^T Q^{-1}(Q^{-1} + C^T R^{-1} C + \Lambda_{\beta_{t+1}})^{-1}(C^T R^{-1}y_{t+1} + h_{\beta_{t+1}}) \enspace.
\end{align}

Given the messages $\alpha_t, \beta_t$ the marginal and pairwise posteriors of the latent states $x_t$ and $(x_{t-1}, x_t)$ are computed as
\begin{align}
    \gamma_t(x_t) := p(x_t \,|\, y) 
        &\propto \alpha_t(x_t) \beta_t(x_t) \\
        \nonumber
        &\propto \mathcal{N}(x_t \, | \, \mu = \Sigma (h_{\alpha_t} + h_{\beta_t}), \Sigma = (\Lambda_{\alpha_t} + \Lambda_{\beta_t})^{-1}) 
\end{align}
\begin{align}
\label{eq:lds_pairwise_marginal}
\gamma_{t-1, t}(x_{t-1}, x_t) := &p(x_{t-1}, x_t \,|\, y) 
    \propto \alpha_{t-1}(x_{t-1} p(y_t, x_t \, | \, x_{t-1}) \beta_{t}(x_t) \\
    \nonumber
    \propto\mathcal{N}\Big(
\begin{bmatrix} x_{t-1}  \\ x_t \end{bmatrix} \, \Big| \,
&\mu = \Sigma \cdot \begin{bmatrix} h_{\alpha_{t-1}}  \\ C^T R_{-1} y_t + h_{\beta_t}\end{bmatrix}, \\
\nonumber
&\Sigma = \begin{bmatrix} \Lambda_{\alpha_{t-1}} + A^T Q^{-1} A &  A^T Q^{-1} \\
               Q^{-1} A  &  C^TR^{-1} C + Q^{-1} + \Lambda_{\beta_t} 
\end{bmatrix}^{-1} \Big) \enspace.
\end{align}

\subsubsection{Gradient Estimator}
We compute the gradient of marginal loglikelihood via Fisher's identity
\begin{align}
\label{eq:lds_marginal_loglike_start}
\grad_A \log p(y | \theta) &=
\sum_{t=1}^T \E_{x | y} \left[ Q^{-1}(x_t - A x_{t-1}) x_{t-1}^T \right]
\\
\grad_{\psi_Q} \log p(y | \theta) &=
\sum_{t=1}^T \E_{x | y} \left[(Q - (x_t - A x_{t-1})(x_t - A x_{t-1})^T) \psi_Q \right]
\\
\grad_C \log p(y | \theta) &=
\sum_{t=1}^T \E_{x | y} \left[ R^{-1}(y_t - C x_t) x_t^T \right]
\\
\grad_{\psi_R} \log p(y | \theta) &=
\sum_{t=1}^T \E_{x | y} \left[(R - (y_t - C x_t)(y_t - C x_t)^T) \psi_R \right]
\label{eq:lds_marginal_loglike_end}
\end{align}
Because each gradient is linear with respect to first and second order terms (e.g. $x_t$, $x_t x_t^T$ and $x_t x_{t-1}^T$),
their expectation of each of these terms is easily computable given $\gamma(x_t, x_{t-1})$.

Let $\gamma_{t,t-1}(x_t, x_{t-1})$ be the Gaussian pairwise marginal posterior from forward backward (see Eq.~\eqref{eq:lds_pairwise_marginal})
\begin{equation}
\gamma_{t-1, t}(x_{t-1}, x_t) = \mathcal{N}\left(
\begin{bmatrix} x_{t-1}  \\ x_t \end{bmatrix} \, \big| \,
\mu = \begin{bmatrix} \mu_{t-1}  \\ \mu_t \end{bmatrix},
\Sigma = \begin{bmatrix} \Sigma_{t-1,t-1} &  \Sigma_{t-1,t} \\
                \Sigma_{t,t-1} &  \Sigma_{t,t}
\end{bmatrix} \right) \enspace.
\end{equation}
Let $M = \Sigma + \mu\mu^T$ be the second moment of $\gamma_{t-1,t}$,
that is $M_{t,t'} := \E[x_t, x_{t'}^T]$.

Then the expectations in the summations of Eqs.~\eqref{eq:lds_marginal_loglike_start}-\eqref{eq:lds_marginal_loglike_end} are
\begin{align}
\E_{x|y} \left[ (x_t - A x_{t-1}) x_{t-1}^T \right] &= M_{t,t-1} - A M_{t-1,t-1}  \\
\E_{x|y} \left[ (x_t - A x_{t-1})(x_t - A x_{t-1})^T \right] &= M_{t,t} - A M_{t-1,t} - M_{t, t-1}A^T + A M_{t-1, t-1} A^T \\
\E_{x|y} \left[ (y_t - C x_t) x_t^T \right] &= y_t \mu_t^T - C M_{t,t} \\
\E_{x | y} \left[ (y_t - C x_t)(y_t - C x_t)^T \right] &= y_t y_t^T - C \mu_t y_t^T - y_t \mu_t^T C^T + C M_{t,t} C^T \enspace.
\end{align}

\subsubsection{Preconditioning}
For the LGSSM, the complete data Fisher information matrix is block diagonal.
With some algebra,
the Fisher information matrix, precondition matrices, and correction term are
\begin{align}
\label{eq:lds_precondition_start}
\mathcal{I}_A = \E[x_t x_t]^T  \otimes Q^{-1} \enspace &\Rightarrow \enspace D_A = I_n \otimes Q \enspace \text{and} \enspace \Gamma(\theta)_{A} = 0
\\
\mathcal{I}_{\psi_Q} = 2 (I_n \otimes Q)  \enspace &\Rightarrow \enspace D_{\psi_Q} = \frac{1}{2} (I_n \otimes Q^{-1}) \enspace \text{and} \enspace \Gamma(\theta)_{\psi_Q} = \psi_Q
\\
\mathcal{I}_C = \E[x_t x_t]^T \otimes R^{-1} \enspace &\Rightarrow \enspace D_C = I_n \otimes R \enspace \text{and} \enspace \Gamma(\theta)_{C} = 0
\\
\mathcal{I}_{\psi_R} = 2 (I_m \otimes R)  \enspace &\Rightarrow \enspace D_{\psi_R} = \frac{1}{2} (I_m \otimes R^{-1}) \enspace \text{and} \enspace \Gamma(\theta)_{\psi_R} = \psi_R
\enspace,
\label{eq:lds_precondition_end}
\end{align}
where $\E[x_t x_t]^T = \sum_{s = 0}^\infty A^s Q (A^s)^T$ for the LGSSM,
In our experiments we chose to replace $\E[x_t x_t]^T$ with the identity matrix $I_n$ to match the ARHMM setup.

\subsubsection{Proof of Lemmas in Section~\ref{sec:models:LGSSM-bounds}}
\label{supp-sec:LGSSM:lemma_proofs}
We now provide proofs to the Lemmas in section~\ref{sec:models:LGSSM-bounds}.
Note that these bound hold pointwise for $\theta$ for all random maps conditioned on any observed sequence $Y_{1:T}$.

We first present a proof of Lemma~\ref{lemma:LGSSM_forward_maps} that shows the forward random maps $f_t$ are contractions if $\| A \| < 1$.

\begin{proof}[Proof of Lemma~\ref{lemma:LGSSM_forward_maps}]
For an LGSSM, the forward smoothing kernel $\mathcal{F}_t$ takes the form
\begin{equation}
    \mathcal{F}_t(x_{t+1}, x_{t} | y) \propto \underbrace{\beta_{t+1}(x_{t+1})p(y_{t+1}\, | \, x_{t+1})}_{p(y_{>t}\,|\, x_{t+1})} p(x_{t+1} \, | \, x_{t})
\end{equation}
where $\beta_{t+1}$ is the backward message at time $t+1$ given by Eq.~\eqref{eq:lds_kalman_backward}.
The recursive formula for $\beta_t$ can be extended to
$p(y_{>t} \, | \, x_{t+1}) \propto \mathcal{N}(x_{t+1} \, | \, \Lambda_{t+1}^{-1}h_{t+1}, \Lambda_{t+1}^{-1})$ with
\begin{align}
\label{eq:lds_backward_extended_lambda}
\Lambda_t &= C^T R^{-1} C + A^T (Q + \Lambda_{t+1}^{-1})^{-1} A \\
h_{t} &= C^T R^{-1} y_{t} + A^T (Q + \Lambda_{t+1}^{-1})^{-1} \Lambda_{t+1} h_{t+1} \enspace.
\end{align}

With this parametrization, the forward smoothing kernel takes the form
\begin{equation}
\mathcal{F}_t(x_{t+1} | x_{t} ) = \mathcal{N}(x_{t+1} \, | \, (Q^{-1} + \Lambda_{t+1})^{-1} (Q^{-1}A x_t + h_{t+1}), (Q^{-1} + \Lambda_{t+1})^{-1}) \enspace.
\end{equation}
Therefore our random map $f_t$ is
\begin{equation}
f_t(x_t) = \underbrace{(Q^{-1} + \Lambda_{t+1})^{-1} Q^{-1}A}_{F^f_t} x_t + \underbrace{(Q^{-1} + \Lambda_{t+1})^{-1} h_{t+1} +  (Q^{-1} + \Lambda_{t+1})^{-1/2} \nu_t}_{\zeta^f_t} \enspace,
\end{equation}
where $\nu_t \sim \mathcal{N}(0, I)$ makes $f_t$ a random map.

The Lipschitz constant for $f_t$ with respect to $x_t$ is $\|F_t^f\| = \| (Q^{-1} + \Lambda_{t+1})^{-1} Q^{-1} A \|$.
From Eq.~\eqref{eq:lds_backward_extended_lambda}, $\Lambda_\theta = C^T R^{-1} C$ is a lower bound on $\Lambda_{t}$ and is tight when $\Lambda_{t+1} = 0$ (at the very beginning of the recursion).

Therefore we have a uniform bound on the Lipschitz constants of $f_t$ :
\begin{equation}
\|F_t^f\| = \| (Q^{-1} + \Lambda_{t+1})^{-1} Q^{-1} A \| < \| (Q^{-1} + \Lambda_\theta)^{-1} Q^{-1} A \| = L_f \enspace.
\end{equation}
\end{proof}

We now present a proof of Lemma~\ref{lemma:LGSSM_backward_maps} that similarly shows the backward random maps $b_t$ are contractions.
We first prove the bound for general prior $p_0(x)$ and then present the special case when the prior variance is less than the steady state variance $V_\infty = (Q + A V_\infty A^T) = \sum_{k=0}^\infty A^k Q (A^T)^k$ and $A$ and $Q$ commute.

\begin{proof}[Proof of Lemma~\ref{lemma:LGSSM_backward_maps}]
The backward smoothing kernel $\mathcal{B}_t$ takes the form
\begin{equation*}
    \mathcal{B}_t(x_{t-1}, x_{t} | y) \propto p(x_t \, | \, x_{t-1}) \alpha_{t-1}(x_{t-1})
\end{equation*}
where $\alpha_{t-1}$ is the forward message at time $t-1$.
Recall from Eq.~\eqref{eq:lds_kalman_forward} the forward messages are
$\alpha_t(x_{t}) \propto \mathcal{N}(x_{t} \, | \, \Lambda_{\alpha_t}^{-1}h_{\alpha_t}, \Lambda_{\alpha_t}^{-1})$.
With this parametrization, the backward smoothing kernel takes the form
\begin{equation*}
\mathcal{B}_t(x_{t-1}, x_{t} | y) = \mathcal{N}(x_{t-1} \, | \, (A^T Q^{-1}A + \Lambda_{\alpha_{t-1}})^{-1} (A^T Q^{-1} x_t + h_{\alpha_{t-1}}), (A^T Q^{-1} A + \Lambda_{\alpha_{t-1}})^{-1}).
\end{equation*}
Our backward random map $b_t$ is thus
\begin{equation*}
b_t(x_t) = (A^T Q^{-1} A + \Lambda_{\alpha_{t-1}})^{-1}(A^T Q^{-1} x_t + h_{\alpha_{t-1}}) +  (A^T Q^{-1}A + \Lambda_{\alpha_{t-1}})^{-1/2} \nu_t \enspace,
\end{equation*}
where $\nu_t \sim \mathcal{N}(0, I)$ with 
\begin{align}
F_t^b &= (A^T Q^{-1} A + \Lambda_{\alpha_{t-1}})^{-1}A^T Q^{-1} \\  
\zeta^b_t &= (A^T Q^{-1} A + \Lambda_{\alpha_{t-1}})^{-1}h_{\alpha_{t-1}} + (A^T Q^{-1}A + \Lambda_{\alpha_{t-1}})^{-1/2} \nu_t \enspace.
\end{align}

The Lipschitz constant for $b_t$ with respect to $x_t$ is
\begin{equation*}
\|F_t^b\| = \| (A^T Q^{-1}A + \Lambda_{\alpha_{t-1}})^{-1} A^TQ^{-1} \| \enspace.
\end{equation*}
From Eq.~\eqref{eq:lds_forward_lambda},
$\Lambda_\theta = C^T R^{-1} C$ is a lower bound on $\Lambda_{\alpha_t}$ and is tight when $\Lambda_{\alpha_{t-1}} = 0$ (at the very beginning of the recursion).
Therefore we have a uniform bound on the Lipschitz constants of $b_t$ for Lemma~\ref{lemma:geometric_wasserstein}:
\begin{equation}
\label{eq:lds_backward_lipbound}
\|F_t^b\| = \| (A^T Q^{-1}A + \Lambda_{\alpha_{t-1}})^{-1} A^TQ^{-1} \| \leq \| (A^T Q^{-1}A + \Lambda_{\theta})^{-1} A^TQ^{-1} \| = L_b \enspace,
\end{equation}
where 
\begin{equation*}
L_b =\|(A^T Q^{-1} A + C^T R^{-1} C\|^{-1}_2 \| A \|_2 \| Q \|_2 = \|A (Q A^T Q^{-1} A + Q C^T R^{-1} C)^{-1} \|_2 \enspace.
\end{equation*}

If the prior variance is less than the steady state variance $V_\infty$,
then $\Lambda_\theta = C^T R^{-1} C + V_\infty^{-1}$ is a larger lower bound on $\Lambda_{\alpha_t}$ as by induction $\Lambda_{\alpha_1} = C^T R^{-1} C + V_\infty^{-1} = \Lambda_\theta$ and from Eq.~\eqref{eq:lds_forward_lambda}
\begin{equation*}
\Lambda_{\alpha_t}^{-1} \leq V_\infty  \, \Rightarrow \, \Lambda_{\alpha_{t+1}} = C^T R^{-1} C + (Q + A \Lambda^{-1}_{\alpha_t} A^T)^{-1} \geq C^T R^{-1} C + (Q + A V_\infty A^T)^{-1} = \Lambda_\theta \enspace.
\end{equation*}

If $A$ and $Q$ commute, then $V_\infty = (Q^{-1} - A^T Q^{-1} A)^{-1}$ as
\begin{equation*}
\underbrace{(Q^{-1} - A^T Q^{-1} A)^{-1}}_{V_\infty} = Q + Q A^T(Q - A Q A^T)^{-1} A Q = \underbrace{Q + A(Q^{-1} - A^T Q^{-1}A)^{-1} A^T}_{Q + A V_\infty A^T} \enspace.
\end{equation*}

Therefore plugging lower bound for $\Lambda_{\alpha_t} \leq C^T R^{-1} C + Q^{-1} - A^T Q^{-1} A$ into Eq.~\eqref{eq:lds_backward_lipbound} gives
we obtain
\begin{equation*}
L_b =\|(Q^{-1} + C^T R^{-1} C\|^{-1}_2 \| A \|_2 \| Q \|_2 = \|A (I_n + Q C^T R^{-1} C)^{-1} \|_2 \enspace.
\end{equation*}
\end{proof}

Finally, we prove Lemma~\ref{lemma:LGSSM_Lipschitz} which bounds the Lipschitz constant for the complete data loglikelihood terms.
\begin{proof}[Proof of Lemma~\ref{lemma:LGSSM_Lipschitz}]
The gradient of the complete data loglikelihood for $\theta = (A, Q, C, R)$ is
\begin{align}
\label{eq:lgssm_complete_grad_begin}
\grad_A \log p(y, x_t \, |\, x_{t-1}, \theta) &=
 Q^{-1}(x_t - A x_{t-1}) x_{t-1}^T
\\
\grad_{\psi_Q} \log p(y, x_t \, |\, x_{t-1}, \theta) &=
(Q - (x_t - A x_{t-1})(x_t - A x_{t-1})^T)\psi_Q
\\
\grad_C \log p(y, x_t \, |\, x_{t-1}, \theta) &=
R^{-1}(y_t - C x_t) x_t^T
\\
\grad_{\psi_R} \log p(y, x_t \, |\, x_{t-1}, \theta) &=
(R - (y_t - C x_t)(y_t - C y_t)^T)\psi_R
\enspace.
\label{eq:lgssm_complete_grad_end}
\end{align}
From Eqs.~\eqref{eq:lgssm_complete_grad_begin}-\eqref{eq:lgssm_complete_grad_end} it is clear that the complete data loglikelihood are quadratic form in $xx^T$ with matrices given by $\Omega$ in Lemma~\ref{lemma:LGSSM_Lipschitz}.
\end{proof}

\subsection{SLDS}
See Section~\ref{sec:models:SLDS} for notation.
As the SLDS does not have a closed form forward-backward algorithm, we instead present the details for the blocked Gibbs sampling scheme (conditional distributions and Initialization) used in Algorithm~\ref{alg:noisygradient-slds}.
\subsubsection{Blocked Gibbs Conditional Distributions}
The conditional posterior distribution of $x$ given $y$ and $z$ follows a time-varying LGSSM.
To sample $x$, we can use the time-varying Kalman filter~\cite{hamilton1994time}.
We first calculate the forward messages $\alpha_t(x_t)$ using the Kalman filter recursion Eq.~\eqref{eq:lds_kalman_forward} with $A_t = A_{z_t}$, $C_t = C$, $Q_t = Q_{z_t}$, and $R_t = R$.
Given $\alpha_t(x_t) \propto \mathcal{N}(x_t \, | \, \mu_{\alpha_t}, \Sigma_{\alpha_t})$,
we sample $x$ using the backward sampler (starting from $t=T$ and descending)
\begin{equation}
\label{eq:slds_x_conditional_gibbs}
x_t \, |\, x_{t-1} \sim \begin{cases}
\mathcal{N}\left( x_T \, | \,
    \mu = \mu_{\alpha_T} \, , \,
    \Sigma = \Sigma_{\alpha_T}\right)
    \text{ if } t = T \enspace, \text{ otherwise }\\
\mathcal{N}\left( x_t \, \Big| \,
    \mu = \Sigma (\Sigma_{\alpha_t}^{-1} \mu_{\alpha_t} + A_{z_{t+1}}^T Q^{-1}_{z_{t+1}} x_{t+1}) \, , \,
    \Sigma = (\Sigma_{\alpha_t}^{-1} + A_{z_{t+1}}^T Q^{-1}_{z_{t+1}} A_{z_{t+1}})^{-1}\right) \\
\end{cases}
\end{equation}

The conditional posterior distribution of $z$ given $y$ and $x$ follows the ARHMM.
To sample $z$, we apply a similar sampler for the ARHMM.
We first calculate the backward messages $\beta_t(z_t)$ using the ARHMM forward messages Eq.~\eqref{eq:hmm_backward_message}, replacing $y$ with $x$.
Given $\alpha_t(z_t)$, we then sample $z$ sequentially in ascending order
using the forward sampler
\begin{equation}
\label{eq:slds_z_conditional_gibbs}
p(z_t = k \, | \, z_{t-1}, x, y) \propto
p(x_t, y_t \, | \, x_{t-1}, z_t = k, \theta) \odot \Pi_{z_{t-1}, k} \odot \beta_t(k)\enspace.
\end{equation}

Finally, the conditional posterior distribution of $z_t$ given $y$ and $z_{\backslash t}$ can be calculated using the forward backward algorithm to marginalize $x$.
Specifically,
\begin{equation}
\label{eq:slds_zt_conditional_gibbs}
p(z_t = k \, | \, z_{\backslash t}, y) \propto \Pi_{z_{t-1}, k} \Pi_{k, z_t} \cdot \int \alpha_{t-1}(x_{t-1}) p(y_t, x_t \, | \, x_{t-1}, z_t = k) \beta_{t}(x_t) \, dx_t dx_{t-1} \enspace,
\end{equation}
where $\alpha_{t-1}, \beta_t$ are calculated using Eqs.~\eqref{eq:lds_kalman_forward}-\eqref{eq:lds_kalman_backward} with $A_{t'} = A_{z_{t'}}, Q_{t'} = Q_{z_{t'}}$ for all $t' \in \BUFFSUBSEQ \backslash \{t\}$.

Note that Eq~\eqref{eq:slds_zt_conditional_gibbs} requires $O(|\BUFFSUBSEQ|)$ time per time point $z_t$; therefore one pass over $z_\BUFFSUBSEQ$ requires $O(|\BUFFSUBSEQ|^2)$.

\subsubsection{Initialization of Blocked Gibbs Sampler}
To sample $z$ from the filtered process,
we recursively sample from the conditional distribution $z_t \, | \, y_t, z_{t-1}$
\begin{equation}
p(z_t = k \, | \, z_{t-1}, y_t) \propto \Pi_{z_{t-1}, k} \cdot \int \alpha_{t-1}(x_{t-1}) p(y_{t}, x_t \, | \, x_{t-1}, z_t = k) \, dx_t dx_{t-1} \enspace,
\end{equation}
where $\alpha_{t-1}$ is calculated using Eq.~\eqref{eq:lds_kalman_forward} with $A_{t'} = A_{z_{t'}}, Q_{t'} = Q_{z_{t'}}$ for all $t' < t$.
Because we do not condition on $y_{>t}$ when $z_t$ is sampled,
we emphasize that this distribution is not the posterior $z \, | \, y$ (it is the \emph{filtered} distribution, not the \emph{smoothed} distribution).
However, it provides a better initialization point than sampling $z$ from the prior.

Alternatively, when $\dim(x) = n \leq \dim(y) = m$, we can initialize $z^{(0)}$ by sampling $z \, | \, x', y, \theta$ using Eq.~\eqref{eq:slds_z_conditional_gibbs} with $x' = y$.

\subsubsection{Gradient Estimator}
For the SLDS, the gradients are similarly a combination of those for the ARHMM Eqs.~\eqref{eq:arhmm_marginal_loglike_start}-\eqref{eq:arhmm_marginal_loglike_end} and the LGSSM Eqs.~\eqref{eq:lds_marginal_loglike_end}-\eqref{eq:lds_marginal_loglike_end}.

\subsubsection{Preconditioning}
For the SLDS, the precondition matrices are similarly a combination of those for the ARHMM Eqs.~\eqref{eq:arhmm_precondition_start}-\eqref{eq:arhmm_precondition_end} and the LGSSM Eqs.~\eqref{eq:lds_precondition_start}-\eqref{eq:lds_precondition_end}.
\begin{align}
\label{eq:slds_precondition_start}
D(\theta)_{\phi_k} = \diag(\phi_k) \enspace &\text{and} \enspace \Gamma(\theta)_{\phi_k} = 1
\\
D(\theta)_{A_k} = I_m \otimes Q_k \enspace &\text{and} \enspace \Gamma(\theta)_{A_k} = 0
\\
D(\theta)_{\psi_{Q_k}} = \frac{1}{2} I_n \otimes Q_k^{-1} \enspace &\text{and} \enspace \Gamma(\theta)_{\psi_{Q_k}} = \psi_{Q_k}
\\
D(\theta)_C = I_n \otimes R \enspace &\text{and} \enspace \Gamma(\theta)_{Q} = 0 
\\
D(\theta)_{\psi_{R_k}} = \frac{1}{2} I_m \otimes R_k^{-1} \enspace &\text{and} \enspace \Gamma(\theta)_{\psi_{R_k}} = \psi_{R_k}
\label{eq:slds_precondition_end} \enspace.
\end{align}

\section{Additional Experiment Details}
\label{supp-sec:experiments}

\subsection{Experiment Hyperparameters}
\subsubsection{Priors}
In our experiments, we use the following (conjugate) priors for $\theta$.

For the discrete latent state sequence transition matrix $\Pi$,
we use a flat-Dirichlet prior
\begin{equation}
\Pr(\Pi_k) \propto \prod_{k'} \Pi_{k,k'}^{\alpha_{k,k'}-1} \enspace, \text{ where } \alpha_{k,k'} = 1 \enspace.
\end{equation}

For the continuous transition matrix $A$,
we use a matrix normal prior
\begin{equation}
\Pr(A) \propto \exp\left(-\tr\left[V^{-1} (A-M)^T U^{-1}(A-M)\right]/2\right) \enspace,
\end{equation}
with mean $M=0$, diagonal column covariance $V = 10^{2} \cdot I_n$, and row variance $U = Q$.

For the noise covariances $Q$ and $R$,
we use flat Wishart priors over $Q^{-1}$ and $R^{-1}$
\begin{equation}
\Pr(Q^{-1}) \propto |Q|^{(n+1-\nu)/2} e^{-\tr(\Psi Q^{-1})/2} \enspace, \enspace
\Pr(R^{-1}) \propto |R|^{(m+1-\nu)/2} e^{-\tr(\Psi R^{-1})/2} \enspace,
\end{equation}
where $\Psi = \nu \cdot I$ and $\nu = n+1$ or $m+1$.

\subsubsection{Sampling Subsequences}
\label{supp-sec:sample_subsequences}
In our experiments, we sample subsequences $\mathcal{S} = \{t_1, \ldots, t_{S}\} \subset \mathcal{T} = \{1, \ldots, T\}$ uniformly from all $T-S+1$ possible contiguous subsequences. That is, $\Pr(t_1 = t) = 1/(T-S+1)$ for $t \in \{1, \ldots, T-S+1\}$and $\Pr(t \in \mathcal{S})$ is given by
\begin{equation}
\label{eq:naive_partition}
\Pr(t \in \mathcal{S}) = \frac{\min \{t, T-t+1, S, T-S+1\}}{T-S+1} \enspace.
\end{equation}

An alternative method for sampling subsequences is to sample $\mathcal{S}$ from separate partitions of $\mathcal{T}$.
That is if $T/S = L$ is a whole number, then $\Pr(t_1 = t) = 1/L$ for $t \in \{1+kL \, | \, k = 0, 1, \ldots, L-1 \}$ and
\begin{equation}
\label{eq:strict_partition}
\Pr(t \in \mathcal{S}) = \frac{1}{L} = \frac{S}{T} \enspace.
\end{equation}

We found both methods work well in practice, but found empirically that the former has reduced variance in the stochastic gradient estimates $\hat{g}(\theta)$; therefore we use the former in our experiments.

\subsubsection{List of Hyperparameters}
\begin{itemize}
\setlength\itemsep{0.5em}
\item Synthetic Gaussian HMM $T = 10^4$
\begin{itemize}
\item Prior: $\Pi_k$ are Dirichlet, $\mu$ is Normal, and $Q^{-1}$ are Wishart. 
\item Initialization: using K-means on $y_t$
\item Stepsizes: 
\begin{center}
\begin{tabular}{ccc|ccc}
\multicolumn{3}{c|}{\underline{SGLD}} & \multicolumn{3}{c}{\underline{SGRLD}} \\
No-Buffer & Buffer & Full & No-Buffer & Buffer & Full \\
\hline
0.001 & 0.001& 0.1 & 0.001 & 0.001 & 1.0 \\
\end{tabular}
\end{center}

\end{itemize}

\item Synthetic Gaussian HMM $T = 10^6$
\begin{itemize}
\item Prior: $\Pi_k$ are Dirichlet, $\mu$ is Normal, and $Q^{-1}$ are Wishart. 
\item Initialization: using K-means on $y_t$
\item Stepsizes: 
\begin{center}
\begin{tabular}{ccc|ccc}
\multicolumn{3}{c|}{\underline{SGLD}} & \multicolumn{3}{c}{\underline{SGRLD}} \\
No-Buffer & Buffer & Full & No-Buffer & Buffer & Full \\
\hline
0.001 & 0.001& 0.1 & 0.001 & 0.01 & 0.1 \\
\end{tabular}
\end{center}
\end{itemize}

\item Ion Channel (Full) HMM
\begin{itemize}
\item Prior: $\Pi_k$ are Dirichlet, $\mu$ is Normal, and $Q^{-1}$ are Wishart. 
\item Initialization: using K-means on $y_t$
\item Stepsizes: 
\begin{center}
\begin{tabular}{cc|cc}
\multicolumn{2}{c|}{\underline{SGLD}} & \multicolumn{2}{c}{\underline{SGRLD}} \\
No-Buffer & Buffer & No-Buffer & Buffer \\
\hline
0.0001 & 0.0001 & 0.01 & 0.01 \\
\end{tabular}
\end{center}

\end{itemize}

\item Ion Channel (Subset) HMM
\begin{itemize}
\item Prior: $\Pi_k$ are Dirichlet, $\mu$ is Normal, and $Q^{-1}$ are Wishart. 
\item Initialization: using K-means on $y_t$
\item Stepsizes: 
\begin{center}
\begin{tabular}{cc|cc}
\multicolumn{2}{c|}{\underline{SGLD}} & \multicolumn{2}{c}{\underline{SGRLD}} \\
No-Buffer & Buffer & No-Buffer & Buffer \\
\hline
0.001 & 0.001 & 0.001 & 0.001 \\
\end{tabular}
\end{center}

\end{itemize}

\item Synthetic ARHMM $T = 10^4$
\begin{itemize}
\item Prior: $\Pi_k$ are Dirichlet, $A$ is matrix Normal, and $Q^{-1}$ are Wishart. 
\item Initialization: using K-means on $[y_t, y_{t-1}]$
\item Stepsizes: 
\begin{center}
\begin{tabular}{ccc|ccc}
\multicolumn{3}{c|}{\underline{SGLD}} & \multicolumn{3}{c}{\underline{SGRLD}} \\
No-Buffer & Buffer & Full & No-Buffer & Buffer & Full \\
\hline
0.0001 & 0.0001& 0.01 & 0.001 & 0.001 & 0.1 \\
\end{tabular}
\end{center}

\end{itemize}

\item Synthetic ARHMM $T = 10^6$
\begin{itemize}
\item Prior: $\Pi_k$ are Dirichlet, $A$ is matrix Normal, and $Q^{-1}$ are Wishart. 
\item Initialization: using K-means on $[y_t, y_{t-1}]$
\item Stepsizes: 
\begin{center}
\begin{tabular}{cccccc}
\multicolumn{3}{c|}{\underline{SGLD}} & \multicolumn{3}{c}{\underline{SGRLD}} \\
No-Buffer & Buffer & Full & No-Buffer & Buffer & Full \\
\hline
0.0001 & 0.0001& 0.1 & 0.0001 & 0.0001 & 0.1 \\
\end{tabular}
\end{center}
\end{itemize}

\item Canine Seizure ARHMM
\begin{itemize}
\item Prior: $\Pi_k$ are Dirichlet, $A$ is matrix Normal, and $Q^{-1}$ are Wishart. 
\item Initialization: using K-means on $[y_t, y_{t-1}]$
\item Stepsizes: SGLD $=0.01$, SGRLD $=0.1$.
\end{itemize}

\item Synthetic LGSSM $T = 10^4$
\begin{itemize}
\item Prior: $A$ is matrix Normal and $Q^{-1}, R^{-1}$ are Wishart. 
\item Initialization: From prior with $\nu = 4, \Psi=4\cdot I_2$ for the Wishart priors.
\item Stepsizes: 
\begin{center}
\begin{tabular}{ccc|ccc}
\multicolumn{3}{c|}{\underline{SGLD}} & \multicolumn{3}{c}{\underline{SGRLD}} \\
No-Buffer & Buffer & Full & No-Buffer & Buffer & Full \\
\hline
0.01 & 0.01& 0.1 & 0.01 & 0.01 & 0.1 \\
\end{tabular}
\end{center}
\end{itemize}

\item Synthetic LGSSM $T = 10^6$
\begin{itemize}
\item Prior: $A$ is matrix Normal and $Q^{-1}, R^{-1}$ are Wishart. 
\item Initialization: From prior with $\nu = 4, \Psi=4\cdot I_2$ for the Wishart priors.
\item Stepsizes: 
\begin{center}
\begin{tabular}{ccc|ccc}
\multicolumn{3}{c|}{\underline{SGLD}} & \multicolumn{3}{c}{\underline{SGRLD}} \\
No-Buffer & Buffer & Full & No-Buffer & Buffer & Full \\
\hline
0.01 & 0.01& 1.0 & 0.01 & 0.01 & 1.0 \\
\end{tabular}
\end{center}

\end{itemize}

\item Synthetic SLDS $T = 10^4$
\begin{itemize}
\item Prior: $\Pi_k$ is Dirichlet, $A_k$ is matrix Normal and $Q_k^{-1}, R^{-1}$ are Wishart. 
\item Initialization: $R$ from Wishart Prior, $\Pi, A, Q$ from $K$-means as in ARHMM.
\item Stepsizes: SGRLD X $=0.5$, SGRLD Z $=0.1$, SGRLD XZ $=0.1$.
\end{itemize}

\item Synthetic SLDS $T = 10^6$
\begin{itemize}
\item Prior: $\Pi_k$ is Dirichlet, $A_k$ is matrix Normal and $Q_k^{-1}, R^{-1}$ are Wishart. 
\item Initialization: $R$ from Wishart Prior, $\Pi, A, Q$ from $K$-means as in ARHMM.
\item Stepsizes: SGRLD X $=0.5$, SGRLD Z $=0.1$, SGRLD XZ $=0.1$.
\end{itemize}

\item Canine Seizure SLDS
\begin{itemize}
\item Prior: $\Pi_k$ is Dirichlet, $A_k$ is matrix Normal and $Q_k^{-1}, R^{-1}$ are Wishart. 
\item Initialization: $R$ from Wishart Prior, $\Pi, A, Q$ from $K$-means as in ARHMM.
\item Stepsizes: SGRLD $=0.1$, SGLD $=0.1$.
\end{itemize}

\item Daily Weather SLDS
\begin{itemize}
\item Prior: $\Pi_k$ is Dirichlet, $A_k$ is matrix Normal and $Q_k^{-1}, R^{-1}$ are Wishart. 
\item Initialization: $R$ from Wishart Prior, $\Pi, A, Q$ from $K$-means as in ARHMM.
\item Stepsizes: SGRLD $=0.1$, SGLD $=0.1$.
\end{itemize}

\item Hourly Weather SLDS
\begin{itemize}
\item Prior: $\Pi_k$ is Dirichlet, $A_k$ is matrix Normal and $Q_k^{-1}, R^{-1}$ are Wishart. 
\item Initialization: $R$ from Wishart Prior, $\Pi, A, Q$ from $K$-means as in ARHMM.
\item Stepsizes: SGRLD $=0.1$, SGLD $=0.01$.
\end{itemize}

\end{itemize}

\subsection{Additional Metric Details}
To assess the `mixing' rate of our MCMC samplers, we measure each sampled chain's kernel Stein divergence (KSD) to the posterior~\cite{liu2016kernelized,gorham2017measuring}.
Given a chain of sampled $\{\theta^{(i)}\}_{1}^N$ (after burnin and thinning),
let $q(\theta)$ be the empirical distribution of the samples, that is
\begin{equation}
q(\theta) = \frac{1}{N}\sum_{i = 1}^{N} \delta_{\theta = \theta^{(i)}} \enspace.
\end{equation}
Then the KSD between $q(\theta)$ and the posterior distribution $p(\theta)$ is
\begin{align}
KSD(q,p) &= \sum_{d = 1}^{\text{dim}(\theta)} \sqrt{\sum_{i,i' = 1}^n \frac{k_0^d(\theta_i, \theta_{i'})}{n^2}} \enspace, \text{ where } \enspace \\
k_0^d(\theta_i, \theta_{i'}) &= \grad_{\theta_d} \log p(\theta_i) k(\theta_i, \theta_{i'}) \grad_{\theta_d} \log p(\theta_{i'}) + \grad \log p(\theta_{i'}) \grad_x k(\theta_i, \theta_{i'}) \\
\nonumber
&\quad + \grad \log p(\theta_i) \grad_y k(\theta_i, \theta_{i'}) + \grad_x \grad_y k(\theta_i, \theta_{i'})
\end{align}
and $k(\cdot, \cdot)$ is a valid kernel function.
Following~\cite{gorham2017measuring}, we use the inverse multiquadratic kernel (IMQ) $k(x,y) = (1+\|x-y\|_2^2)^{-0.5}$ in our experiments.
As full gradient evaluations $\grad \log p(\theta)$ are computationally intractable for our long time series, we replace them with stochastic estimates based on Eq.~\eqref{eq:efficient_potential_estimate} using $S = 10^4$ and $B=100$ when $T > 10^4$.

To measure the recovery of discrete latent state variables $z_t$ when the true latent states are known (e.g. in synthetic experiments),
we use normalized mutual information (NMI).
NMI is an information theoretic measure of similarity between discrete assignments~\cite{vinh2010information}.
\begin{equation}
\text{NMI}(Z_i, Z_*) = \frac{I(Z_i, Z_*)}{\sqrt{H(Z_i)H(Z_*)}} \enspace, \text{ with } Z_i = (z^{(i)}_1, \ldots, z^{(i)}_T) \enspace,
\end{equation}
 where $I(X,Y)$ is mutual information and $H(X)$ is entropy.
NMI is maximized at 1 when the assignments are equal up to a permutation and minimized at 0 when the assignments share no information.
This serves as `clustering' or segmentation metric for measuring the coherence between our model's inferred latent states and the true latent states.

To measure the recovery of continuous latent state variables $x_t$ when the true latent states are known,
we use root mean-squared error (RMSE) $\text{RMSE}(x, x') = \sum_t \| x_t - x_t'\|_2$.

\subsection{Synthetic Gaussian HMM}
\label{supp-sec:experiments-synthHMM}
Following~\cite{foti2014stochastic,ma2017stochastic}, we generate data from a Gaussian HMM
with $K = 8$ latent states (see Figure~\ref{fig:gausshmm_grad_error} (left))
This \textit{reversed cycles} (RC) dataset strongly transitions between two cycles over three states, each in opposite directions.

\begin{figure}[!htb]
    \centering
    \begin{minipage}[t]{.25\textwidth}
    \centering
    \includegraphics[width=\textwidth]{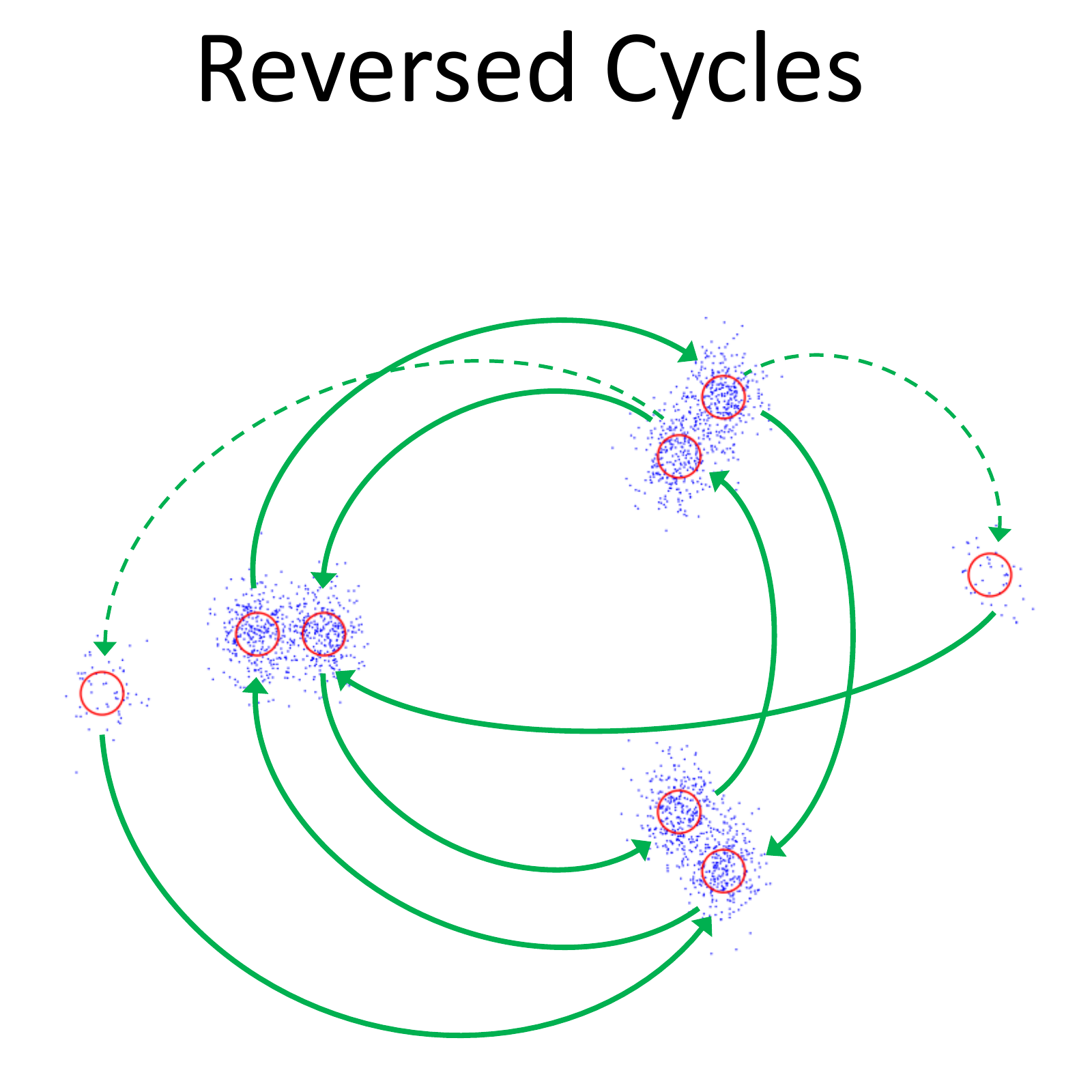}
    \end{minipage}
    \begin{minipage}[t]{.35\textwidth}
    \centering
        \includegraphics[width=\textwidth]{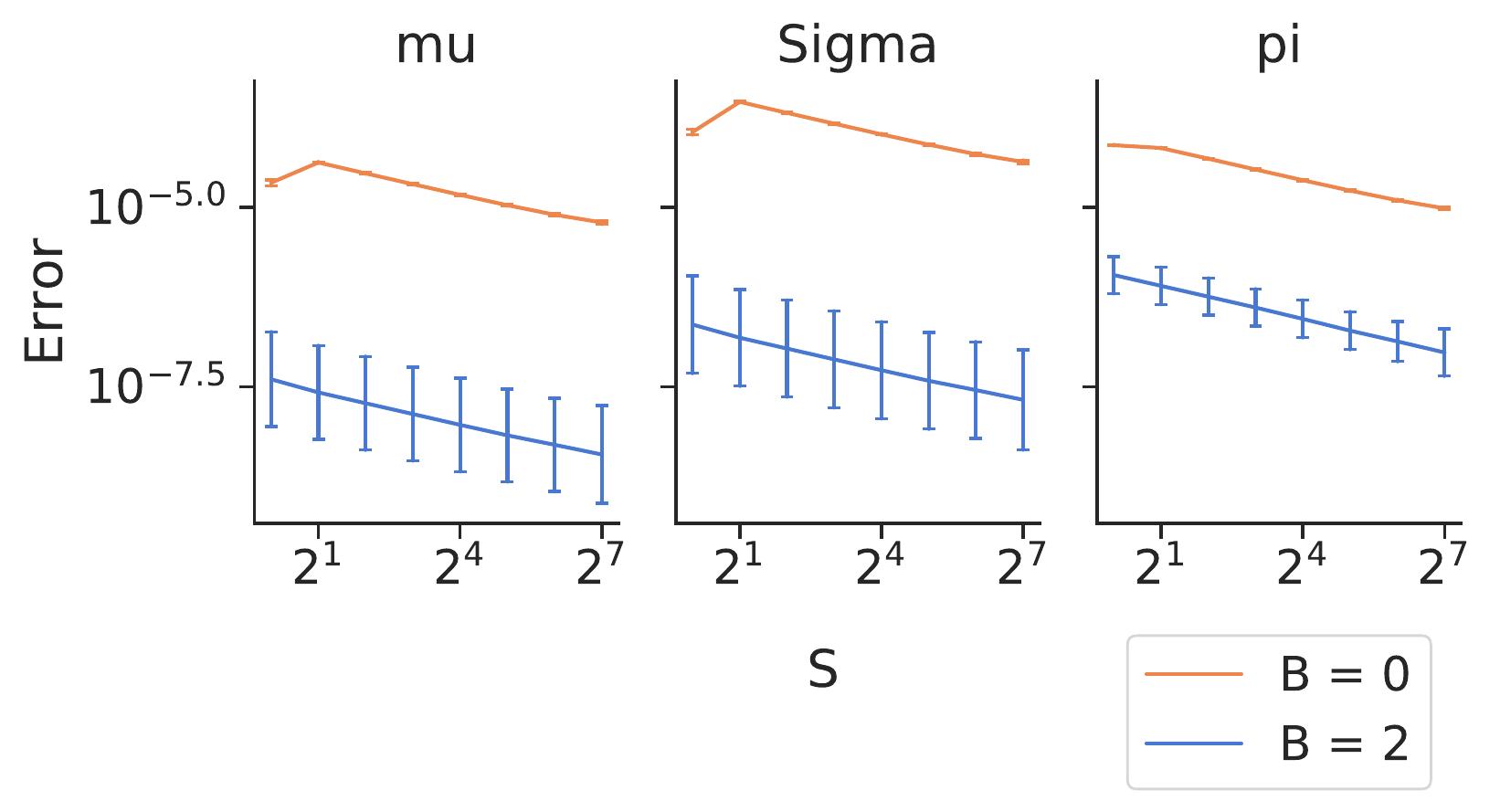}
    \end{minipage}
    \begin{minipage}[t]{.35\textwidth}
    \centering
        \includegraphics[width=\textwidth]{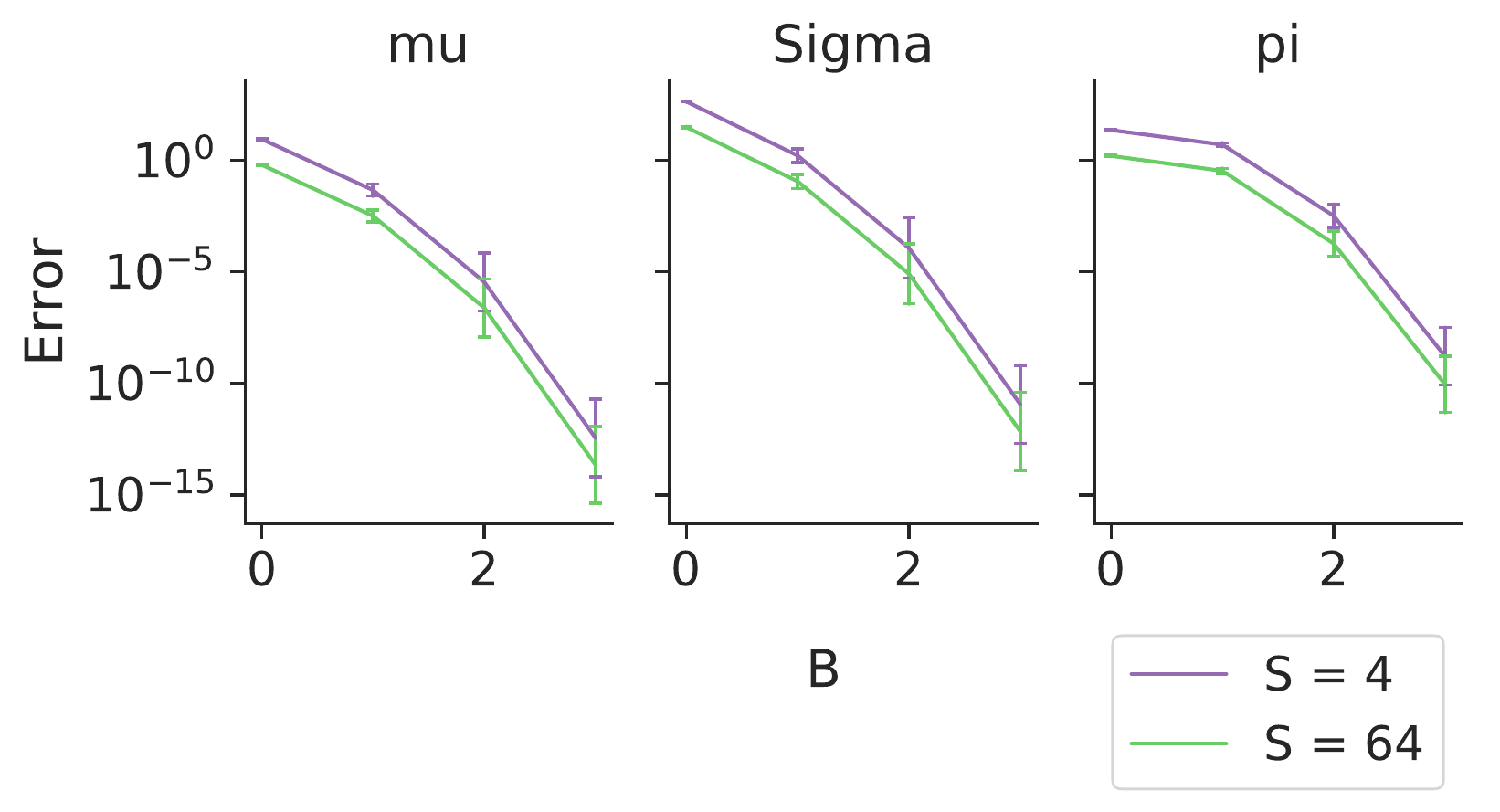}
    \end{minipage}
    \caption{(Left) Sample dataset; arrows indicate Markov transitions. Stochastic gradient error $\E_\SUBSEQ\|\bar{g}(\theta) - \tilde{g}(\theta)\|_2$: (center) varying subsequence size $S$ for no-buffer $B=0$ and buffer $B=5$, (right) varying buffer size $B$ for $S = 2$ and $S=50$. Error bars are SD over $100$ randomly generated datasets.
}
    \label{fig:gausshmm_grad_error}
\end{figure}
Figure~\ref{fig:gausshmm_grad_error} (right-pair) are plots of the stochastic gradient error $\E_\SUBSEQ\|\bar{g}(\theta) - \tilde{g}(\theta)\|_2$ between the unbiased and buffered estimates evaluated at the true model parameters $\theta=\theta^*$.
Similar to the ARPHMM and LGSSM, we see that the error decays $O(1/S)$ and that buffering deceases the error by orders of magnitude in Figure~\ref{fig:gausshmm_grad_error} (center).
In Figure~\ref{fig:gausshmm_grad_error} we see that the error decays geometrically in buffer size $O(L^B)$.
For this RC dataset, the geometric decay rate $L$ is very small; thus small buffers (e.g. $B=2$) reduce the error drastically.

From Figure~\ref{fig:gausshmm_grad_error} (center), we see that the stochastic gradients are heavily biased without buffering (orange) for small subsequence lengths, as they fail to capture the structured transitions between states. However this bias disappears with buffering (blue).
From Figure~\ref{fig:gausshmm_grad_error} (right), we see that the stochastic gradient decays quickly with increasing buffer size $B$ for small subsequence $S = 2$ (purple).
The bias in the stochastic gradients of observations parameters $(\mu, \Sigma)$ is less extreme than for transition matrix $\Pi$ which is associated with the latent states;
we include their error plots in the Supplement.

\begin{figure}[!htb]
\centering
    \begin{minipage}[c]{.1\textwidth}
    \centering
    \hbox{\rotatebox{90}{\hspace{1em} $T = 10^4$}}
    \end{minipage}
    \hspace{0.2em}
    \begin{minipage}[c]{.4\textwidth}
    \centering
        \includegraphics[width=\textwidth]{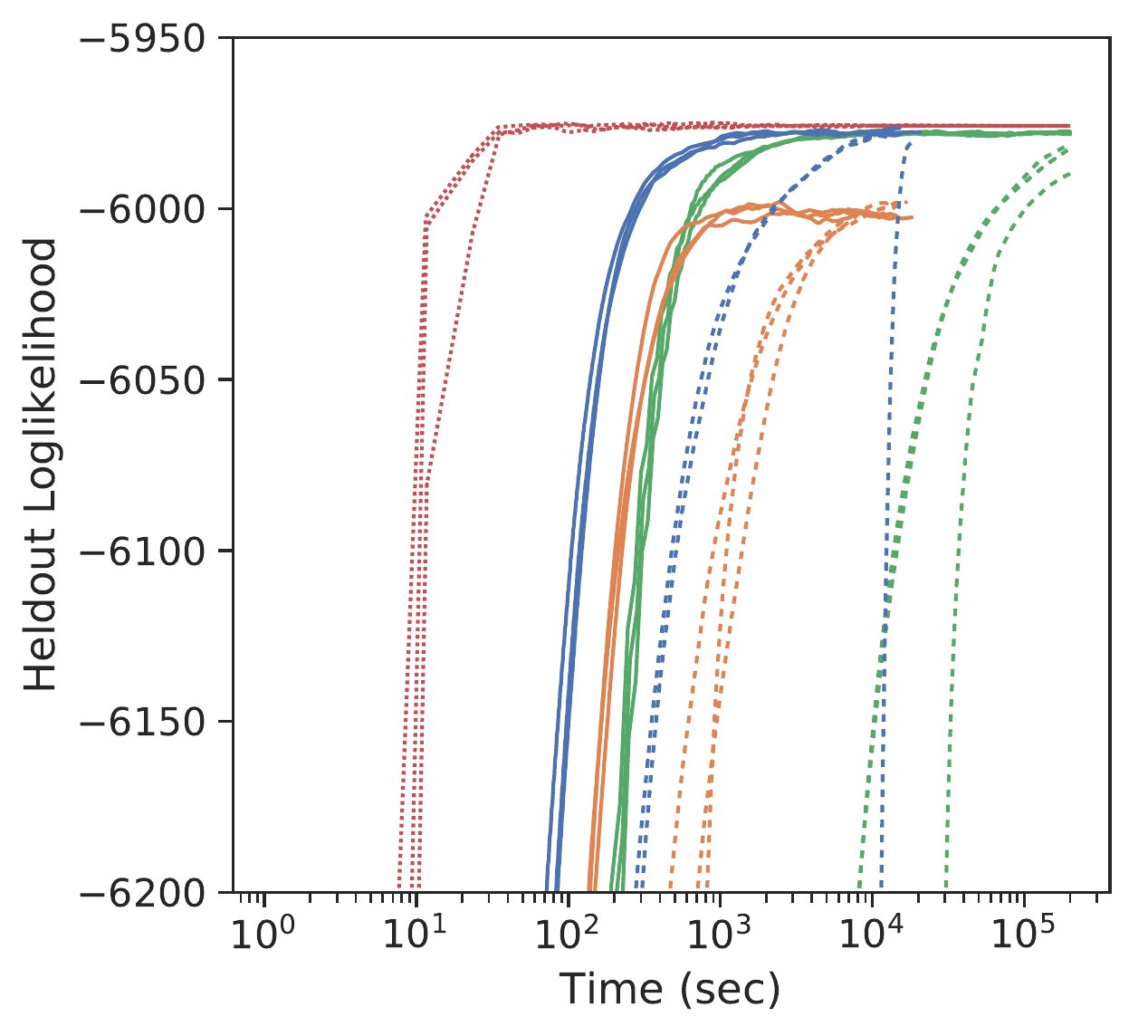}
    \end{minipage}
    \begin{minipage}[c]{.4\textwidth}
    \centering
        \includegraphics[width=\textwidth]{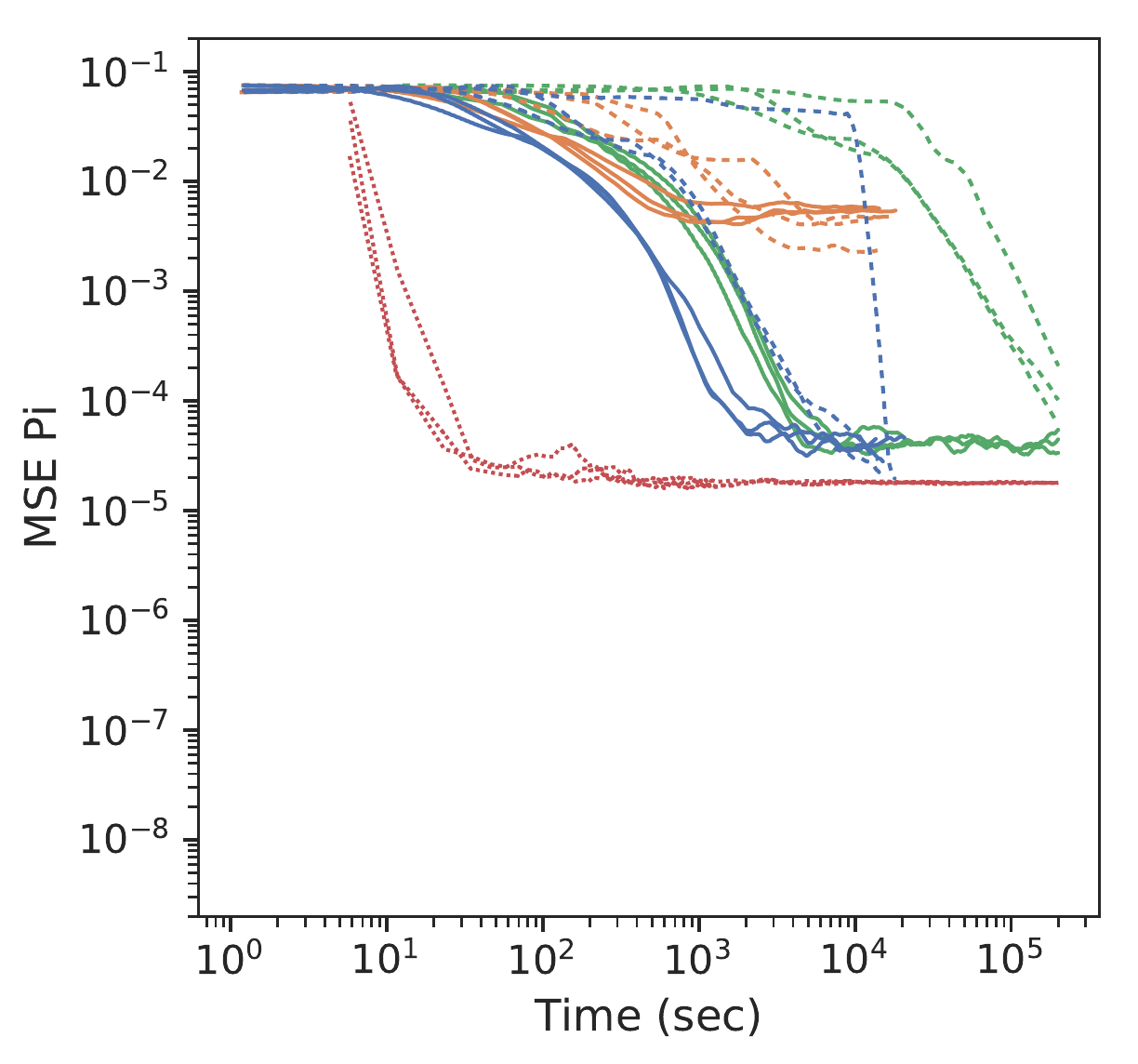}
    \end{minipage}

    \begin{minipage}[c]{.1\textwidth}
    \centering
    \hbox{\rotatebox{90}{\hspace{1em} $T = 10^6$}}
    \end{minipage}
    \hspace{0.2em}
    \begin{minipage}[c]{.4\textwidth}
    \centering
        \includegraphics[width=\textwidth]{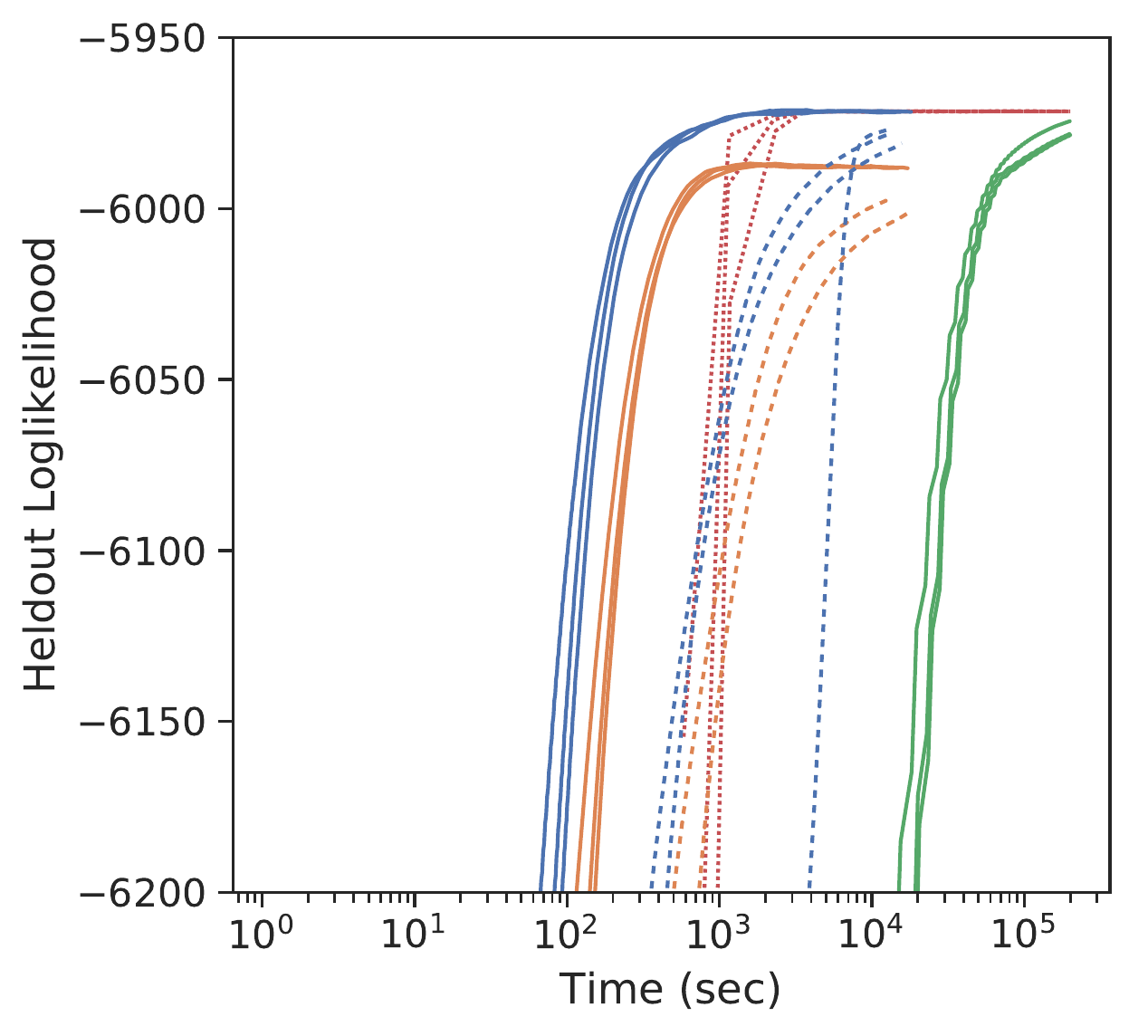}
    \end{minipage}
    \begin{minipage}[c]{.4\textwidth}
    \centering
        \includegraphics[width=\textwidth]{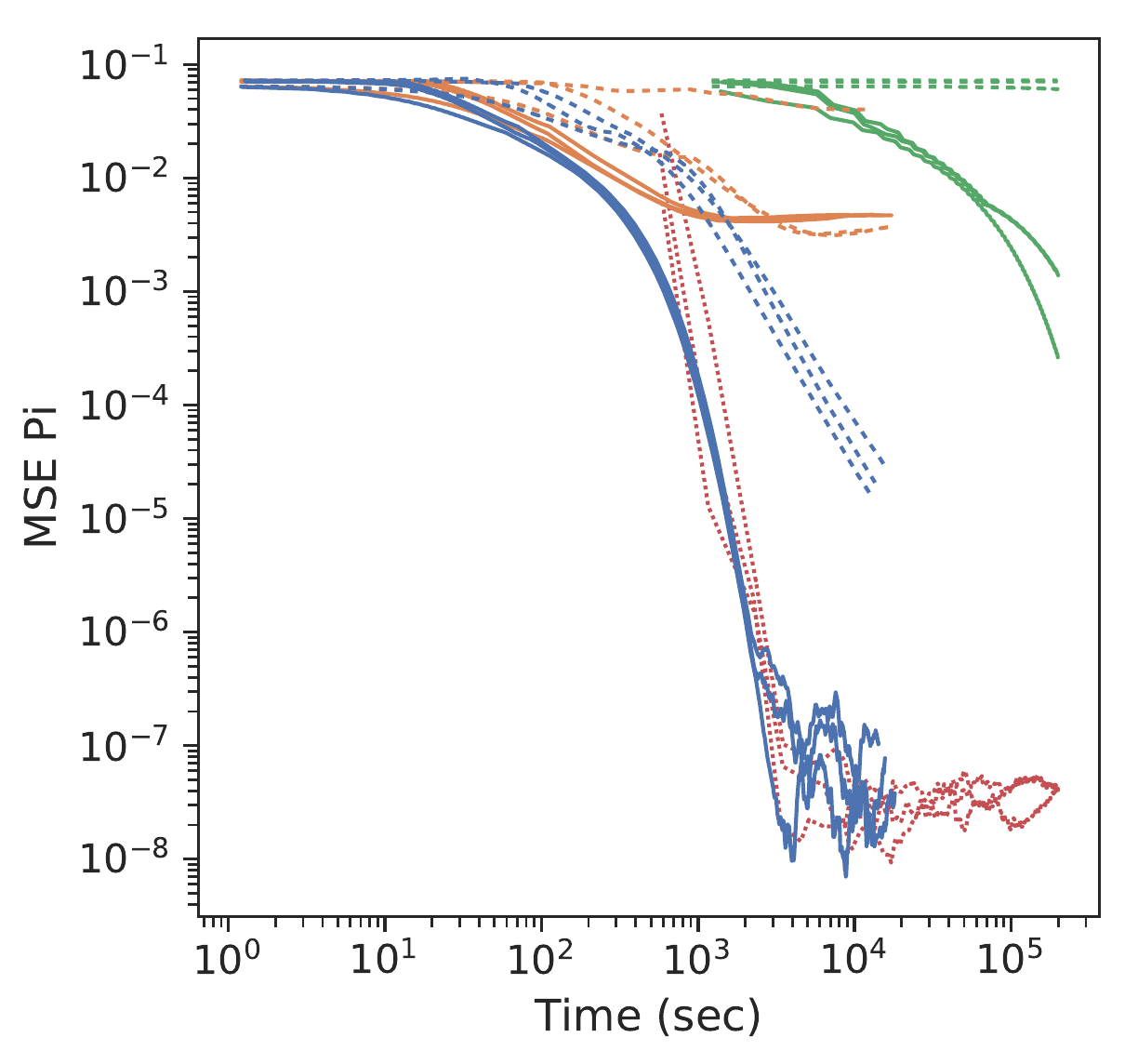}
    \end{minipage}
    \caption{Metrics vs Runtime on RC data with $T = 10^4$ (top), $T=10^6$ (bottom), for different methods:
    \textcolor{BrickRed}{(Gibbs)}, \textcolor{ForestGreen}{(Full)},  \textcolor{Orange}{(No Buffer)} and \textcolor{Blue}{(Buffer)} SGMCMC.
    For SGMCMC methods, solid (\fullline) and dashed (\dashedline) lines indicate SGRLD and SGLD respectively.
    The different metrics are: (left) heldout loglikelihoood and (right) transition matrix estimation error $MSE(\hat\Pi^{(s)}, \Pi^*)$.
.
}
    \label{fig:gausshmm_rc_metrics}
\end{figure}
\begin{figure}[!htb]
    \begin{minipage}[c]{.9\textwidth}
    \centering
        \includegraphics[width=\textwidth]{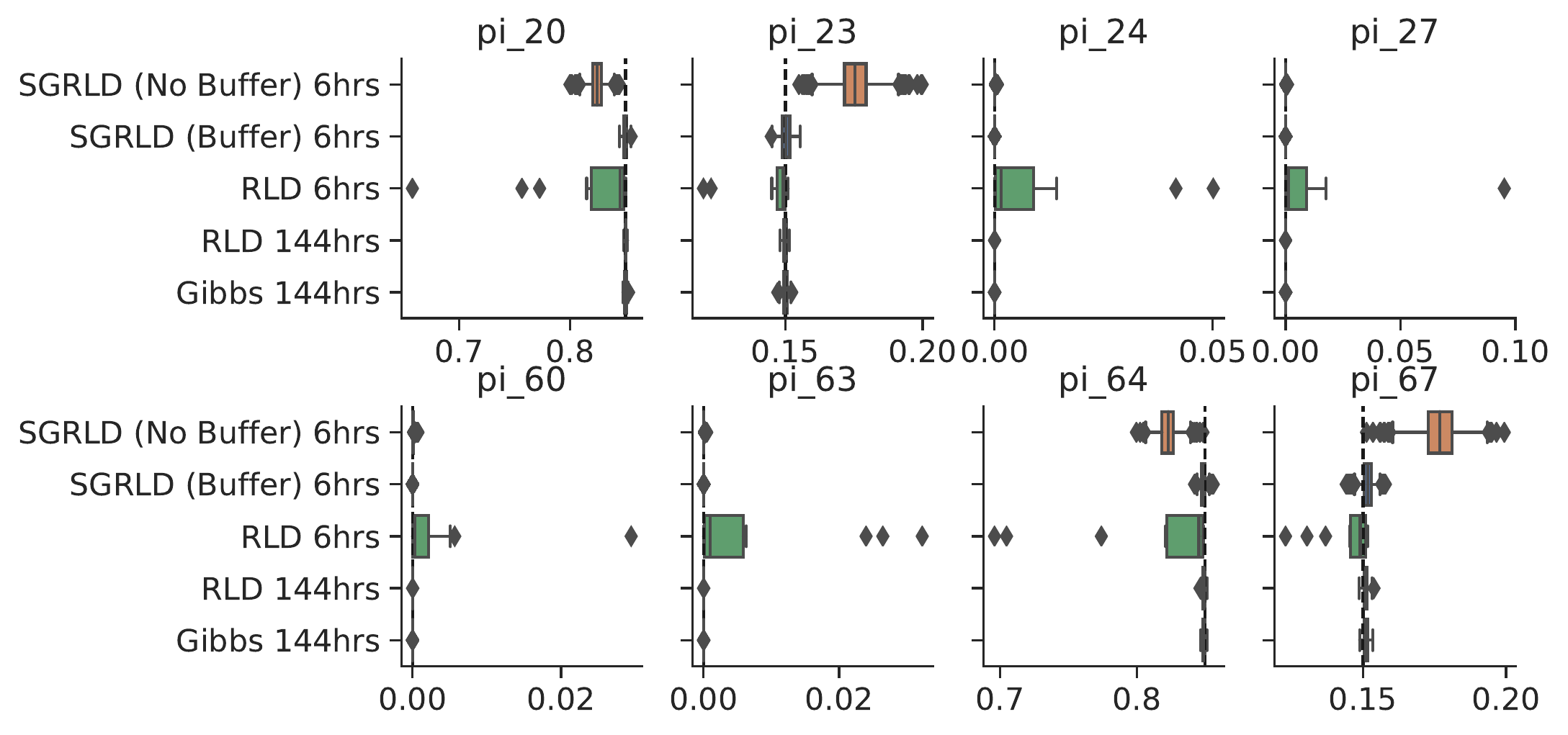}
    \end{minipage}
    \caption{Boxplot of MCMC samples for RC data $T = 10^6$ for select values of $\Pi$.
}
\label{fig:gausshmm_rc_boxplot}
\end{figure}
In Figures~\ref{fig:gausshmm_rc_metrics} and \ref{fig:gausshmm_rc_boxplot},
we compare SGLD (no-buffer and buffer), SGRLD (no-buffer and buffer), and Gibbs.
We run our samplers on one training sequence and evaluate performance on another test sequence.
We consider sequence lengths of $T = 10^4$ and $T=10^6$.
For the SGMCMC methods,
we use a subsequence size of $S = 2$ and a buffer size of $B=0$ (no-buffer) or $B=2$ (buffer).
From Figure~\ref{fig:gausshmm_rc_metrics} we again see that preconditioning helps convergence and mixing as SGRLD outperforms SGLD and from figure~\ref{fig:gausshmm_rc_boxplot} that buffering is necessary to properly estimate $\Pi$.

Note that for the $T = 10^4$ case, we observe that SGRLD underestimates the variance of $\Pi$ (Figure~\ref{fig:gausshmm_rc_metrics} bottom-left).
This is due to the preconditioner $D(\theta)_{\Pi} = \Pi$, creating absorbing states in the discretized dynamics (see comment in Section~\ref{supp-sec:hmm_precond}).

\begin{table}[ht]
\centering
\caption{$\log_{10}$(KSD) by variable of RC samplers. Mean and (SD) over runs in Figure~\ref{fig:gausshmm_rc_metrics}.}
\label{table:gausshmm_rc_ksd}
\begin{tabular}{rrlll}
  \hline
\abovestrut{1em}
 & Sampler & $\pi$ & $\mu$ & $\Sigma$ \\
  \hline
 \multirow{7}{*}{\rotatebox[origin=c]{90}{$|\FULLSEQ| = 10^4$}}
\abovestrut{1em}
& SGLD (No Buffer)  &  1.95 (0.05) &  1.12 (0.06) &  2.46 (0.05) \\
& SGLD (Buffer)     &  1.33 (0.15) &  1.16 (0.17) &  1.99 (0.10) \\
& LD       &  1.99 (0.07) &  1.50 (0.39) &  2.10 (0.72) \\
\cline{2-5}
\abovestrut{1em}
& SGRLD (No Buffer) &  1.69 (0.01) &  0.77 (0.03) &  2.49 (0.03) \\
& SGRLD (Buffer)    &  0.81 (0.01) &  0.53 (0.01) &  2.09 (0.05) \\
& RLD      &  0.85 (0.03) &  0.54 (0.06) &  2.09 (0.06) \\
\cline{2-5}
\abovestrut{1em}
& Gibbs             &  0.77 (0.01) &  0.38 (0.06) &  1.74 (0.07) \\
 \hline
 \hline
 \multirow{7}{*}{\rotatebox[origin=c]{90}{$|\FULLSEQ| = 10^6$}}
\abovestrut{1em}
& SGLD (No Buffer)  &  4.25 (0.41) &  2.93 (0.52) &  4.63 (0.44) \\
& SGLD (Buffer)     &  3.34 (0.12) &  2.84 (0.47) &  3.94 (0.04) \\
& LD       &  5.42 (0.03) &  4.07 (0.35) &  5.30 (0.41) \\
\cline{2-5}
\abovestrut{1em}
& SGRLD (No Buffer) &  3.67 (0.01) &  2.59 (0.05) &  3.99 (0.09) \\
& SGRLD (Buffer)    &  2.07 (0.04) &  2.38 (0.08) &  3.78 (0.09) \\
& RLD     &  3.91 (0.05) &  3.17 (0.11) &  4.76 (0.03) \\
\cline{2-5}
\abovestrut{1em}
& Gibbs             &  3.11 (0.05) &  3.10 (0.07) &  4.65 (0.05) \\
 \hline
\end{tabular}
\end{table}

Table~\ref{table:gausshmm_rc_ksd} shows the KSD of different sampling methods for different components of $\theta$.
Although full sequence methods performs well for small $T$, they perform worse for larger $T$ due to increase time between iterations.
We also see that buffered SGRLD outperforms the other SGMCMC methods on $\Pi$, as the non-buffered methods are sampling from the incorrect distribution and SGLD suffers from extreme autocorrelation.

\subsection{Downsampled Ion Channel Recordings}
We now consider a downsampled version of the ion channel recording data presented in Section~\ref{sec:experiments-HMM-ion}.
In particular, we consider downsampling the data by a factor of $50$ (as in~\cite{ma2017stochastic}), resulting in $|\FULLSEQ| = 209,634$ observations.
We again train on the first 90\% and evaluate on the last 10\% after applying a log-transform and normalizing the observations to use Gaussian emissions.
For our SGMCMC methods we again use a subsequence size of $S = 10$ and a buffer size of $B=0$ (no-buffer) or $B=10$ (buffer).
Figure~\ref{fig:down_sampled_ion_channel} presents our results including comparisons to Gibbs sampling (red).
For this (shorter) downsampled data, Gibbs sampling outperforms the SGMCMC methods.
We see that the performance of the SGMCMC method is similar to the full sample case (compare to Figure~\ref{fig:ion_channel}) and that SGRLD with buffering quickly reaches the same mode as Gibbs.

\begin{figure}[!htb]
    \centering
    \begin{minipage}[c]{.33\textwidth}
    \centering
        \includegraphics[width=\textwidth]{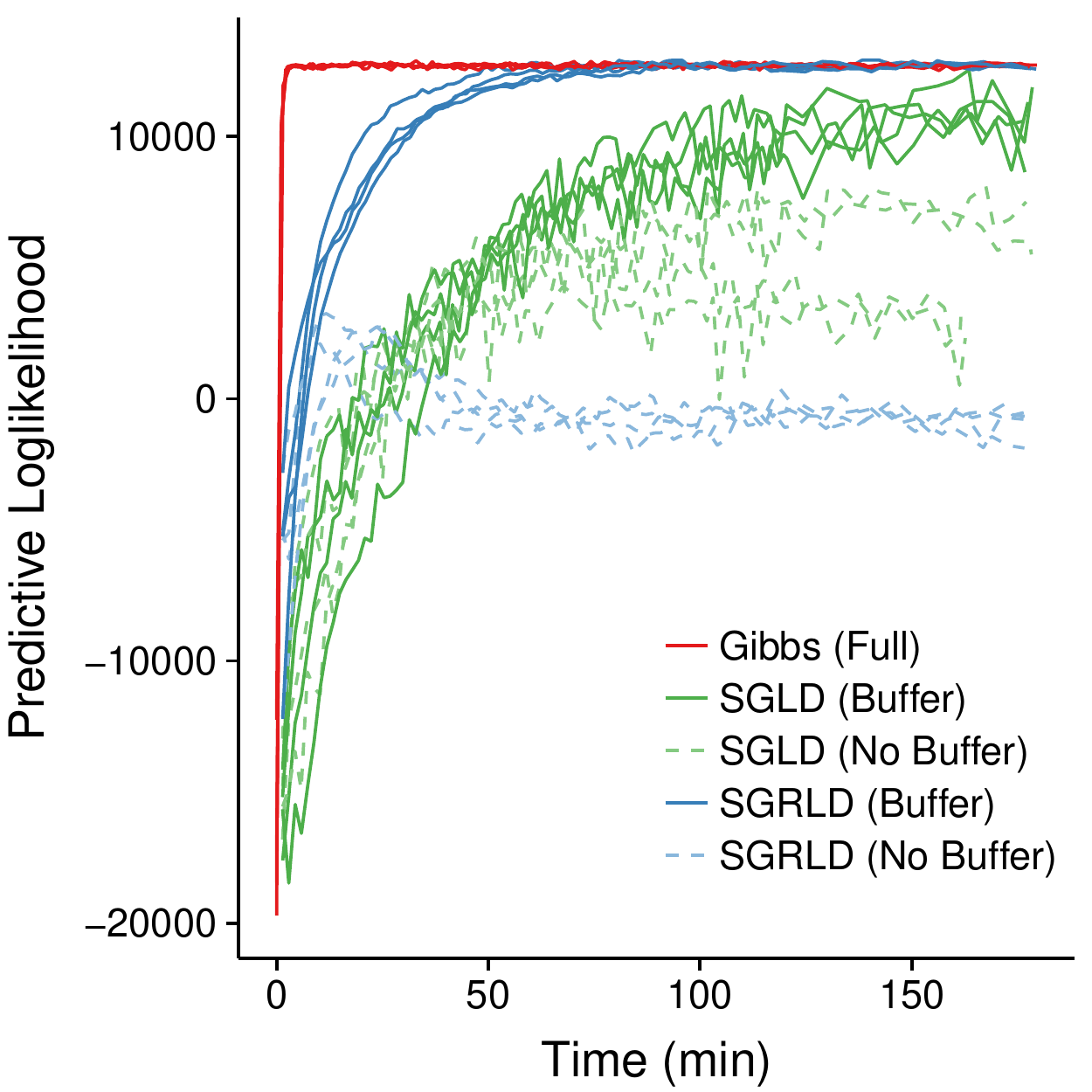}
    \end{minipage}
    \begin{minipage}[c]{.31\textwidth}
    \centering
    \includegraphics[width=\textwidth]{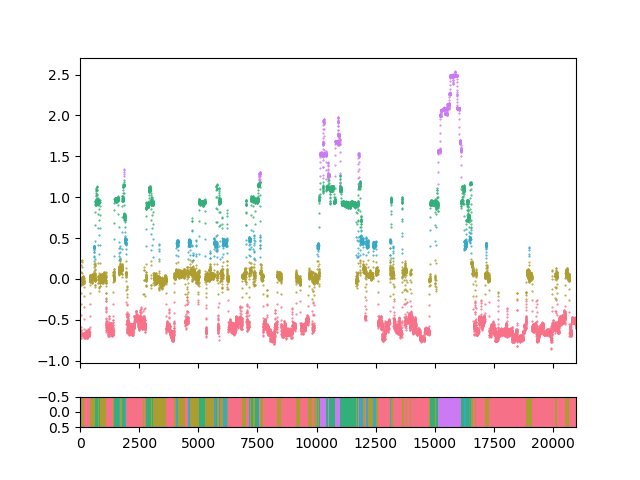}
    \end{minipage}
    \begin{minipage}[c]{.31\textwidth}
    \centering
    \includegraphics[width=\textwidth]{figures/sgrld_gausshmm_ion-channel_fit.png}
    \end{minipage}
    \caption{Ion Channel Recordings: (Left) predictive loglikelihood vs runtime. (Center) example segmentation using Gibbs (Right) example segmentation using SGRLD.
}
\label{fig:down_sampled_ion_channel}
\end{figure}

\subsection{Additional Synthetic Experiment Plots}
We now present additional plots for the synthetic data experiments.
These plots show the MSE for `other' components of $\theta$ to the true parameters of $\theta^*$ as well as other measures of fit such as predictive loglikelihood or recovery of the latent state sequence (NMI or RMSE).
\subsubsection{Gaussian HMM}
The parametrization of the RC data set is as follows:
\begin{equation}
\Pi = \begin{bmatrix}
.01 & .99 & 0 & 0 & 0 & 0 & 0 & 0 \\
0 & .01 & .99 & 0 & 0 & 0 & 0 & 0  \\
.85 & 0 & 0 & .15 & 0 & 0 & 0 & 0 \\
0  & 0 & 0 & 0 & 1 & 0 & 0 & 0 \\
0 & 0 & 0 & 0 & .01 & .99 & 0 & 0 \\
0 & 0 & 0 & 0 & 0 & .01 & .99 & 0 \\
0 & 0 & 0 & 0 & .85 & 0 & 0 & .15\\
1 & 0 & 0 & 0 & 0 & 0 & 0 & 0
\end{bmatrix} 
\end{equation}
with
\begin{equation}
\mu_{1:8} = \left\{ (-50,0); (30,-30); (30,30); (-100,-10); (40,-40); (-65,0); (40, 40); (100,10)  \right\},
\end{equation}
and $\Sigma_k = 20*I_2$ for all states $k = 1:K$.
Figure~\ref{fig:gausshmm_rc_extra_metrics} are plots of additional metrics for the Gaussian HMM experiment on the RC data set.
We see a bigger difference between the buffered and non-buffered methods in predictive loglikelihood as it is more sensitive to $\Pi$.
For RC data, there is less difference between the buffered and non-buffered methods for estimating $A$ and $Q$ (Figure~\ref{fig:gausshmm_rc_extra_metrics} (bottom)).

\begin{figure}[!htb]
\centering
    \begin{minipage}[c]{.05\textwidth}
    \centering
    \hbox{\rotatebox{90}{\hspace{1em} $T = 10^4$}}
    \end{minipage}
    \begin{minipage}[c]{.3\textwidth}
    \centering
        \includegraphics[width=\textwidth]{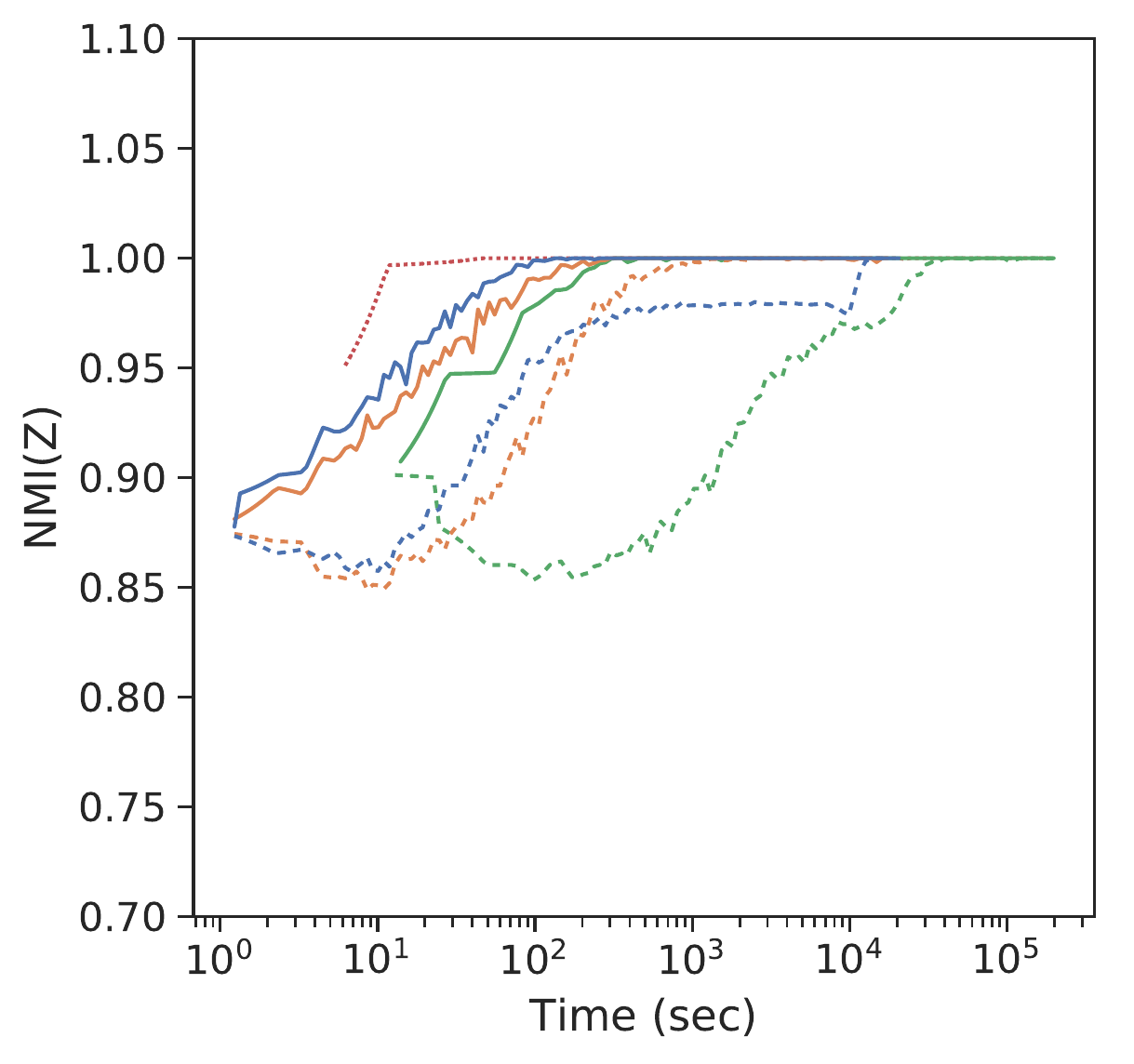}
    \end{minipage}
    \begin{minipage}[c]{.3\textwidth}
    \centering
        \includegraphics[width=\textwidth]{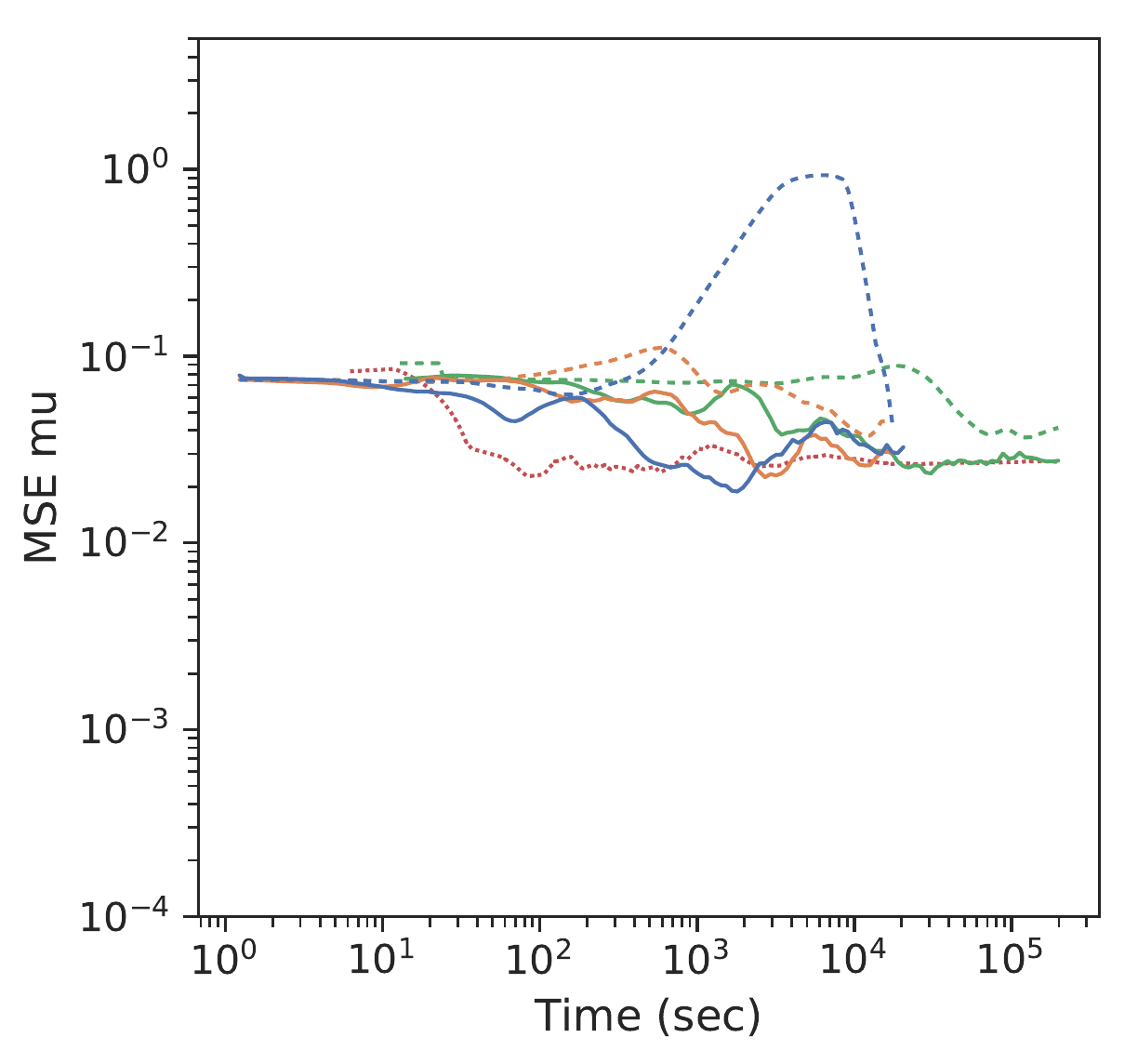}
    \end{minipage}
    \begin{minipage}[c]{.3\textwidth}
    \centering
        \includegraphics[width=\textwidth]{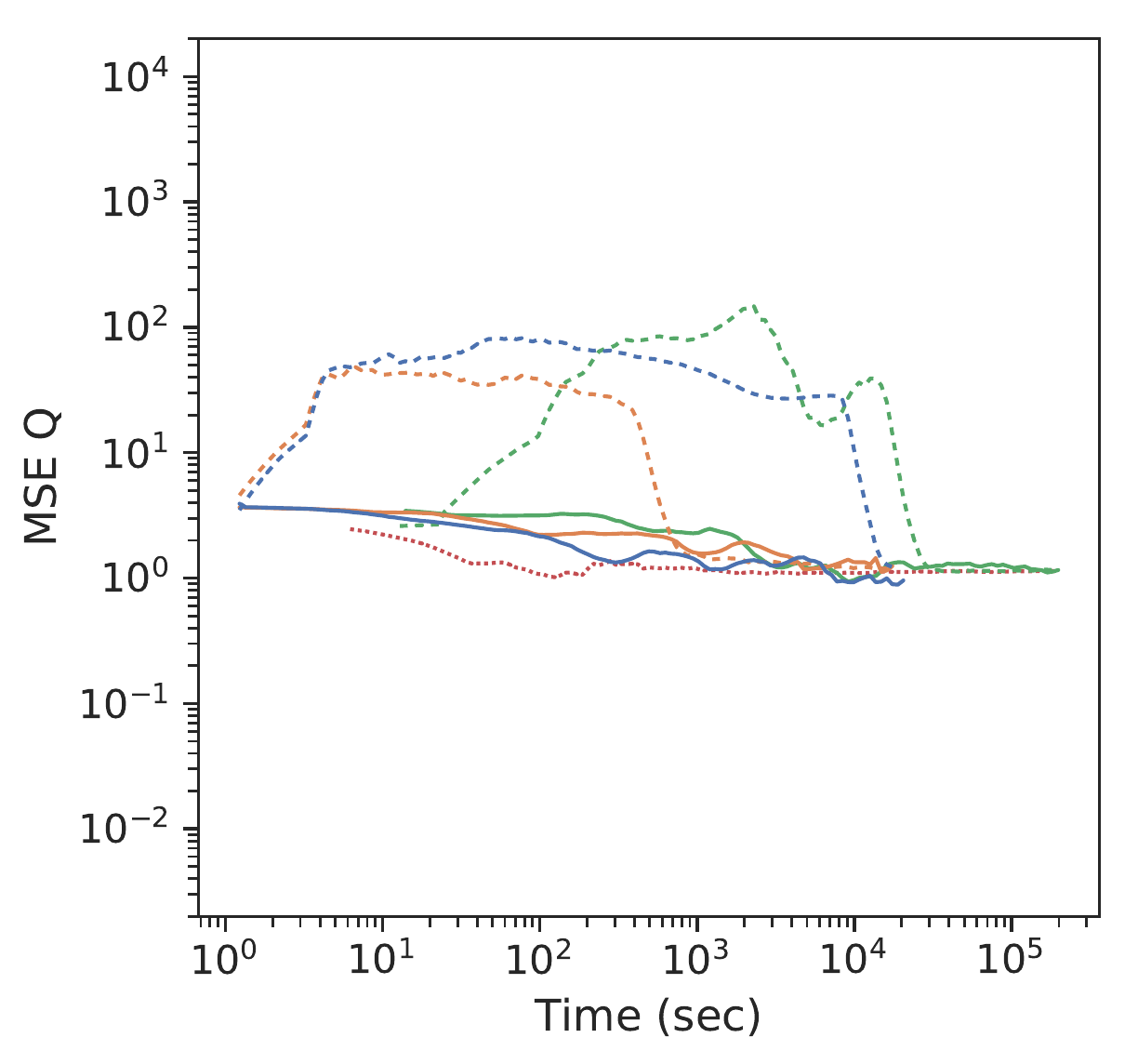}
    \end{minipage}

    \begin{minipage}[c]{.05\textwidth}
    \centering
    \hbox{\rotatebox{90}{\hspace{1em} $T = 10^6$}}
    \end{minipage}
    \begin{minipage}[c]{.3\textwidth}
        \includegraphics[width=\textwidth]{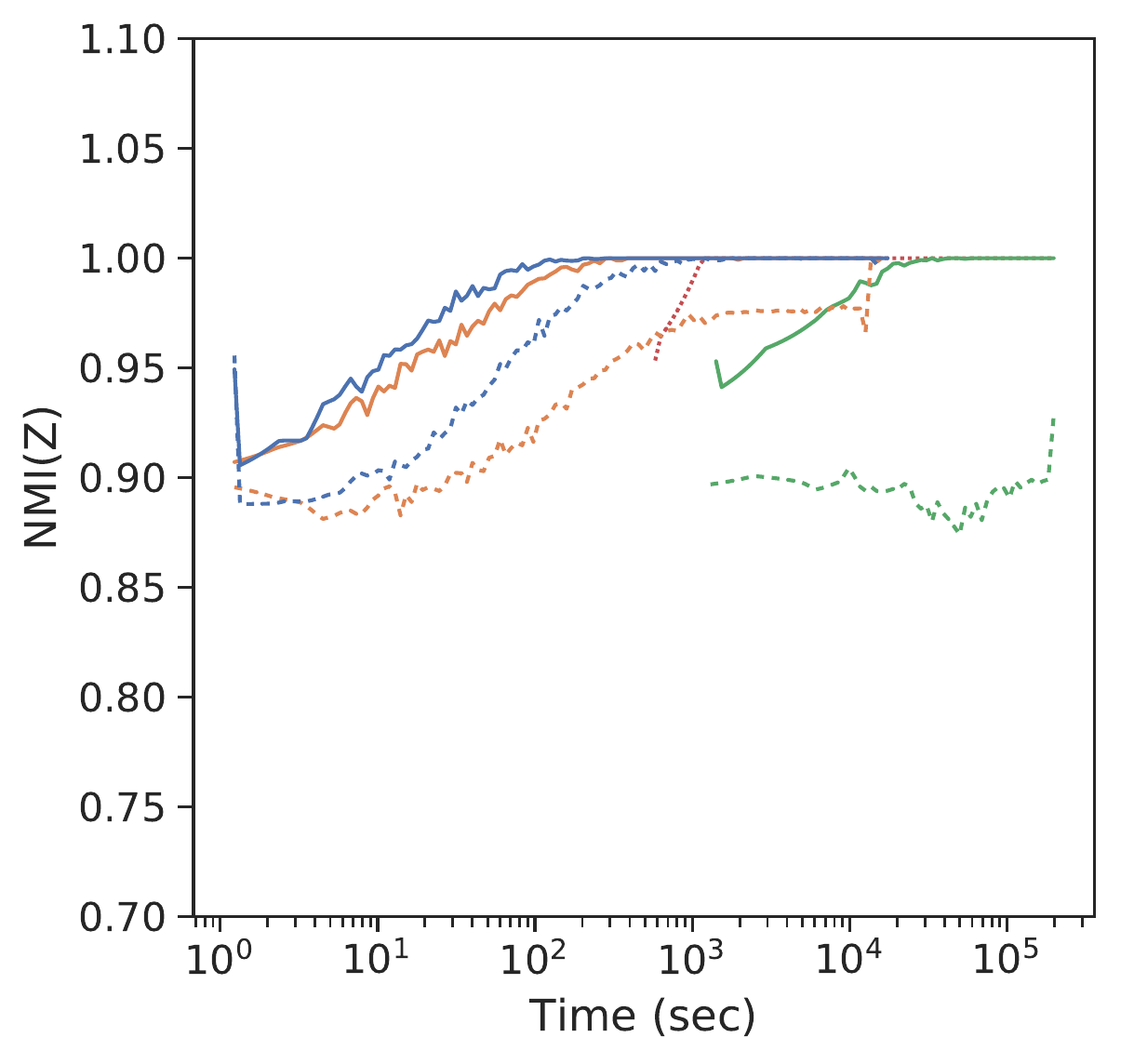}
    \end{minipage}
    \begin{minipage}[c]{.3\textwidth}
    \centering
        \includegraphics[width=\textwidth]{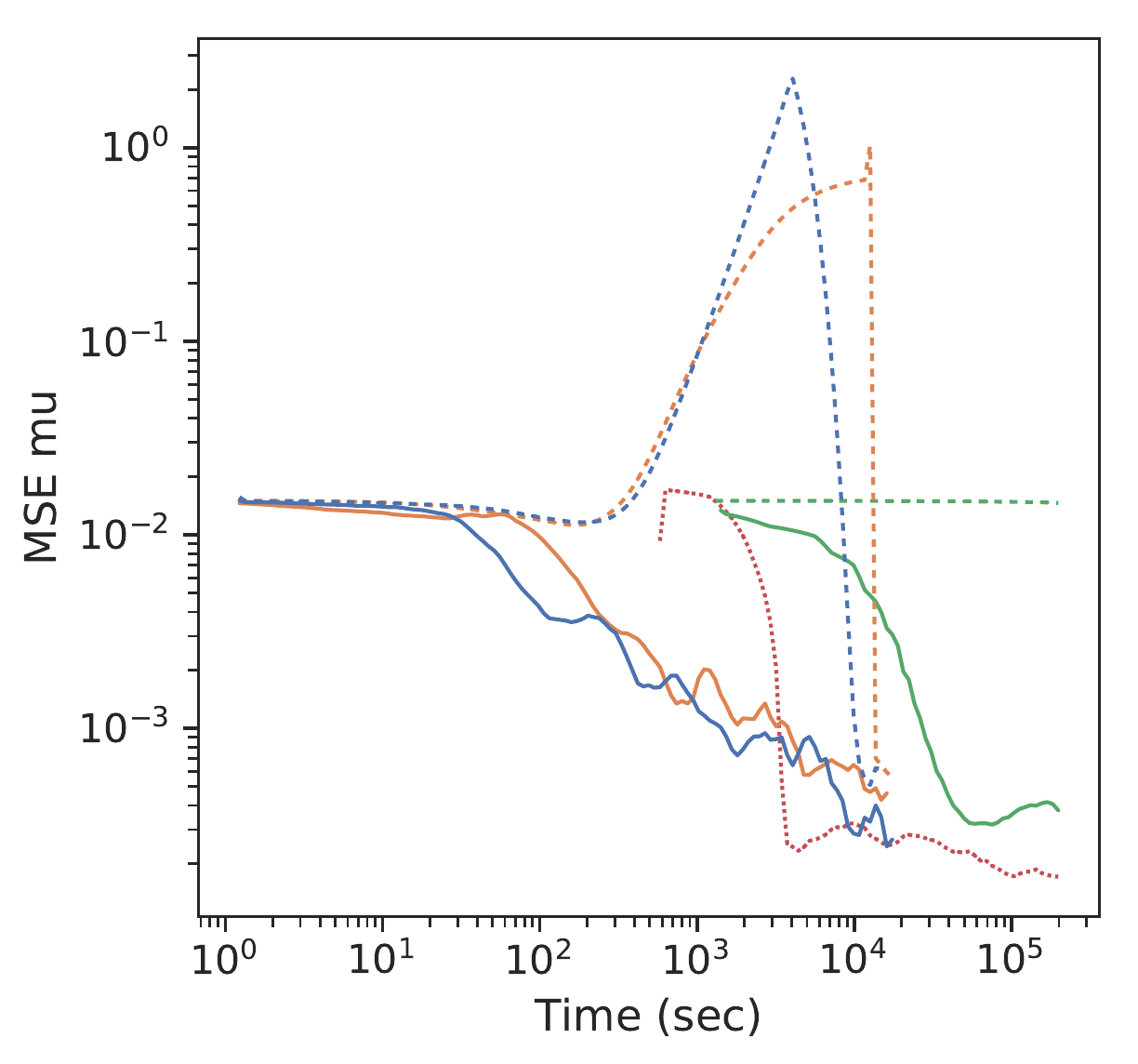}
    \end{minipage}
    \begin{minipage}[c]{.3\textwidth}
    \centering
        \includegraphics[width=\textwidth]{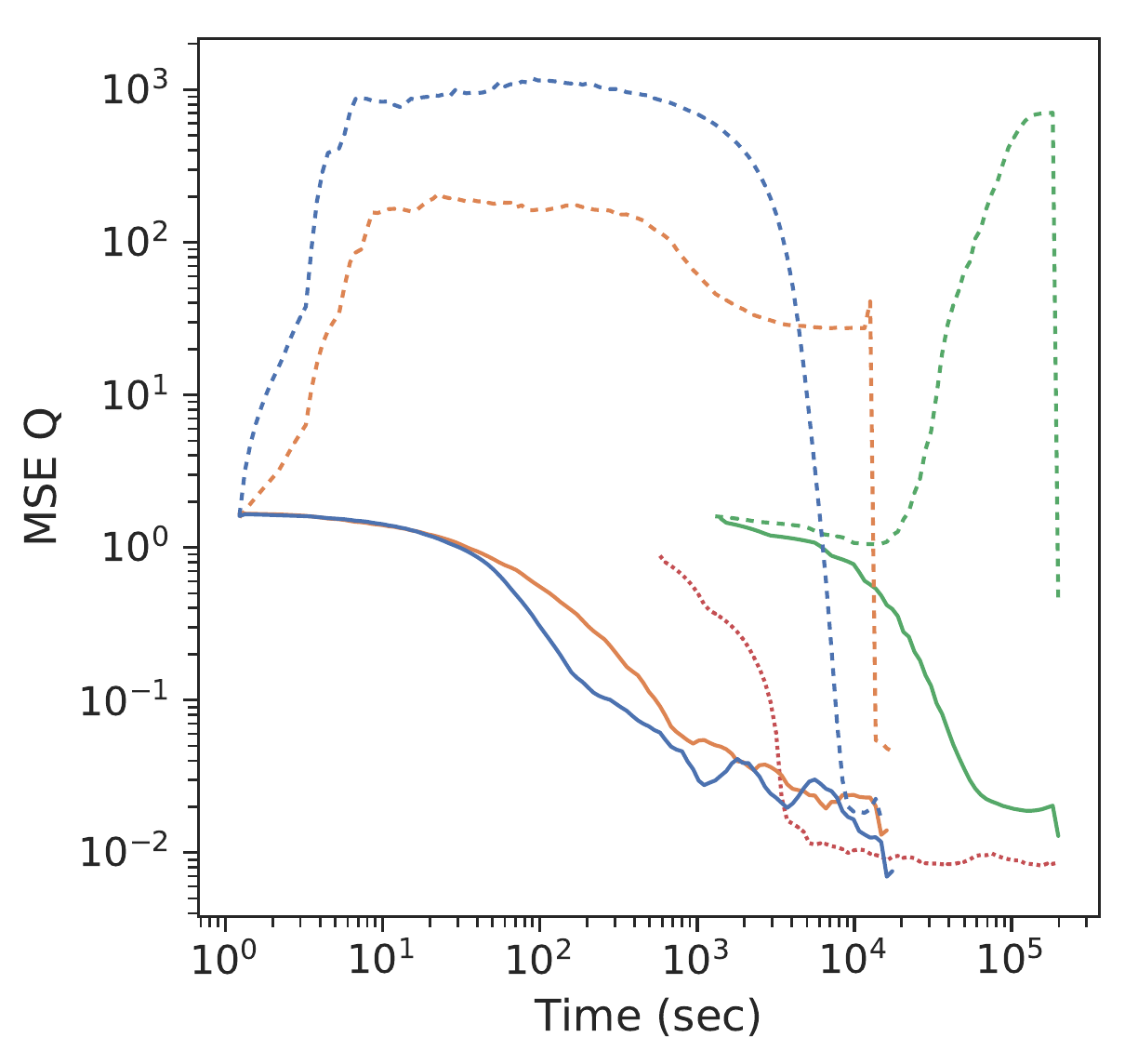}
    \end{minipage}

    \caption{Additional Metrics vs Runtime on RC data with $T = 10^4$ (top), $T=10^6$ (bottom), for different methods:
    \textcolor{BrickRed}{(Gibbs)}, \textcolor{ForestGreen}{(Full)},  \textcolor{Orange}{(No Buffer)} and \textcolor{Blue}{(Buffer)} SGMCMC.
    For SGMCMC methods, solid (\fullline) and dashed (\dashedline) lines indicate SGRLD and SGLD respectively.
    The different metrics are: (left) NMI, (center) estimation error $MSE(\hat{\mu}^{(s)}, \mu^*)$ (right) estimation error $MSE(\hat{Q}^{(s)}, Q^*)$ .
}
    \label{fig:gausshmm_rc_extra_metrics}
    \vspace{-0.5em}
\end{figure}

\subsubsection{ARHMM}
Figure~\ref{fig:arhmm_synth_extra_metrics} are plots of additional metrics for the ARHMM.
\begin{figure}[!htb]
\vspace{-1.2em}
\centering
    \begin{minipage}[c]{.05\textwidth}
    \centering
    \hbox{\rotatebox{90}{\hspace{1em} $T = 10^4$}}
    \end{minipage}
    \begin{minipage}[c]{.3\textwidth}
    \centering
        \includegraphics[width=\textwidth]{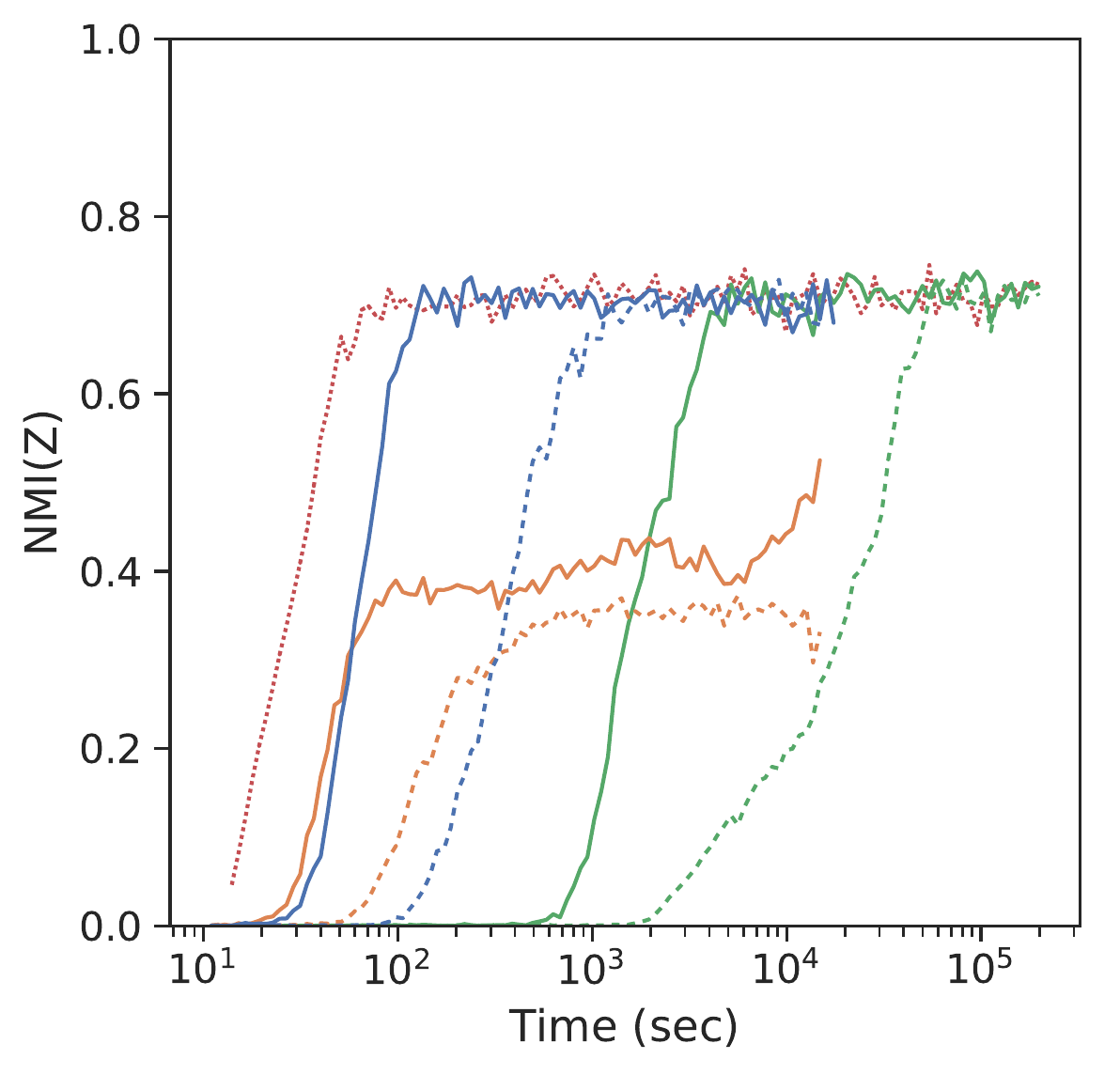}
    \end{minipage}
    \begin{minipage}[c]{.3\textwidth}
    \centering
        \includegraphics[width=\textwidth]{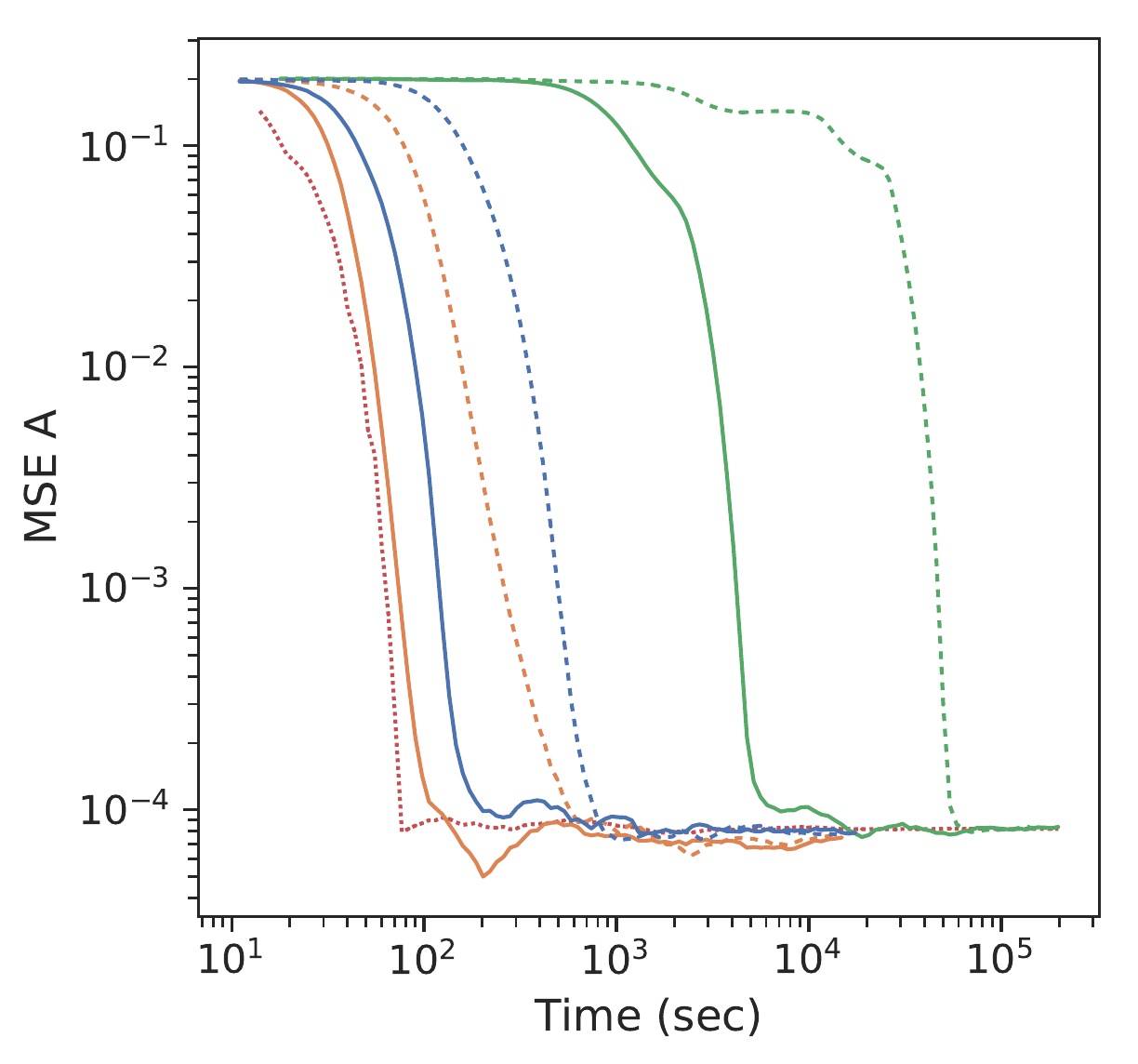}
    \end{minipage}
    \begin{minipage}[c]{.3\textwidth}
    \centering
        \includegraphics[width=\textwidth]{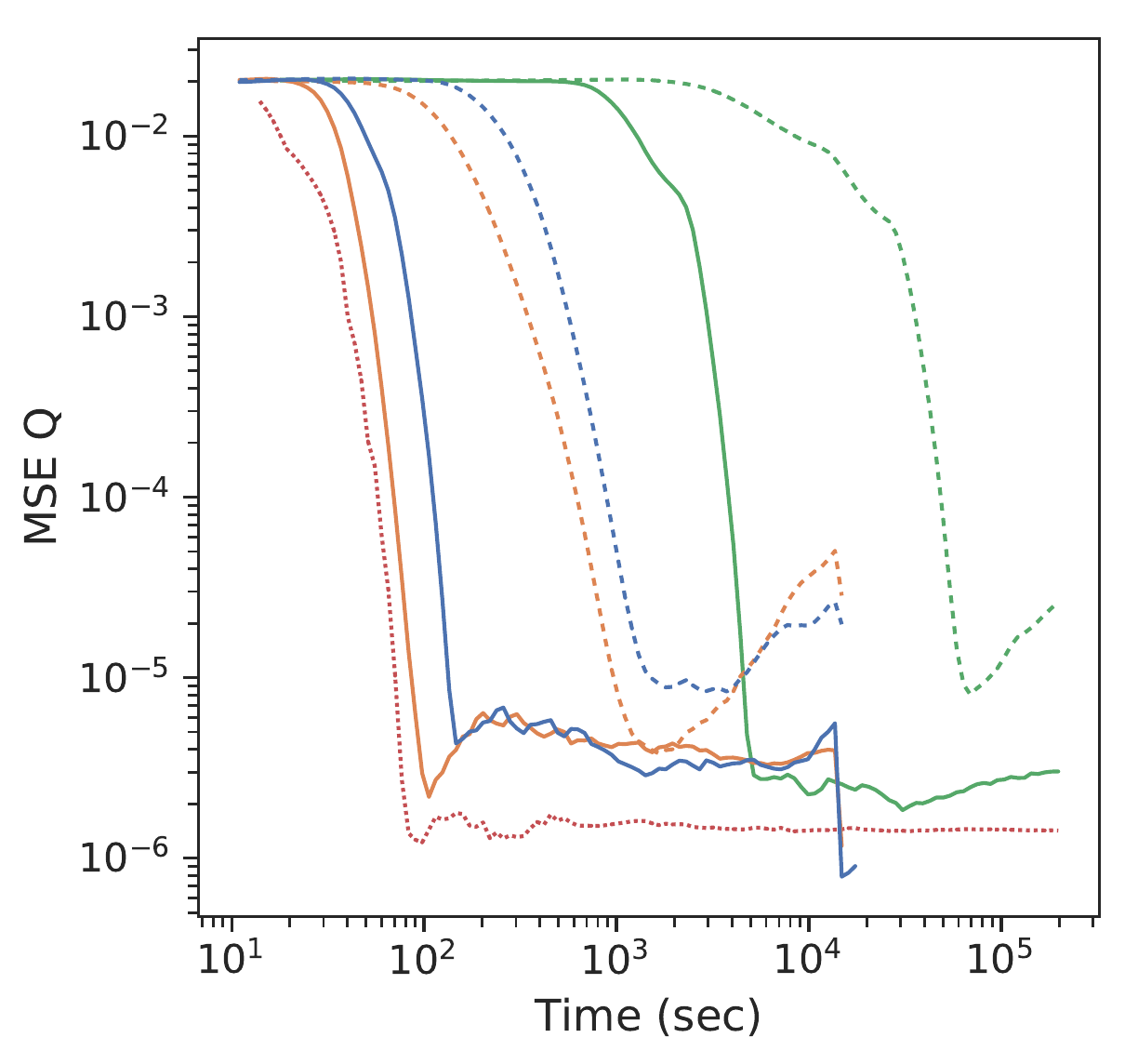}
    \end{minipage}

    \begin{minipage}[c]{.05\textwidth}
    \centering
    \hbox{\rotatebox{90}{\hspace{1em} $T = 10^6$}}
    \end{minipage}
    \begin{minipage}[c]{.3\textwidth}
    \centering
        \includegraphics[width=\textwidth]{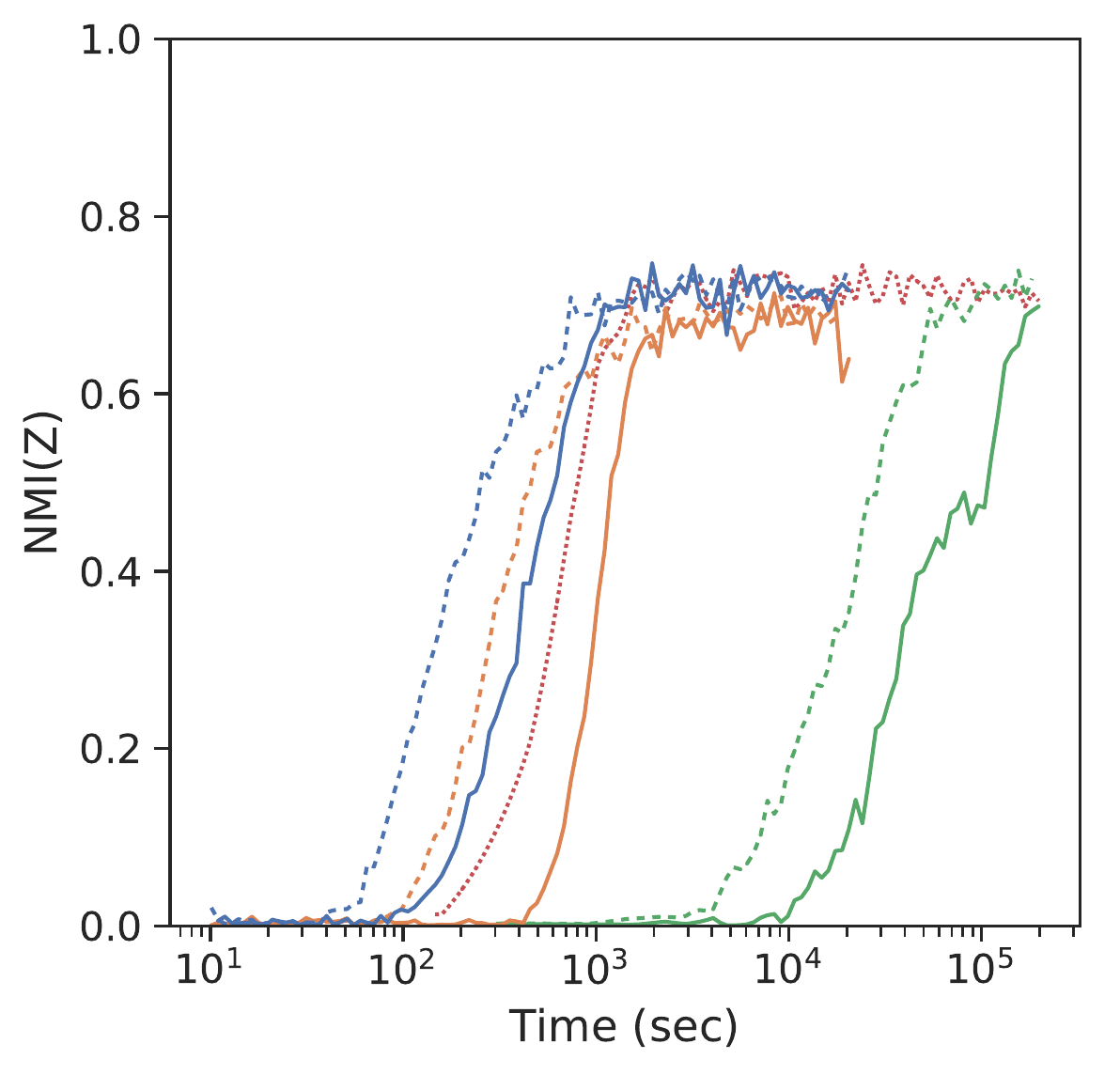}
    \end{minipage}
    \begin{minipage}[c]{.3\textwidth}
    \centering
        \includegraphics[width=\textwidth]{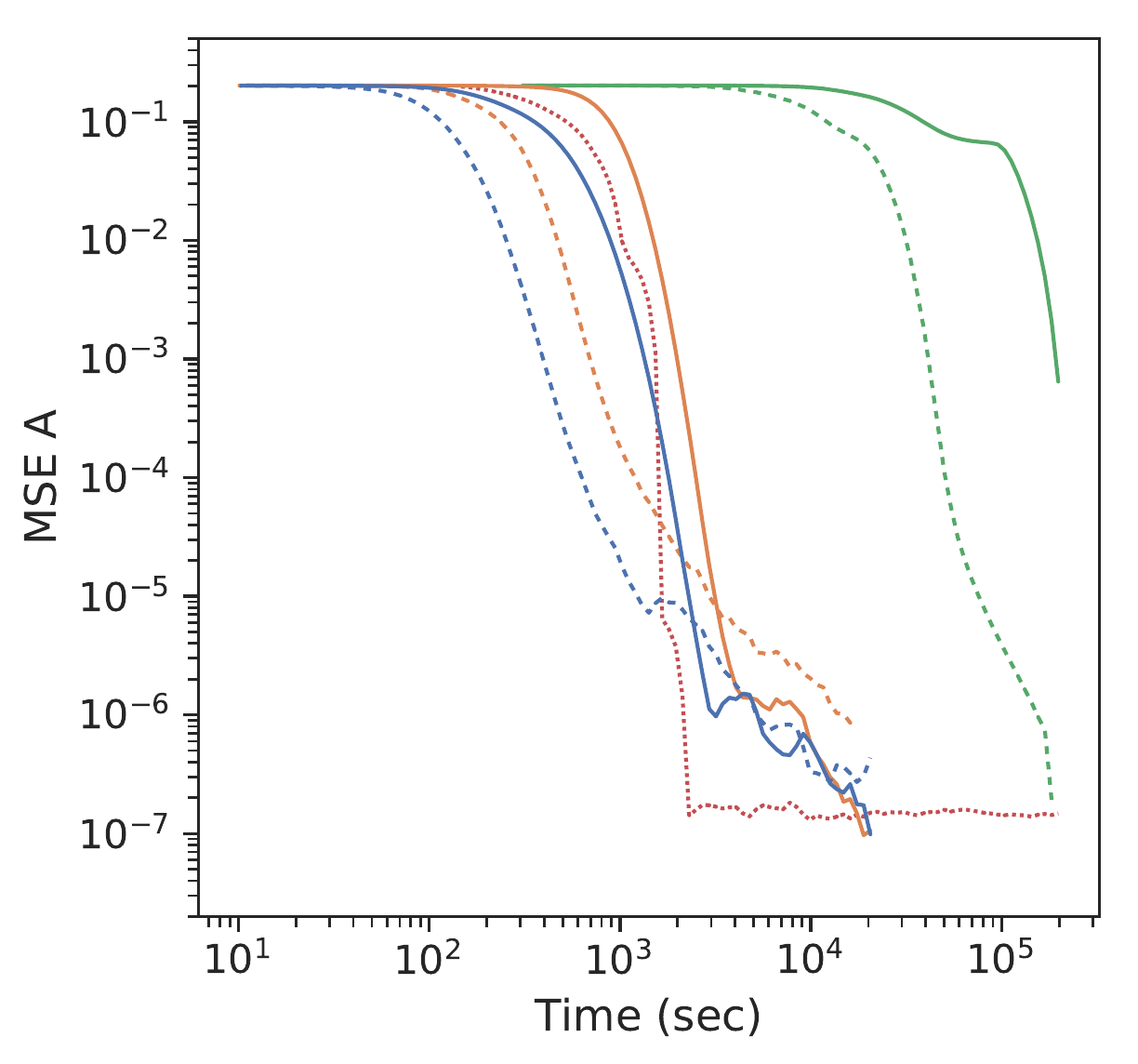}
    \end{minipage}
    \begin{minipage}[c]{.3\textwidth}
    \centering
        \includegraphics[width=\textwidth]{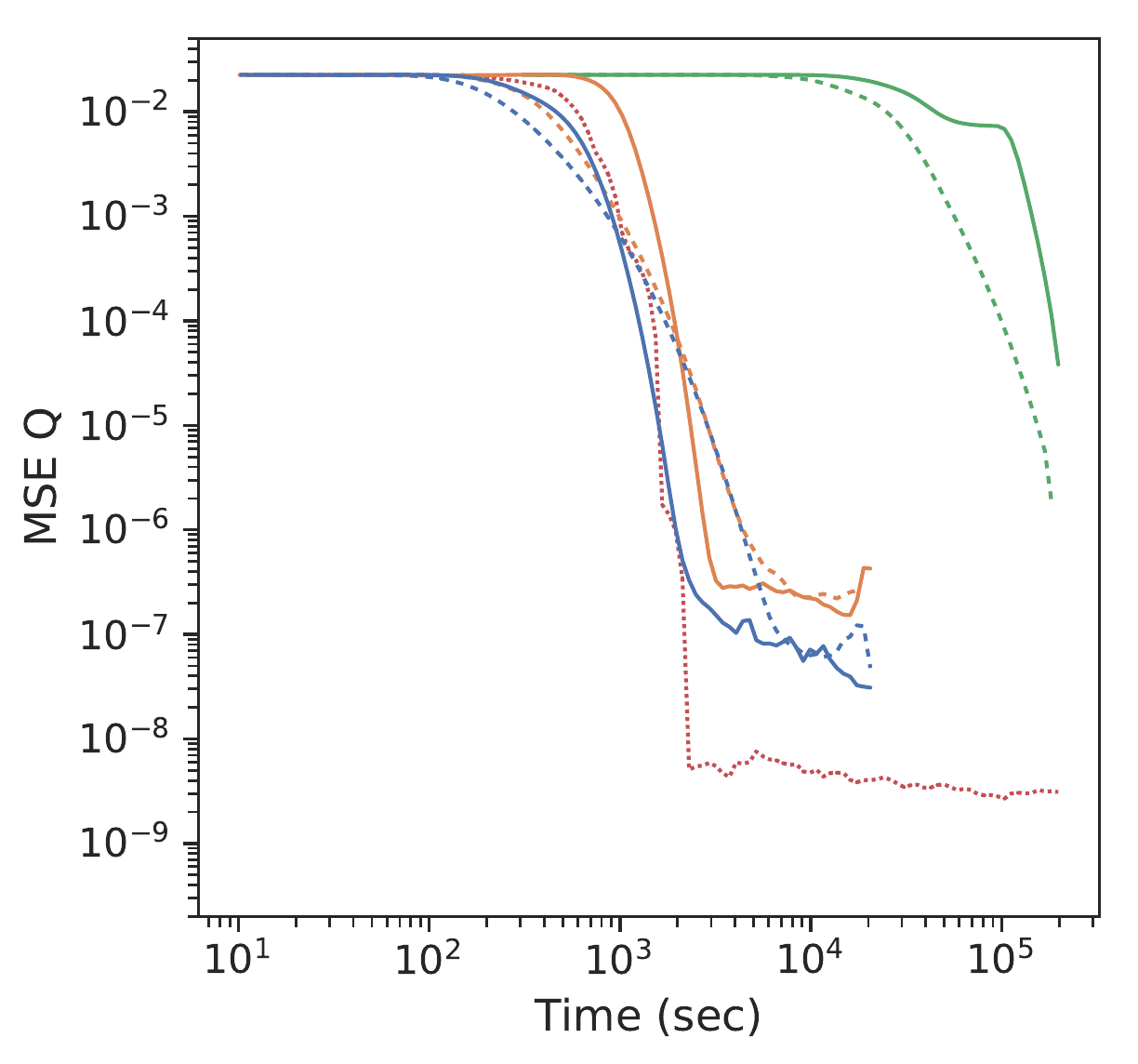}
    \end{minipage}
    \caption{Additional Metrics vs Runtime on ARHMM data with $T = 10^4$ (top), $T=10^6$ (bottom), for different methods:
    \textcolor{BrickRed}{(Gibbs)}, \textcolor{ForestGreen}{(Full)},  \textcolor{Orange}{(No Buffer)} and \textcolor{Blue}{(Buffer)} SGMCMC.
    For SGMCMC methods, solid (\fullline) and dashed (\dashedline) lines indicate SGRLD and SGLD respectively.
    The different metrics are: (left) NMI, (center) $MSE(\hat{A}^{(s)}, A^*)$ (right) $MSE(\hat{Q}^{(s)}, Q^*)$.
}
    \label{fig:arhmm_synth_extra_metrics}
\end{figure}

\subsubsection{LGSSM}
Figure~\ref{fig:lds_extra_metrics} are plots of additional metrics for the LGSSM synthetic data.

\begin{figure}[!htb]
\centering
    \begin{minipage}[c]{.05\textwidth}
    \centering
    \hbox{\rotatebox{90}{\hspace{1em} $T = 10^4$}}
    \end{minipage}
    \begin{minipage}[c]{.3\textwidth}
    \centering
        \includegraphics[width=\textwidth]{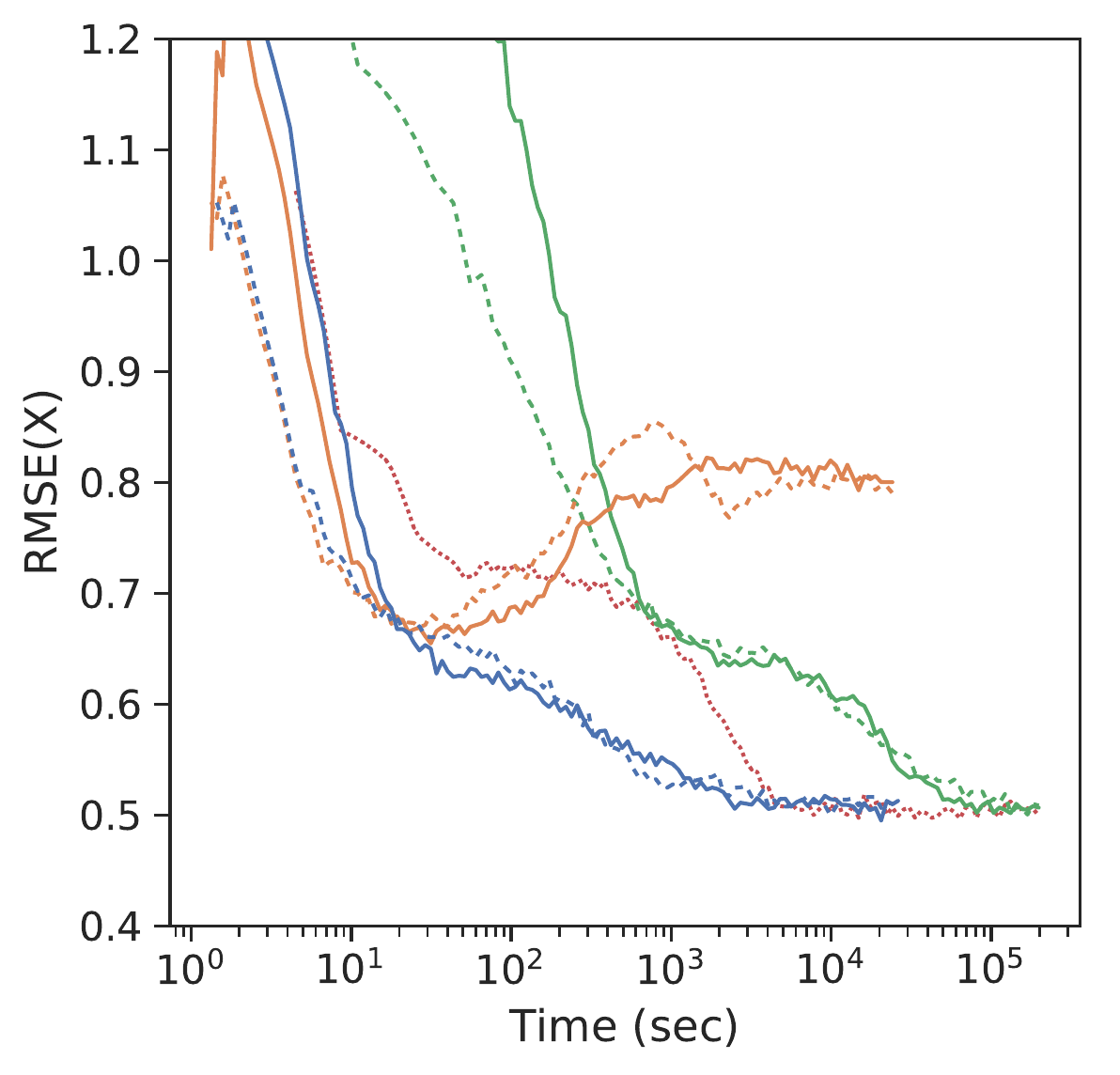}
    \end{minipage}
    \begin{minipage}[c]{.3\textwidth}
    \centering
        \includegraphics[width=\textwidth]{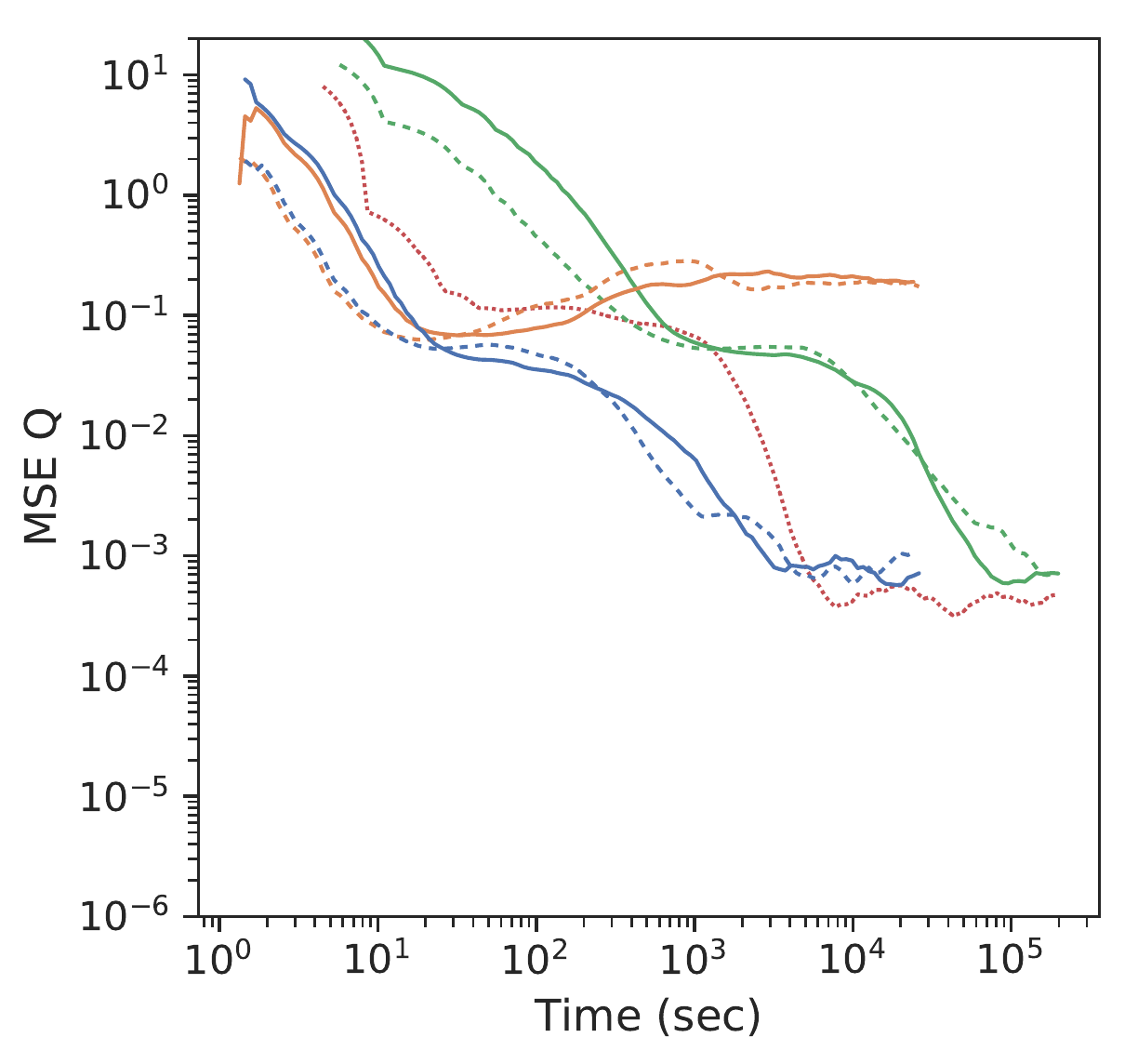}
    \end{minipage}
    \begin{minipage}[c]{.3\textwidth}
    \centering
        \includegraphics[width=\textwidth]{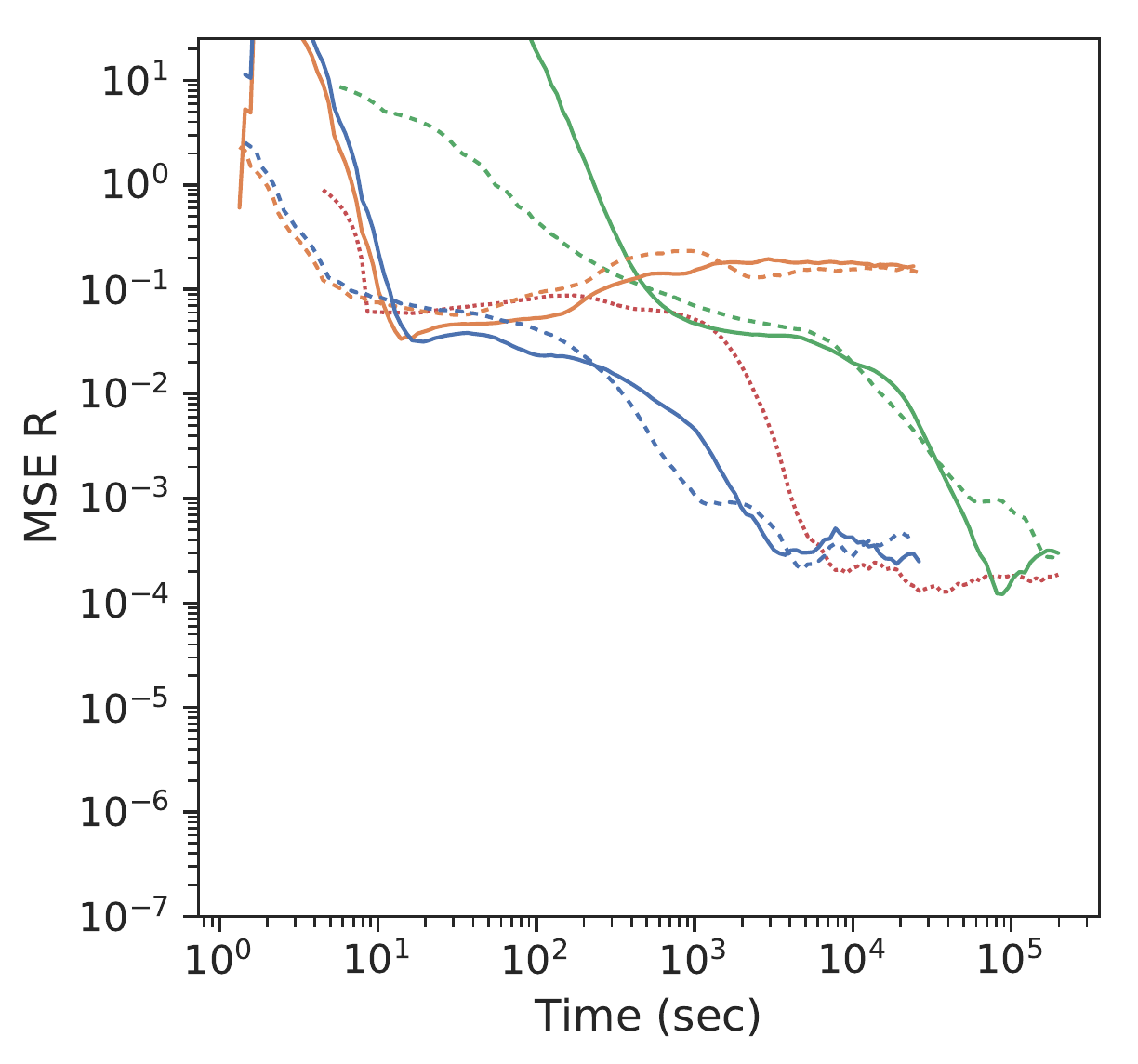}
    \end{minipage}
    
    \begin{minipage}[c]{.05\textwidth}
    \centering
    \hbox{\rotatebox{90}{\hspace{1em} $T = 10^6$}}
    \end{minipage}
    \begin{minipage}[c]{.3\textwidth}
    \centering
        \includegraphics[width=\textwidth]{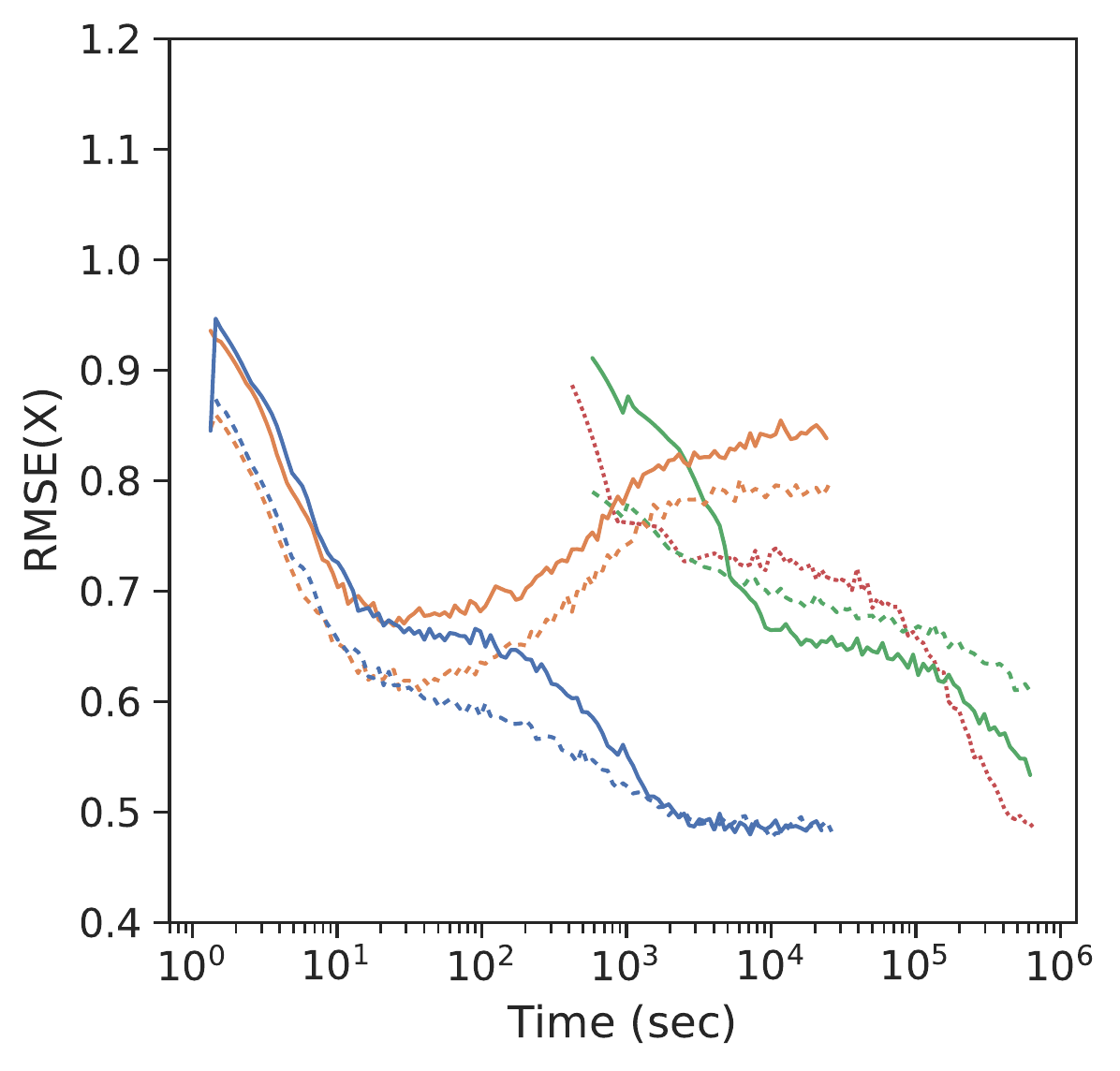}
    \end{minipage}
    \begin{minipage}[c]{.3\textwidth}
    \centering
        \includegraphics[width=\textwidth]{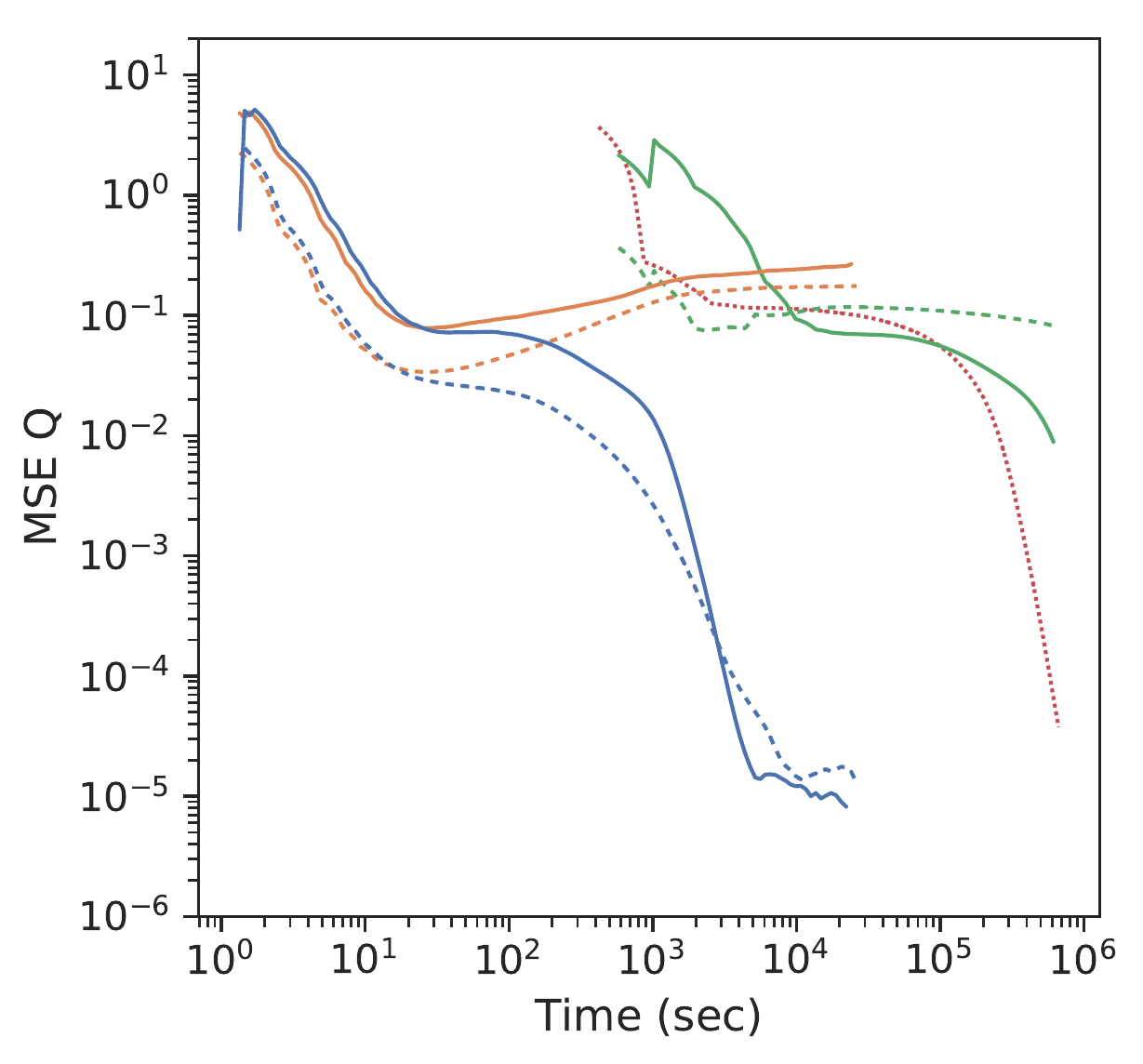}
    \end{minipage}
    \begin{minipage}[c]{.3\textwidth}
    \centering
        \includegraphics[width=\textwidth]{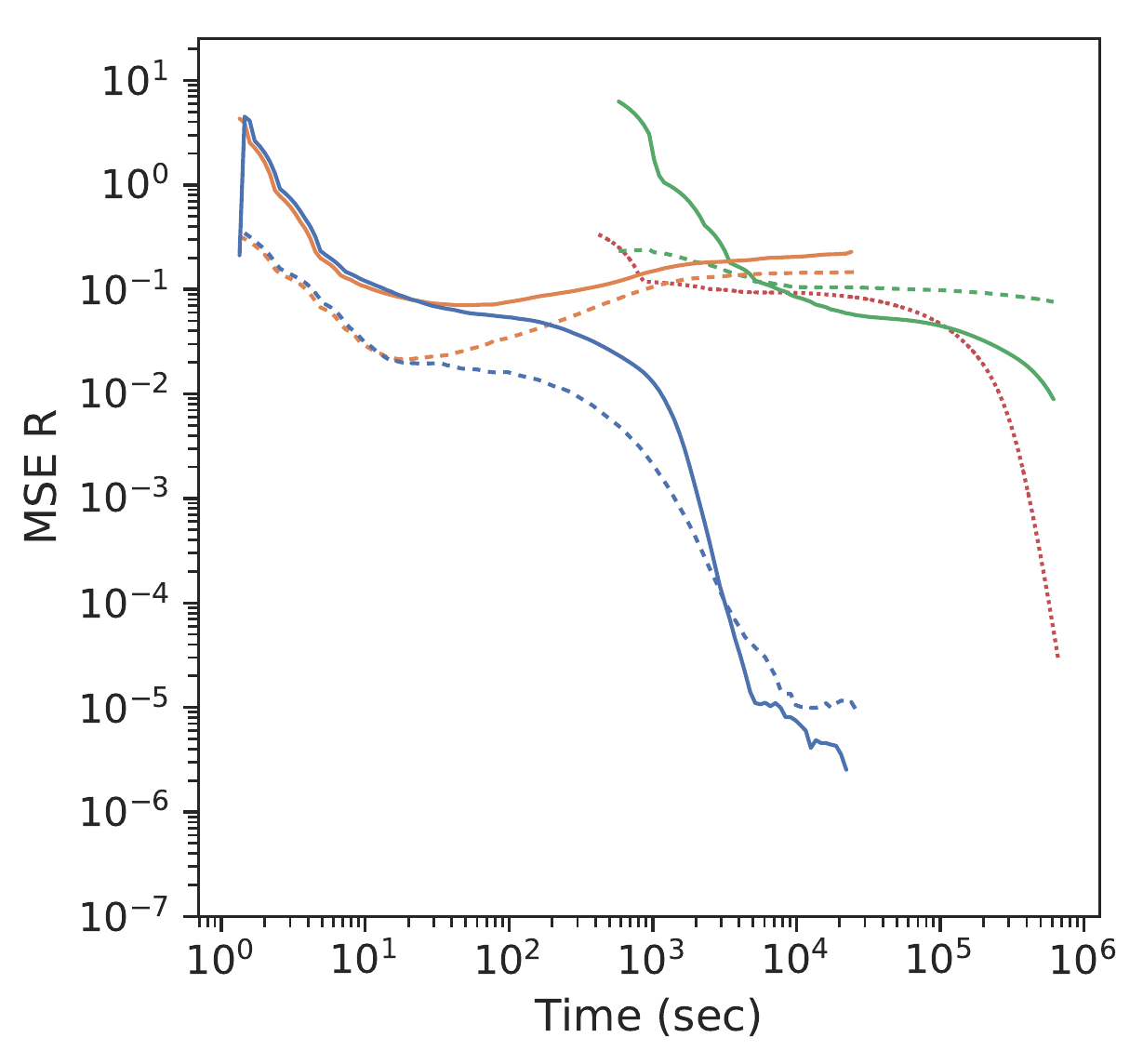}
    \end{minipage}

    \caption{Additional Metrics vs Runtime on LGSSM data with $T = 10^4$ (top), $T=10^6$ (bottom), for different methods:
    \textcolor{BrickRed}{(Gibbs)}, \textcolor{ForestGreen}{(Full)},  \textcolor{Orange}{(No Buffer)} and \textcolor{Blue}{(Buffer)} SGMCMC.
    For SGMCMC methods, solid (\fullline) and dashed (\dashedline) lines indicate SGRLD and SGLD respectively.
    The different metrics are: (left) root-mean squared error (RMSE) between $\hat{x}$ and $x^*$, (center) estimation error $MSE(\hat{Q}^{(s)}, Q^*)$ (right) estimation error $MSE(\hat{R}^{(s)}, R^*)$ .
}
    \label{fig:lds_extra_metrics}
\end{figure}

\subsubsection{SLDS}
Figure~\ref{fig:slds_extra_metrics} are plots of additional metrics for the SLDS data.
\begin{figure}[!htb]
\centering
    \begin{minipage}[c]{.05\textwidth}
    \centering
    \hbox{\rotatebox{90}{\hspace{1em} $|\FULLSEQ| = 10^4$}}
    \end{minipage}
    \begin{minipage}[c]{.25\textwidth}
    \centering
        \includegraphics[width=\textwidth]{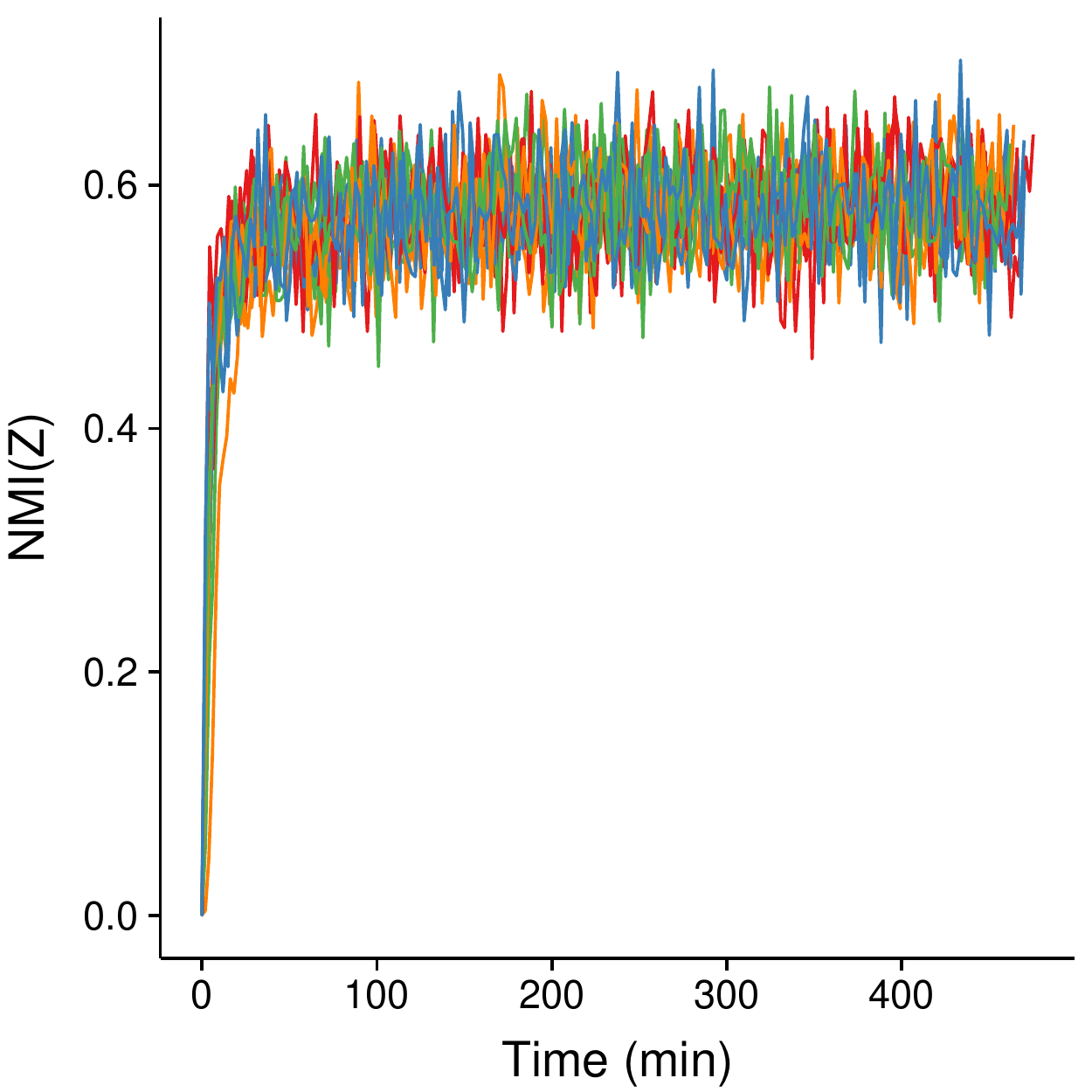}
    \end{minipage}
    \begin{minipage}[c]{.25\textwidth}
    \centering
        \includegraphics[width=\textwidth]{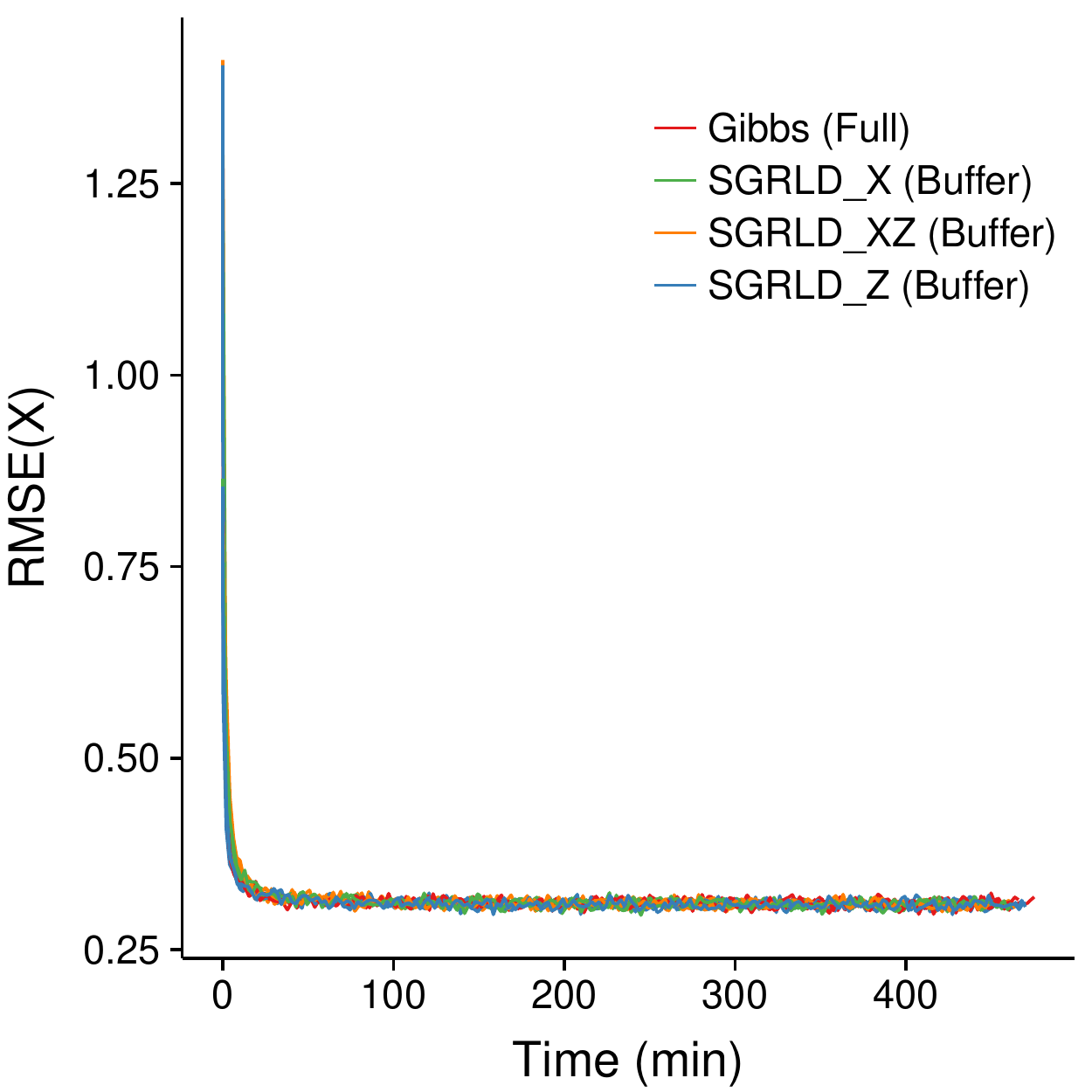}
    \end{minipage}
    \begin{minipage}[c]{.4\textwidth}
    \centering
        \includegraphics[width=\textwidth]{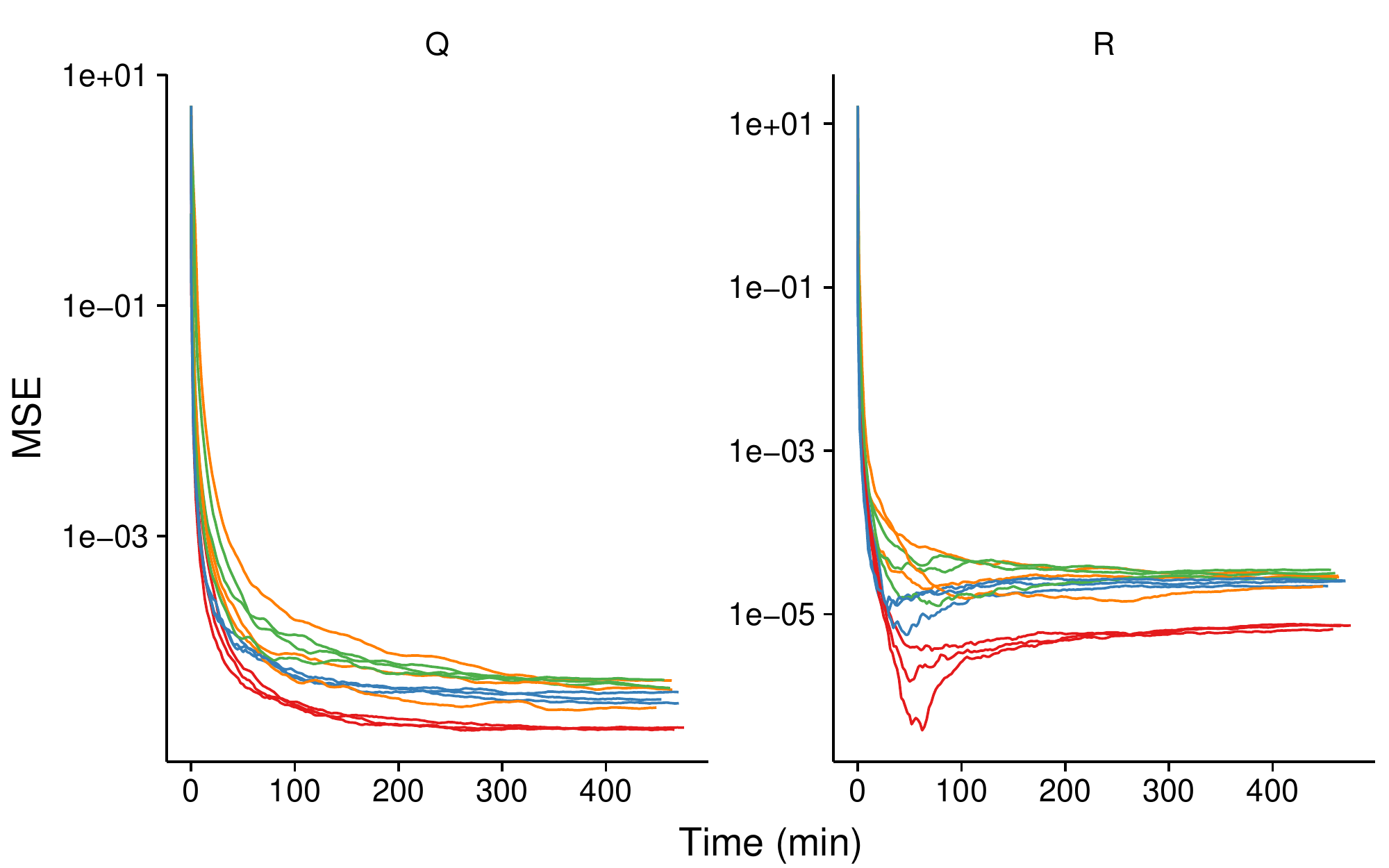}
    \end{minipage}

    \begin{minipage}[c]{.05\textwidth}
    \centering
    \hbox{\rotatebox{90}{\hspace{1em} $|\FULLSEQ| = 10^6$}}
    \end{minipage}
    \begin{minipage}[c]{.25\textwidth}
    \centering
        \includegraphics[width=\textwidth]{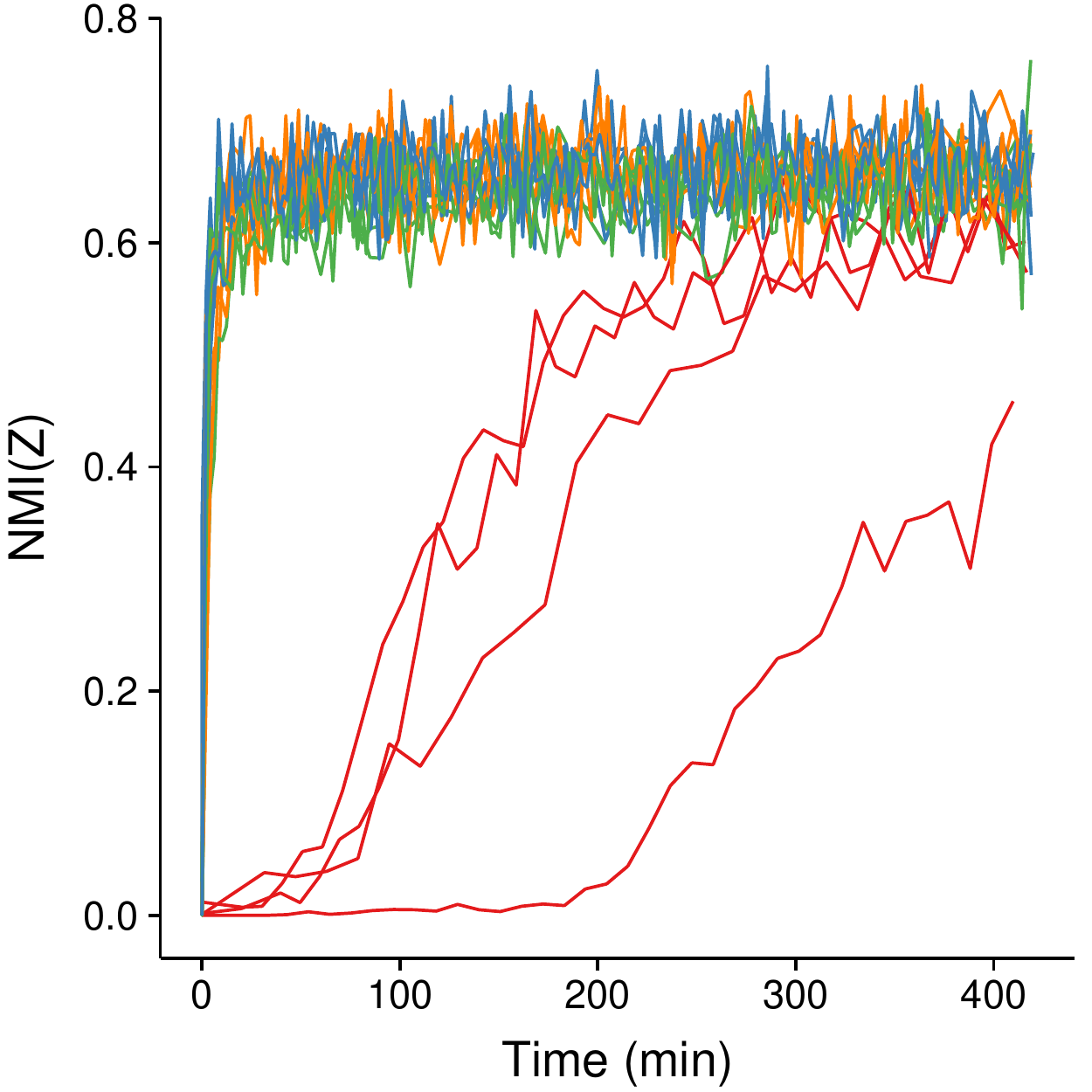}
    \end{minipage}
    \begin{minipage}[c]{.25\textwidth}
    \centering
        \includegraphics[width=\textwidth]{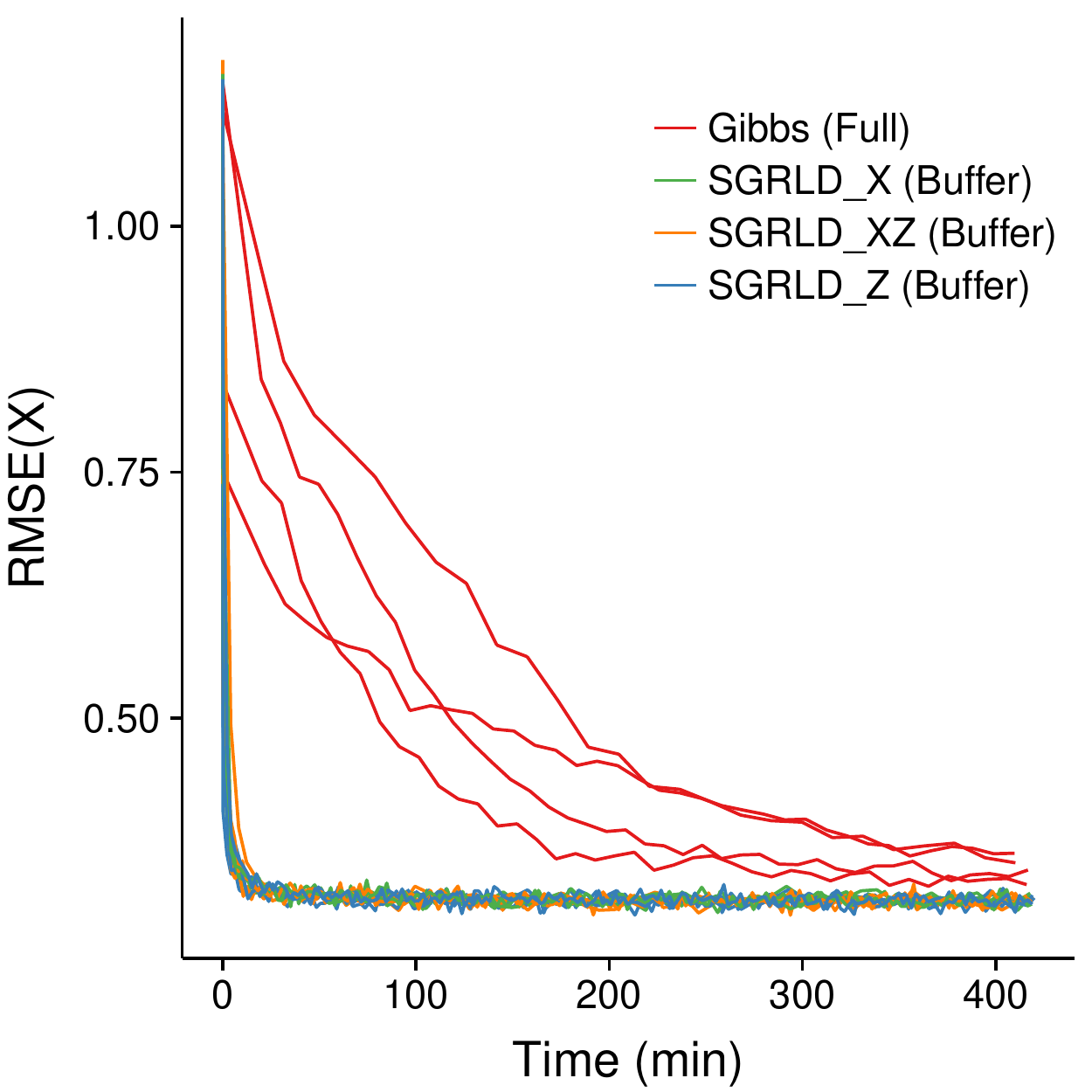}
    \end{minipage}
    \begin{minipage}[c]{.4\textwidth}
    \centering
        \includegraphics[width=\textwidth]{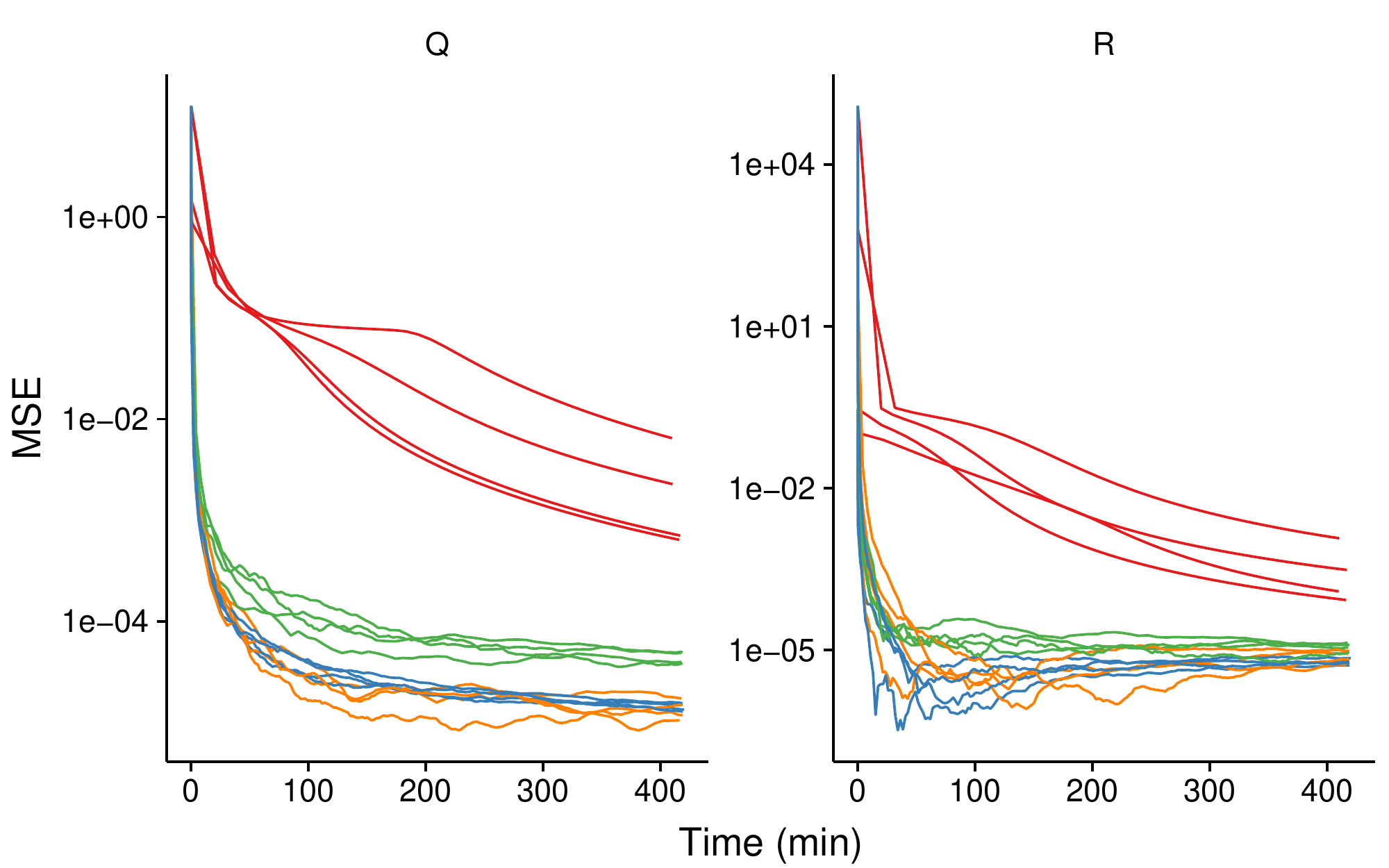}
    \end{minipage}
    \caption{Additional Metrics vs Runtime on SLDS data: (Top) $|\FULLSEQ| = 10^4$, (Bottom) $|\FULLSEQ| = 10^6$. (Left)  NMI between $\hat{z}$ and $z^*$. (Center) root-mean square error (RMSE) between $\hat{x}$ and $x^*$, (Right) estimation error $\| \theta^{(s)} - \theta^* \|$.
Methods: \textcolor{BrickRed}{Gibbs}, \textcolor{ForestGreen}{SGRLD X}, \textcolor{YellowOrange}{SGRLD XZ}, and \textcolor{Blue}{SGLRD Z}.
}
    \label{fig:slds_extra_metrics}
\end{figure}

\end{document}